%% file: arXiv-dpdt.tex
\documentclass[12pt]{article}
\usepackage[margin=1in]{geometry}

\usepackage{natbib}
\usepackage[utf8]{inputenc} 
\usepackage[T1]{fontenc}    
\usepackage{url}            
\usepackage{booktabs}       
\usepackage{amsfonts}       
\usepackage{nicefrac}       
\usepackage{microtype}      

\usepackage{color,colortbl,soul,upgreek}
\usepackage{arydshln}
\usepackage{wrapfig}
\usepackage{soul}
\usepackage{amsmath}
\usepackage{empheq}
\usepackage{mdframed,enumerate}

\usepackage{graphicx}
\usepackage{caption}
\usepackage{subcaption}

\usepackage{makecell}

\usepackage{microtype}
\usepackage{booktabs} 
\usepackage[table]{xcolor}

\usepackage{morenotations,rotating}

\usepackage{caption}
\usepackage{tcolorbox}

\title{\papertitle}
\author{
  Richard Nock$^{\dagger,\ddagger}$ \qquad Wilko Henecka$^{\dagger}$ \\
 $^{\dagger}$Data61, $^\ddagger$the Australian National University\\
{\normalsize \texttt{\{firstname.lastname\}@data61.csiro.au}} \\
}
\begin{document}

\date{}

\maketitle

\begin{abstract}
\input{content-arxiv/abstract}
\end{abstract}

\newpage

\input{content-arxiv/introduction}

\input{content-arxiv/definitions}

\input{content-arxiv/dilemma}
\input{content-arxiv/matsu}

\input{content-arxiv/boost}

\input{content-arxiv/dilemma-alpha}
\input{content-arxiv/experiments}

\input{content-arxiv/conclusion}
\input{content-arxiv/acknowledgments}

\bibliographystyle{icml2020}
\bibliography{references,bibgen}

\clearpage
\newpage

\input{content-arxiv/appendix-org}

\end{document}

%% file: content-arxiv/abstract.tex
Boosted ensemble of decision tree (DT) classifiers are extremely
popular in international competitions, yet
to our knowledge nothing is formally known on how to make them \textit{also}
differential private (DP), up to the point that random forests
currently reign supreme in the DP stage. Our paper starts with the proof that the
privacy vs boosting picture for DT involves a notable and general technical tradeoff: 
the sensitivity tends to increase with the boosting rate of the
loss, for any proper loss. DT induction algorithms being
fundamentally iterative, our finding implies non-trivial choices to select or tune the
loss to balance noise against utility to split nodes. To address
this, we craft a new parametererized proper loss, called the
M$\alpha$-loss, which, as we show,
allows to finely tune the tradeoff in the complete spectrum
of sensitivity vs
boosting guarantees.  We then introduce \textit{objective calibration}
as a method to adaptively tune the tradeoff during DT
induction to limit the privacy budget spent while formally being able
to keep boosting-compliant convergence on limited-depth nodes with
high probability. 
Extensive experiments on 19 UCI domains reveal that objective
calibration is highly competitive, even in the DP-free setting. Our
approach tends to very significantly beat
random forests, in particular on high DP regimes ($\varepsilon
\leq 0.1$) and even
with boosted ensembles containing ten times less trees, which
could be crucial to keep a key feature of DT models under differential privacy: interpretability. 

%% file: content-arxiv/introduction.tex
\section{Introduction}\label{sec:int}

The past decade has seen considerable growth of the subfield of machine learning (ML) tackling the augmentation of the classical models with additional constraints that are now paramount in applications \citep{adwFR,kmmsDP,agltvQC,dljfTD,jkczthakQA,jkmprsuDP}.
One challenge posed by such constraints is the potentially risky design process for new approaches: it may not be hard to modify the state of the art to accomodate for the new constraint(s), but if not cared for enough,
the modification may come at a hefty price tag for accuracy. Differential privacy (DP) is a very good example of a now popular constraint, which essentially proceeds by randomizing parts of the whole process to reduce the output's sensitivity to local changes in the input \citep{drTA}. DP possesses a toolbox of simple randomisation mechanisms that can allow for simple modifications of ML algorithms to make them private. 
However, a careful optimization of the utility (accuracy) under the DP constraints typically requires rethinking the training process, as exemplified by the output perturbation mechanism to train kernel machines in \citet{cmsDP}.\\
There is to date no such comparable achievement in the case of Decision Trees (DTs) induction, a crucial problem to address: decision trees have been popular in machine learning for decades \citep{bfosCA,qC4}, they are widely used, in particular for tabular data, and recognised for their accuracy, interpretability, and efficiency; they are virtually present in almost every Kaggle competition \cite{ahPR}, with extremely popular implementations like \cite{cgXA,kmfwcmylLA}. On the DP side, there is to our knowledge no extension of boosting properties to DP. We attribute the fact that random forests (RFs) currently "reign supreme" in DP \citep[Section 6]{fiDT} as more a consequence of the lack of formal results for boosting rather than following from any negative result.

\noindent\textbf{Our first contribution} shows a tradeoff to address to solve this problem. On the accuracy side, it has been known for a long time that the curvature of the Bayes risk used conditions the convergence rate in the boosting model \citep{kmOT,nnOD}. In this paper, we first investigate the privacy side and show that the \textit{sensitivity} of the splitting criterion has the \textit{same} dependence on the curvature: in few words, faster rate goes along with putting more noise to pick the split. Since the total privacy budget spent grows with the size of the tree, there is therefore a nontrivial tradeoff to solve between rate and noise injection to get sufficient accuracy under DP budget constraints.

\noindent\textbf{Our second contribution} brings a nail to hammer for this tradeoff: a new proper loss, properness being the minimal requirement that Bayes rule achieves the optimum of the loss. This loss, that we call M$\alpha$-loss, admits parameter $\alpha \in [0,1]$ which finely tunes the boosting convergence vs privacy budget tradeoff. As $\alpha \rightarrow 1$, boosting rate converges to the optimal rate while as $\alpha \rightarrow 0$, sensitivity converges to the minimum. In addition, we provide the full picture of boosting rates for the M$\alpha$-loss, of independent interest since generalizing the results of \citet{kmOT}.

\noindent\textbf{Our third contribution} brings a possible hammer for this nail. We show how to \textit{tune} the loss during induction to limit the privacy budget spent while keeping the \textit{same} boosting rates as in the noise-free case for a subtree of the tree with the same root, with a guaranteed probability. As the training sample increase in size, all else being equal, this probability converges to 1 and the subtree converges to the full boosted tree. This technique, that we nickname \textbf{objective calibration}, picks at the beginning of the induction a splitting criterion with optimal boosting convergence, thus paying significant privacy budget, and then reduces the budget spent as we split deeper nodes, thus also reducing convergence. Ultimately, the budget converges to the smallest splitting budget as the tree converges to consistency on training. 

\noindent\textbf{Our fourth contribution} provides extensive experiments on 19 UCI domains \citep{dgUM}. An extensive comparison of our approach with two SOTA RFs reveals that our approach tends to very significantly beat RFs, even with ensembles more than ten times smaller. Our results display the benefits of combining boosting with DP, as well as the fact that objective calibration happens to be competitive also in the noise-free case.

\noindent The rest of this paper follows the order of contributions: after some definition in Section $\S$ \ref{sec:def}, the tradeoff between privacy and accuracy is developed in Section $\S$ \ref{sec:priv}, the M$\alpha$-loss is presented in $\S$ \ref{sec:matalpha}, results on boosting with the M$\alpha$-loss are given in $\S$ \ref{sec-boost}, objective calibration is presented in $\S$ \ref{sec-solv}, experiments are summarized in $\S$ \ref{sec-exp} and a last Section, $\S$ \ref{sec-conc}, concludes the paper. In order not to laden the main body's content, all proofs and considerably more detailed experiments have been pushed to an appendix (\sm), available from pp \pageref{tabcon} (proofs) and from pp \pageref{exp_expes} (experiments).

%% file: content-arxiv/definitions.tex
\section{Definitions}\label{sec:def}

\noindent $\triangleright$ \textbf{Batch learning}: most of our notations from \citet{nwLO}. We use the
shorthand notations $[n]
\defeq \{1, 2, ..., n\}$ for $n \in \mathbb{N}_*$ and $z' + z \cdot [a, b]
\defeq [z' + za, z' + zb]$ for $z \geq 0, z'\in \mathbb{R}, a \leq  b \in \mathbb{R}$. We also let
$\overline{\mathbb{R}} \defeq [-\infty, \infty]$.
In the batch supervised
learning setting, one is given a training set of $m$ examples
${\mathcal{S}} \defeq \{({\ve{x}}_i, y_i), i \in [m]\}$, where ${\ve{x}}_i
\in {\mathcal{X}}$ is an observation
(${\mathcal{X}}$ is called the domain: often,  ${\mathcal{X}}\subseteq {\mathbb{R}}^n$) and $y_i
\in \mathcal{Y} \defeq \{-1,1\}$ is a label, or class. The objective
is to learn a \textit{classifier}, \textit{i.e.} a function $h
: \mathcal{X} \rightarrow \mathbb{R}$ which belongs to a given
set $\mathcal{H}$. The first class of models we consider are decision trees (DTs).
A (binary) DT $h$ makes a recursive partition of a domain. There are two types of nodes: internal nodes are indexed by a binary test and leaves are indexed by a real number. The depth of a node (resp. a tree) is the minimal path length from the root to the node (resp. the maximal node depth). Thus, depth(root) is zero. The classification of some $\ve{x} \in \mathcal{X}$ is achieved by taking the sign of the real number whose leaf is reached by $\ve{x}$ after traversing the tree from the root, following the path of the tests it satisfies. The other types of classifiers we consider are linear combinations of base classifiers, now hugely popular when base classifiers are DTs, after the advents of bagging \citep{bBP} and boosting \citep{fhtAL}.

\noindent $\triangleright$ \textbf{Losses}: the goodness of fit of some $h$ on $\mathcal{S}$ is
evaluated by a given \textit{loss}. There are two dual views of losses to train domain-partitioning classifiers (like DTs) and linear combinations of base classifiers \citep{nnBD}. Both views start from the definition of a \textit{loss for class probability estimation}, $\losscpe : \mathcal{Y} \times [0,1]
\rightarrow \mathbb{R}$,
\begin{eqnarray}
\losscpe(y,u) & \defeq & \iver{y=1}\cdot \partialloss{1}(u) +
                     \iver{y=-1}\cdot \partialloss{-1}(u), \label{eqpartialloss}
\end{eqnarray}
where $\iver{.}$ is Iverson's bracket. Functions $\partialloss{1}, \partialloss{-1}$ are called \textit{partial} losses; we refer to \citet{rwCB} for the additional background on partial losses. We consider
\textit{symmetric} losses for which $\partialloss{1}(u) = \partialloss{-1}(1-u),
\forall u \in [0,1]$ \citep{nnOT} (in particular, this assumes that there is no class-dependent misclassification loss). For example, the square loss has $\partialsqloss{1}(u) \defeq (1/2)\cdot (1-u)^2$
and $\partialsqloss{-1}(u) \defeq (1/2)\cdot u^2$. The
log loss has $\partiallogloss{1}(u) \defeq -\log u$
and $\partiallogloss{-1}(u) \defeq -\log(1-u)$. The 0/1 loss has $\partialZOloss{-1}(u) \defeq \iver{\pi \geq 1/2}$ and $\partialZOloss{1}(u) \defeq \iver{\pi \leq 1/2}$. All these losses are symmetric. The associated (pointwise) \textit{Bayes} risk is 
\begin{eqnarray}
\bayesrisk(\pi) & \defeq & \inf_u \expect_{\Y \sim \mathrm{B}(\pi)} [\losscpe(\Y, u)], \label{pbr}
\end{eqnarray}
where $\mathrm{B}(\pi)$ denotes a Bernoulli for picking label $\Y = 1$. Most DT induction algorithms follow the greedy minimisation of a loss which is in fact a Bayes risk \citep{kmOT}. For example, up to a multiplicative constant that plays no role in its minimisation, the square loss gives Gini criterion, $\bayessqrisk(\pi) = (1/2) \cdot \pi (1-\pi)$ \citep{bfosCA}; the log loss gives the information gain, $\bayeslogrisk(u) = -\pi\log(\pi)-
(1-\pi)\log(1-\pi)$ \citep{qC4} and the 0/1 loss gives the empirical risk $\bayesZOrisk(u) = \min\{\pi, 1-\pi\}$. To follow \citet{kmOT}, we assume wlog that all Bayes risks are normalized so that $\bayesrisk(1/2)=1$, which is the maximum for any symmetric proper loss \citep{nnOT}, and $\bayesrisk(0) = \bayesrisk(1) = 0$ (the loss is \textit{fair}, \citet{rwCB}). Any Bayes risk is concave \citep{rwCB}. So, if $h$ is a DT, then the loss minimized to greedily learn $h$, $F(h;\mathcal{S})$, can be defined in general as: 
\begin{eqnarray}
F(h;\mathcal{S}) & \defeq & \expect_{\mathcal{S}}[\bayesrisk(q(\ell(\ve{x}_i)))],\label{defLOSSGENPROB}
\end{eqnarray}
where $\ell(.)$ is the leaf reached by $\ve{x}_i$ in $H$\footnote{Not to be confused with the general notation of a loss for class probability estimation, $\losscpe$.} and $q(\ell) \defeq \hat{p}[\Y = 1 | \ell]$ is the relative proportion of class $1$ in the examples reaching $\ell$. To ensure that a real valued classification is taken at each leaf of $h$, the predicted value for leaf $\leaf$ is 
\begin{eqnarray}
h(\ell) & \defeq & -\bayesrisk'(q(\ell)) \in \mathbb{R}.\label{predDT} 
\end{eqnarray}
Function $-\bayesrisk'$ is called the \textit{canonical link} of the loss \citep{bssLF,nwLO,rwCB}. If the loss is non differentiable, the canonical link is obtained from any selection of its subdifferential.

If $H$ is a linear combination of base classifiers, we adopt the convex dual formulation of (negative) the Bayes risk which, by the property of Bayes risk, admits a domain that can be the full $\mathbb{R}$ \citep{bvCO}. In this case, we replace \eqref{defLOSSGENPROB} by the following loss:
\begin{eqnarray}
F(h;\mathcal{S}) & \defeq & \expect_{\mathcal{S}}[(-\bayesrisk)^\star(-y_i h(\ve{x}_i))],\label{defLOSSGENR}
\end{eqnarray}
where $\star$ denotes the Legendre conjugate
of $F$, $F^\star(z)\defeq \sup_{z' \in \mathrm{dom}(F)}\{zz' - F(z')\}$
\citep{bvCO}. Losses like \eqref{defLOSSGENR} are sometimes called balanced convex losses \citep{nnOT} and belong to a broad class of losses also known as margin losses \citep[Section 2.3]{mvAV}. The most popular losses are particular cases of \eqref{defLOSSGENR}, like the square or logistic losses \citep{mvAV}. It can be shown that if a DT has its outputs mapped to $\mathbb{R}$ following the canonical link \eqref{predDT}, then minimizing \eqref{defLOSSGENR} to learn the DT is equivalent to minimizing \eqref{defLOSSGENPROB}, which therefore make both views equivalent \citep[Theorem 3]{nnBD}. Finally, the \textit{empirical risk} of $H$, $\emprisk(H)$, is \eqref{defLOSSGENR} in which the inside brackets is predicate $\iver{y_i h(\ve{x}_i) < 0}$.

\noindent $\triangleright$ \textbf{Differential privacy} (DP) essentially relies on randomized mechanisms to guarantee that \textit{neighbor} inputs to an algorithm $\mathcal{M}$ should not change too much its distribution of outputs \citep{dmnsCN}. In our context, $\mathcal{M}$ is a learning algorithm and its input is a training sample (omitting additional inputs for simplicity) and two training samples ${\mathcal{S}}$ and ${\mathcal{S}}'$ are neighbors, noted ${\mathcal{S}}\approx {\mathcal{S}}'$ iff they differ by at most one example. The output of $\mathcal{M}$ is a classifier $h$.
\begin{definition}\label{defEDPRIV} Fix $\epsilon\geq 0$. $\mathcal{M}$ gives $\epsilon$-DP if $p[\mathcal{M} ({\mathcal{S}}) = h] \leq \exp(\epsilon) \cdot  p[\mathcal{M} ({\mathcal{S}}') = h], \forall {\mathcal{S}}\approx {\mathcal{S}}', \forall h$,
where the probabilities are taken over the coin flips of $\mathcal{M}$. 
\end{definition}
The smaller $\epsilon$, the more private the algorithm. Privacy comes with a price which is in general the noisification of $\mathcal{M}$. A fundamental quantity that allows to finely calibrate noise to the privacy parameters relies on the \textit{sensitivity} of a function $f(.)$, defined on the same inputs as $\mathcal{M}$, which is just the maximal possible difference of $f$ among two neighbor inputs. Assuming $\mathrm{Im} f\subseteq {\mathbb{R}}^n$, the global sensitivity of $f(.)$, $\Delta_f$, is $\Delta_f \defeq \max_{{\mathcal{S}}\approx {\mathcal{S}}'} \|f({\mathcal{S}}) - f({\mathcal{S}}')\|_1$ \citep{dmnsCN}.
DP offers two standard tools to devise general mechanisms with $\varepsilon$-DP guarantees, one to protect real values and the other to protect a choice in a fixed set \citep{drTA,mtMD}. The former, the Laplace mechanism, adds $\mathrm{Lap}(b)$ noise to a real-valued input, with $b\defeq \Delta_f /\epsilon$ is the scale parameter.
The latter is the exponential mechanism: let $\{g : g\in {\mathcal{G}}\}$ denote a set of alternatives and $f: {\mathcal{R}} \rightarrow {\mathbb{R}}$ a function that scores each of them (the higher, the better), whose values depend of course on $\mathcal{S}$. The exponential mechanism outputs $g\in {\mathcal{G}}$ with probability $\propto \exp(\epsilon f(g)/(2\Delta_f))$, thus tending to favor the highest scores. 
Finally, the \textit{composition theorem}, particularly useful when training $h$ is iterative like for DTs, states that the sequential application of $\epsilon_i$-DP mechanisms ($i=1, 2, ...$), provides $\left( \sum_i \epsilon_i \right)$-DP \citep{dmnsCN}.

%% file: content-arxiv/dilemma.tex
\section{The Privacy vs Boosting Dilemma for DT}\label{sec:priv}

Let $\leafset (h)$ denote the set of leaves of tree $h$. Let $\ve{w} \in (0,1]^m$ denote a set of non-normalized weights over the training sample $\mathcal{S}$. Because $h$ produces a partition of $\mathcal{S}$, we rewrite the loss \eqref{defLOSSGENPROB} as $w(\mathcal{S}) \cdot F(h;\mathcal{S}) =  \sum_{\leaf \in \leafset (h)} f_{\bayesrisk}(h, \leaf, \mathcal{S})$ with\footnote{We multiply both sides by $w(\mathcal{S})$ to follow \citet{fsDM}; $w(\mathcal{S})$ is indeed constant when growing a tree and does not influence the exponential mechanism.}
\begin{eqnarray}
f_{\bayesrisk}(h, \leaf, \mathcal{S}) & \defeq & w(\leaf) \cdot \bayesrisk\left( \frac{w^1(\leaf)}{w(\leaf)} \right),\label{defMECH}
\end{eqnarray}
and $w(\mathcal{S}) \defeq 1^\top \ve{w}, w^1(\leaf) \defeq \sum_i \iver{(i \in \leaf) \wedge (y_i = 1)}\cdot w_i, w(\leaf) \defeq \sum_i \iver{i \in \leaf}\cdot w_i$ and $i \in \leaf$ is the predicate "observation $\ve{x}_i$ reaches leaf $\leaf$ in $h$". Following \citet{fsDM}, we want to compute the sensitivity of $f$, 
\begin{eqnarray}
\Delta_{\bayesrisk}(h, \lambda) & \defeq & \sup_{{\mathcal{S}}' \approx {\mathcal{S}}} |f_{\bayesrisk}(h, \leaf, \mathcal{S}')-f_{\bayesrisk}(h, \leaf, \mathcal{S})|\label{gensen}
\end{eqnarray}
(we sometimes note $\Delta_{\bayesrisk}$ to save readability). To compute it, we need a definition from convex analysis, \textit{perspectives}. 
\begin{definition}\citep{mOA1,mOA2}\label{defpt}
Given closed convex function $f$, the \textbf{perspective} of $f$, noted
$\check{f}(x,y)$ is:
\begin{eqnarray}
\check{f}(x,y) & \defeq & y\cdot f (x/y)\:\:, \mbox{
if } y>0\:\:,\label{defPF}
\end{eqnarray}
and otherwise $\check{f}(x,y) \defeq f0^+(x)$ if $y=0$ and
$\check{f}(x,y) \defeq +\infty$ if $y<0$. Here, $f0^+$ is the
recession function of $f$.
\end{definition}
To save notations, we extend this notion to Bayes risks, that are concave, and therefore write for short $\check\bayesrisk \defeq -\left(\reallywidecheck{-\bayesrisk}\right)$. 
\begin{theorem} \label{th3P}
$\Delta_{\bayesrisk}(h, \lambda) \leq \max\{3, 1 + \check{\bayesrisk}(1, m+1)\}$. 
\end{theorem}
Theorem \ref{th3P} generalizes SOTA in two ways, first because only up to 4 Bayes risks were covered \citep{fsDM}, and second because classical analyses have $\ve{w}$ uniform (which precludes boosting). We now show that the variation of a perspective transform of a Bayes risk is linked to its \textit{weight} (or curvature, \citet{rwCB}). 
\begin{lemma}\label{lemcurv}
For any twice differentiable $\bayesrisk$, for any $m$, there exists $a\in [0,1/(m+1)]$ such that
\begin{eqnarray}
 (\check{\bayesrisk}')(1, m+1) & = & (m+1)^{-2}\cdot(-\bayesrisk'')(a)\:\:.\label{eqCURV1}
\end{eqnarray}
\end{lemma}
The proof of Theorem \ref{th3P} includes the proof that the bound is in fact almost tight as some neighboring samples admit $\Delta_{\bayesrisk}(h, \lambda) = \check{\bayesrisk}(1, m)$, so the variation in DP budget with $m$ is directly linked to \eqref{eqCURV1}. In other words, the larger the weight ($-\bayesrisk''$, \citet{rwCB}), the more expensive becomes DP with $m$ when relying on $\Delta_{\bayesrisk}$ as sensitivity measure --- such as the exponential mechanism in \citet{fsDM}. It turns out that it has long been known that \textit{boosting}'s convergence works the exact \textit{same} way: the larger the weight, the better is the rate guaranteed under boosting-compliant assumptions \citep{kmOT}. Since the top-down induction of a greedy tree gradually spends privacy budget to split each node, the boosting vs privacy dilemma is thus to guarantee fast enough convergence --- because it also saves budget as we converge in less iterations --- while keeping the privacy budget within required bounds. We now give an example of the budget required for popular Bayes risks using Theorem \ref{th3P}. $\bayesmatrisk(u) \defeq 2\sqrt{u(1-u)}$ is Bayes risk of Matsushita loss \citep{nnOT,nnBD}, which guarantees optimal boosting convergence \citep{kmOT} and thus, as expectable, is the most "expensive" DP-wise.
\begin{lemma}\label{lemPhiProof}
$\forall \bayesrisk \in \{\bayesmatrisk, \bayeslogrisk, \bayessqrisk, \bayesZOrisk\}$, we have $\Delta_{\bayesrisk} \leq \max\{3, 1 + \Delta^*_{\bayesrisk}(m)\}$ where $\Delta^*_{\bayesmatrisk} = 2\sqrt{m}$, $\Delta^*_{\bayeslogrisk} = (1+\log(m+1))\cdot \log^{-1} 2 $, $\Delta^*_{\bayessqrisk} = 4m/(m+1)$, $\Delta^*_{\bayesZOrisk} = 2$.
\end{lemma}
We note that  all our bounds are within 1 of the bounds known for $\bayeslogrisk, \bayessqrisk, \bayesZOrisk$ \citep{fsDM}, so the generality of Theorem \ref{th3P} (all applicable Bayes risks, non-uniform weights over examples) comes at reduced price. 

%% file: content-arxiv/matsu.tex
\section{The M$\alpha$-loss}\label{sec:matalpha}
\begin{figure*}[t]
\begin{tabular}{cc}
\resizebox{.70\linewidth}{!}{
\hspace{-1.8cm} \begin{minipage}{\linewidth}
\begin{eqnarray*}
\alphalink(u) \hspace{-0.3cm} & \in & \hspace{-0.3cm} \alpha \cdot \frac{2u - 1}{\sqrt{u(1-u)}} -2(1-\alpha)\cdot \left\{
\begin{array}{rcl}
1 & \mbox{ if } & u < 1/2\\
\big[ -1,1 \big] & \mbox{ if } & u = 1/2\\
-1 & \mbox{ if } & u > 1/2
\end{array}
\right. ,\nonumber\\
{\alphalink}^{-1}(z) \hspace{-0.3cm} & = & \hspace{-0.3cm} \frac{1}{2} \cdot \left( 1 + \iver{z \not\in 2(1-\alpha)\cdot
      [-1,1]}\cdot \frac{\frac{z}{2} - \mathrm{sign}(z) \cdot  (1-\alpha)}{\sqrt{\alpha^2+\left(\frac{|z|}{2} - (1-\alpha)\right)^2} }\right) ,\\
\alphasur(z) \hspace{-0.3cm} & = & \hspace{-0.3cm}  1 - \frac{z}{2} + \iver{z \not\in 2(1-\alpha)\cdot [-1,1]}\cdot \left(\sqrt{\alpha^2+\left(\frac{|z|}{2} - (1-\alpha)\right)^2} - \alpha\right).
\end{eqnarray*}
\end{minipage}
} &  \hspace{-2cm} \begin{minipage}{\linewidth}
\begin{tabular}{cc}
\includegraphics[trim=10bp 10bp 30bp
10bp,clip,width=0.18\linewidth]{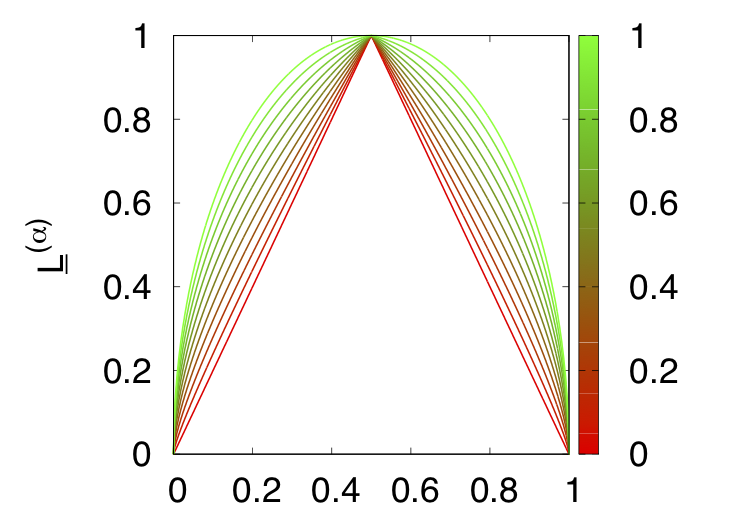} & \hspace{-0.5cm}\includegraphics[trim=10bp 10bp 30bp
10bp,clip,width=0.18\linewidth]{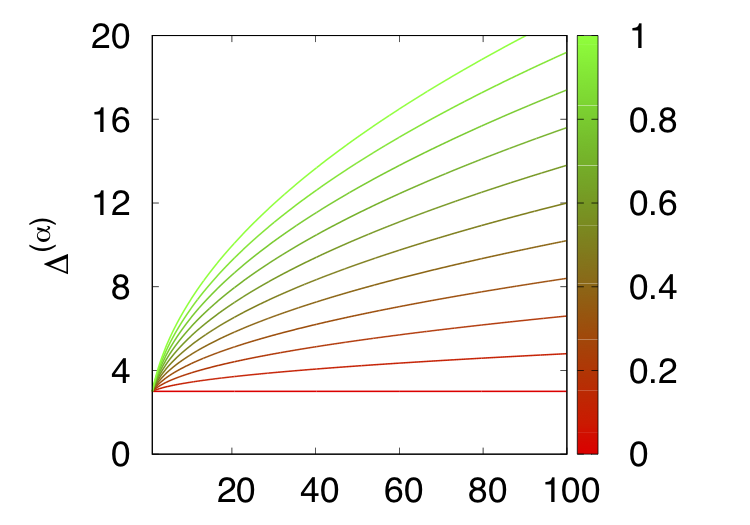}\\
\includegraphics[trim=10bp 10bp 30bp
10bp,clip,width=0.18\linewidth]{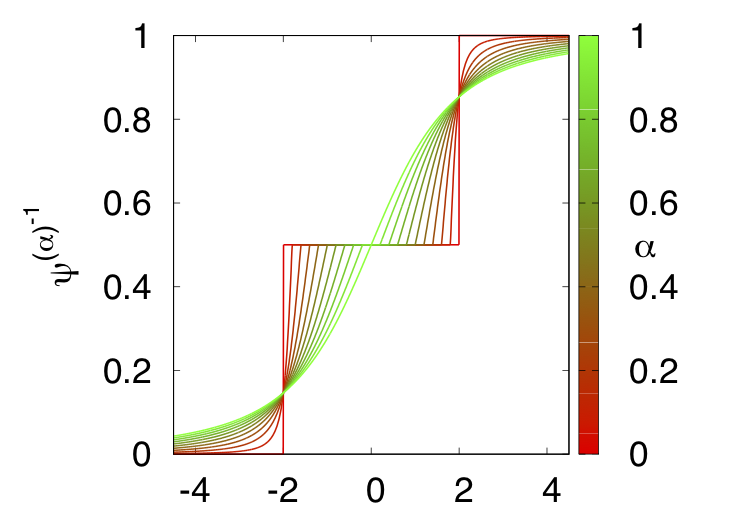} & \hspace{-0.5cm}
\includegraphics[trim=10bp 10bp 30bp
10bp,clip,width=0.18\linewidth]{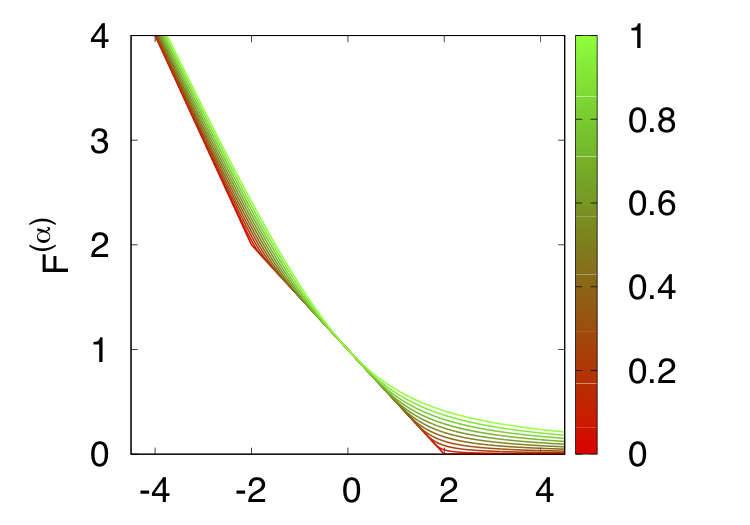}
\end{tabular}
\end{minipage}
\end{tabular}
\caption{\textit{Left}: canonical link $\alphalink$, inverse canonical link
  ${\alphalink}^{-1}$ and convex surrogate $\alphasur$ for the
  M$\alpha$-loss. \textit{Right}: plots of Bayes risk
  $\bayesalpharisk(q)$, sensitivity $\Delta^{(\alpha)}\defeq
  \Delta_{\bayesalpharisk}(m)$, ${\alphalink}^{-1}(z)$ and
  $\alphasur(u)$ for the M$\alpha$-loss,
for various $\alpha$s (colors).}
  \label{f-alphaM}
\end{figure*}
In the boosting vs DP picture, there are two extremal losses. The 0/1
loss is the one that necessitates the smallest DP budget (Lemma
\ref{lemPhiProof}) but achieves the poorest convergence guarantee
\citep[Section 5.1]{kmOT}. On the other side of the spectrum,
Matsushita loss guarantees the optimal convergence rate
\citep{kmOT,nnOD} but necessitates a considerable DP budget (Lemmata
\ref{lemcurv}, \ref{lemPhiProof}). We address
the challenge of tuning the convergence rate vs DP budget by creating a
new proper symmetric loss, allowing to stand anywhere in between these
extremes via a simple tunable parameter $\alpha$.
\begin{definition}\label{defMatsualpha}
The M$\alpha$-loss is defined for any $\alpha \in [0,1]$ by the
following partial losses, for $y \in \{-1,1\}$:
\begin{eqnarray*}
\partialalphaloss{y}(u) & \defeq & 2\alpha \cdot \left(\frac{1-u}{u}\right)^{\nicefrac{y}{2}} + 2(1-\alpha)\cdot \iver{yu\leq y/2}.
\end{eqnarray*}
\end{definition}
It is easy to check that the M$\alpha$-loss is 
proper (strictly if $\alpha > 0$) and symmetric, as well as its Bayes risk is
a convex combination of those of the 0/1 and
Matsushita losses: 
\begin{eqnarray*}
\bayesalpharisk(u) & = & \alpha \cdot \bayesmatrisk(u) + (1-\alpha)
                         \cdot \bayesZOrisk(u).
\end{eqnarray*}
It is also not hard to show that the sensitivity intrapolates between
both losses' sensitivities using Lemma \ref{lemPhiProof}.
\begin{corollary}\label{sensALPHA}
The sensitivity \eqref{gensen} of the M$\alpha$-loss satisfies
$\Delta_{\bayesalpharisk} \leq \Delta^*_{\bayesalpharisk}(m) \defeq
  3+2\alpha(\sqrt{m} - 1)$.
\end{corollary}
Because of the 0/1 loss is not differentiable, getting the inverse canonical link and the
convex surrogate is trickier. 
\begin{theorem}\label{theoremalpha}
The canonical link $\alphalink$, inverse canonical link
${\alphalink}^{-1}$ and surrogate $\alphasur$ of the
M$\alpha$-loss are as in Fig. \ref{f-alphaM}.
\end{theorem}

%% file: content-arxiv/boost.tex
\section{Boosting with the M$\alpha$-loss}\label{sec-boost}

\noindent$\triangleright$ \textbf{Boosting decision trees}: We know from the last Section that the M$\alpha$-loss allows, by
tuning $\alpha$, to continuously change the sensitivity of the
criterion between the minimal ($\alpha = 0$) and a maximal one ($\alpha = 1$). We are now going to show that the criterion allows as well to
intrapolate between optimal boosting regime ($\alpha = 1$) and a
"minimal" convergence guarantee ($\alpha = 0$), thereby completing the
boosting vs privacy picture for the M$\alpha$-loss. We first tackle the
induction of a single DT as in \citet{kmOT}. Boosting start
by formulating a \textit{Weak Learning Assumption} (WLA) which gives a weak
form of correlation with labels for the elementary block of a
classifier. In the case of a DT, such a block is a split. So, consider
leaf $\leaf$ and a test $g : {\mathcal{X}} \rightarrow \{0,1\}$ that
splits the leaf in two, the examples going to the left (for which $g=0$)
and those going to the right (for which $g=1$). The relative weight of
positive examples reaching $\leaf$ is $q \in (0,1)$, where $q\neq 0,1$
ensures that the leaf is not pure. Define the balanced
weights at leaf $\leaf$ to be (a) $\tilde{w}_i = 0$ if $\ve{x}_i\not\in\leaf$, else (b) $\tilde{w}_i = 1/(2q)$ if $y_i = +1$, else (c)
$\tilde{w}_i = 1/(2(1-q))$. Let $\tilde{\ve{w}}_{\leaf}$ denote the complete
distribution and $g^{\nicefrac{+}{-}} \in \{-1, 1\}^{\mathcal{X}}$ as $g^{\nicefrac{+}{-}}
\defeq -1 + 2g$. We adopt
the \textit{edge}
notation $\eta(\tilde{\ve{w}}, h) \defeq \sum_i \tilde{w}_i y_i h(\ve{x}_i)$ for any $h
\in \mathbb{R}^{\mathcal{X} }$. Suppose $\upgamma > 0$ a constant. 
\begin{definition}(WLA for DT)\label{wlaKMDT}
Split $g$ at leaf $\leaf$ satisfies the $\upgamma$-WLA iff $|\eta(\tilde{\ve{w}}_\leaf, g^{\nicefrac{+}{-}})| \geq \upgamma$.
\end{definition}
Definition \ref{wlaKMDT} is not the same as \citet{kmOT}, but it is
equivalent (\sm, $\S$ \ref{proof_thBoostDT1}) and in fact more 
convenient for our framework. A random split would not satisfy the WLA so the
WLA enforces the existence of splits at least moderately correlated
with the class. Top-down DT induction usually does
\textit{not} proceed by optimizing the split based on the WLA, but in
fact it can be shown that the WLA \textit{implies} good splits according to
top-down DT induction criteria \citet[Section 5.3]{kmOT}. 
So we let $h_\ell$ denote the current DT with $\ell$
leaves and $\ell-1$ internal nodes. We grow it to get $h_{\ell+1}$ by minimizing
$\bayesalphariskparam{\ell}$ as:
\begin{eqnarray}
\bayesalphariskparam{\ell}(h_{\ell+1}) & \defeq & \alpha_{\ell} \cdot \bayesmatrisk(h_{\ell+1})  + (1-\alpha_{\ell}) \cdot \bayeserrrisk(h_{\ell+1}),\nonumber
\end{eqnarray}
where $h_{\ell+1}$ is $h_\ell$ with a leaf $\leaf \in \leafset(h_\ell)$
replaced by a split. Noting $K\geq 0$ a constant, $w \defeq \sum_i w_i$ and
\begin{eqnarray}
\tilde{w}_\ell & \defeq & \left(\frac{1}{w}\right) \cdot \sum_i w_i \iver{i\in \leaf_\ell} \quad
                       \in
                       [0,1]\label{defNW}
\end{eqnarray}
 the total \textit{normalized} weight of the examples reaching
$\leaf_\ell$, we say that the sequence of $\alpha$s is \textit{$K$-monotonic}
iff $\alpha_\ell \leq \alpha_{\ell-1} \cdot \exp\left(K \tilde{w}_\ell
  (1-\alpha_{\ell-1})\right)$ for any $\ell>2$ (and $\alpha_1 \in [0,1]$). Since the parameter in the
$\exp$ is $\geq 0$, $K$-monotonicity prevents the sequence from
growing too fast.
\begin{theorem}\label{thBoostDT1}
Suppose all splits satisfy the $\upgamma$-WLA and the sequence of
$\alpha$s is $(\upgamma^2/16)$-monotonic. Then $\forall \xi \in (0,1]$, the empirical risk of $h_L$ satisfies $\emprisk(h_L) \leq \xi$ as long as
\begin{eqnarray}
\sum_{\ell=1}^L \tilde{w}_\ell \alpha_\ell & \geq & \left(\frac{16}{\upgamma^2}\right) \cdot \log \left(\frac{1}{\xi}\right).\label{eqBOOST11}
\end{eqnarray}
\end{theorem}
It is worth remarking
that this is indeed a generalization of \cite{kmOT}: suppose
$\alpha_{\ell} = \alpha$ constant (which is
$(\upgamma^2/16)$-monotonic $\forall \upgamma \geq 0$) and we pick at each
iteration the heaviest leaf to split. We thus have $\tilde{w}_\ell\geq
1/\ell$, assuming further it satisfies the WLA. Since $\sum_{\ell=1}^L (1/\ell) \geq
\int_1^{L+1} \mathrm{d}z/z = \log(L+1)$, \eqref{eqBOOST11} is
guaranteed if
\begin{eqnarray}
L & \geq & \left(\frac{1}{\xi}\right)^{\frac{16}{\alpha \upgamma^2}}, \label{boostCOND}
\end{eqnarray}
which, for $\alpha=1$, is in fact the square root of the bound in \citet[Theorem
10]{kmOT} and is thus significantly better. Rather than a quantitative
improvement, we were seeking for a qualitative one as \citet{kmOT}
pick the heaviest leaf to split, which means using DP budget to find
it. To see how we can get essentially the same guarantee
\textit{without} this contraint, suppose instead that we split
\textit{all} current leaves, \textit{all} of them satisfying the
WLA\footnote{This happens to be reasonable on domains big enough, for small trees or when the
  set from which $g$ is picked is rich enough.}. Since
$\sum_{\leaf} \tilde{w}(\leaf) = 1$ (those weights are normalized), once we remark that it
takes one split for the root, then two, then four and so on to fully
split the current leaves, $L$ boosting iterations guarantee a full
split up to depth $O(\log(L))$, which delivers the same condition as
\eqref{boostCOND} with an eventual change in the exponent
constant. Our result is also a generalization of \citet{kmOT} since it allows
to \textit{tune} $\alpha$ during learning, which is important
for us ($\S$ \ref{sec-solv}).\\

\noindent$\triangleright$ \textbf{Boosting linear combinations of
  classifiers}: we now consider that we build a Linear Combination
(LC) of classifiers, $H_T \defeq
\sum_{t=1}^T \beta_t h_t$, where $h_t : \mathcal{X} \rightarrow
\mathbb{R}$ is a real valued classifier --- this could be a DT or any
other applicable classifier. We tackle the problem of achieving
boosting-compliant convergence when building $H_T$, which means we
have a WLA on each $h_t$. We also assume
$\exists M>0$ such that $|h_t|\leq M$. Let $\ve{w}_t
\in [0,1]^m$ an \textit{unnormalized} weight vector on $\mathcal{S}$, $t$ denoting the
iteration number from which $h_t$ is obtained. Noting $\tilde{w}_t \defeq \left(\nicefrac{1}{m}\right)\cdot \sum_i w_{ti} \in [0,1]$
the expected unnormalized weight at iteration $t$, we also let
$\tilde{\ve{w}}_t \defeq (1/(m \tilde{w}_t))\cdot \ve{w}_t$ denote the
\textit{normalized} weight vector at iteration $t$. The WLA is as follows.
\begin{definition}(WLA for LC)\label{wlaKMLC}
$h_t$ obtained at iteration $t$ satisfies the $\upgamma$-WLA iff
$|\eta(\tilde{\ve{w}}_t, h_t)| \geq \upgamma M$.
\end{definition}
Remark that this definition is similar to Definition \ref{wlaKMDT},
since $|g^{\nicefrac{+}{-}}| = 1$. All the crux is now how to get the weight vectors $\ve{w}_t$ so that
we can prove a boosting-compliant convergence rate using the
M$\alpha$-loss. We do so using a standard mechanism, which consists in
initializing $\bm{w}_1 \defeq (1/2)\cdot \ve{1}$ (unnormalized) and then using the
mirror update of the M$\alpha$-loss to update weights after $h_t$ has
been received:
\begin{eqnarray}
w_{(t+1)i} \hspace{-0.3cm} & \leftarrow & \hspace{-0.3cm}{\alphalink}^{-1}\left( -\beta_t y_{i}
                          h_t(\bm{x}_i) + {\alphalink}
                          (w_{ti})\right) \label{defwunMF},
\end{eqnarray}
where $\beta_t$ is a leveraging coefficient for $h_t$ in the final
classifier, taken to be $\beta_t \leftarrow a \tilde{w}_t \cdot
\eta(\tilde{\ve{w}}_t, h_t)$, where 
$a$ is a constant chosen beforehand anywhere in interval $
(\alpha/M^2)\cdot \left[ 1 - \pi, 1 + \pi\right]$, $\pi \in [0,1)$ quantifying the
freedom in choosing $a$. This is sufficient to complete the
description of the algorithm (also given \textit{in
  extenso} in \sm, $\S$ \ref{proof_thBoostLC1}).
\begin{theorem}\label{thBoostLC1}
Suppose all $h_t$ satisfy the $\upgamma$-WLA. Then $\forall \xi
\in [0,1]$, we have $\emprisk(H_T)\leq
\xi$ as long as:
\begin{eqnarray}
\sum_{t=1}^T
                                                  \tilde{w}^2_t & \geq &\frac{2
                                                                    (1-\xi)}{(1-\pi^2)\upgamma^2
                                                                         \alpha}.\label{eqCONST111M}
\end{eqnarray}
\end{theorem}
This Theorem
has a very similar flavour on boosting conditions as we had in Theorem
\ref{thBoostDT1} for DTs but its dependence on $\xi$ is comparatively
misleading. What Theorem \ref{thBoostLC1} indeed tells us is
boosting for LC is efficient under the WLA as long as $w_t$ is "large"
enough in $[0,1]$. The weight update in \eqref{defwunMF} meets the classical
boosting property that an example has its weight directly correlated
to classification: the better, the smaller its weight (\textit{Cf} the plot of ${\alphalink}^{-1}$ in Figure
\ref{f-alphaM}). Hence, as classification gets better, the sum on the
LHS of \eqref{eqCONST111M} increases at smaller rate and if
$\xi$ is too small, this means a potentially larger number of
iterations to meet \eqref{eqCONST111M}.

%% file: content-arxiv/dilemma-alpha.tex
\section{Privacy and boosting: objective calibration}\label{sec-solv}

We have so far described the complete picture of DP for DT with any noisification mechanism that relies on the sensitivity of a Bayes risk, and the complete but noise-free boosting picture for the M$\alpha$-loss for DT and LC. We now assemble them. In an iterative boosted combination of DT, two locations of privacy budget spending can make the full classifier meet DP: (a) node splitting in trees, (b) leaf predictions in trees. The protection of the leveraging coefficients $\beta_t$ can be obtained in two ways: either we multiply each leaf prediction by $\beta_t$, then replace $\beta_t \leftarrow 1$ and then carry out (b), or use the faster but more conservative approach to just do (b) \textit{e.g.} with the Laplace mechanism from which follows the protection of $\beta_t$ ($\S$ \ref{sec-boost}). \textit{We do not carry out pruning} as boosting alone can be sufficient for good generalization, see \textit{e.g.} \citet[Section 2.1]{sfblBT}, \citet[Theorems 16, 17]{bmRA}, and pruning also requires privacy budget \citep[$\S$ 3.5]{fiDT}. The public information is the attribute domain, which is standard \citep{fiDT}, and we consider that each continuous attributes is regularly quantized using a public number $\nvpriv$ of values. This makes sense for many common attributes like age, percentages, $\$$-value, and this can contribute to ease interpretation; this also has three technical justifications: (1) a private approaches requires budget, (2) $\nvpriv$ allows to tightly control the computational complexity of the whole DT induction, (3) boosting does not require exhaustive split search provided $\nvpriv$ is not too small (more in \sm, $\S$\ref{sub-sec-gen}). 

\noindent $\triangleright$ \textbf{Private induction of a DT: objective calibration}. The overall privacy budget $\epsilon$ is split in two proportions: $\betatree$ for node splitting (a) and $\betapred\defeq 1-\betatree$ leaves' predictions (b). The basis of our approach to split nodes is the nice --- but never formally analyzed --- trick of \citet{fsDM} which consists in using the exponential mechanism to choose splits. Let $\mathcal{G}$ denote the whole set of splits. The probability to pick $g \in \mathcal{G}$ to split leaf $\leaf \in \leafset(h_\ell)$ is:
\begin{eqnarray}
\pexpm ((g, \leaf)) \hspace{-0.3cm} & \propto &  \hspace{-0.3cm}
                    \exp\left(-\frac{\epsilon(h_\ell, \leaf) w(\mathcal{S}) F(h_\ell\oplus(g, \leaf))}{2\Delta^*_{\bayesalpharisk}(m)} \right),\label{expSPLIT}
\end{eqnarray}
where notation $h_\ell\oplus(g, \leaf)$ refers to decision tree $h_\ell$ in which leaf $\leaf$ is replaced by split $g$, $w(\mathcal{S}) \cdot F(h_\ell\oplus(g, \leaf))$ is the unnormalized Bayes risk (Section \ref{sec:priv}) and $\Delta^*_{\bayesalpharisk}(m)$ is given in Lemma \ref{sensALPHA}. $\epsilon(h_\ell, \leaf)$ is the fraction of the total privacy budget allocated to the split.  So far, all recorded approaches consider uniform budget spending \citep{fiDT} but such a strategy is clearly oblivious to the accuracy vs privacy dilemma as explained in Section \ref{sec:priv}. 
We now introduce a more sophisticated approach exploiting our result, allowing to bring strong probabilistic guarantees on boosting \textit{while} being private. The intuition behind is simple: the "support" (total unnormalized weight) of a node is monotonic decreasing on any root-to-leaf path. Therefore, we should typically increase the budget spent in low-depth splits because (i) it impacts more examples and (ii) it increases the likelihood of picking the splits that meet the WLA in the exponential mechanism \eqref{expSPLIT}. Consequently, we also should pick $\alpha$ larger for low-depth splits, to increase the early boosting rate and drive as fast as possible the empirical risk to the minimum, yet monitoring the dependency of the exponential mechanism in $\alpha$ to control the probability of picking the splits that meet the WLA. This may look like a quite intricate set of dependences between privacy and boosting, but here is a solution that matches all of them.
If we denote $h_1$ the tree reduced to a leaf from which $h_\ell$ was built, $\depth(.)$ as the depth of a node, $d$ the maximal depth of a tree and $T$ the number of trees in the combination, then we let:
\begin{eqnarray}
\alpha_\ell  & \defeq & \frac{\emprisk(h_\ell)}{\emprisk(h_1)} \quad(\in
[0,1]),\label{fixALPHA} \\
\epsilon(h_\ell, \leaf) & = & \frac{\betatree}{Td2^{\depth(\leaf)}}\cdot
                 \epsilon\label{fixEPSILON}.
\end{eqnarray}
The choice of  $\alpha_\ell$ makes it decreasing along every path from the root: while we split the root using Matsushita loss ($\alpha = 1$), which guarantees optimal boosting rate, we gradually move in deeper leaves to using more of the Bayes risk of the 0/1 loss, which may reduce the rate but reduces privacy budget used as well. Referring to objective perturbation which noisifies the loss \citep{cmsDP}, we call our method that \textit{tunes} the loss \textbf{objective calibration} (O.C).

We formally analyze O.C. First, remark that the total budget spent for one tree is $\betatree \epsilon/T$, which fits in the global budget $\varepsilon$. To develop the boosting picture, we build on the $\upgamma$-WLA. We first remark that for any $h$, leaf $\leaf \in  \leafset(h)$ and split $g \in \mathcal{G}$, there exists $u\geq 0$ such that
\begin{eqnarray}
\hspace{-0.3cm}\bayesalpharisk(h) - \bayesalpharisk(h
      \oplus (g, \leaf)) \hspace{-0.3cm} & = & \hspace{-0.3cm}
  u \cdot \frac{\upgamma^2 \alpha \tilde{w}(\lambda) }{16} \cdot \bayesalpharisk(h).\label{condGAP}
\end{eqnarray}
This is a simple consequence of the concavity of any Bayes risk. Interestingly, for \textit{all} splits that satisfy the WLA, it can be shown that we can pick $u\geq 1$ (\sm, $\S$ \ref{proof_thBOOSTDP1}). Let us denote $\setfat \subseteq \mathcal{G}$ the whole set of such boosting amenable splits, and let $\setslim \defeq \mathcal{G} \backslash \setfat$ denote the remaining splits. The exponential mechanism might of course pick splits in $\setslim$ but let us assume that there is at least a small "gap" between those splits and those of $\setfat$, in such a way that for any split in $\setslim$, \eqref{condGAP} holds only for $u\leq \delta$ for some $\delta < 1$. This property always holds for \textit{some} $\delta<1$ but let us assume that this $\delta$ is a constant, just like the $\upgamma$ of the WLA, and call it the \textit{$\delta$-Gap assumption}. Let $\nodeset(h)$ denotes the set of nodes of $h$, including
leaves in $\leafset(h)$. The \textit{tree-efficiency} of $\node \in \nodeset(h)$ in
$h$ is defined as 
\begin{eqnarray}
J(\node, h) & \defeq & \frac{8\tilde{w}(\node)
                                                 \emprisk(h)^2}{2^{\depth(\node)}}
                      \quad \in [0,1],
\end{eqnarray}
where $\tilde{w}(\node)$ is the normalized weights of examples reaching $\node$. Let $\mathcal{L}$ be a subset of indexes of the leaves split from $h_1$ to create a depth-$d$ tree, with unnormalized weights $\ve{w}$. Each element $\ell$ refers to a couple $(\leaf_\ell, h_\ell)$ where $h_\ell$ is the tree in which $\leaf_\ell$ was replaced by a split.

\begin{theorem}\label{thBOOSTDP1}
Suppose the exponential mechanism is implemented with $\alpha_\ell$ and $\epsilon_\ell$ as in \eqref{fixALPHA}, \eqref{fixEPSILON}. Suppose $Td \leq \log m$, $m\geq 3$ and 
both the $\upgamma$-WLA and $\delta$-Gap assumptions hold. Suppose that $\forall \ell \in \mathcal{L}$, 
\begin{eqnarray*}
|{\setfat}_\ell| & \geq & |{\setslim}_\ell| \cdot \exp\left(- \Omega\left(  J(\leaf_\ell, h_\ell) \cdot \frac{
         \epsilon \sqrt{m}}{\log m}\right)\right).
\end{eqnarray*} 
Then, for any $\xi>0$, if
\begin{eqnarray}
\min_{\ell \in \mathcal{L}} J(\leaf_\ell, h_\ell) & = & \Omega \left(\frac{\log m}{
         \epsilon \sqrt{m}} \log \frac{|\mathcal{L}|}{\xi}\right),\label{constJJTREE}
\end{eqnarray}
then with probability $\geq 1 - \xi$, \textbf{all} splits chosen by the exponential mechanism to split the leaves indexed in $\mathcal{L}$ satisfy the WLA.
\end{theorem}
The proof (\sm, $\S$ \ref{proof_thBOOSTDP1}), explicits all hidden constants. We insist on the message that Theorem \ref{thBOOSTDP1} carries about the exponential mechanism: under the WLA/Gap assumptions and a size constraint on each ${\setfat}_t$ (which by the way authorises it to be reasonably smaller than ${\setslim}_t$), the exponential mechanism has essentially \textit{no negative impact} on boosting with high probability. This, we believe, is a very strong incentive in favor of the exponential mechanism as designed in \citet{fsDM}. Finally, the condition $Td \leq \log m$ could be replaced by a low-degree polylog but even without doing so, it actually fits well in a series of experimental work \citep{fiDT}, for example \citet{mbacSA} ($Td = 4$), \citet{fsDM} ($Td = 5$), \citet{fiAD} ($Td = 20$).

\begin{remark} Theorem \ref{thBOOSTDP1} reveals another reason why we should indeed put emphasis on boosting on low-depth nodes: for any node of $h$, if its tree efficiency is above a threshold, then so is the case for all nodes along a shortest path from this node to the root of $h$. Hence the largest set $\mathcal{L}$ for which \eqref{constJJTREE} holds corresponds to a \textit{subtree} of $h$ with the same root.
 \end{remark}
To get a simple idea of how \eqref{constJJTREE} vanishes with $m$, remark that $|\mathcal{L}| \leq 2^{d+1}-1$. Condition
\eqref{condJJ} is therefore satisfied if for example $\epsilon, \xi, d$ are related to $m$ as 
\begin{eqnarray*}
\frac{\log (m)}{\sqrt{m}}  & = & o(\epsilon),\\
d, \log \nicefrac{1}{\xi} & = & o\left(\frac{\sqrt{m}}{\log m}\right),
\end{eqnarray*}
and in this case the constraint on $\min_{\ell \in \mathcal{L}} J(\leaf_\ell, h)$ in
\eqref{condJJ2} vanishes with $m$. This makes that strong privacy regimes can fit to Theorem \ref{thBOOSTDP1}, \textit{e.g.} with $\epsilon = \log^{1+c}(m)/\sqrt{m}$ for $c>0$ a constant.

\noindent $\triangleright$ \textbf{Private predictions at the leaves}: because our trees output real values, we use the Laplace mechanism. This fits well with the WLA using $|h_t|\leq M$ (Definition \ref{wlaKMLC}), for the sensitivity of the mechanism \citep{drTA}.

%% file: content-arxiv/experiments.tex
\section{Experiments}\label{sec-exp}

\begin{table}[t]
\begin{center}
\scalebox{.93}{\begin{tabular}{cccccc}\hline \hline
$\varepsilon $ & $0.01$ & $0.1$ & $1$ &
                                                                   $10$
  & $25$\\
(O.C, 0.1, 1.0) & (14,3,5) & (13,2,6) & (9,3,9) & (5,4,11)
  & (8,6,6) \\ \hline\hline
\end{tabular}}
\end{center}
\caption{Summary, on the 19 domains, of the $\#$ of domains
  for which one strategy for $\alpha$ in $\{$objective calibration
  (O.C), $0.1$, $1,0$$\}$ leads to the best result (ties lead to sums $>19$).}
  \label{f-DP-alpha}
\end{table}

\begin{figure*}[t]
\begin{center}
\scalebox{1.0}{
\begin{tabular}{cc?cc}\hline \hline
\hspace{-0.3cm}  perf. wrt $\alpha$s, w/o DP \hspacegss & \hspacegss
                                                          perf. wrt
                                                          $\alpha$s,
                                                          with DP &
                                                                    $\#$leaves,
                                                                    w/o
                                                                    DP
                                                                    \hspacegss
  & \hspacegss $\#$leaves, with DP \\ \hline
\hspace{-0.3cm} \includegraphics[trim=90bp 5bp 75bp
15bp,clip,width=0.2\linewidth]{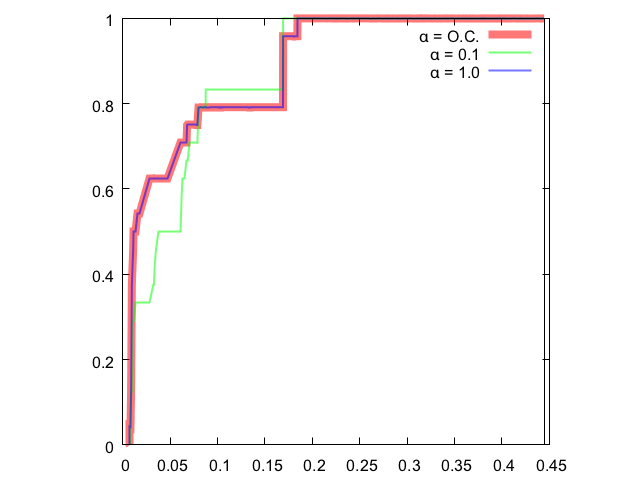}
   \hspacegs& \hspacegs\includegraphics[trim=90bp 5bp 75bp
15bp,clip,width=0.2\linewidth]{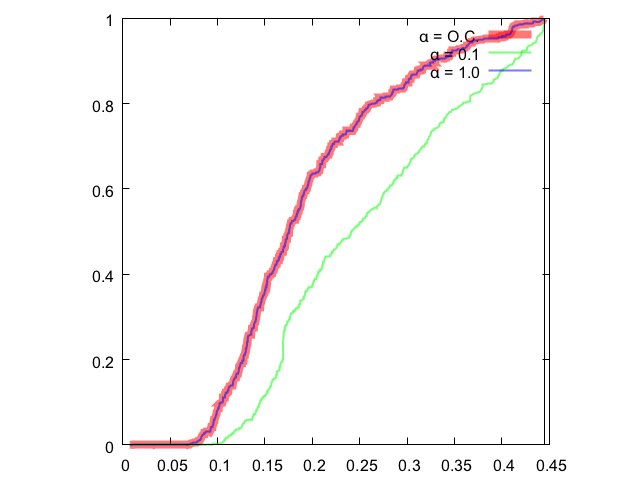}
            &
              \pushgraphicsLeafDepth{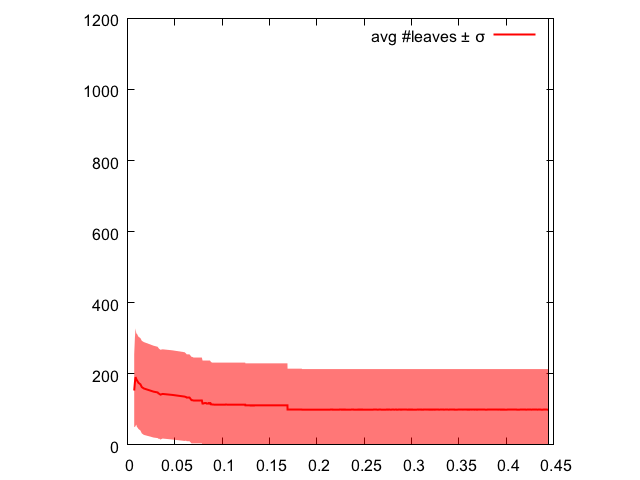}  \hspacegss & \hspacegss  \pushgraphicsLeafDepth{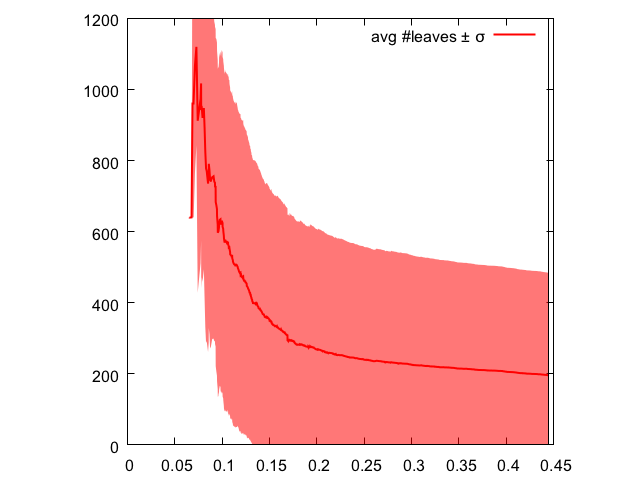} \\ \hline\hline
\end{tabular}}
\end{center}
\caption{UCI domain \texttt{banknote}: in each plot, $x$ depicts test errors and $y$ a cumulated $\%$ of runs of
  \algoname~having test error at most $x$. In each pane (left, right),
  the left plot is without DP and the right plot is with
  DP. \textit{Left pane}: comparison of \alphaboost~for three strategies on $\alpha$
  (see text). \textit{Right pane}: mean $\pm$ stddev for the number of
  leaves in the related trees (see text).}
  \label{f-banknote}
\end{figure*}

We have performed 10-folds stratified CV experiments on 19 UCI domains, detailed in \sm,
Section \ref{sub-sec-exp-res}, ranging from $m\cdot n<$ 3 000 to $m\cdot n>$ 200 000 . We have compared our approach,
\alphaboost, to two state of the art implementation of RFs
based on \cite{FLETCHER201716} but replacing the smooth sensitivity by
the global sensitivity (Definition \ref{defEDPRIV}). RFs have the appealing property for DP that privacy budget needs
only be spent at the leaves: we have tried both the Laplace
(\rflap) and the exponential (\rfexp) mechanisms (see \sm, $\S$
\ref{sub-sec-sum-imp}) with RFs containing $T=21$ trees to prevent ties.
We have performed three kinds of
experiments: (i) check that \alphaboost~performs well and complies
with the boosting theory in the privacy-free case, (ii) compare the
various flavours of \alphaboost~in the private case, (iii) compare
\alphaboost~vs RFs in the private case. 
We ran \algoname, both private and not private, for all combinations
of $T \in \{2, 5, 10, 20\}, \alpha \in \{0.1,
1.0, \mbox{O.C}\}$, depth $\in \{1, 2, 3, 4,
5, 6\}, \epsilon \in \{0.01, 0.1,
1.0, 10.0, 25.0\}, \betatree \in \{0.1, 0.5, 0.9\}$, and even more
parameters (see \sm, $\S$ \ref{sub-sec-gen}), for a total number of
boosting experiments alone that far exceeds the million ensemble
models learned. When there is no DP constraint, we add in
\alphaboost~the test of whether a leaf is pure -- \textit{i.e.} is
not reached by examples of both classes -- before attempting to
split it (we do not split further pure leaves). When there is DP however, we do not make the test in order
not to spend privacy budget, and so \alphaboost~
builds trees in which all leaves are at the required depth.
The \sm, $\S$ \ref{exp_expes}, gives the experiments in
greater details, summarized here.\\
\noindent $\triangleright$ \textbf{\alphaboost, with and without
  noise}: Figure \ref{f-banknote}, left pane, displays a picture that
can be observed more or less over all domains: \alphaboost~with
$\alpha=1$ tends to obtain better results than with $\alpha=0.1$,
which complies with Theorems
\ref{thBoostDT1} and \ref{thBoostLC1}, and with the boosting theory
more generally \citep{kmOT}. Three additional observations 
emerge: (a) objective calibration (O.C) is competitive without noise, (b)
 this also holds \textit{with noise}, which we believe indicate a good
compromise between convergence rates and safekeeping privacy budget in \alphaboost~
(Section \ref{sec:priv}) and contributes to experimentally validate
our theory in Section \ref{sec-solv}; (c) DP curves display predictable degradations due to noise,
but on many domains noisification still gives interesting results
compared to the noise-free setting: in \texttt{banknote} for example
(Figure \ref{f-banknote}), more than 2/3 of the private runs with
O.C get test error $\leq 20\%$, an
upperbound test error for noise-free boosting.\\
\noindent $\triangleright$ \textbf{\alphaboost~in various privacy regimes}:
Table \ref{f-DP-alpha} is an extremal experiments which looks at the
best models that can be learned under DP under various
$\varepsilon$. The picture that seems to emerge is that objective
calibration is the best technique for high privacy demand, which we
take as a good sign given our theory (Section
\ref{sec-solv}). Obviously, the experiments aggregate a number of
parameters for each $\epsilon$, such as $T, d, \betatree$, so to
really get the best of a regime for $\alpha$, one should be able to
have clues on how to fix those other parameters. It turns out that the
experiments display that this should be possible. In particular, for
each domain, the value $\betatree$ does not seem to significantly matter to get the
best results \textit{but} the model size parameters seem to matter a
lot more: for each domain, there is a particular regime of $d, T$ that
tends to give the top DP results (like rather deep trees for
\texttt{banknote}, Figure \ref{f-banknote}). \sm, Section \ref{sub-sec-sumdT} presents
the whole list details. This, we believe, is important, in
particular for domains where $T$ is small like \texttt{page}, as some RFs approaches
fit huge sets reducing
interpretability \citep[Table I]{fiDT}.\\
\noindent $\triangleright$ \textbf{\alphaboost~vs (\rflap~and \rfexp)}:
a Table (\ref{t-bvsrf}, given in \sm, S$\S$ \ref{sub-sec-AlphavsRF})
computes over all 19 domains the $\%$ of runs where \alphaboost~beats
RFs, among all runs for which one approach statistically significantly
($p_{\mbox{\tiny val}}<0.01$)
beats the other. The scale heavily tips in favor of \alphaboost~when
it boosts $T=20$ trees: O.C and $\alpha=1.0$ are
significantly superior than \rfexp~and
\rflap~on more than \textbf{80} $\%$ of such cases (less than 4$\%$ of
the differences are not significant). This means two things: first, for these strategies of \alphaboost, there is not much care needed
to optimize some parameters of \alphaboost~($d, \betapred$) to get to
or beat SOTA, which
is good news; second, this suggests that we can compete
with RFs on much smaller trees, which is indeed displayed in the
left pane of Table \ref{t-bvsrf} where \alphaboost~fits less than
\textbf{ten times} trees than RFs, and still beat those in
a majority of cases, which is good news for interpretability. When we drill down into the
results as a function of $\varepsilon$, we observe that
\alphaboost~tends to be especially good against RFs for high
privacy regimes (\textit{e.g.}$\varepsilon = 0.01$). \\
\noindent $\triangleright$ \textbf{\alphaboost~in the $\nvpriv=10$ vs
  $\nvpriv=50$ regime}: the previous summarizes experiments for a
regular quantization with $\nvpriv=10$ of the continuous
attributes. Our experiments (\sm, Section \ref{sub-sec-1050}) also
contain a summary of the comparisons for \alphaboost~when we rather
use $\nvpriv=50$. Notice that multiplying by five the potential number
of splits significantly affects the time
complexity of the algorithm. The results display that the impact
varies as a function of the domain at hand. There can be significant
improvements: \texttt{qsar} and \texttt{winewhite} are two domains for
which $\nvpriv=50$ buys more than 2$\%$ improvement for
objective calibration, a clear winner among all tested strategies for
$\alpha$. On  \texttt{banknote}, the improvement is more in favor of
$\alpha = 1.0$. On \texttt{winered}, there is no significant
improvement for the best strategy and apart from a seemingly better "concentration" of
more than 3/4 of the runs of objective calibration towards its best
results with $\nvpriv = 50$, there is no apparent gain otherwise.

%% file: content-arxiv/conclusion.tex
\section{Conclusion}\label{sec-conc}

While boosted ensemble of DTs have long shown their accuracy in international competitions, to our knowledge nothing is known on how to fit them in a differentially private framework while keeping some of the boosting guarantees, a setting in which random forests have been reigning supreme. In this paper, we first establish the existence of a nontrivial tradeoff to push boosting methods in a differentially private framework. To address this tradeoff, we first create a tunable proper canonical loss, whose boosting rate and sensitivity can be controlled up to optimal boosting rate, or minimal sensitivity. We then show guaranteed boosting rates for both the induction of DTs and ensembles using this loss, of independent interest. We introduce objective calibration as a way to dynamically tune this loss and make the most of boosting under a given privacy budget with high probability. Experiments reveal that our approach manages to significantly beat random forests, that the best private models tend to be learned by objective calibration, and that our technique appears all the better on high privacy regimes.

%% file: content-arxiv/acknowledgments.tex
\section*{Acknowledgments}
\label{sec:ackno}

The authors thank Sam Fletcher and Borja Balle for comments on this material.

%% file: content-arxiv/appendix-org.tex
\section*{Appendix}
\section{Table of contents}\label{tabcon}

\noindent \textbf{Appendix on proofs} \hrulefill Pg
\pageref{proof_proofs}\\
\noindent Proof of Theorem \ref{th3P}\hrulefill Pg
\pageref{proof_th3P}\\
\noindent Proof of Lemma \ref{lemcurv}\hrulefill Pg
\pageref{proof_lemcurv}\\
\noindent Proof of Lemma \ref{lemPhiProof}\hrulefill Pg
\pageref{proof_lemPhiProof}\\
\noindent Proof of Theorem \ref{theoremalpha}\hrulefill Pg
\pageref{proof_theoremalpha}\\
\noindent Proof of Theorem \ref{thBoostDT1}\hrulefill Pg
\pageref{proof_thBoostDT1}\\
\noindent Proof of Theorem \ref{thBoostLC1}\hrulefill Pg
\pageref{proof_thBoostLC1}\\
\noindent Proof of Theorem \ref{thBOOSTDP1}\hrulefill Pg
\pageref{proof_thBOOSTDP1}\\

\noindent \textbf{Appendix on experiments} \hrulefill Pg
\pageref{exp_expes}\\
\noindent Implementation\hrulefill Pg \pageref{sub-sec-sum-imp}\\
\noindent General setting\hrulefill Pg \pageref{sub-sec-gen}\\
\noindent Domain summary Table\hrulefill Pg \pageref{sub-sec-sum}\\
\noindent UCI \texttt{transfusion}\hrulefill Pg \pageref{sub-sec-transfusion}\\
\noindent UCI \texttt{banknote}\hrulefill Pg \pageref{sub-sec-banknote}\\
\noindent UCI \texttt{breastwisc}\hrulefill Pg \pageref{sub-sec-breastwisc}\\
\noindent UCI \texttt{ionosphere}\hrulefill Pg \pageref{sub-sec-ionosphere}\\
\noindent UCI \texttt{sonar}\hrulefill Pg \pageref{sub-sec-sonar}\\
\noindent UCI \texttt{yeast}\hrulefill Pg \pageref{sub-sec-yeast}\\
\noindent UCI \texttt{winered}\hrulefill Pg \pageref{sub-sec-winered}\\
\noindent UCI \texttt{cardiotocography}\hrulefill Pg \pageref{sub-sec-cardiotocography}\\
\noindent UCI \texttt{creditcardsmall}\hrulefill Pg \pageref{sub-sec-creditcardsmall}\\
\noindent UCI \texttt{abalone}\hrulefill Pg \pageref{sub-sec-abalone}\\
\noindent UCI \texttt{qsar}\hrulefill Pg \pageref{sub-sec-qsar}\\
\noindent UCI \texttt{page}\hrulefill Pg \pageref{sub-sec-page}\\
\noindent UCI \texttt{mice}\hrulefill Pg \pageref{sub-sec-mice}\\
\noindent UCI \texttt{hill+noise}\hrulefill Pg \pageref{sub-sec-hillnoise}\\
\noindent UCI \texttt{hill+nonoise}\hrulefill Pg \pageref{sub-sec-hillnonoise}\\
\noindent UCI \texttt{firmteacher}\hrulefill Pg \pageref{sub-sec-firmteacher}\\
\noindent UCI \texttt{magic}\hrulefill Pg \pageref{sub-sec-magic}\\
\noindent UCI \texttt{eeg}\hrulefill Pg \pageref{sub-sec-eeg}\\
\noindent Summary in $d, T$ for the best DP results in \alphaboost\hrulefill Pg \pageref{sub-sec-sumdT}\\
\noindent Summary of the comparison \alphaboost~vs RFs with DP\hrulefill Pg \pageref{sub-sec-AlphavsRF}\\
\noindent Summary comparison $\nvpriv=10$ vs $\nvpriv=50$ ($M = 10$)\hrulefill Pg \pageref{sub-sec-1050}

\newpage

\input{content-arxiv/appendix}

\clearpage
\newpage

\input{content-arxiv/appendix-suppexp}

%% file: content-arxiv/appendix.tex
\section*{Appendix on Proofs}\label{proof_proofs}

\section{Proof of Theorem \ref{th3P}}
\label{proof_th3P}
The proof is split in three parts. The two first being the following
two Lemmata.
\begin{lemma}\label{lemI}
Fix $u\geq 0$. $\check{\bayesrisk}(u,v)$ is non-decreasing over $v\geq
u$.
\end{lemma}
\begin{proof}
We know that $-\bayesrisk$ is convex and therefore $D_{-\bayesrisk}(a\|b)$ is non negative, $D_{-\bayesrisk}$ being the Bregman divergence with generator $-\bayesrisk$ \citep{nnBD}. We obtain, with $a=0, b=u/v$,
\begin{eqnarray}
D_{-\bayesrisk}(a\|b) = \bayesrisk\left(\frac{u}{v}\right) - \bayesrisk (0) - \left(0-\frac{u}{v}\right)\cdot (-\bayesrisk)' & \geq  & 0\:\:,\nonumber
\end{eqnarray}
where $(-\bayesrisk) ' \in \partial (-\bayesrisk) (u/v)$, $\partial$
denoting the subdifferential. Simplifying ($\bayesrisk (0) = 0$) yields
\begin{eqnarray}
\bayesrisk\left(\frac{u}{v}\right) + \frac{u}{v} \cdot (-\bayesrisk) ' & \geq & 0\:\:.\label{defphi}
\end{eqnarray}
We then remark that 
\begin{eqnarray}
\partial_v \check{\bayesrisk}(u,v) & = & \left\{\bayesrisk\left(\frac{u}{v}\right) + \frac{u}{v} \cdot (-\bayesrisk) '\:\:, \bayesrisk '\in \partial \bayesrisk\left(\frac{u}{v}\right)\right\}\:\:.\label{SI_1}
\end{eqnarray}
Therefore, $\check{\bayesrisk}(u,v)$ is non-decreasing when $v \geq u$, and we obtain the statement of Lemma \ref{lemI}.
\end{proof}
The next Lemma shows a few more facts about $\bayesrisk$.
\begin{lemma}\label{lemI2}
The following holds true:
\begin{itemize}
\item [(A)] $\bayesrisk$ is not decreasing (resp. not increasing) over
  $[0,1/2]$ (resp. $[1/2, 1]$);
\item [(B)] For any $0\leq p\leq q\leq
1/2$, or any $1/2\leq q\leq p\leq 1$, we have
\begin{eqnarray}
0\leq \bayesrisk(q) - \bayesrisk(p) \leq \bayesrisk(|q-p|)\:\:.
\end{eqnarray}
\item [(C)] Suppose $m\geq 2$. For any $0<v \leq m+1$, $0<u\leq
  \min\{1, v\}$,
  $\check{\bayesrisk}(u,v) \leq \check{\bayesrisk}(1,m+1)$
\item [(D)] For any $x \geq 2$
\begin{eqnarray}
\bayesrisk\left(\frac{1}{2}\right) - \bayesrisk\left(\frac{1}{2} -
    \frac{1}{x}\right) &\leq & \frac{2}{x}.
\end{eqnarray}
\end{itemize}
\end{lemma}
\begin{proof}
A fact that we will use repeatedly hereafter is the fact that a
concave function sits above all its chords. We first prove (A): if $\bayesrisk$ were decreasing somewhere on
$[0,1/2]$, there would be some
$a\geq b, a, b \in [0,1/2]$ such that $\bayesrisk(a) >
\bayesrisk(b)$. Since $\bayesrisk(a) \leq \bayesrisk(1/2)$, $(a,
\bayesrisk(a))$ sits below the chord $(b, \bayesrisk(b)),(1/2,1)$,
which is impossible. The case $[1/2,1]$ is obtained by symmetry.\\
\noindent We now prove (B). We prove it for the case $0\leq p\leq q\leq
1/2$, the other following from the symmetry of $\bayesrisk$.
Non-negativity follows from (A) and the fact that $\bayesrisk(0) = 0$. The right inequality follows from the concavity of $\bayesrisk$:
indeed, since $\bayesrisk(0) = 0$, this inequality is equivalent to proving
\begin{eqnarray}
\frac{\bayesrisk(q) - \bayesrisk(p)}{q-p} \leq \frac{\bayesrisk(q-p) - \bayesrisk(0)}{q-p-0}\:\:,
\end{eqnarray}
which since $p\geq 0$, is just stating that slopes of chords that
intersect $\bayesrisk$ at points of constant difference between abscissae do not
increase, \textit{i.e.} $\bayesrisk$ is concave. \\
\noindent We prove (C). To get the result, we just need to
write:
\begin{eqnarray}
\check{\bayesrisk}(u,v) & \defeq & v\cdot \bayesrisk\left(\frac{u}{v}\right)\nonumber\\
& \leq & (m +1) \cdot \bayesrisk\left(\frac{u}{m+1}\right) \label{eqaa}\\ 
 & \leq & (m+1) \cdot \bayesrisk\left(\frac{1}{m+1}\right) \label{eqbb},
\end{eqnarray}
where Ineq. (\ref{eqaa}) follows from Lemma
\ref{lemI} and $v\leq m$. Ineq. \eqref{eqbb} follows from $u/m\leq 1/m
\leq 1/2$. We finally prove (D). We have
\begin{eqnarray}
\bayesrisk(x) & \geq & 2x \:\:, \forall x
\in [0,1/2]\:\:,\label{eqbinf}
\end{eqnarray}
since $y=2x$ is a chord for $\bayesrisk$ over
$[0,1/2]$ and $\bayesrisk$ is concave (it therefore sits over its
chords). We get for any $x\geq 2$,
\begin{eqnarray}
\bayesrisk\left(\frac{1}{2}\right) - \bayesrisk\left(\frac{1}{2} -
    \frac{1}{x}\right) & = & 1 - \bayesrisk\left(\frac{1}{2} -
    \frac{1}{x}\right)\nonumber\\
 & \leq & 1 - 2\cdot \left(\frac{1}{2} -
    \frac{1}{x}\right)\nonumber\\
 & & = \frac{2}{x}\:\:,\nonumber
\end{eqnarray}
as claimed. We obtain the statement of Lemma \ref{lemI2}.
\end{proof}
We now embark on the proof of Theorem \ref{th3P}. 
Let us fix for short 
\begin{eqnarray}
\Delta & \defeq & |f_{\bayesrisk}(h, \leaf,
                  \mathcal{S}')-f_{\bayesrisk}(h, \leaf,
                  \mathcal{S})|\nonumber\\
 & & = \left| w'(\leaf) \cdot \bayesrisk\left(
     \frac{{w'}^1(\leaf)}{w'(\leaf)}\right) - w(\leaf) \cdot \bayesrisk\left( \frac{w^1(\leaf)}{w(\leaf)}\right)\right|,
\end{eqnarray} 
and let us assume
without loss of generality that samples contain at least two examples
(otherwise $\Delta = 0$). The only eventual difference between $\mathcal{S}$ and $\mathcal{S}'$
that can make $\Delta > 0$ is on a weight and / or class change for the
switched example. So, for some $\delta \leq 1$, we consider the
following cases.

\noindent \textbf{Case A}: the total weight in leaf $\leaf$ changes vs
it does not change
\begin{eqnarray}
A_1 \defeq "w'(\leaf) = w(\leaf)" & ; & A_2 \defeq "w'(\leaf) = w(\leaf) + \delta".
\end{eqnarray}
\noindent \textbf{Case B}: the total weight for class 1 in leaf $\leaf$ changes vs
it does not change
\begin{eqnarray}
B_1 \defeq "{w'}^1(\leaf) = w^1(\leaf) + \delta" & ; & B_2 \defeq "{w'}^1(\leaf) = w^1(\leaf)".
\end{eqnarray}
And we also consider different cases depending on the relationship between the weight of class 1
and the total weight in leaf $\leaf$ in $\mathcal{S}$: $\exists u \in
(w(\leaf)/2) \cdot [-1,1]$ such that
\begin{eqnarray}
w^1(\leaf) & = & \frac{w(\leaf)}{2} + u.
\end{eqnarray}
We also suppose wlog that $\delta > 0$ (otherwise, we permute
${\mathcal{S}}$ and ${\mathcal{S}}'$, which does not change $\Delta$
because of $|.|$). We also remark that if we prove the result for
$u\leq 0$, then because of the symmetry of $\bayesrisk$, we get the
result for $u\geq 0$ as well -- this just amounts to reasoning on
negative examples instead of positive examples, changing notations but
not the reasoning.

\noindent $\hookrightarrow$ Case $A_1 \wedge B_1 \wedge (u \leq
-\delta/2)$. Because of the constraint on $u$, either
$w^1(\leaf)+\delta \leq w(\leaf)/2$ (if $u \leq -\delta$), or when $u\in
(-\delta, -\delta/2]$, we have both $w^1(\leaf) / w(\leaf)\leq 1/2$,
$(w^1(\leaf) + \delta) / w(\leaf) > 1/2$ and $(w^1(\leaf) + \delta) /
w(\leaf) - (1/2) \leq (1/2) - w^1(\leaf) / w(\leaf)$. So,
\begin{eqnarray}
\bayesrisk\left(\frac{w^1(\leaf)+\delta}{w(\leaf)}\right) & \geq & \bayesrisk\left(\frac{w^1(\leaf)}{w(\leaf)}\right),
\end{eqnarray}
and therefore using $A_1, B_1$, we get
\begin{eqnarray}
\Delta & = & w(\leaf) \cdot \left| \bayesrisk\left(\frac{w^1(\leaf)+\delta}{w(\leaf)}\right) - \bayesrisk\left(\frac{w^1(\leaf)}{w(\leaf)}\right) \right|\nonumber\\
 & = & w(\leaf) \cdot \left(
       \bayesrisk\left(\frac{w^1(\leaf)+\delta}{w(\leaf)}\right) -
       \bayesrisk\left(\frac{w^1(\leaf)}{w(\leaf)}\right) \right)
       \nonumber.
\end{eqnarray}
We now have two sub-cases,\\
$\bullet$ If $(w^1(\leaf)+\delta)/ w(\leaf) \leq 1/2$, then we directly get
\begin{eqnarray}
 \Delta & \leq & w(\leaf) \cdot
                 \bayesrisk\left(\frac{\delta}{w(\leaf)}\right) \label{eq2}\\
& & = \check{\bayesrisk}(\delta, w(\leaf))\nonumber\\
& \leq & m \cdot
                 \bayesrisk\left(\frac{1}{m}\right) \:\:.\label{eq4}
\end{eqnarray}
\eqref{eq2} holds because of Lemma
\ref{lemI2} (B). Ineq. (\ref{eq4}) follows from Lemma
\ref{lemI2} (C). \\
$\bullet$ If $(w^1(\leaf)+\delta)/ w(\leaf) > 1/2$, then we know that
since $w^1(\leaf) / w(\leaf) \leq 1/2$, 
\begin{eqnarray}
\frac{1}{2} - \frac{w^1(\leaf)}{w(\leaf)} & \leq & \frac{\delta}{w(\leaf)},
\end{eqnarray}
and so
\begin{eqnarray}
\Delta & = & w(\leaf) \cdot \left(
       \bayesrisk\left(\frac{w^1(\leaf)+\delta}{w(\leaf)}\right) -
       \bayesrisk\left(\frac{w^1(\leaf)}{w(\leaf)}\right) \right) \nonumber\\
& \leq & w(\leaf) \cdot \left( \bayesrisk\left(\frac{1}{2}\right) -
         \bayesrisk\left(\frac{1}{2} - v\right)
         \right)\nonumber,
\end{eqnarray}
with therefore
\begin{eqnarray}
v & \defeq & \frac{w(\leaf) - 2w^1(\leaf)}{2w(\leaf)} \leq \frac{\delta}{w(\leaf)}.
\end{eqnarray}
We get
\begin{eqnarray}
\Delta & \leq & w(\leaf) \cdot \left( \bayesrisk\left(\frac{1}{2}\right) -
         \bayesrisk\left(\frac{1}{2} - \frac{\delta}{w(\leaf)}\right)
         \right) \nonumber\\
& \leq & w(\leaf) \cdot \left( \frac{2\delta}{w(\leaf)}
         \right) \nonumber\\
& &=  2\delta \leq 2,\label{eq42}
\end{eqnarray}
because of Lemma \ref{lemI2} (D) ($x \defeq w(\leaf) / \delta \geq 2$),
$\delta \leq 1$, $\bayesrisk(1/2) = 1$ and Lemma \ref{lemI}.\\
\noindent $\hookrightarrow$ Case $A_1 \wedge B_1 \wedge (u \in
(-\delta/2, \delta/2))$.  We now
obtain:
\begin{eqnarray}
\Delta & = & w(\leaf) \cdot \left(
             \bayesrisk\left(\frac{w^1(\leaf)}{w(\leaf)}\right) -
             \bayesrisk\left(\frac{w^1(\leaf)+\delta}{w(\leaf)}\right)
             \right) \nonumber\\
& \leq & w(\leaf) \cdot \left( \bayesrisk\left(\frac{1}{2}\right) -
         \bayesrisk\left(\frac{w^1(\leaf)+\delta}{w(\leaf)}\right)
         \right)\nonumber\\
& & =  w(\leaf) \cdot \left( \bayesrisk\left(\frac{1}{2}\right) -
         \bayesrisk\left(\frac{1}{2} + v\right)
         \right)\nonumber
\end{eqnarray}
with
\begin{eqnarray}
v & \defeq & \frac{2w^1(\leaf)+2\delta - w(\leaf)}{2w(\leaf)}.
\end{eqnarray}
we remark that $v\geq 0$ because it is equivalent to
\begin{eqnarray}
w^1(\leaf) & \geq & \frac{w(\leaf)}{2} - \delta,
\end{eqnarray}
which indeed holds because $u > -\delta/2 \geq -\delta$ (we recall
$\delta > 0$). We also obviously have $v \geq 1/2$, so using the
symmetry of $\bayesrisk$ around $1/2$, we get
\begin{eqnarray}
\Delta & \leq & w(\leaf) \cdot \left( \bayesrisk\left(\frac{1}{2}\right) -
         \bayesrisk\left(\frac{1}{2} - v\right)
         \right)\nonumber\\
& \leq & w(\leaf) \cdot \left( \frac{2w^1(\leaf)+2\delta - w(\leaf)}{w(\leaf)}
         \right)\label{eqUSE1}\\
& & = 2w^1(\leaf)+2\delta - w(\leaf)\nonumber\\
& = & 2(u+\delta) \leq 3\delta \leq 3 .\label{eq43}
\end{eqnarray}
\eqref{eqUSE1} follows from Lemma \ref{lemI2} (D).\\
\noindent $\hookrightarrow$ Case $A_1 \wedge B_1 \wedge (u \geq
\delta/2)$.  Since $\bayesrisk$ is symmetric around $1/2$, this boils
down to case $(u \leq
-\delta/2)$ with the negative examples.\\
\noindent $\hookrightarrow$ Case $A_1 \wedge B_2$. In this case, $\Delta = 0\leq
\check{\bayesrisk}(1,m)$.\\
\noindent $\hookrightarrow$ Case $A_2 \wedge B_1 \wedge \left( u \leq -\frac{\delta}{2} \cdot \frac{w(\leaf)}{2 w(\leaf) + \delta}\right)$. In this
case, there is no class flip, just a change in weight. We first show
that because of the constraint on $u$,
\begin{eqnarray}
\bayesrisk\left(\frac{w^1(\leaf)+\delta}{w(\leaf)+\delta}\right) & \geq
  & \bayesrisk\left(\frac{w^1(\leaf)}{w(\leaf)}\right).
\end{eqnarray}
A sufficient condition for this to happen is
$(w^1(\leaf)+\delta)/(w(\leaf)+\delta) \leq 1/2$, which, after
reorganising yields
\begin{eqnarray}
w^1(\leaf) & \leq & \frac{w(\leaf)}{2} - \frac{\delta}{2},
\end{eqnarray}
and so is covered by the fact that $u\leq -\delta / 2$, or, given the
symmetry of $\bayesrisk$, can also be achieved if the following
conditions are met:
\begin{eqnarray}
\frac{w^1(\leaf)+\delta}{w(\leaf)+\delta} & \geq &
                                                 \frac{1}{2}, \label{eqBM1}\\
\frac{w^1(\leaf)}{w(\leaf)} & \leq & \frac{1}{2}, \label{eqBMZ}\\
\frac{w^1(\leaf)+\delta}{w(\leaf)+\delta} -
                                                 \frac{1}{2} & \leq & \frac{1}{2} - \frac{w^1(\leaf)}{w(\leaf)}\label{eqB11}.
\end{eqnarray}
To satisfy all these inequalities, we need, respectively, $u \geq
-\delta/2$, $u\leq 0$ and 
\begin{eqnarray}
u & \leq & -\frac{\delta}{2} \cdot \frac{w(\leaf)}{2 w(\leaf) + \delta},
\end{eqnarray}
all of which are then implied if 
\begin{eqnarray}
u & \in & -\frac{\delta}{2}\left(  1,  \frac{w(\leaf)}{2 w(\leaf) + \delta}\right),
\end{eqnarray}
which, together with the previous case results in the Case condition
on $u$. We therefore have:
\begin{eqnarray}
\Delta & = & (w(\leaf)+\delta) \cdot
  \bayesrisk\left(\frac{w^1(\leaf)+\delta}{w(\leaf)+\delta}\right) -
  w(\leaf) \cdot \bayesrisk\left(\frac{w^1(\leaf)}{w(\leaf)}\right) 
\end{eqnarray}
We now have two sub-cases:\\
$\bullet$ If 
\begin{eqnarray}
\frac{w^1(\leaf)+\delta}{w(\leaf)+\delta} & \leq & \frac{1}{2},
\end{eqnarray}
then since we have as well $w^1(\leaf) / w(\leaf) \leq
(w^1(\leaf)+\delta)/(w(\leaf)+\delta)$ and $\bayesrisk$ is non
decreasing over $[0,1/2]$, we get directly from Lemma \ref{lemI2} (B),
\begin{eqnarray}
\Delta & = & (w(\leaf)+\delta) \cdot \left(
  \bayesrisk\left(\frac{w^1(\leaf)+\delta}{w(\leaf)+\delta}\right)
             -\bayesrisk\left(\frac{w^1(\leaf)}{w(\leaf)}\right) \right)
             + \delta \cdot \bayesrisk\left(\frac{w^1(\leaf)}{w(\leaf)}\right) \nonumber\\
& \leq & (w(\leaf)+\delta) \cdot \left(
          \bayesrisk\left(\frac{w^1(\leaf)+\delta}{w(\leaf)+\delta}\right)
          - \bayesrisk\left(\frac{w^1(\leaf)}{w(\leaf) +\delta}\right)
          \right) + \delta \cdot \bayesrisk\left(\frac{w^1(\leaf)}{w(\leaf)}\right) \nonumber\\
& \leq & (w(\leaf)+\delta) \cdot 
          \bayesrisk\left(\frac{\delta}{w(\leaf)+\delta} \right) + \delta \cdot \bayesrisk\left(\frac{w^1(\leaf)}{w(\leaf)}\right) \nonumber\\
& \leq & (w(\leaf)+\delta) \cdot 
          \bayesrisk\left(\frac{\delta}{w(\leaf)+\delta} \right) + \delta .\label{eqBEF1}
\end{eqnarray}
We also remark that 
\begin{eqnarray}
(w(\leaf)+\delta) \cdot 
          \bayesrisk\left(\frac{\delta}{w(\leaf)+\delta} \right) & \leq
  & (m+1) \cdot 
          \bayesrisk\left(\frac{\delta}{m+1} \right)\nonumber\\
& \leq & (m+1) \cdot 
          \bayesrisk\left(\frac{1}{m+1} \right),\nonumber
\end{eqnarray}
respectively because of Lemma \ref{lemI} and $1/(m+1) \leq 1/2$, which
is in the regime where $\bayesrisk$ is non decreasing. We get from
\eqref{eqBEF1} that 
\begin{eqnarray}
\Delta & \leq & \check{\bayesrisk}(1,m+1) + 1.\label{eq44}
\end{eqnarray}
$\bullet$ If 
\begin{eqnarray}
\frac{w^1(\leaf)+\delta}{w(\leaf)+\delta} & > & \frac{1}{2},\label{eqBM2}
\end{eqnarray}
then we can use \eqref{eqB11} and get
\begin{eqnarray}
\Delta & = & (w(\leaf)+\delta) \cdot \left(
          \bayesrisk\left(\frac{w^1(\leaf)+\delta}{w(\leaf)+\delta}\right)
          - \bayesrisk\left(\frac{w^1(\leaf)}{w(\leaf)}\right)
          \right) + \delta \cdot \bayesrisk\left(\frac{w^1(\leaf)}{w(\leaf)}\right) \nonumber\\
& \leq &  (w(\leaf)+\delta) \cdot \left(
          \bayesrisk\left(\frac{1}{2}\right)
          - \bayesrisk\left(\frac{w^1(\leaf)}{w(\leaf)}\right)
          \right) + \delta \cdot \bayesrisk\left(\frac{w^1(\leaf)}{w(\leaf)}\right) \nonumber\\
& & =  (w(\leaf)+\delta) \cdot \left(
          \bayesrisk\left(\frac{1}{2}\right)
          - \bayesrisk\left(\frac{1}{2} - v\right)
          \right) + \delta \cdot \bayesrisk\left(\frac{w^1(\leaf)}{w(\leaf)}\right) ,\label{eqSZ1}
\end{eqnarray}
with
\begin{eqnarray}
v & \defeq & \frac{1}{2} - \frac{w^1(\leaf)}{w(\leaf)}.
\end{eqnarray}
Remark that
\begin{eqnarray}
\frac{w^1(\leaf)+\delta}{w(\leaf)+\delta} & = &
                                              \frac{w^1(\leaf)}{w(\leaf)} + \frac{\delta(w(\leaf)-w^1(\leaf))}{w(\leaf)(w(\leaf)+\delta)},\label{eqCOND12}
\end{eqnarray}
so to get both \eqref{eqBM2} and \eqref{eqBMZ}, we need
\begin{eqnarray}
\frac{w^1(\leaf)}{w(\leaf)} & \geq & \frac{1}{2} -
                                   \frac{\delta(w(\leaf)-w^1(\leaf))}{w(\leaf)(w(\leaf)+\delta)}, \label{eqCOND13}
\end{eqnarray}
which implies
\begin{eqnarray}
v & \leq & \frac{\delta(w(\leaf)-w^1(\leaf))}{w(\leaf)(w(\leaf)+\delta)},
\end{eqnarray}
and so \eqref{eqSZ1} and Lemma \ref{lemI2} (D) yields
\begin{eqnarray}
\Delta & \leq & (w(\leaf)+\delta) \cdot \left(
          \frac{2\delta(w(\leaf)-w^1(\leaf))}{w(\leaf)(w(\leaf)+\delta)}
          \right) + \delta \cdot \bayesrisk\left(\frac{w^1(\leaf)}{w(\leaf)}\right) \nonumber\\
& & = 2\delta \cdot \left(1 - \frac{w^1(\leaf)}{w(\leaf)}\right) +\delta
    \leq
    3\delta \leq 3.\label{eq45}
\end{eqnarray}
\noindent $\hookrightarrow$ Case $A_2 \wedge B_1 \wedge \left( u \in
  \left(-\frac{\delta}{2} \cdot \frac{w(\leaf)}{2 w(\leaf) + \delta},
    \frac{\delta}{2} \cdot \frac{w(\leaf)}{2 w(\leaf) +
      \delta}\right)\right)$. 
In this case, we have
\begin{eqnarray}
\bayesrisk\left(\frac{w^1(\leaf)+\delta}{w(\leaf)+\delta}\right) & \leq
  & \bayesrisk\left(\frac{w^1(\leaf)}{w(\leaf)}\right),\label{eqPROP11}
\end{eqnarray}
and therefore, by virtue of the triangle inequality,
\begin{eqnarray}
\Delta & = & \left| (w(\leaf)+\delta) \cdot
  \bayesrisk\left(\frac{w^1(\leaf)+\delta}{w(\leaf)+\delta}\right) -
  w(\leaf) \cdot \bayesrisk\left(\frac{w^1(\leaf)}{w(\leaf)}\right)
             \right|\nonumber\\
& = & \left| (w(\leaf)+\delta) \cdot
  \left(\bayesrisk\left(\frac{w^1(\leaf)+\delta}{w(\leaf)+\delta}\right) -
  \bayesrisk\left(\frac{w^1(\leaf)}{w(\leaf)}\right) \right) +
      \delta\cdot
      \bayesrisk\left(\frac{w^1(\leaf)}{w(\leaf)}\right)\right|\nonumber\\
& \leq & \left| (w(\leaf)+\delta) \cdot
  \left(\bayesrisk\left(\frac{w^1(\leaf)+\delta}{w(\leaf)+\delta}\right) -
  \bayesrisk\left(\frac{w^1(\leaf)}{w(\leaf)}\right) \right) \right| +
      \delta\cdot
      \bayesrisk\left(\frac{w^1(\leaf)}{w(\leaf)}\right)\nonumber\\
 & = & (w(\leaf)+\delta) \cdot \left(\bayesrisk\left(\frac{w^1(\leaf)}{w(\leaf)}\right)-
          \bayesrisk\left(\frac{w^1(\leaf)+\delta}{w(\leaf)+\delta}\right)
          \right) +
      \delta\cdot
      \bayesrisk\left(\frac{w^1(\leaf)}{w(\leaf)}\right)\nonumber.
\end{eqnarray}
We have two sub-cases.\\
\noindent $\bullet$ $w^1(\leaf) / w(\leaf) \geq 1/2$. In this case, we
apply Lemma \ref{lemI2} (B) and get
\begin{eqnarray}
\Delta & \leq & (w(\leaf)+\delta) \cdot
                \bayesrisk\left(\frac{w^1(\leaf)+\delta}{w(\leaf)+\delta}-\frac{w^1(\leaf)}{w(\leaf)}\right)
      +\delta\cdot
      \bayesrisk\left(\frac{w^1(\leaf)}{w(\leaf)}\right)\nonumber\\
& & = (w(\leaf)+\delta) \cdot
                \bayesrisk\left(\frac{\delta(w(\leaf)-w^1(\leaf))}{w(\leaf)(w(\leaf)+\delta)}\right)
      +\delta\cdot
      \bayesrisk\left(\frac{w^1(\leaf)}{w(\leaf)}\right).
\end{eqnarray}
Fixing $u\defeq \frac{\delta(w(\leaf)-w^1(\leaf))}{w(\leaf)}\leq \delta$ and $v
\defeq w(\leaf)+\delta$, we remark that $u\leq v$ and $v\leq m+1$ so we can apply Lemma
\ref{lemI2} (C) and get
\begin{eqnarray}
\Delta & \leq & \check{\bayesrisk}(1,m+1) + \delta\cdot
      \bayesrisk\left(\frac{w^1(\leaf)}{w(\leaf)}\right)\nonumber\\
& \leq & \check{\bayesrisk}(1,m+1) + \delta \leq \check{\bayesrisk}(1,m+1) + 1.\label{eq46}
\end{eqnarray}
\noindent $\bullet$ $w^1(\leaf) / w(\leaf) < 1/2$. In this case, we
remark that \eqref{eqPROP11} implies $(w^1(\leaf) +\delta) / (w(\leaf) +
\delta) > 1/2$, in which case since we still get \eqref{eqCOND12}, to
get $w^1(\leaf) / w(\leaf) < 1/2$, we must have
\begin{eqnarray}
\frac{w^1(\leaf)+\delta}{w(\leaf)+\delta} & \leq & \frac{1}{2} +
                                   \frac{\delta(w(\leaf)-w^1(\leaf))}{w(\leaf)(w(\leaf)+\delta)}, \label{eqCOND14}
\end{eqnarray}
and combining this with the fact that (i) $\bayesrisk$ is maximum in
$1/2$ and non increasing afterwards, and symmetric around $1/2$,
\begin{eqnarray}
\Delta & \leq & (w(\leaf)+\delta) \cdot \left(\bayesrisk\left(\frac{1}{2}\right)-
          \bayesrisk\left(\frac{w^1(\leaf)+\delta}{w(\leaf)+\delta}\right)
          \right) +
      \delta\cdot
      \bayesrisk\left(\frac{w^1(\leaf)}{w(\leaf)}\right)\nonumber\\
& \leq & (w(\leaf)+\delta) \cdot \left(\bayesrisk\left(\frac{1}{2}\right)-
          \bayesrisk\left(\frac{1}{2} +
                                   \frac{\delta(w(\leaf)-w^1(\leaf))}{w(\leaf)(w(\leaf)+\delta)}\right)
          \right) +
      \delta\cdot
      \bayesrisk\left(\frac{w^1(\leaf)}{w(\leaf)}\right)\nonumber\\
&  & = (w(\leaf)+\delta) \cdot \left(\bayesrisk\left(\frac{1}{2}\right)-
          \bayesrisk\left(\frac{1}{2} -
                                   \frac{\delta(w(\leaf)-w^1(\leaf))}{w(\leaf)(w(\leaf)+\delta)}\right)
          \right) +
      \delta\cdot
      \bayesrisk\left(\frac{w^1(\leaf)}{w(\leaf)}\right)\nonumber\\
&  \leq & (w(\leaf)+\delta) \cdot \left(\frac{2\delta(w(\leaf)-w^1(\leaf))}{w(\leaf)(w(\leaf)+\delta)}
          \right) +
      \delta\cdot
      \bayesrisk\left(\frac{w^1(\leaf)}{w(\leaf)}\right)\label{eqTHF1}\\
& & = 2 \delta \cdot \left(1 - \frac{w^1(\leaf)}{w(\leaf)}\right) +
      \delta\cdot
      \bayesrisk\left(\frac{w^1(\leaf)}{w(\leaf)}\right) \nonumber\\
& \leq & 3\delta \leq 3 \label{eq47}.
\end{eqnarray}
We have used Lemma \ref{lemI2} (D) in \eqref{eqTHF1}.\\
\noindent $\hookrightarrow$ Case $A_2 \wedge B_1 \wedge \left(u \geq
\frac{\delta}{2} \cdot \frac{w(\leaf)}{2 w(\leaf) + \delta}\right)$.  Since $\bayesrisk$ is symmetric around $1/2$, this boils
down to case $\left(u \leq
-\frac{\delta}{2} \cdot \frac{w(\leaf)}{2 w(\leaf) + \delta}\right)$ with the negative examples.\\
\noindent $\hookrightarrow$ Case $A_2 \wedge B_2 \wedge \left(u \leq \frac{\delta}{2} \cdot \frac{2w(\leaf)}{2w(\leaf) + \delta}
\right)$. This time, we can immediately write, independently from the
condition on $u$,
\begin{eqnarray}
\Delta & = & \left| (w(\leaf)+\delta) \cdot
  \bayesrisk\left(\frac{w^1(\leaf)}{w(\leaf)+\delta}\right) -
  w(\leaf) \cdot \bayesrisk\left(\frac{w^1(\leaf)}{w(\leaf)}\right)
             \right|\nonumber\\
& = & \left| (w(\leaf)+\delta) \cdot
  \left(\bayesrisk\left(\frac{w^1(\leaf)}{w(\leaf)+\delta}\right) -
  \bayesrisk\left(\frac{w^1(\leaf)}{w(\leaf)}\right) \right) +
      \delta\cdot
      \bayesrisk\left(\frac{w^1(\leaf)}{w(\leaf)}\right)\right|\nonumber\\
& \leq & \left| (w(\leaf)+\delta) \cdot
  \left(\bayesrisk\left(\frac{w^1(\leaf)}{w(\leaf)+\delta}\right) -
  \bayesrisk\left(\frac{w^1(\leaf)}{w(\leaf)}\right) \right) \right| +
      \delta\cdot
      \bayesrisk\left(\frac{w^1(\leaf)}{w(\leaf)}\right) \nonumber\\
& \leq & \left| (w(\leaf)+\delta) \cdot
  \left(\bayesrisk\left(\frac{w^1(\leaf)}{w(\leaf)+\delta}\right) -
  \bayesrisk\left(\frac{w^1(\leaf)}{w(\leaf)}\right) \right) \right| +
      \delta\label{eqB2B}.
\end{eqnarray}
We first examine the condition under which
\begin{eqnarray}
\bayesrisk\left(\frac{w^1(\leaf)}{w(\leaf)+\delta}\right) & \leq & \bayesrisk\left(\frac{w^1(\leaf)}{w(\leaf)}\right).\label{const1B}
\end{eqnarray}
Again, $u\leq 0$ is a sufficient condition. Otherwise, if therefore
$w^1(\leaf)/w(\leaf) \geq 1/2$, then we need
\begin{eqnarray}
\frac{w^1(\leaf)}{w(\leaf)+\delta} & < & \frac{1}{2},\nonumber\\
\frac{1}{2} - \frac{w^1(\leaf)}{w(\leaf)+\delta} & \geq &
                                                        \frac{w^1(\leaf)}{w(\leaf)}
                                                        - \frac{1}{2};
\end{eqnarray}
the latter constraint is equivalent to
\begin{eqnarray}
\frac{w^1(\leaf)}{w(\leaf)} &\leq & \frac{w(\leaf)+\delta}{2w(\leaf)+\delta},
\end{eqnarray}
and therefore
\begin{eqnarray}
\frac{u}{w(\leaf)} \defeq \frac{w^1(\leaf)}{w(\leaf)}
                                                        - \frac{1}{2}
  & \leq & \frac{w(\leaf)+\delta}{2w(\leaf)+\delta} - \frac{1}{2} = \frac{\delta}{2w(\leaf)+\delta},
\end{eqnarray}
which leads to our constraint on $u$ and gives 
\begin{eqnarray}
\Delta & \leq &  (w(\leaf)+\delta) \cdot
  \left(\bayesrisk\left(\frac{w^1(\leaf)}{w(\leaf)}\right)  - \bayesrisk\left(\frac{w^1(\leaf)}{w(\leaf)+\delta}\right) 
  \right)  +
      \delta\label{eqB2B2}.
\end{eqnarray}
We have two sub-cases.\\
\noindent $\bullet$ $w^1(\leaf) / w(\leaf)\leq 1/2$. In thise case, we
get directly from Lemma \ref{lemI2} (B),
\begin{eqnarray}
\Delta & \leq &  (w(\leaf)+\delta) \cdot
  \bayesrisk\left(\frac{w^1(\leaf)}{w(\leaf)}-\frac{w^1(\leaf)}{w(\leaf)+\delta}\right)   +
      \delta\nonumber\\
& & = (w(\leaf)+\delta) \cdot
  \bayesrisk\left(\frac{w^1(\leaf) \delta}{w(\leaf)(w(\leaf)+\delta)}\right)   +
      \delta\nonumber\\
& \leq & \check{\bayesrisk}(1,m+1) + 1,\label{eq48}
\end{eqnarray}
where we have used Lemma \ref{lemI2} (C) with $u \defeq w^1(\leaf)\delta
/w(\leaf) \leq 1$ and $v \defeq w(\leaf)+\delta \leq m+1$. We also check
that $u\leq \delta\leq v$.\\
\noindent $\bullet$ $w^1(\leaf) / w(\leaf)\geq 1/2$. In this case, we
remark that
\begin{eqnarray}
\frac{w^1(\leaf)}{w(\leaf)} & = &  \frac{w^1(\leaf)}{w(\leaf)+\delta} + \frac{w^1(\leaf) \delta}{w(\leaf)(w(\leaf)+\delta)},
\end{eqnarray}
and since we need $w^1(\leaf)/(w(\leaf)+\delta)\leq 1/2$ (otherwise,
\eqref{const1B} cannot hold), then it implies
\begin{eqnarray}
\frac{w^1(\leaf)}{w(\leaf)+\delta} & \geq & \frac{1}{2} - \frac{w^1(\leaf) \delta}{w(\leaf)(w(\leaf)+\delta)},
\end{eqnarray}
and so the fact that $\bayesrisk$ is non-decreasing before $1/2$ and Lemma \ref{lemI2} (D) yield
\begin{eqnarray}
\Delta & \leq & (w(\leaf)+\delta) \cdot
  \left(\bayesrisk\left(\frac{1}{2}\right)  - \bayesrisk\left(\frac{w^1(\leaf)}{w(\leaf)+\delta}\right) 
  \right)  +
      \delta\nonumber\\
& \leq & (w(\leaf)+\delta) \cdot
  \left(\bayesrisk\left(\frac{1}{2}\right)  - \bayesrisk\left(\frac{1}{2} - \frac{w^1(\leaf) \delta}{w(\leaf)(w(\leaf)+\delta)}\right) 
  \right)  +
      \delta\nonumber\\
& \leq & (w(\leaf)+\delta) \cdot
  \frac{2 w^1(\leaf) \delta}{w(\leaf)(w(\leaf)+\delta)}  +
      \delta\nonumber\\
&  & = 
  \frac{2 w^1(\leaf) \delta}{w(\leaf)}  +
      \delta \leq 3\delta \leq 3\label{eq49}.
\end{eqnarray}
To complete the proof of the Case, suppose now that
\begin{eqnarray}
\bayesrisk\left(\frac{w^1(\leaf)}{w(\leaf)+\delta}\right) & \geq & \bayesrisk\left(\frac{w^1(\leaf)}{w(\leaf)}\right),\label{const1B3}
\end{eqnarray}
which therefore imposes
\begin{eqnarray}
\frac{w^1(\leaf)}{w(\leaf)} \geq \frac{w^1(\leaf)}{w(\leaf)+\delta} \geq  \frac{1}{2},
\end{eqnarray}
so using Lemma \ref{lemI2} (B) yields
\begin{eqnarray}
\Delta & \leq &  (w(\leaf)+\delta) \cdot
  \left(\bayesrisk\left(\frac{w^1(\leaf)}{w(\leaf)+\delta} - \bayesrisk\left(\frac{w^1(\leaf)}{w(\leaf)}\right)  \right) 
  \right)  +
      \delta\nonumber\\
& \leq & (w(\leaf)+\delta) \cdot
  \bayesrisk\left(\frac{w^1(\leaf)}{w(\leaf)}-\frac{w^1(\leaf)}{w(\leaf)+\delta}\right)  +
      \delta\nonumber\\
& & = (w(\leaf)+\delta) \cdot
  \bayesrisk\left(\frac{w^1(\leaf) \delta}{w(\leaf)(w(\leaf)+\delta)}\right)   +
      \delta\nonumber\\
& \leq & \check{\bayesrisk}(1,m+1) + 1,\label{eq410}
\end{eqnarray}
where we have used Lemma \ref{lemI2} (C) with $u \defeq w^1(\leaf)\delta
/w(\leaf) \leq 1$ and $v \defeq w(\leaf)+\delta \leq m+1$. We also check
that $u\leq \delta\leq v$.\\
\noindent $\hookrightarrow$ Case $A_2 \wedge B_2 \wedge \left(u > \frac{\delta}{2} \cdot \frac{2w(\leaf)}{2w(\leaf) + \delta}
\right)$. Since $\bayesrisk$ is symmetric around $1/2$, this boils
down to case $\left(u \leq \frac{\delta}{2} \cdot
  \frac{2w(\leaf)}{2w(\leaf) + \delta}\right)$ with the negative
examples.\\

We can now finish the upperbound on $\Delta$ by taking all bounds in
\eqref{eq4}, \eqref{eq42}, \eqref{eq43}, \eqref{eq44}, \eqref{eq45},
\eqref{eq46}, \eqref{eq47}, \eqref{eq48}, \eqref{eq49} and
\eqref{eq410}:
\begin{eqnarray}
\Delta & \leq & \max\{\check{\bayesrisk}(1,m), 2, 3,
                1+\check{\bayesrisk}(1,m+1)\} = \max\{3,  1+\check{\bayesrisk}(1,m+1)\},
\end{eqnarray}
as claimed, using Lemma \ref{lemI2} (C).

\noindent \textbf{Remark}: We can prove that $\Delta = \check{\bayesrisk}(1, m)$ can be
realized: consider set $\mathcal{S}$ with $m$ examples with unit weight, 1 of which
each is from the positive class
class. In ${\mathcal{S}}'$, we flip this class. We get:
\begin{eqnarray}
\Delta & = & m\cdot \bayesrisk\left(\frac{1}{m}\right) - m\cdot \bayesrisk\left(\frac{0}{m}\right)\nonumber\\
 & = & m\cdot \bayesrisk\left(\frac{1}{m}\right) - m\cdot \bayesrisk\left(0\right)\nonumber\\
 & = & m\cdot \bayesrisk\left(\frac{1}{m}\right) = \check{\bayesrisk}(1, m)\:\:,
\end{eqnarray}
as claimed (since $\bayesrisk(0) = 0$).

\section{Proof of Lemma \ref{lemcurv}}\label{proof_lemcurv}

We perform a Taylor expansion of $\bayesrisk$ up to second order and obtain:
\begin{eqnarray*}
\bayesrisk(0)  & = &  \underbrace{\bayesrisk\left(\frac{1}{x}\right) + \left( 0-\frac{1}{x}\right)\cdot \bayesrisk'\left(\frac{1}{x}\right) }_{\defeq J} \nonumber\\
 & & + \left( 0-\frac{1}{x}\right)^2\cdot \bayesrisk''(a)\:\:,
\end{eqnarray*}
for some $a \in [0,x]$. There remains to see that $J = \check{\bayesrisk}'(1, x)$ (eq. (\ref{SI_1}) in the Appendix), fix $x = m+1$ and reorder given $\bayesrisk(0) = 0$.

\section{Proof of Lemma \ref{lemPhiProof}}\label{proof_lemPhiProof}

\noindent We have
\begin{eqnarray}
\bayesmatrisk (1,m+1) & = & (m+1)\cdot 2\sqrt{\frac{1}{m+1}
  \cdot \frac{m}{m+1}}  \nonumber\\
 & = & 2\sqrt{m}\:\:,
\end{eqnarray}
as claimed.\\

\noindent We have (we make the distinction $\log$ base-2 and $\ln$
base-$e$\footnote{In the main body, $\log$ is base-$e$ by default.})
\begin{eqnarray}
\bayeslogrisk (1,m+1) & = & (m+1)\cdot \left( -\frac{1}{m+1}\log
  \frac{1}{m+1} -  \frac{m}{m+1}\log \frac{m}{m+1}\right) \nonumber\\
 & = & \log(m+1) + m\log\frac{m+1}{m}\nonumber\\
 & \leq & \log(m+1) + \frac{1}{\ln 2}\:\:.
\end{eqnarray}
The last inequality follows from \citet[Claim 1]{fsDM}.\\

\noindent We have
\begin{eqnarray}
\bayessqrisk (1,m+1) & = & (m+1)\cdot \frac{4}{m+1} \cdot \frac{m}{m+1} \nonumber\\
 & = & \frac{4m}{m+1}\:\:,
\end{eqnarray}
as claimed.\\

\noindent Finally, we have
\begin{eqnarray}
\bayesZOrisk (1,m+1) & = & (m+1)\cdot 2
\min\left\{\frac{1}{m+1}, \frac{m}{m+1}\right\}\nonumber\\
 & = & 2\:\:,
\end{eqnarray}
as claimed.

\section{Proof of Theorem \ref{theoremalpha}}\label{proof_theoremalpha}

That Matsushita's $\alpha$-loss is symmetric is a direct consequence
of its definition. It is proper because it is a convex combination of
two proper losses, Matsushita loss and the 0/1-loss \citet[Table
1]{rwCB}. As a consequence, its pointwise Bayes risk is the convex
combination of the Bayes risks:
\begin{eqnarray}
\bayesalpharisk(u) = 2 \cdot(\alpha \cdot \sqrt{u(1-u)} + (1-\alpha) \cdot\min\{u, 1-u\}).
\end{eqnarray}
We get the canonical link in the subdifferential of negative
the pointwise Bayes risk:
\begin{eqnarray}
\alphalink (u) \defeq -\partial \bayesalpharisk(u)  & = & \alpha \cdot \frac{2u - 1}{\sqrt{u(1-u)}} -2(1-\alpha)\cdot \left\{
\begin{array}{rcl}
1 & \mbox{ if } & u < 1/2\\
\big[ -1,1 \big] & \mbox{ if } & u = 1/2\\
-1 & \mbox{ if } & u > 1/2
\end{array}
\right. ,
\end{eqnarray}
and we immediately get the weight function from the fact that
$\weightalphaloss \defeq - \bayesalpharisk''$ \citet[Theorem 6]{rwCB}.
We get the corresponding convex surrogate of the proper loss by taking the convex conjugate of negative the pointwise Bayes risk:
\begin{eqnarray}
\leaf_\alpha(z) & \defeq & \sup_{u \in [0,1]} \{zu + 2 \cdot(\alpha \cdot \sqrt{u(1-u)} + (1-\alpha) \cdot\min\{u, 1-u\})\}.\label{negent}
\end{eqnarray}
We remark that if $z<0$ then the $\sup$ is going to be attained for $u$ closer to $0$ than $1$ (thus $u \leq 1/2$), and if $z>0$, it is the opposite: the $\sup$ is going to be attained for $u$ closer to $1$ than to $0$ (thus $u \geq 1/2$). If $z = 0$, the $\sup$ is trivially going to hold for $u=1/2$ (that is, $\leaf_\alpha(0) = 1/2$).\\
\noindent \textbf{Case 1: $\alpha = 0$} -- when $z< -2$ (resp. $z> 2$), the $\sup$ is attained for $u=0$ (resp. $u = 1$). Otherwise, the $\sup$ is attained for $u = 1/2$. Hence
\begin{eqnarray}
\leaf_0(z) & = & \left\{
\begin{array}{lcr}
0 & \mbox{ if } & z< -2\\
1 + \frac{z}{2} & \mbox{ if } & z\in 2\cdot [-1, 1]\\
z &  \mbox{ if } & z > 2
\end{array}
\right. .\label{csurzero}
\end{eqnarray}
\noindent \textbf{Case 2: $\alpha \neq 0$} -- Let us find the values of $z$ for which the argument $u = 1/2$ in \eqref{negent}, that is we want to find $z$ such that
\begin{eqnarray}
\left\{ 
\begin{array}{rcl}
zu + 2\alpha\sqrt{u(1-u)} + 2(1-\alpha)u & \leq & 1 + \frac{z}{2}, \forall u \in [0,1/2]\\
zu + 2\alpha\sqrt{u(1-u)} + 2(1-\alpha)(1-u) & \leq & 1 + \frac{z}{2}, \forall u \in [1/2,1]
\end{array}\right. .\label{feq1}
\end{eqnarray}
We consider the topmost condition in \eqref{feq1}. Reorganising, we want $2\alpha \sqrt{u(1-u)} \leq 1 + (z/2) - (z+2(1-\alpha))u$ for $u \in [0,1/2]$. Fix $z \defeq -2(1-\alpha)+\delta$, which gives the condition 
\begin{eqnarray}
2\sqrt{u(1-u)} & \leq & 1 + \frac{\delta}{\alpha}\cdot (1-u), \forall u \in [0,1/2].
\end{eqnarray}
This condition obviously holds when $\delta \geq 0$, and it is in fact violated when $\delta < 0$ because the LHS can be made as close as desired to $1$. So the topmost condition holds for $z\geq -2(1-\alpha)$. Regarding the bottommost condition, we now want  $2\alpha \sqrt{u(1-u)} \leq 1 - 2(1-\alpha) + (z/2) - (z-2(1-\alpha))u$ for $u \in [1/2,1]$, which, after letting $z \defeq 2(1-\alpha)+\delta$, gives equivalently 
\begin{eqnarray}
2\sqrt{u(1-u)} & \leq & 1 - \frac{\delta}{\alpha}\cdot \left(u-\frac{1}{2}\right), \forall u \in [1/2,1].
\end{eqnarray}
While the condition trivially holds when $\delta \leq 0$, it is in fact violated when $\delta > 0$ because the LHS can be made as close as desired to $1$. To summarize, the trivial argument $u = 1/2$ giving us \eqref{negent} is obtained when $z \in [-2(1-\alpha), +\infty) \cap (-\infty, 2(1-\alpha)] = 2(1-\alpha)\cdot [-1,1]$, and we get
\begin{eqnarray}
\leaf_\alpha (z) & = & 1 + \frac{z}{2} \mbox{ if } z \in 2(1-\alpha)\cdot [-1,1],
\end{eqnarray}
which, we also remark, gives the mid condition in \eqref{csurzero} when $\alpha \rightarrow 0$.\\

\noindent Now, when $z \not\in 2(1-\alpha)\cdot [-1,1]$, we can differentiate \eqref{negent} to find the argument $u$ realising the max. 
Let
\begin{eqnarray}
h_-(u) & \defeq & (z+2(1-\alpha)) \cdot u + 2 \alpha \cdot \sqrt{u(1-u)}\nonumber\\
 & & = \alpha \cdot \underbrace{\left(Z_- u + 2 \sqrt{u(1-u)}\right)}_{\defeq g_-(u)},\nonumber\\
h_+(u) & \defeq & 2(1-\alpha) + (z-2(1-\alpha)) \cdot u + 2 \alpha \cdot \sqrt{u(1-u)}\nonumber\\
& & = 2(1-\alpha) + \alpha \cdot \underbrace{\left(Z_+ u + 2 \sqrt{u(1-u)}\right)}_{\defeq g_+(u)},
\end{eqnarray}
with $Z_- \defeq  (z+2(1-\alpha))/\alpha, Z_+ \defeq  (z-2(1-\alpha))/\alpha$.
We compute $\max_{[0,1/2]} h_-(u)$ and $\max_{[1/2,1]} h_+(u)$, granted that the max of the two will give us the convex conjugate. 

\noindent Let us focus on $h_-(u)$. The argument $u$ we seek satisfies, after derivating $g_-(u)$,
\begin{eqnarray}
Z_- +\frac{1-2u}{\sqrt{u(1-u)}} & = & 0,\label{firsteqZ1}
\end{eqnarray}
\textit{i.e.} $1 - 2 u = - Z_-\sqrt{u(1-u)}$, or $1-(4+Z_-^2)u+(4+Z_-^2)u^2 = 0$, which brings the solution $u^*(z)$,
\begin{eqnarray}
u^*(z) & = & \frac{4+Z_-^2 \pm |Z_-|\sqrt{4+Z_-^2}}{2(4+Z_-^2)} = \frac{1}{2} \pm \frac{|Z_-|}{2 \sqrt{4+Z_-^2}} = \frac{1}{2} - \frac{|Z_-|}{2 \sqrt{4+Z_-^2}} ,
\end{eqnarray}
because we maximize $g_-$ in $[0,1/2]$. We get:
\begin{eqnarray}
h_-(u^*(z)) & = & \frac{\alpha Z_-}{2} - \frac{\alpha |Z_-|Z_-}{2 \sqrt{4+Z_-^2}} + 2\alpha \sqrt{\frac{1}{4} - \frac{Z_-^2}{4(4+Z_-^2)}}\nonumber\\
 & = & \frac{\alpha Z_-}{2} - \frac{\alpha |Z_-|Z_-}{2 \sqrt{4+Z_-^2}} + \alpha \sqrt{1 - \frac{Z_-^2}{4+Z_-^2}}\nonumber\\
 & = & \frac{\alpha Z_-}{2} - \frac{\alpha |Z_-|Z_-}{2 \sqrt{4+Z_-^2}} + \frac{2\alpha}{\sqrt{4+Z_-^2}}\nonumber\\
 & = & \alpha \cdot \left(\frac{Z_-}{2} + \frac{4 - |Z_-|Z_-}{2 \sqrt{4+Z_-^2}}\right)  \nonumber\\
 & = & \frac{z + 2(1-\alpha)}{2} + \frac{4\alpha^2 - |z + 2(1-\alpha)|(z + 2(1-\alpha))}{2 \alpha \sqrt{4+\left(\frac{z + 2(1-\alpha)}{\alpha}\right)^2}} \nonumber\\
 & = & 1 - \alpha+ \frac{z}{2} + \frac{4\alpha^2 - |z + 2(1-\alpha)|(z + 2(1-\alpha))}{2 \sqrt{4\alpha^2+(z + 2(1-\alpha))^2}} \defeq h^*_-(z). \label{eqzz1}
\end{eqnarray}
\noindent We now focus on $h_+(u)$. It is straightforward to check that \eqref{firsteqZ1} still holds but with $Z_+$ replacing $Z_-$ and 
\begin{eqnarray}
u^*(z) & = & \frac{1}{2} + \frac{|Z_+|}{2 \sqrt{4+Z_+^2}} \geq 1/2,
\end{eqnarray}
leading to
\begin{eqnarray}
h_+(u^*(z)) & = & 2 (1-\alpha) + \frac{\alpha Z_+}{2} + \frac{\alpha |Z_+|Z_+}{2 \sqrt{4+Z_+^2}} + 2\alpha \sqrt{\frac{1}{4} - \frac{Z_+^2}{4(4+Z_+^2)}}\nonumber\\
 & = & 2 (1-\alpha) +\frac{z - 2(1-\alpha)}{2} + \frac{4\alpha^2 + |z - 2(1-\alpha)|(z - 2(1-\alpha))}{2 \sqrt{4\alpha^2+(z - 2(1-\alpha))^2}}\nonumber\\
 & = &  1 - \alpha+ \frac{z}{2}  + \frac{4\alpha^2 + |z - 2(1-\alpha)|(z - 2(1-\alpha))}{2 \sqrt{4\alpha^2+(z - 2(1-\alpha))^2}} \defeq h^*_+(z).\label{eqzz2}
\end{eqnarray}
To finish up, we need to compute $\leaf_\alpha(z) = \max\{h^*_-(z), h^*_+(z)\}$ for $z \not\in 2(1-\alpha)\cdot [-1,1]$. 

\noindent \textbf{Case 2.1: $z < -2(1-\alpha)$} --- In this case,
\begin{eqnarray}
h^*_-(z) & = & 1 - \alpha+ \frac{z}{2} + \frac{4\alpha^2 + (z + 2(1-\alpha))^2}{2 \sqrt{4\alpha^2+(z + 2(1-\alpha))^2}},\nonumber\\
 & = & 1 - \alpha+ \frac{z}{2} + \frac{\sqrt{4\alpha^2+(z + 2(1-\alpha))^2}}{2}.\nonumber\\
h^*_+(z) & = & 1 - \alpha+ \frac{z}{2}  + \frac{4\alpha^2 - (z - 2(1-\alpha))^2}{2 \sqrt{4\alpha^2+(z - 2(1-\alpha))^2}},
\end{eqnarray}
and it is easy to check that $h^*_-(z) > h^*_+(z)$. 

\noindent \textbf{Case 2.1: $z > 2(1-\alpha)$} --- In this case,
\begin{eqnarray}
h^*_-(z) & = & 1 - \alpha + \frac{z}{2}  + \frac{4\alpha^2 - (z + 2(1-\alpha))^2}{2 \sqrt{4\alpha^2+(z + 2(1-\alpha))^2}},\nonumber\\
h^*_+(z) & = & 1 - \alpha + \frac{z}{2}  + \frac{4\alpha^2 + (z - 2(1-\alpha))^2}{2 \sqrt{4\alpha^2+(z - 2(1-\alpha))^2}}\nonumber\\
& = & 1 - \alpha+ \frac{z}{2} + \frac{\sqrt{4\alpha^2+(z - 2(1-\alpha))^2}}{2}.
\end{eqnarray}
and it is easy to check that $h^*_+(z) > h^*_-(z)$. 

\noindent To summarize \textbf{Case 2}, we get the convex conjugate and surrogate loss for Matsushita $\alpha$-entropy:
\begin{eqnarray}
\leaf_\alpha (z) & = & 
\left\{\begin{array}{rcl}
1 - \alpha+ \frac{z}{2} + \frac{\sqrt{4\alpha^2+(z + 2(1-\alpha))^2}}{2} & \mbox{ if } & z < -2(1-\alpha)\\
1 + \frac{z}{2} & \mbox{ if } & z \in 2(1-\alpha)\cdot [-1,1]\\
1 - \alpha+ \frac{z}{2} + \frac{\sqrt{4\alpha^2+(z - 2(1-\alpha))^2}}{2} & \mbox{ if } & z > 2(1-\alpha)
\end{array}
\right.,
\end{eqnarray}
which can be further simplified to
\begin{eqnarray}
\leaf_\alpha (z) & = & 1 + \frac{z}{2} + \iver{z \not\in 2(1-\alpha)\cdot [-1,1]}\cdot \left(\sqrt{\alpha^2+\left(\frac{|z|}{2} - (1-\alpha)\right)^2} - \alpha\right),
\end{eqnarray}
and the convex surrogate is just by definition 
\begin{eqnarray}
\alphasur(z) & = &
\leaf_\alpha (-z), \label{propCSUR}
\end{eqnarray}
 as claimed. We also get the inverse canonical link by
differentiating $\leaf_\alpha$, giving
\begin{eqnarray}
{\alphalink}^{-1}(z) & \defeq & \leaf'_\alpha (z) \nonumber\\
& = & \frac{1}{2} \cdot \left( 1 + \iver{z \not\in 2(1-\alpha)\cdot
      [-1,1]}\cdot \mathrm{sign}(z) \cdot \frac{\frac{|z|}{2} - (1-\alpha)}{\sqrt{\alpha^2+\left(\frac{|z|}{2} - (1-\alpha)\right)^2} }\right)
\end{eqnarray}
This achieves the proof of Theorem \ref{theoremalpha}.

\section{Proof of Theorem \ref{thBoostDT1}}\label{proof_thBoostDT1}

The proof proceeds in two steps. First we give some notations and
explain why our WLA in Definition \ref{wlaKMDT} is equivalent to
\citet[Section 3]{kmOT}. We then proceed to the proof itself.\\

\noindent $\triangleright$ \textbf{Notations and the Weak Learning Assumption}: recall that our objective is to minimise 
\begin{eqnarray}
\bayesalpharisk(h) & \defeq & \sum_{\leaf \in \leafset}
                            w(\leaf) \bayesalpharisk(q(\leaf)),\label{eqENTR-SM}
\end{eqnarray}
where $h$ is a tree and $\leafset$ is its set of leaves. Note also
that $\sum_{\leaf} w(\leaf) = w({\mathcal{S}})$, which is \textit{not}
normalized. Even when un-normalizing makes no difference, we are going
to stick to \citet{kmOT}'s setting and assume that our loss in
\eqref{eqENTR-SM} is \textit{normalized} (thus divided by
$w({\mathcal{S}})$). We shall remove this assumption at the end of the proof.

We have alleviated the boosting iteration index in $w$, so that $w(\leaf) \defeq  \sum_i w_i \cdot \iver{i \in \leaf}$. $q(\leaf) \in [0, 1]$ is the relative proportion of positive examples reaching leaf $\leaf$, 
\begin{eqnarray}
q(\leaf) & \defeq & (1/w(\leaf)) \cdot \sum_i \iver{(i \in \leaf) \wedge (y_i = +1)}\cdot w_i.
\end{eqnarray}
It should be clear at this stage that because we spend part of our DP budget each time we learn a split in a tree, we need to minimise \eqref{eqENTR-SM} as fast as possible under the weakest possible assumptions. Boosting gives us a very convenient framework to do so. Notations used are now simplified as summarized in Figure \ref{f-tree-not}, so that for example $q \defeq q(\leaf)$. 

\begin{figure}[t]
\begin{center}
\begin{tabular}{c}
\includegraphics[trim=50bp 580bp 680bp
10bp,clip,width=0.80\linewidth]{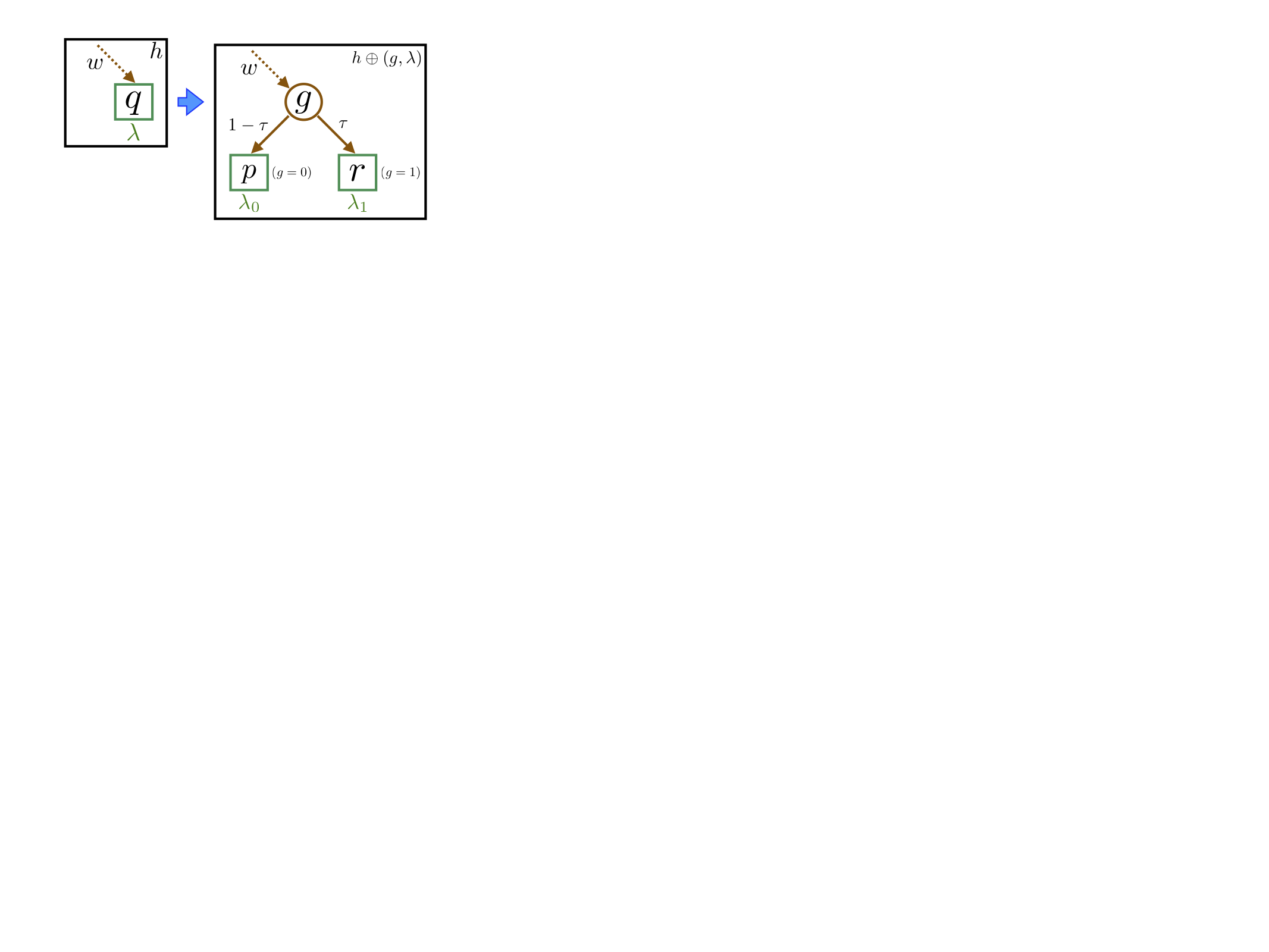}
\end{tabular}
\end{center}
\caption{Notations used in our proof of Theorem \ref{thBoostDT1}: leaf
  $\lambda$ in tree $h$ is replaced by subtree indexed by binary
  subtree with root test $g : \mathbb{R} \rightarrow \{0, 1\}$ and two
  new leaves $\lambda_0$ and $\lambda_1$ in grown tree $h\oplus(g, \leaf)$. The total proportion of examples reaching $\lambda$ (and therefore subject to test $g$) is $w$; the relative proportion of those for which $g(.)=0$ (resp. $g(.) = 1$) is $1-\tau$ (resp. $\tau$). The relative proportion of positive examples in $\lambda$ (resp. $\lambda_0$; resp. $\lambda_1$) is $q$ (resp $p$; resp. $r$).}
  \label{f-tree-not}
\end{figure}

We first review the weak learning assumption (WLA) for decision trees as carried out in \citet{kmOT}, which imposes a weak correlation between split $g$ and the labels of the examples reaching $\leaf$ for the split to meet the WLA. This correlation is measured not with respect to the current weights $w$ but to a distribution restricted to leaf $\leaf$ and giving equal weight to positive and negative examples: let 
\begin{eqnarray}
w_{\leaf, i} & \defeq & w_i \cdot \left\{
\begin{array}{rcl}
0 & \mbox{ if } & i \not\in \leaf\\
\frac{1}{2q} & \mbox{ if } & (i \in \leaf) \wedge (y_i = +1)\\
\frac{1}{2(1-q)} & \mbox{ if } & (i \in \leaf) \wedge (y_i = -1)
\end{array}
\right. .
\end{eqnarray}
\begin{definition}(Weak learning assumption, \citet{kmOT})\label{wlaKM}
Fix $\upgamma > 0$. Split $g$ at leaf $\leaf$ satisfies the $\upgamma$-weak learning assumption (WLA for short, omitting $\upgamma$) iff 
\begin{eqnarray}
\left| \sum_i w_{\leaf, i} \cdot \iver{( (g(\ve{x}_i) = 0) \wedge (y_i = +1) ) \vee ( (g(\ve{x}_i) = 1) \wedge (y_i = -1) )} - \frac{1}{2}\right| & \geq & \upgamma. \label{wladef1}
\end{eqnarray}
\end{definition}
It is not hard to check that, provided the splits are closed under negation (that is, if $g$ is a potential split then so is $\neg g$), then Definition \ref{wlaKM} is equivalent to the weak hypothesis assumption of \citet[Lemma 2]{kmOT}. To better see the correlation, define $g^{\nicefrac{+}{-}} \defeq -1 + 2g \in \{-1, 1\}$. Then it is not hard to check that
\begin{eqnarray*}
\lefteqn{\sum_i w_{\leaf, i} \cdot \iver{( (g(\ve{x}_i) = 0) \wedge (y_i = +1) ) \vee ( (g(\ve{x}_i) = 1) \wedge (y_i = -1) )}}\\
& = & \frac{1}{2} \cdot \sum_i w_{\leaf, i} \cdot (1 - y_i g^{\nicefrac{+}{-}}(\ve{x}_i))\\
& = & \frac{1}{2} \cdot \left( 1 - \sum_i w_{\leaf, i} \cdot y_i g^{\nicefrac{+}{-}}(\ve{x}_i) \right),
\end{eqnarray*}
so the WLA is equivalent to $|\sum_i w_{\leaf, i} \cdot y_i
g^{\nicefrac{+}{-}}(\ve{x}_i)| \geq 2 \upgamma$, that is, using the edge
notation $\eta(\ve{w}, h) \defeq \sum_i w_i y_i h(\ve(x)_i)$ with $h :
\mathcal{X} \rightarrow \mathbb{R}$ and $\ve{w}$ defines a discrete
distribution over the training sample $\mathcal{S}$, we can
reformulate the weak learning assumption as: split $g$ at leaf $\leaf$
satisfies the $\upgamma$-WLA iff $|\eta(\ve{w}_\leaf, g^{\nicefrac{+}{-}})| \geq
\upgamma$, which is Definition \ref{wlaKMDT} and is therefore
equivalent to Definition \ref{wlaKM} up to a factor 2 in the weak
learning guarantee.

\noindent $\triangleright$ \textbf{Proof of the Theorem}: we now
embark on the proof of Theorem \ref{thBoostDT1}. The proof follows the same
schema as \cite{kmOT} with some additional details to handle the
change of $\alpha$ in the course of training a DT. We first summarize
the high-level details of the proof. Denote $h\oplus(g, \leaf)$ tree $h$ in which a leaf
$\leaf$ has been replaced by a split indexed with some $g: \mathbb{R}
\rightarrow \{0,1\}$ satisfying the weak learning assumption (Figure \ref{f-tree-not}). The
decrease in $\bayesrisk(.)$, $\Delta \defeq
\bayesrisk(h)-\bayesrisk(h\oplus(g, \leaf))$, is lowerbounded as a function of
$\upgamma$ and then used to lowerbound the number of iterations (each
of which is the replacement of a leaf by a binary subtree) to get to a
given value of $\bayesrisk(.)$. It follows that $\Delta \defeq \omega(\leaf) \cdot \Delta_{\bayesalpharisk}(q, \tau, \delta)$, with
\begin{eqnarray}
\Delta_{\bayesalpharisk}(q, \tau, \delta) & \defeq & \bayesalpharisk(q) -
                                                 (1-\tau)
                                                 \bayesalpharisk(q-\tau\delta) -\tau \bayesalpharisk(q+(1-\tau)\delta)\label{defDELTA1}
\end{eqnarray}
with $\delta \defeq
\upgamma q(1-q)/(\tau(1-\tau))$ with $\tau$ denoting the
\textit{relative} proportion of examples for which $g = +1$ in leaf
$\leaf$, following \cite{kmOT}. We thus have
\begin{eqnarray}
\tau & \defeq & \frac{\sum_i w_i \cdot \iver{(i\in \leaf) \wedge
                (g(\ve{x}_i) = 1)}}{\sum_i w_i \cdot \iver{i\in \leaf}}.
\end{eqnarray}
We also introduce normalized weights with notation
$\tilde{w}_i \defeq w_i / w(\mathcal{S})$, so the total normalized
weight of examples reaching leaf $\leaf$ can also be denoted with the
tilda: $\tilde{w}(\leaf) \defeq \sum_i \tilde{w}_i \cdot \iver{i\in
  \leaf}$.

\begin{figure}[t]
\begin{center}
\begin{tabular}{c}
\includegraphics[trim=20bp 630bp 580bp
30bp,clip,width=0.80\linewidth]{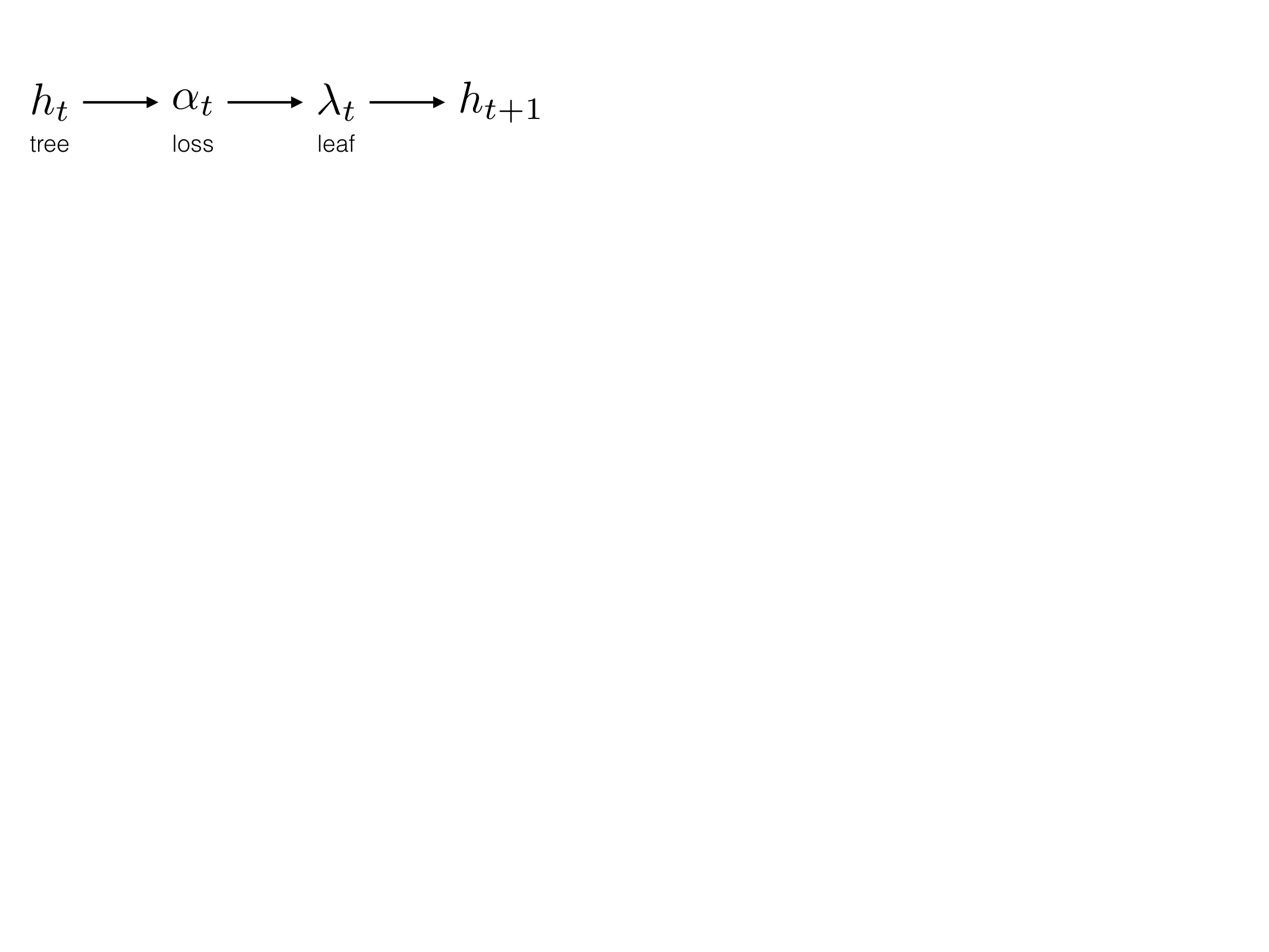}
\end{tabular}
\end{center}
\caption{Sequence of key parameters for the induction of a DT, which
  leads to tree $h_{t+1}$ after having split leaf $\leaf_t$ in
  $h_t$. $\alpha_t$ is the parameter chosen for the M$\alpha$-loss.}
  \label{f-learn}
\end{figure}

\noindent  We now let $h_\ell$ denote the current DT with $\ell$
leaves and $\ell-1$ internal nodes, the first tree being thus the single
root leaf $h_1$. We obtain $h_{\ell+1}$ by splitting a leaf $\leaf_\ell \in
\leafset(h_\ell)$, chosen to minimize
\begin{eqnarray}
\bayesalphariskparam{\ell}(h_{\ell+1}) & \defeq & \alpha_{\ell} \cdot \bayesmatrisk(h_{\ell+1})  + (1-\alpha_{\ell}) \cdot \bayeserrrisk(h_{\ell+1})\nonumber
\end{eqnarray}
over all possible leaf splits in $\leafset(h_\ell)$. Figure \ref{f-learn}
summarizes the whole process of getting $h_{\ell+1}$ from $h_\ell$.
\begin{lemma}\label{lemBCONV}
Suppose the sequence of $\alpha_\ell$ satisfies:
\begin{eqnarray}
\alpha_\ell & \leq & \alpha_{\ell-1} \cdot \exp\left(\frac{\upgamma^2 \tilde{w}_\ell }{16} \cdot (1-\alpha_{\ell-1})\right), \forall \ell>0,\label{condKMOT1}
\end{eqnarray}
with $\tilde{w}_\ell$ the total normalized weight of examples reaching leaf $\leaf_\ell$ split at iteration $\ell$.
Then for any $\xi \in (0,1]$, the empirical risk of $h_L$ satisfies $\emprisk(h_L) \leq \xi$ as long as
\begin{eqnarray}
\sum_{\ell=1}^L \tilde{w}_\ell \alpha_\ell & \geq & \frac{16}{\upgamma^2} \cdot \log \frac{1}{\xi}.
\end{eqnarray}
\end{lemma}
\begin{proof}
We first need a technical Lemma, in which we replace $\alpha_\ell$ by
$\alpha$ for the sake of readability.
\begin{lemma}\label{lemBINFDELTA}(Equivalent of \citet[Lemma 13]{kmOT}
  for $\Delta_{\bayesalpharisk}$) Fix $\alpha \in [0,1]$.
If $\upgamma < 0.2$ and $q$ is sufficiently small, then $\Delta_{\bayesalpharisk}$ is minimized by $\tau \in [0.4, 0.6]$.
\end{lemma}
\begin{proof}
We have 
\begin{eqnarray}
\Delta_{\bayesalpharisk}(q, \tau, \delta) & = & \alpha \cdot \Delta_{\bayesmatrisk}(q, \tau, \delta) + (1-\alpha) \cdot \Delta_{\bayeserrrisk}(q, \tau, \delta).
\end{eqnarray}
Suppose without loss of generality that $p \leq q \leq r$. It follows that if $r \leq 1/2$ or $p \geq 1/2$, $\Delta_{\bayeserrrisk}(q, \tau, \delta) = 0$ so we get the result directly from \citet[Lemma 13]{kmOT}. Otherwise, we have two cases.

\noindent \textbf{Case 1}: $q \leq 1/2, r>1/2$. In this case, 
\begin{eqnarray*}
\Delta_{\bayeserrrisk}(q, \tau, \delta) & = & 2q -
                                                2 (1-\tau)(q-\tau\delta) -2\tau (1-(q+(1-\tau)\delta))\\
& = & 2\tau \cdot \left( 2q + 2(1-\tau) \delta  -1 \right)\\
& = & 2\tau\cdot \left( 2q + \frac{2\upgamma q(1-q)}{\tau} - 1\right)\\
& = & 2\tau (2q - 1) + 4 \upgamma q(1-q),
\end{eqnarray*}
under the additional condition (for $r > 1/2$)
\begin{eqnarray}
\tau & < & \frac{2\gamma q(1-q)}{1-2q} \nonumber\\
 & & \sim_0 4\gamma q\label{condR}.
\end{eqnarray}
We get $\partial \Delta_{\bayeserrrisk}(q, \tau, \delta) / \partial \tau = 2 (2q-1)$ and so 
\begin{eqnarray}
\frac{\partial \Delta_{\bayesalpharisk}(q, \tau, \delta)}{\partial \tau} & = & \alpha \cdot \frac{\partial \Delta_{\bayesmatrisk}(q, \tau, \delta)}{\partial \tau} + 2(1-\alpha)(2q-1)\nonumber\\
& \leq & \alpha \cdot \frac{\partial \Delta_{\bayesmatrisk}(q, \tau, \delta)}{\partial \tau}
\end{eqnarray}
since $q \leq 1/2$, and it comes from Lemma 13 in \cite{kmOT} that $\partial \Delta_{\bayesalpharisk}(q, \tau, \delta) / \partial \tau \leq 0$ for $\tau \leq 0.4$, and under the condition of their Lemma ($q$ is sufficiently small, $\upgamma < 0.2$), then \eqref{condR} precludes $\tau \geq 0.6$ on Case 1.

\noindent \textbf{Case 2}: $q \geq 1/2, p<1/2$. In this case, we remark that $\Delta_{\bayesalpharisk}$ is invariant to the change 
$p\mapsto 1-p$, 
$q\mapsto 1-q$, 
$r\mapsto 1-r$, which brings us back to Case 1.
\end{proof}
The following Lemma brings the key brick to the proof of Lemma \ref{lemBCONV}.
\begin{lemma}\label{tekLEM}
Using notations of Figure \ref{f-tree-not}, suppose the split put at left $\leaf_\ell$ in $h_\ell$ satisfies the $\upgamma$-Weak Learning Assumption and furthermore the sequence of $\alpha$s satisfies \eqref{condKMOT1}.
Then we have 
\begin{eqnarray}
\bayesalphariskparam{\ell} (h_{\ell+1}) & \leq & \left(1 - \frac{\upgamma^2 \tilde{w}_\ell \alpha_\ell}{16}\right)\cdot \bayesalphariskparam{\ell-1} (h_{\ell}).\label{bONEITER}
\end{eqnarray}
\end{lemma}
\textbf{Remark}: the key result for Matsushita's loss in \citet[Theorem 10]{kmOT} follows from the particular case of Lemma \ref{tekLEM} for $\alpha_\ell = 1, \forall \ell$ (for which condition \eqref{condKMOT1} obviously holds for any $\upgamma$ and $\tilde{w}_\ell$).
\begin{proof}
We use the notations of Figures \ref{f-tree-not} and \ref{f-learn}. As long as the split satisfies the $\upgamma$-Weak Learning Assumption, we get from the proof of \citet[Theorem 10]{kmOT}
\begin{eqnarray}
\bayesmatrisk(h_{\ell+1}) & \leq & \left(1 - \frac{\upgamma^2 \tilde{w}_\ell }{16}\right) \cdot \bayesmatrisk(h_{\ell}),\label{eqKM1}
\end{eqnarray}
further noting that the use of Lemma \ref{lemBINFDELTA} is "hidden" in this bound, but proceeds as in the proof of \citet[Theorem 10]{kmOT}. We remind that if we tune $\alpha$ then by definition
\begin{eqnarray}
\bayesalphariskparam{\ell} (h_{\ell+1}) & \defeq & \alpha_\ell \cdot \bayesmatrisk(h_{\ell+1})  + (1-\alpha_\ell) \cdot \bayeserrrisk(h_{\ell+1}) ,\nonumber\\
\bayesalphariskparam{\ell-1} (h_{\ell}) & \defeq & \alpha_{\ell-1} \cdot \bayesmatrisk(h_{\ell})  + (1-\alpha_{\ell-1}) \cdot \bayeserrrisk(h_{\ell}) .\nonumber
\end{eqnarray}
Now we have, successively because of \eqref{eqKM1} and $\bayeserrrisk(h_{\ell+1}) \leq \bayeserrrisk(h_{\ell})$ (error cannot increase as the partition of $\mathcal{X}$ achieved by $h_{\ell+1}$ is finer than that of $h_\ell$),
\begin{eqnarray}
\bayesalphariskparam{\ell} (h_{\ell+1}) & \leq & \alpha_{\ell} \cdot \left(1 - \frac{\upgamma^2 \tilde{w}_\ell }{16}\right) \cdot \bayesmatrisk(h_{\ell}) + (1-\alpha_\ell) \cdot \bayeserrrisk(h_{\ell+1}) \nonumber\\
& \leq & \alpha_{\ell} \cdot \left(1 - \frac{\upgamma^2 \tilde{w}_\ell }{16}\right) \cdot \bayesmatrisk(h_{\ell}) + (1-\alpha_\ell) \cdot \bayeserrrisk(h_{\ell})\nonumber\\
& & = \alpha_{\ell} \cdot \left(1 - \frac{\upgamma^2 \tilde{w}_\ell }{16}\right) \cdot \bayesmatrisk(h_{\ell}) + Q \cdot \bayeserrrisk(h_{\ell}) \nonumber\\
& & + (1-\alpha_{\ell-1}) \cdot \left(1- \frac{\upgamma^2 \tilde{w}_\ell \alpha_{\ell}}{16}\right) \cdot \bayeserrrisk(h_{\ell})\label{llKM2},
\end{eqnarray}
with 
\begin{eqnarray}
Q & \defeq & \alpha_{\ell-1} - \alpha_\ell + \frac{\upgamma^2 \tilde{w}_\ell }{16} \cdot \alpha_\ell(1-\alpha_{\ell-1}).
\end{eqnarray}
Now, if
\begin{eqnarray}
\alpha_\ell & \leq & \frac{\alpha_{\ell-1}}{1 - \frac{\upgamma^2 \tilde{w}_\ell }{16} \cdot (1-\alpha_{\ell-1})},\label{approxEQ1}
\end{eqnarray}
then $Q \geq 0$. Since $\bayeserrrisk(h_{\ell}) \leq \bayesmatrisk(h_{\ell})$, 
\begin{eqnarray}
\lefteqn{\alpha_{\ell} \cdot \left(1 - \frac{\upgamma^2 \tilde{w}_\ell }{16}\right) \cdot \bayesmatrisk(h_{\ell}) + Q \cdot \bayeserrrisk(h_{\ell})}\nonumber \\
  & \leq & \alpha_{\ell} \cdot \left(1 - \frac{\upgamma^2 \tilde{w}_\ell }{16}\right) \cdot \bayesmatrisk(h_{\ell}) + Q \cdot \bayesmatrisk(h_{\ell}) \nonumber\\
& & =\left(\alpha_\ell - \frac{\upgamma^2 \tilde{w}_\ell \alpha_\ell}{16} + \alpha_{\ell-1} - \alpha_\ell + \frac{\upgamma^2 \tilde{w}_\ell }{16} \cdot \alpha_\ell(1-\alpha_{\ell-1})\right)\cdot  \bayesmatrisk(h_{\ell})\nonumber\\
& = & \alpha_{\ell-1} \cdot \left(1 - \frac{\upgamma^2 \tilde{w}_\ell \alpha_\ell}{16}\right)\cdot  \bayesmatrisk(h_{\ell}),\nonumber
\end{eqnarray}
and so, assembling with \eqref{llKM2}, we get
\begin{eqnarray}
\bayesalphariskparam{\ell} (h_{\ell+1}) & \leq & \alpha_{\ell-1} \cdot \left(1 - \frac{\upgamma^2 \tilde{w}_\ell \alpha_\ell}{16}\right)\cdot  \bayesmatrisk(h_{\ell}) + (1-\alpha_{\ell-1}) \cdot \left(1- \frac{\upgamma^2 \tilde{w}_\ell \alpha_{\ell}}{16}\right) \cdot \bayeserrrisk(h_{\ell})\nonumber\\
& & = \left(1 - \frac{\upgamma^2 \tilde{w}_\ell \alpha_\ell}{16}\right)\cdot (\alpha_{\ell-1} \cdot \bayesmatrisk(h_{\ell}) + (1-\alpha_{\ell-1}) \cdot \bayeserrrisk(h_{\ell}))\nonumber\\
& = & \left(1 - \frac{\upgamma^2 \tilde{w}_\ell \alpha_\ell}{16}\right)\cdot \bayesalphariskparam{\ell-1} (h_{\ell}),
\end{eqnarray}
which achieves the proof of Lemma \ref{tekLEM} once we use the fact
that $1-z \leq \exp(-z)$ on the denominator of \eqref{approxEQ1},
which yields a lower-bound on its right-hand side and thus a
sufficient condition of this inequality to hold, which, after
simplification, is \eqref{condKMOT1} and the definition of
$\Gamma$-monotonicity in the main file. Notice finally that the first
split, on $h_1$ to get $h_2$ ($t\defeq 1$) introduces a dependence on
$\alpha_0 \in [0,1]$ to compute the M$\alpha_0$-loss of the root
leaf. Since $\bayesalpharisk(q)\leq \bayesmatrisk(q), \forall q \in
[0,1]$, we just pick $\alpha_0 = 1$, which implies complete freedom to
pick $\alpha_1 \in [0,1]$ under $\Gamma$-monotonicity.
\end{proof}

To finish the proof of Lemma \ref{lemBCONV}, we use the fact that $1-z \leq \exp(-z)$ and unravel \eqref{bONEITER}: after $L$ iterations of boosting, under the conditions of Lemma \ref{tekLEM}, we get
\begin{eqnarray}
\bayesalpharisk(h_{L}) & \leq & \exp\left(-\frac{\upgamma^2}{16}\cdot
                                \sum_{\ell=1}^L \tilde{w}_\ell
                                \alpha_\ell\right), \label{eqCONV1C}
\end{eqnarray}
from which, since $\alpha_\ell \in [0,1], \forall \ell$, we have the empirical risk of $h_L$, $\emprisk(h_L)$, satisfy $\emprisk(h_L) = \bayeserrrisk(h_{L}) \leq \bayesalpharisk(h_{L})$ and a sufficient condition for $\emprisk(h_L) \leq \xi$ is thus
\begin{eqnarray}
\sum_{\ell=1}^L \tilde{w}_\ell \alpha_\ell & \geq & \frac{16}{\upgamma^2} \cdot \log \frac{1}{\xi},
\end{eqnarray}
which is the statement of Lemma \ref{lemBCONV}.
\end{proof}
Remark that Lemma \ref{lemBCONV} is Theorem \ref{thBoostDT1} \textit{with normalized
  weights}. If we consider unnormalized weights in $\bayesalpharisk$
then we need to multiply the right hand side of \eqref{eqCONV1C} by
$w({\mathcal{S}})$, but we also have in this case $\emprisk(h_L) \leq
\bayesalpharisk(h_{L}) / w({\mathcal{S}})$, which in fact does not
change the statement for normalized weights. We also remark that the
Weak Learning Assumption is not affected by this change in
normalization, so we get the statement of Theorem \ref{thBoostDT1} for unnormalized
weights as well.

\section{Proof of Theorem \ref{thBoostLC1}}\label{proof_thBoostLC1}

\begin{algorithm}[t]
\caption{\maboost}\label{cboost}
\begin{algorithmic}
  \STATE  \textbf{Input} sample ${\mathcal{S}} = \{(\bm{x}_i, y_i), i
  = 1, 2, ..., m\}$, number of iterations $T$, loss and update parameters
\begin{eqnarray}
\alpha & \in & (0, 1]\nonumber\\
\pi & \in & [0, 1)\nonumber\\
a & \in & \frac{\alpha}{M^2}\cdot \left[ 1 - \pi, 1 + \pi\right];
\end{eqnarray}
\STATE  Step 1 : let $w_i = 1/2, \forall i = 1, 2, ..., m$; // initial weights
\STATE  Step 2 : \textbf{for} $t = 1, 2, ..., T$
\STATE  \hspace{1.1cm} Step 2.1 : let $h_t \leftarrow
\weak({\mathcal{S}}, \bm{w}_t)$\; // weak classifier
\STATE  \hspace{1.1cm} Step 2.2 : let $\beta_t \leftarrow (a/m) \cdot \sum_{i}
     {w_{ti} y_{i} h_t(\bm{x}_i)}$\; //
leveraging coefficient
\STATE  \hspace{1.1cm} Step 2.3 : \textbf{for} $i = 1, 2, ..., m$, let
\begin{eqnarray}
w_{(t+1)i} & \leftarrow & {\alphalink}^{-1}\left( -\beta_t y_{i}
                          h_t(\bm{x}_i) + {\alphalink} (w_{ti})\right)
                          \quad(\in [0,1])\:\:; \label{defwun}
\end{eqnarray}
\STATE \textbf{Return} $H_T = \sum_t \beta_t h_t$.
\end{algorithmic}
\end{algorithm}

\begin{figure}[t]
\begin{center}
\begin{tabular}{c}
\includegraphics[trim=5bp 0bp 20bp
10bp,clip,width=0.45\linewidth]{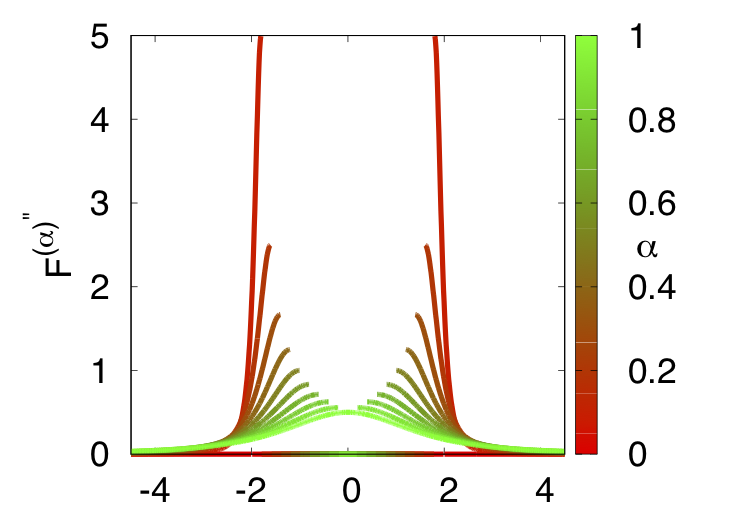}
\end{tabular}
\end{center}
\caption{Second derivative of the convex surrogate $\alphasur$, for
  various values of $\alpha$. The color code follows Figure
  \ref{f-alphaM} in the main file.}
  \label{f-der2}
\end{figure}

We first display in Algorithm \maboost~the complete pseudo-code of our
approach to boosting using the M$\alpha$-loss. In stating the
algorithm, we have simplified notations; in particular we can indeed
check that the leveraging coefficient of $h_t$ satisfies:
\begin{eqnarray}
\beta_t & = & a \tilde{w}_t \eta(\tilde{\ve{w}}_t, h_t).
\end{eqnarray}
We make use of the same
proof technique as in \citet[Theorem 7]{nwLO}. We
sketch here the main steps. A first quantity we define is:
\begin{eqnarray}
X & \defeq & \expect_{{\mathcal{S}}}\left[ (y_i
                                             H_{t}(\ve{x}_i) - y_i
                                             H_{t+1}(\ve{x}_i)) \alphasur '(y_i
                                             H_{t}(\ve{x}_i))\right]
        \nonumber\\
& = & \beta_t \expect_{{\mathcal{S}}}\left[ -y_i h_t (\ve{x}_i) \cdot - {\alphalink}^{-1}\left(-y_i
                                             H_{t}(\ve{x}_i)\right)\right]\label{fEQ1}\\
& = & \beta_t \expect_{{\mathcal{S}}}\left[w_{ti} y_i h_t (\ve{x}_i) \right]\label{fEQ2}\\
& = & \beta_t \cdot \frac{1}{m}\cdot \sum_i w_{ti} y_i h_t (\ve{x}_i)\label{fEQ3}\\
& = & a \tilde{w}^2_t \eta^2(\tilde{\ve{w}}_t, h_t).\label{bX}
\end{eqnarray}
\eqref{fEQ1} holds because of \eqref{propCSUR} and the fact that
$H_{t+1}(\ve{x}_i) = H_{t}(\ve{x}_i) + y_i h_t (\ve{x}_i)$ by
definition. \eqref{fEQ2} holds because of the definition of $w_{ti}$
and \eqref{fEQ3} is just a rewriting using the distribution of
examples in $\mathcal{S}$. A second quantity we define is
\begin{eqnarray}
Y(\mathcal{Z}) & \defeq & \expect_{{\mathcal{S}}}\left[ (y_i
                                             H_{t}(\ve{x}_i) - y_i
                                             H_{t+1}(\ve{x}_i))^2
     \alphasur ''(z_i)\right] \label{defYY},
\end{eqnarray}
where $\mathcal{Z} \defeq \{z_1, z_2, ..., z_m\} \subset
\mathbb{R}^m$. We then need to compute the second derivative of
$\alphasur$, which we find to be (Figure \ref{f-der2})
\begin{eqnarray}
\alphasur''(z) & = & \left\{
\begin{array}{ccl}
0 & \mbox{ if } & z \in 2(1-\alpha)\cdot(-1,1)\\
\frac{4\alpha^2}{\left(4\alpha^2+(|z|-2(1-\alpha))^2\right)^{\frac{3}{2}}}
  & \mbox{ if } & z \not\in 2(1-\alpha)\cdot[-1,1]\\
\mbox{undefined} & \mbox{ if } & z \in 2(1-\alpha)\cdot\{-1,1\}
\end{array}
\right..
\end{eqnarray}
from which we easily find
\begin{eqnarray}
\sup_z \alphasur'' & = & \frac{1}{2\alpha},
\end{eqnarray}
and therefore for any $\mathcal{Z} \subset
\mathbb{R}^m$,
\begin{eqnarray}
Y(\mathcal{Z}) & \leq & \frac{1}{2\alpha}\cdot \expect_{{\mathcal{S}}}\left[ (y_i
                                             H_{t}(\ve{x}_i) - y_i
                                             H_{t+1}(\ve{x}_i))^2\right]\nonumber\\
& & =  \frac{1}{2\alpha}\cdot \expect_{{\mathcal{S}}}\left[ (a\eta_t\cdot
    h_t(\ve{x}_i))^2\right]\nonumber\\
& \leq & \frac{a^2 \tilde{w}^2_t \eta^2(\tilde{\ve{w}}_t, h_t) M^2}{2\alpha}\label{bY}.
\end{eqnarray}
We then get from the proof of \citet[Theorem 7]{nwLO} and \eqref{bX}, \eqref{bY} that there
exists a set $\mathcal{Z} \subset \mathbb{R}^m$ such that
\begin{eqnarray}
\expect_{{\mathcal{S}}}\left[ \alphasur (y_i H_{t}(\ve{x}_i))\right] - \expect_{{\mathcal{S}}}\left[
  \alphasur (y_i H_{t+1}(\ve{x}_i))\right] & \geq & X - Y(\mathcal{Z})
                                                    \nonumber\\
& \geq & a \tilde{w}^2_t \eta^2(\tilde{\ve{w}}_t, h_t) - \frac{a^2 \tilde{w}^2_t \eta^2(\tilde{\ve{w}}_t, h_t) M^2}{2\alpha}\nonumber\\
& & \left(1 - \frac{aM^2}{2\alpha}\right)\cdot a \eta_t^2 .
\end{eqnarray}
Suppose
\begin{eqnarray}
a & \in & \frac{\alpha}{M^2}\cdot \left[ 1 - \pi, 1 + \pi\right]
\end{eqnarray}
for some $\pi \in [0,1]$. We then have:
\begin{eqnarray}
\expect_{{\mathcal{S}}}\left[ \alphasur (y_i H_{t}(\ve{x}_i))\right] - \expect_{{\mathcal{S}}}\left[
  \alphasur (y_i H_{t+1}(\ve{x}_i))\right] & \geq &
                                                    \frac{(1-\pi^2)\alpha}{2M^2}\cdot \eta_t^2,
\end{eqnarray}
so after combining $T$ classifiers in the linear combination, we get
\begin{eqnarray}
\expect_{{\mathcal{S}}}\left[
  \alphasur (y_i H_{T}(\ve{x}_i))\right] & \leq & \alphasur(0) -
                                                  \frac{(1-\pi^2)\alpha}{2M^2}\cdot
                                                  \sum_{t=1}^T
                                                  \tilde{w}^2_t \eta^2(\tilde{\ve{w}}_t, h_t)\nonumber\\
& & = 1 - \frac{(1-\pi^2)\alpha}{2M^2}\cdot
                                                  \sum_{t=1}^T
                                                  \tilde{w}^2_t \eta^2(\tilde{\ve{w}}_t, h_t).\label{eqSEQ}
\end{eqnarray}
To summarize, if the sequence of edges satisfies 
\begin{eqnarray}
\frac{1}{M^2} \cdot \sum_{t=1}^T
                                                  \tilde{w}^2_t \eta^2(\tilde{\ve{w}}_t, h_t) & \geq & \frac{2
                                                                    (1-\xi)}{(1-\pi^2)\alpha},\label{eqCONST1111}
\end{eqnarray}
then
\begin{eqnarray}
\expect_{{\mathcal{S}}}\left[ \alphasur (y_i H_{T}(\ve{x}_i))\right] & \leq &
                                                    \xi.\label{eqCONST1112}
\end{eqnarray}
Since for any $\alpha > 0$, $\alphasur$ is strictly decreasing and non
negative, for
any $\theta\geq 0$, if $\pr_{{\mathcal{S}}}\left[
  \iver{y_i H_{T}(\ve{x}_i)\leq \theta}\right] > \xi$, then
\begin{eqnarray}
\expect_{{\mathcal{S}}}\left[
  \alphasur (y_i H_{T}(\ve{x}_i))\right] & > & \xi
                                               \alphasur (\theta)
                                               +
                                               (1-\xi)\inf_z\alphasur
                                               (z)\nonumber\\
& & \geq \xi
                                               \alphasur (\theta).
\end{eqnarray}
Hence, we get from \eqref{eqCONST1111} and \eqref{eqCONST1112} that if
the sequence of edges satisfies
\begin{eqnarray}
\sum_{t=1}^T
                                                  \tilde{w}^2_t \eta^2(\tilde{\ve{w}}_t, h_t) & \geq &
                                                                    \frac{2M^2
                                                                    (1-\xi
                                                                    \alphasur
                                                                    (\theta))}{(1-\pi^2)\alpha},\label{eqCONST111}
\end{eqnarray}
then $\expect_{{\mathcal{S}}}\left[
  \alphasur (y_i H_{T}(\ve{x}_i))\right] \leq \xi
\alphasur(\theta)$ and so
\begin{eqnarray}
\expect_{{\mathcal{S}}}\left[
  \iver{y_i H_{T}(\ve{x}_i)\leq \theta}\right] & \leq & \xi.\label{eqBB1}
\end{eqnarray}
There remains to remark that $\emprisk(H_T) \leq \expect_{{\mathcal{S}}}\left[
  \iver{y_i H_{T}(\ve{x}_i)\leq 0}\right] $, and therefore pick
$\theta = 0$ for which $\alphasur(\theta) = 1$. Under the
$\upgamma$-WLA, we note that
\begin{eqnarray}
\tilde{w}^2_t \eta^2(\tilde{\ve{w}}_t, h_t) & \geq &
                                                     \tilde{w}^2_t\upgamma^2 M^2,\nonumber
\end{eqnarray}
and so, to summarise, under the $\upgamma$-WLA, if the sequence of 
expected weights satisfies
\begin{eqnarray}
\sum_{t=1}^T
                                                  \tilde{w}^2_t & \geq &
                                                                    \frac{2
                                                                    (1-\xi)}{(1-\pi^2)\upgamma^2
                                                                         \alpha},\label{eqCONST111}
\end{eqnarray}
then $\emprisk(H_T)\leq \xi$. This ends the proof of Theorem \ref{thBoostLC1}.

\section{Proof of Theorem \ref{thBOOSTDP1}}
\label{proof_thBOOSTDP1}

We first prove a preliminary result used in the main file.
\begin{lemma}\label{lemPRELIM}
For any $\alpha_\ell \in [0,1]$, any split $g$ on leaf $\leaf$ that
satisfies the $\upgamma$-Weak Learning Assumption on $h_\ell$ yields
\begin{eqnarray}
\bayesalphariskparam{\ell}(h_\ell
      \oplus (g, \leaf))  & \leq & \left(1 - \frac{\upgamma^2 \alpha_\ell \tilde{w}(\leaf)
      }{16}\right) \cdot \bayesalphariskparam{\ell}(h_\ell).
\end{eqnarray}
\end{lemma}
\begin{proof}
As long as split $g$ on leaf $\leaf$ satisfies the $\upgamma$-Weak Learning Assumption, we get from the proof of \citet[Theorem 10]{kmOT}
\begin{eqnarray}
\bayesmatrisk(h_\ell
      \oplus (g, \leaf)) & \leq & \left(1 - \frac{\upgamma^2 \tilde{w}_\ell }{16}\right) \cdot \bayesmatrisk(h_{\ell}),\label{eqKM122}
\end{eqnarray}
It yields, $\forall \alpha_\ell \in [0,1]$,
\begin{eqnarray}
\lefteqn{\bayesalphariskparam{\ell}(h_\ell
      \oplus (g, \leaf)) \defeq \alpha_\ell \bayesmatrisk(h_\ell
      \oplus (g, \leaf)) + (1-\alpha_\ell) \bayeserrrisk(h_\ell
      \oplus (g, \leaf))}\nonumber\\
 & \leq & \alpha_\ell \left(1 - \frac{\upgamma^2 \tilde{w}_\ell }{16}\right)
          \cdot \bayesmatrisk(h_{\ell}) +  (1-\alpha_\ell) \bayeserrrisk(h_\ell
      \oplus (g, \leaf))\nonumber\\
& \leq & \alpha_\ell \left(1 - \frac{\upgamma^2 \tilde{w}_\ell }{16}\right)
          \cdot \bayesmatrisk(h_{\ell}) +  (1-\alpha_\ell)
         \bayeserrrisk(h_\ell)\label{eq1BSUP}\\
& & = \alpha_\ell \left(1 - \frac{\upgamma^2 \alpha_\ell \tilde{w}_\ell }{16}\right)
          \cdot \bayesmatrisk(h_{\ell}) +  (1-\alpha_\ell) \left(1 -
    \frac{\upgamma^2 \alpha_\ell \tilde{w}_\ell }{16}\right) \cdot 
         \bayeserrrisk(h_\ell)+ Q\nonumber\\
& = & \left(1 - \frac{\upgamma^2 \alpha_\ell \tilde{w}(\leaf)
      }{16}\right) \cdot \bayesalphariskparam{\ell}(h_\ell) + Q,
\end{eqnarray}
where \eqref{eq1BSUP} holds because the partition achieved by $h_\ell
      \oplus (g, \leaf)$ is finer than that achieved by $h_\ell$ (hence,
      its empirical risk cannot be greater), with
\begin{eqnarray}
Q & \defeq & \left[ \alpha_\ell \left(1 - \frac{\upgamma^2 \tilde{w}_\ell
             }{16}\right)- \alpha_\ell \left(1 - \frac{\upgamma^2
             \alpha_\ell \tilde{w}_\ell }{16}\right)\right]\cdot
             \bayesmatrisk(h_{\ell}) \nonumber\\
& & + \left[(1-\alpha_\ell) -  (1-\alpha_\ell) \left(1 -
    \frac{\upgamma^2 \alpha_\ell \tilde{w}_\ell }{16}\right) \right]\cdot 
         \bayeserrrisk(h_\ell)\nonumber\\
& = & - \frac{\upgamma^2 \alpha_\ell \tilde{w}_\ell
             }{16}(1-\alpha_\ell) \cdot
             \bayesmatrisk(h_{\ell}) + \frac{\upgamma^2 \alpha_\ell \tilde{w}_\ell
             }{16}(1-\alpha_\ell) \cdot
             \bayeserrrisk(h_{\ell}) \nonumber\\
& = & - \frac{\upgamma^2 \alpha_\ell \tilde{w}_\ell
             }{16}(1-\alpha_\ell) \cdot
             (\bayesmatrisk(h_{\ell}) -
             \bayeserrrisk(h_{\ell})) \nonumber\\
& \leq & 0
\end{eqnarray}
because $\bayesmatrisk(h_{\ell}) \geq \bayeserrrisk(h_{\ell})$ for any
$\alpha_\ell, h_\ell$. This ends the proof of Lemma \ref{lemPRELIM}
\end{proof}

\begin{figure}[t]
\begin{center}
\begin{tabular}{c}
\includegraphics[trim=20bp 600bp 630bp
40bp,clip,width=0.80\linewidth]{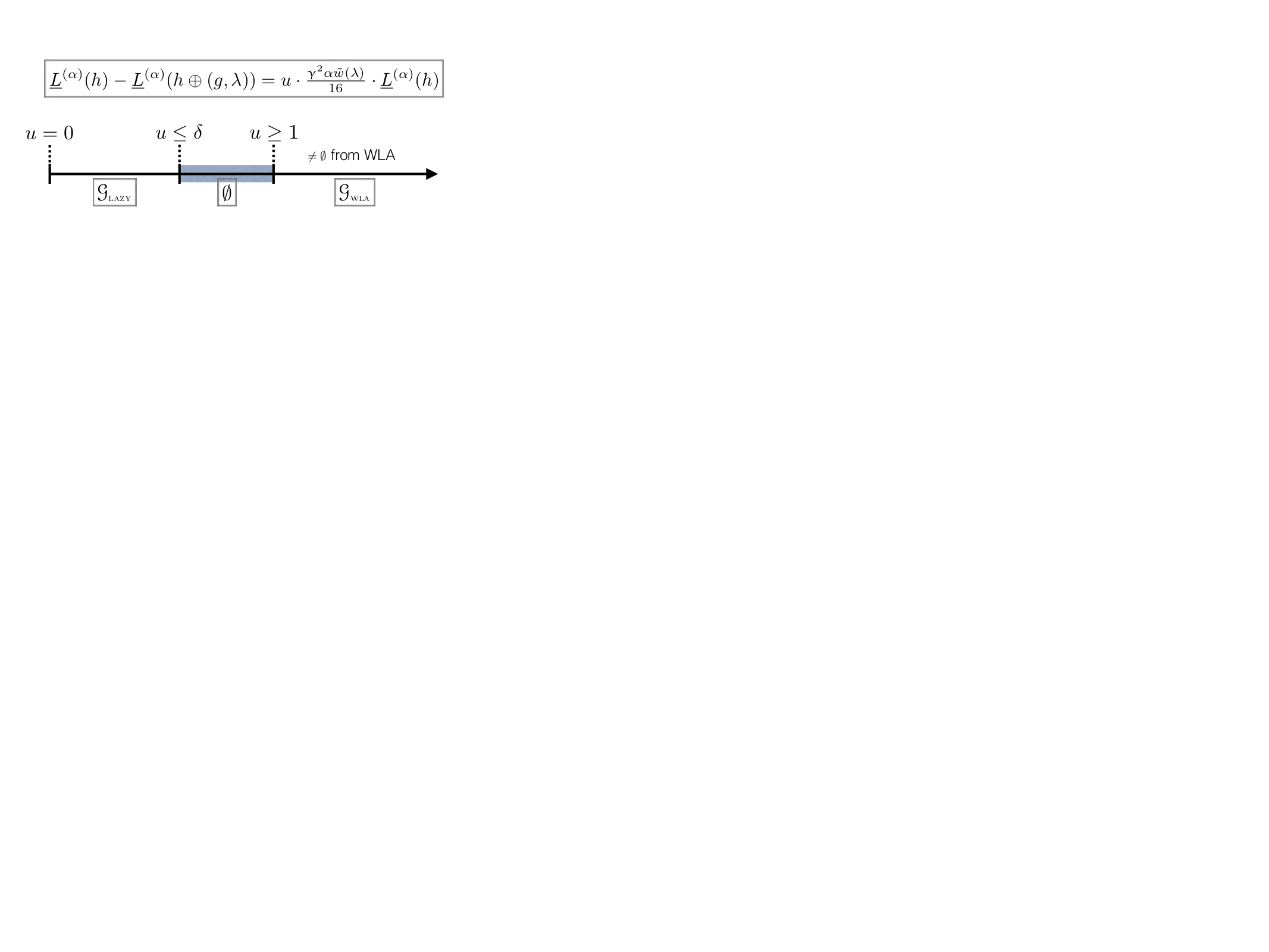}
\end{tabular}
\end{center}
\caption{In the $\delta$-Gap model of boosting, the total set of potential
  splits $\mathcal{G}$ contains two subsets in regard to the current
  leaf that is being split, $\leaf$. A subset $\setfat$ contains all
  splits that guarantee a moderate decrease in the Bayes risk -- this
  set is guaranteed non empty under the Weak Learning
  Assumption (Lemma \ref{lemPRELIM}). Another set, $\setslim$, contains all the other splits,
  supposed to yield a decrease in the Bayes risk at least smaller by
  factor $\delta < 1$. In the main file, we have assumed for
  simplicity that we can fix $\delta = \upgamma$ but the proof of
  Theorem \ref{thBOOSTDP1} below relaxes this assumption.}
  \label{f-gapmodel}
\end{figure}

Notations are as follows: $\mathcal{G}$ denotes the complete set of possible splits and
\begin{eqnarray}
\kappa & \defeq & \frac{\epsilon}{2\Delta^*_{\bayesalpharisk}(m)},
\end{eqnarray}
which depends on $\epsilon, m, \alpha, \leaf$ (See Corollary \ref{sensALPHA} in the
main file). $\nodeset(h)$ denotes the set of nodes of $h$, including
leaves in $\leafset(h)$. 

\begin{definition}
For any node $\node \in \nodeset(h)$, let $\depth(\node)$ denote its
depth in $h$ and $\tilde{w}(\node) \in [0,1]$ the normalized weight of
examples reaching $\node$. The \textbf{tree-efficiency} of $\node$ in
$h$ is:
\begin{eqnarray}
J(\node, h) & \defeq & \frac{8\tilde{w}(\node)
                                                 \emprisk(h)^2}{2^{\depth(\node)}}
                      \quad \in [0,1].
\end{eqnarray}
\end{definition}

\begin{figure}[t]
\begin{center}
\begin{tabular}{c}
\includegraphics[trim=50bp 600bp 740bp
30bp,clip,width=0.60\linewidth]{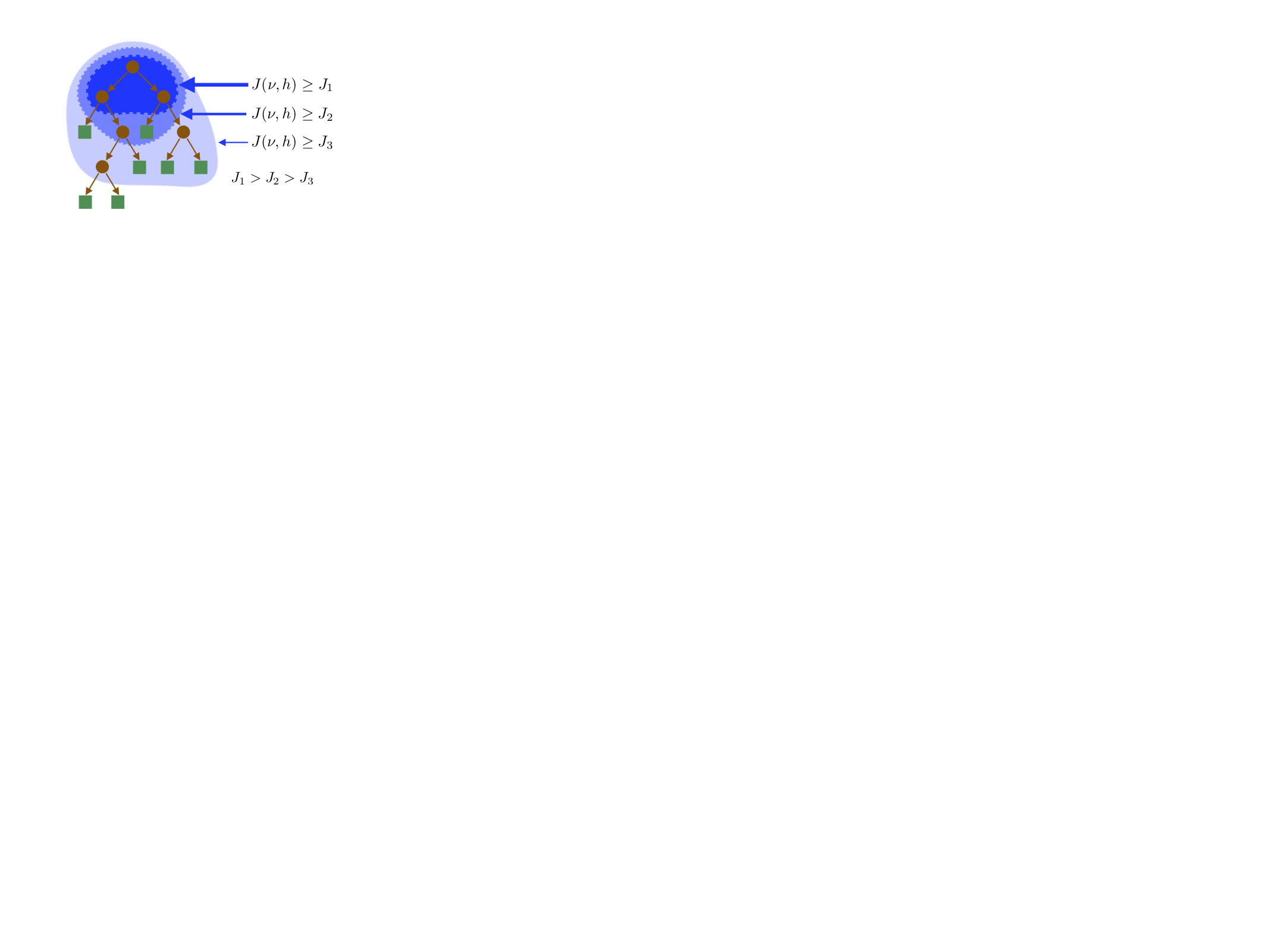}
\end{tabular}
\end{center}
\caption{Visualisation of Lemma \ref{lem-rtnode}: root-to-node tree
  efficiency is decreasing.}
  \label{f-mono}
\end{figure}

The following Lemma gives a key property of the tree efficiency of a node. 
\begin{lemma}\label{lem-rtnode}
(Tree efficiency is root-to-node decreasing) For any decision tree $h$, consider any path of nodes $\node_1, \node_2,
..., \node_k \in \nodeset(h)$ where $\node_1$ is the root of $h$ and
$\depth(\node_{i+1}) = \depth(\node_{i})+1$, $\forall i$. Then the
tree efficiency is strictly decreasing along this path:
$J(\node_i, h) > J(\node_{i+1}, h), \forall i$.
\end{lemma}
The proof of this Lemma comes from the fact that along such a path,
$\tilde{w}(.)$ is non-increasing while depth strictly
increases. Figure \ref{f-mono} gives a sketch visualisation of Lemma \ref{lem-rtnode}.\\

We now prove Theorem \ref{thBOOSTDP1}. We consider two cases, starting
first with the simplified case of a single split and then investigate a set of splits.\\

$\triangleright$ \textbf{Single split}: notation $h\oplus(g, \leaf)$ indicates decision tree $h$
in which leaf $\leaf \in \leafset(h)$ is replaced by split $g \in \mathcal{G}$.
It follows from \cite{fsDM} that the probability to pick
split $g$ for leaf $\leaf \in h$ following the exponential mechanism, $\pexpm ((g, \leaf))$,
\begin{eqnarray}
\pexpm ((g, \leaf)) & = & \frac{1}{Z}\cdot
                    \exp\left(-\kappa \cdot w(\mathcal{S}) \cdot F(h\oplus(g, \leaf))\right),
\end{eqnarray}
where $Z \defeq \sum_{g'\in \mathcal{G}} \exp\left(-\kappa
  w(\mathcal{S}) \cdot F(h\oplus(g, \leaf)) \right)$. Notice that the part in the sum in
$F(h\oplus(g, \leaf))$ that does not depend on $\leaf$ can be factored
thanks to the $\exp$, which allows us to simplify
\begin{eqnarray}
\pexpm ((g, \leaf)) & = & \frac{1}{Z}\cdot \exp\left(-\kappa \cdot
                    \left[w(\leaf\wedge g) \cdot
                    \bayesrisk\left( \frac{w^1(\leaf\wedge
                    g)}{w(\leaf\wedge g)} \right) +
                    w(\leaf\wedge \neg g) \cdot \bayesrisk\left(
                    \frac{w^1(\leaf\wedge \neg g)}{w(\leaf\wedge \neg
                    g)} \right)\right]\right)\nonumber\\
& \propto & \exp\left(\kappa \cdot
                    \left[\bayesalpharisk(h) - \bayesalpharisk(h
      \oplus (g, \leaf))\right]\right)
\end{eqnarray}
and $Z$ is the normalization coefficient modified
accordingly. Suppose, $h$ and $\leaf$ being fixed, that we have two subsets,
$\setfat$ and $\setslim$ such that
\begin{eqnarray}
\bayesalpharisk(h) - \bayesalpharisk(h
      \oplus (g, \leaf)) & \geq
  & \frac{\upgamma^2 \alpha \tilde{w}(\leaf) }{16} \cdot \bayesalpharisk(h),
    \forall g \in \setfat\label{eqKM1bb},\\
\bayesalpharisk(h) - \bayesalpharisk(h
      \oplus (g, \leaf)) & \leq
  & \frac{\delta^2 \upgamma^2 \alpha\tilde{w}(\leaf) }{16} \cdot \bayesalpharisk(h),
    \forall g \in \setslim\label{eqKM1cc},
\end{eqnarray}
where we remind that $\tilde{w}(\leaf)$ is the total normalized weight of
examples reaching leaf $\leaf$ \eqref{defNW}.
Assuming $\mathcal{G} = \setfat\cup \setslim$ and letting $\rho \defeq
|\setfat|/|\setslim|$, we get
\begin{eqnarray}
\frac{\pexpm (g \in \setfat | \leaf)}{\pexpm (g \in \setslim | \leaf)}
  & \geq &
           \rho 
           \cdot \exp\left(\frac{(1-\delta^2)\upgamma^2 \alpha\epsilon
           \tilde{w}(\leaf)
           }{32}\cdot \frac{\bayesalpharisk(h)}{\Delta^*_{\bayesalpharisk}(m)}\right).\label{eqBOUND22}
\end{eqnarray}
We want $\pexpm (g \in \setfat | \leaf) \geq \exp( - \xi)$ for some $\xi>0$. From
\eqref{eqBOUND22}, this
shall be the case if
\begin{eqnarray}
\frac{(1-\delta^2)\upgamma^2 \alpha\epsilon
           \tilde{w}(\leaf)
           }{32}\cdot
  \frac{\bayesalpharisk(h)}{\Delta^*_{\bayesalpharisk}(m)} &
                                                                    \geq
  & \log\left(\frac{1}{\exp(\xi)-1}\right) -\log\rho\nonumber\\
& & = \logsur^{-1}(\xi) - \log \rho,
\end{eqnarray}
where $\logsur$ is the convex surrogate of the $\log$-loss. This can
also be inverted to get all $\xi$s for which this applies using the
fact that $\logsur$ is strictly decreasing,
as
\begin{eqnarray}
\xi & \geq & \logsur\left( \frac{(1-\delta^2)\upgamma^2 \alpha\epsilon
           \tilde{w}(\leaf)
           }{32}\cdot
  \frac{\bayesalpharisk(h)}{\Delta^*_{\bayesalpharisk}(m)} + \log \rho\right).
\end{eqnarray}

$\triangleright$ \textbf{Sequence $\mathcal{L}$ of split}: we now index
quantities $\leaf_\ell$ (replacing notation $\tilde{w}(\leaf_\ell)$ by
$\tilde{w}_\ell$ to follow Theorem \ref{thBoostDT1}), $h_\ell, \rho_\ell,
\alpha_\ell, \xi_\ell, \epsilon_\ell$. In particular, the exponential mechanism
to pick $g \in \mathcal{G}$ to split $\leaf_\ell$ in $h_\ell$
now becomes 
\begin{eqnarray}
\pexpm ((g, \leaf_\ell)) \hspace{-0.3cm} & \propto &  \hspace{-0.3cm}
                    \exp\left(-\frac{\epsilon_\ell w(\mathcal{S}) F(h_\ell\oplus(g, \leaf_\ell))}{2\Delta^*_{\bayesalpharisk}(m)} \right),\label{expSPLIT2}
\end{eqnarray}
We constrain the analysis to indexes $\ell$
in a specific set $\mathcal{L}$ of size $|\mathcal{L}|$. We get that for any
$\xi_\ell$,
\begin{eqnarray}
\xi_\ell \geq \logsur\left( \frac{(1-\delta^2)\upgamma^2 
           }{32}\cdot
  \frac{\alpha_\ell\epsilon_\ell \tilde{w}_\ell\bayesalphariskparam{\ell}(h_\ell)}{\Delta^*_{\bayesalphariskparam{\ell}}(m)}
  + \log \rho_\ell\right) & \Rightarrow & \pexpm (g \in ({\setfat})_\ell | \leaf_\ell) \geq \exp( - \xi_\ell),
\end{eqnarray}
with the simplifying assumption that $\forall \ell, \mathcal{G} =
({\setfat})_\ell \cup ({\setslim})_\ell$. Because of Theorem
\ref{eqBOOST11}, whenever the sequence of $\alpha_\ell$ is
$\upgamma^2/16$-monotonic, letting
\begin{eqnarray}
Q \defeq \frac{(1-\delta^2)\upgamma^2 
           }{32} & , & A_\ell \defeq 
  \frac{\epsilon_\ell
           \tilde{w}_\ell \alpha_\ell\bayesalphariskparam{\ell}(h_\ell)}{\Delta^*_{\bayesalphariskparam{\ell}}(m)},
\end{eqnarray}
if furthermore $\xi_\ell \geq \logsur\left(Q A_\ell + \log \rho_\ell\right), \forall \ell$,
then with probability $\geq \exp(-\sum_\ell \xi_\ell)$, \textit{all} splits
in $\mathcal{L}$
satisfy the $\upgamma$-WLA and therefore the boosting condition in
\eqref{eqBOOST11} is met. In other words, the use of the exponential
mechanism to make splits differentially private does not endanger at all convergence
with high probability. We now have two competing objectives in a
differentially private induction of a top-down decision tree:\\
\noindent (i) we need to pick the $\epsilon_\ell$s so as to match the
total privacy budget allowed for the induction of a single tree, 
\begin{eqnarray}
\frac{\betatree \epsilon}{T} & \defeq & \sum_\ell
\epsilon_\ell,
\end{eqnarray}
(composition theorem).\\
\noindent (ii) we want  to find $\xi_\ell, \ell = 1, 2, ..., L$ such that we have
simultaneously, for some $\xi>0$,
\begin{eqnarray}
\sum_\ell \xi_\ell & \leq & \log \frac{1}{1-\xi},\label{bb1}\\
\xi_\ell & \geq & \logsur\left(Q A_\ell + \log \rho_\ell\right), \forall \ell,\label{bb2}
\end{eqnarray}
because then we can lowerbound the probability that all splits chosen
comply with the WLA:
\begin{eqnarray}
\pexpm \left(\wedge_\ell (g \in ({\setfat})_\ell | \leaf_\ell)\right) & \geq & 1 - \xi,
\end{eqnarray}

Note that, in particular for the first tree induced, $w(\mathcal{S}) =
m/2 = \Omega(m)$ and in all cases, $w(\mathcal{S}) \leq m =  O(m)$, so
suppose $w(\mathcal{S}) = \xi' m$ with $\xi' \in (0,1)$ a constant\footnote{The
  boosting weight update \eqref{defwunMF} prevents zero / unit weights if the
number of boosting iterations $T\ll \infty$.}. We have
\begin{eqnarray}
\bayesalphariskparam{\ell}(h_\ell) & = & \sum_{\leaf_\ell \in \leafset(h_\ell)} w(\leaf_\ell) \cdot
                    \bayesalphariskparam{\ell}\left(
                                   \frac{w^1(\leaf_\ell)}{w(\leaf_\ell)}
                                   \right)\nonumber\\
& = & w(\mathcal{S}) \cdot \sum_{\leaf_\ell \in \leafset(h_\ell)} \frac{w(\leaf_\ell)}{w(\mathcal{S})} \cdot
                    \bayesalphariskparam{\ell}\left(
                                   \frac{w^1(\leaf_\ell)}{w(\leaf_\ell)}
                                   \right)\nonumber\\
& \geq & 2\xi' m \cdot \emprisk(h_\ell).
\end{eqnarray}
Then we can
refine and lowerbound
\begin{eqnarray}
A_\ell & = & \frac{\epsilon_\ell
           \tilde{w}_\ell \alpha_\ell w(\mathcal{S}) \cdot \bayesalphariskparam{\ell}\left(h_\ell
                                   \right)}{3+2\alpha_\ell(\sqrt{m} - 1)}\nonumber\\
& \geq & \epsilon_\ell
           \tilde{w}_\ell \cdot \frac{2 \alpha_\ell m\xi' \emprisk(h_\ell) }{3+2\alpha_\ell(\sqrt{m} - 1)}.\nonumber
\end{eqnarray}
Suppose we fix\footnote{We note that $ \emprisk(h_\ell) \leq 1/2, \forall
  h_\ell$.} 
\begin{eqnarray}
\alpha_\ell  & \defeq & \frac{\emprisk(h_\ell)}{\emprisk(h_1)} \quad(\in
[0,1]),
\end{eqnarray}
 which, since $\emprisk(h_\ell)$ is
non increasing, is therefore $\upgamma^2/16$-monotonic as a
sequence. We get
\begin{eqnarray}
A_\ell & \geq & \xi' \epsilon_\ell
           \tilde{w}_\ell \cdot \frac{4 m \emprisk(h_\ell)^2 }{3
             \emprisk(h_1)+4 \emprisk(h_\ell) (\sqrt{m} - 1)}.\nonumber
\end{eqnarray}
Define for $r\geq 0$
\begin{eqnarray}
t(z) & \defeq & \frac{4z^2}{3r+4z}.
\end{eqnarray}
We can check that if $z \geq (3qr)/(4(1-q))$ for some $q>0$, then $t(z)
\geq q z$. Now,
\begin{eqnarray}
A_\ell & \geq & \xi' \epsilon_\ell
           \tilde{w}_\ell \cdot \frac{4 m \emprisk(h_\ell)^2 }{3+4
             \emprisk(h_\ell) \sqrt{m}}\nonumber\\
& & = \xi' \epsilon_\ell
           \tilde{w}_\ell \cdot \frac{4 z^2 }{3+4
             z}
\end{eqnarray}
for $z \defeq \emprisk(h_\ell) \sqrt{m}$. We get
\begin{eqnarray}
A_\ell & \geq & \xi' \epsilon_\ell
           \tilde{w}_\ell\emprisk(h_\ell)^2 \sqrt{m},
\end{eqnarray}
provided $\emprisk(h_\ell) \sqrt{m} \geq (3 
\emprisk(h_\ell) \xi')/(4(1-\emprisk(h_\ell)))$, which simplifies in
\begin{eqnarray}
m & \geq & \frac{9{\xi'}^2}{16(1-\emprisk(h_\ell))^2},
\end{eqnarray}
and since $\xi'\leq 1, \emprisk(h_\ell)\leq 1/2$, holds whenever
\begin{eqnarray}
m & \geq & \frac{9}{4}.\label{boundM1}
\end{eqnarray}
We then have
\begin{eqnarray}
\logsur\left(Q A_\ell + \log \rho_\ell\right) & \leq &
                                                 \logsur\left(\frac{(1-\delta^2)\upgamma^2 \xi' \epsilon_\ell
           \tilde{w}_\ell \emprisk(h_\ell)^2 
           }{32} \cdot \sqrt{m} + \log \rho_\ell\right).
\end{eqnarray}
Suppose 
\begin{eqnarray}
m & \geq & 3,
\end{eqnarray}
which implies \eqref{boundM1}. Fix now
\begin{eqnarray}
\epsilon_\ell & = & \frac{\betatree}{Td2^{\depth(\leaf_\ell)}}\cdot
                 \epsilon,\label{vllEPSILONt}\\
Td & \leq & \log m.\label{boundSIZE}
\end{eqnarray}
We recall that $d$ is the maximum depth of a tree and $T$ is the
number of trees in the boosted combination. $Td$ is therefore a proxy for the maximal number of tests in trees to
classify an observation.
\begin{eqnarray}
\logsur\left(Q A_\ell + \log \rho_\ell\right) & \leq &
                                                 \logsur\left(\frac{\betatree (1-\delta^2)\upgamma^2{\xi'} \epsilon 
           }{32Td} \cdot \frac{\tilde{w}_\ell
                                                 \emprisk(h_\ell)^2\sqrt{m}}{2^{\depth(\leaf_\ell)}} + \log \rho_\ell\right)\nonumber\\
& \leq & \logsur\left( \frac{\betatree (1-\delta^2)\upgamma^2{\xi'} }{256}
           \cdot J(\leaf_\ell, h) \cdot \frac{
         \epsilon \sqrt{m}}{\log m} + \log \rho_\ell\right),
\end{eqnarray}
with
\begin{eqnarray}
J(\leaf_\ell, h) & \defeq & \frac{8\tilde{w}_\ell
                                                 \emprisk(h_\ell)^2}{2^{\depth(\leaf_\ell)}}
                      \quad \in [0,1].
\end{eqnarray}
Suppose now that 
\begin{eqnarray}
\log \rho_\ell & \geq & - \frac{\betatree (1-\delta^2)\upgamma^2{\xi'} }{256}
           \cdot J(\leaf_\ell, h) \cdot \frac{
         \epsilon \sqrt{m}}{\log m},
\end{eqnarray}
which is equivalent to
\begin{eqnarray}
\frac{|\setfat|}{|\setslim|} & \geq & \exp\left(- \frac{\betatree (1-\delta^2)\upgamma^2{\xi'} }{256}
           \cdot J(\leaf_\ell, h) \cdot \frac{
         \epsilon \sqrt{m}}{\log m}\right),
\end{eqnarray}
or
\begin{eqnarray}
|\setfat| & \geq & \frac{|\mathcal{G}|}{1 + \exp\left(\frac{\betatree (1-\delta^2)\upgamma^2{\xi'} }{256}
           \cdot J(\leaf_\ell, h) \cdot \frac{
         \epsilon \sqrt{m}}{\log m}\right)}
\end{eqnarray}
and thus $\setfat$ cannot be vanishing (or at least too fast as a
function of $m$) with respect to $\mathcal{G}$. This implies 
\begin{eqnarray}
\logsur\left(Q A_\ell + \log \rho_\ell\right) & \leq &\logsur\left( Q'
           \cdot J(\leaf_\ell, h) \cdot \frac{
         \epsilon \sqrt{m}}{\log m} \right),
\end{eqnarray}
with
\begin{eqnarray}
Q' & \defeq & \frac{\betatree{\xi'} (1-\delta^2)\upgamma^2}{256}
              \quad \in (0, 1/256].\nonumber
\end{eqnarray}
Notice that $Q' = \theta(1)$, \textit{i.e.} it is a constant. The concavity of $\log$ yields 
\begin{eqnarray}
\sum_{\ell \in  \mathcal{L}} \logsur\left( Q'
           \cdot J(\leaf_\ell, h) \cdot \frac{
         \epsilon \sqrt{m}}{\log m} \right) & \leq & |\mathcal{L}|
                                                     \log\left(1 +
                                                     \expect_{\mathcal{L}}
                                                     \exp\left(-Q'
           \cdot J(\leaf_\ell, h) \cdot \frac{
         \epsilon \sqrt{m}}{\log m} \right)\right),
\end{eqnarray}
and so if we pick
\begin{eqnarray}
\xi_\ell & \defeq & \logsur\left( Q'
           \cdot J(\leaf_\ell, h) \cdot \frac{
         \epsilon \sqrt{m}}{\log m} \right),
\end{eqnarray}
then a sufficient condition to have \eqref{bb1} is 
\begin{eqnarray}
\expect_{\mathcal{L}}
                                                     \exp\left(-Q'
           \cdot J(\leaf_\ell, h) \cdot \frac{
         \epsilon \sqrt{m}}{\log m} \right) & \leq &
                                                     \left(\frac{1}{1-\xi}\right)^{\frac{1}{|\mathcal{L}|}}
                                                     -1.\label{cond1a1}
\end{eqnarray}
We also have $\forall \xi \in [0,1], |\mathcal{L}|\geq 1$,
\begin{eqnarray}
\left(\frac{1}{1-\xi}\right)^{\frac{1}{|\mathcal{L}|}}
                                                     -1 & \geq & \frac{\xi}{|\mathcal{L}|},
\end{eqnarray}
so to get \eqref{cond1a1} it is sufficient that
\begin{eqnarray}
\expect_{\mathcal{L}}
                                                     \exp\left(-Q'
           \cdot J(\leaf_\ell, h) \cdot \frac{
         \epsilon \sqrt{m}}{\log m} \right) & \leq & \frac{\xi}{|\mathcal{L}|},
\end{eqnarray}
which is ensured if 
\begin{eqnarray}
\min_{\ell \in \mathcal{L}} J(\leaf_\ell, h) & \geq & \frac{1}{Q'}\cdot \frac{\log m}{
         \epsilon \sqrt{m}} \log \frac{|\mathcal{L}|}{\xi}\nonumber\\
& = & \Omega \left(\frac{\log m}{
         \epsilon \sqrt{m}} \log \frac{|\mathcal{L}|}{\xi}\right) \label{condJJ}.
\end{eqnarray}
This ends the proof of Theorem \ref{thBOOSTDP1}.\\

\noindent \textbf{Remark}: Notice that $|\mathcal{L}| \leq 2^{d+1}-1$, so we get that condition
\eqref{condJJ} is satisfied if for example
\begin{eqnarray}
\min_{\ell \in \mathcal{L}} J(\leaf_\ell, h) & = & \Omega \left(\frac{\log m}{
         \epsilon \sqrt{m}} \cdot \left(d + \log \frac{1}{\xi}\right)\right) \label{condJJ2}.
\end{eqnarray}
As long as for example
\begin{eqnarray}
\frac{\log m}{\sqrt{m}} & = & o(\epsilon),\\
d, \log \frac{1}{\xi} & = & o\left(\frac {\sqrt{m}}{\log m}\right),
\end{eqnarray}
then the constraint on $\min_{\ell \in \mathcal{L}} J(\leaf_\ell, h)$ in
\eqref{condJJ2} will vanish.

%% file: content-arxiv/appendix-suppexp.tex
\section{Appendix on Experiments}\label{exp_expes}

\subsection{General setting}\label{sub-sec-gen}

\noindent $\triangleright$ \textbf{Public information} is as
follows. First, the attribute domain is public, which is standard in
the field \citep{fiDT}. Several authors have tried to compute the
threshold information for continuous attributes in a private way
\citep{fiDT,fsDM}. This is not necessarily a good approach: it requires
privacy budget, it can require weakening privacy and does not
necessarily buys improvements \citep[Section 3.2.2]{fiDT}. Since the
attribute domain is public, there is a simple alternative that does
not suffer most of these workarounds: the regular quantisation of the
domain using a public number of values. This particularly makes sense
\textit{e.g.} for many commonly used attribute classes like age, percentages, $\$$-value,
mileages, distances, or for any attribute for which the key segments
are known from the specialists, such as in life sciences or medical
domain. This also has three technical justifications: (1) a private
approaches requires budget, (2) $\nvpriv$ allows to tightly control
the computational complexity of the whole DT induction, and most importantly (3) boosting
does not require exhaustive split search. It indeed just assumes the
WLA, which essentially requires $\nvpriv$ not too
small, even more if the tree is not too deep \\

\noindent $\triangleright$ \textbf{Parameters} for \alphaboost. We ran out approach, both private and not private, for all combinations
of $T \in \{2, 5, 10, 20\}$, $\alpha \in \{0.1,
1.0, \mbox{O.C}\}$ (O.C = Objective Calibration), $d\in \{1, 2, 3, 4,
5, 6\}$. Finally, we have tried a quantisation in $\nvpriv\in
\{10, 50\}$ values, for all numeric attributes (Section \ref{sec-solv}
in the main file). In order not to give a potential advantage to
noise-free boosting in its tests that would not come from the absence
of noise, we also use this regular
quantisation for the noise-free boosting tests of our approach.\\

For the private version, in
addition to all these combinations, we considered $\epsilon \in \{0.01, 0.1,
1.0, 10.0, 25.0\}$ and 
$\betatree \in \{0.1, 0.5, 0.9\}$. For the private trees, after having
noisified the leaf predictions, we clamp the output
values of the private trees to a maximal $M \in \{1,
10, 100\}$, which is another parameter. In the private setting, once
the depth is fixed, all tree induced have each of their leaves at the
same depth: this means that we even split leaves that are pure if they
are below the required depth, to prevent using DP budget to test for
purity (which we do when there is no DP, as we do not split pure
leaves in this case).\\

Altogether, this represents more than 1.3 million (ensemble) models
learned using our approach. Obviously, increasing $\nvpriv$ tends to
improve accuracy but significantly increases time complexity for
\alphaboost, in particular to split the nodes, a task carried out
repeatedly for both the non private but also for the exponential
mechanism in the differential
privacy case, adding an further computational burden in this case. Because of the size of the experiments, we report here
the results obtained for $M=10, \nvpriv = 10$, which seems to lead to
a good compromise between accuracy and execution time.

\subsection{Implementation}\label{sub-sec-sum-imp} 

We give here a few details on the implementation.\\

\noindent $\triangleright$ \textbf{Boosting}: For boosting
algorithms, we clamp the value $q(\ell) \in [\zeta, 1 -
\zeta]$ with $\zeta = 10^{-4}$ to prevent infinite predictions and
NaNs via the link function. Then the value is noisifed if DP, and if
DP, after that, the maximal value is clamped to a maximum value,
$M$. Since in theory weights cannot be 0 or 1 when $\alpha \neq 0$ but
numerical precision errors can result in 0 or 1 weights in exceptional
cases, we replace such weights by a corresponding value in $\{\zeta', 1 -
\zeta'\}$. \\

\noindent $\triangleright$ \textbf{Random forests}: A random decision
forest is an ensemble of random decision trees \citep{Fan03israndom}. A random decision tree is constructed by choosing the split features purely at random.
\citet{FLETCHER201716} showed that this independence of the training data can be favourable for learning differentially private classifier, as the construction of the tree does not incur any privacy costs.
 
We implemented random decision forest based on the ideas from those
papers. However, instead of smooth sensitivity, we use global
sensitivity, not just to rely on the exact same definition of
sensitivity: our code was written with federated learning in mind,
and, as smooth sensitivity is data dependent, it is an open problem if
you can cooperatively compute smooth sensitivity over distributed
datasets without leaking information. Since privacy is spent at the
leaves' predictions, we have implemented two mechanisms to make those
private: the exponential mechanism using the class counts, and the
Laplace mechanism, still on the class counts, splitting evenly the
privacy budget among the leaves prior to applying each mechanism. We
refer to the two random forest approaches as \rfexp~and \rflap,
respectively for the exponential and Laplace mechanisms.
 
\subsection{Additional experimental results}\label{sub-sec-exp-res}

\subsection*{Domain summary Table}\label{sub-sec-sum}

\begin{table}[h]
\begin{center}
\begin{tabular}{|crr|}
\hline \hline
Domain & \multicolumn{1}{c}{$m$} & \multicolumn{1}{c|}{$n$}  \\ \hline 
Transfusion & 748 & 4  \\
Banknote & 1 372 & 4  \\
Breast wisc & 699 & 9 \\
Ionosphere & 351 & 33  \\
Sonar & 208 & 60  \\
Yeast & 1 484 & 7\\
Wine-red & 1 599 & 11 \\
Cardiotocography (*)  & 2 126 & 9  \\
CreditCardSmall (**) & 1 000 & 23\\
Abalone  & 4 177 & 8  \\
Qsar & 1 055 & 41\\
Wine-white & 4 898 & 11 \\
Page & 5 473 & 10  \\
Mice & 1 080 & 77\\
Hill+noise & 1 212 & 100\\
Hill+nonoise & 1 212 & 100\\
Firmteacher & 10 800 & 16\\
Magic & 19 020 & 10  \\
EEG & 14 980 & 14\\ \hline\hline      
\end{tabular}
\end{center}
\caption{UCI domains considered in our experiments ($m=$ total number
  of examples, $n=$ number of features), ordered in
  increasing $m \times n$. (*) we used features 13-21 as descriptors; (**) we used the first 1 000 examples of the
  UCI domain.}
  \label{t-s-uci}
\end{table}

\clearpage
\newpage
\section*{Results for $\nvpriv=10,  M = 10$}

Due to the excessive number of files/plots, results on a subset of the
domains are shown here. Contact the authors for a more comprehensive
non-ArXiv version of the paper.

\subsection*{$\triangleright$ UCI \texttt{transfusion}}\label{sub-sec-transfusion}

\begin{figure}[h]
\begin{center}
\begin{tabular}{cc||cc}\hline \hline
\multicolumn{2}{c||}{without DP} & \multicolumn{2}{c}{with DP}\\ \hline
 \pushgraphicsLeafDepth{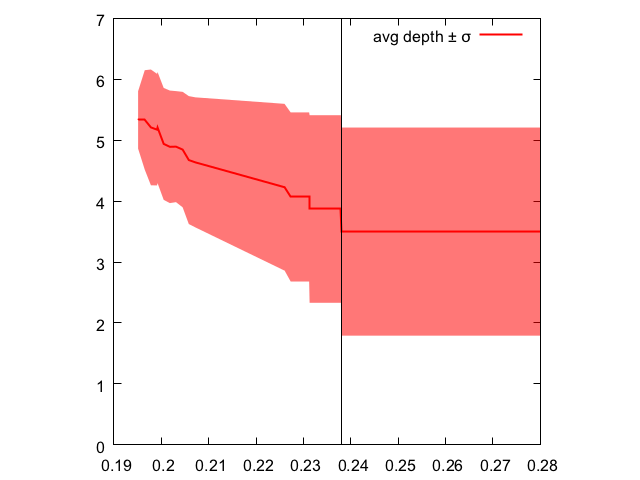}
  \hspacegss & \hspacegss \pushgraphicsLeafDepth{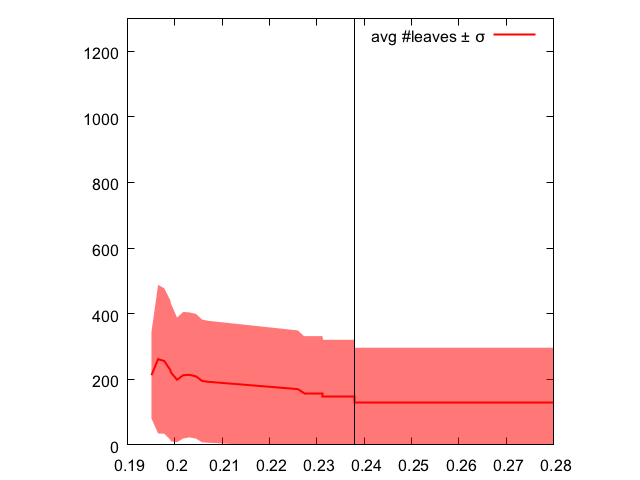}  &  \pushgraphicsLeafDepth{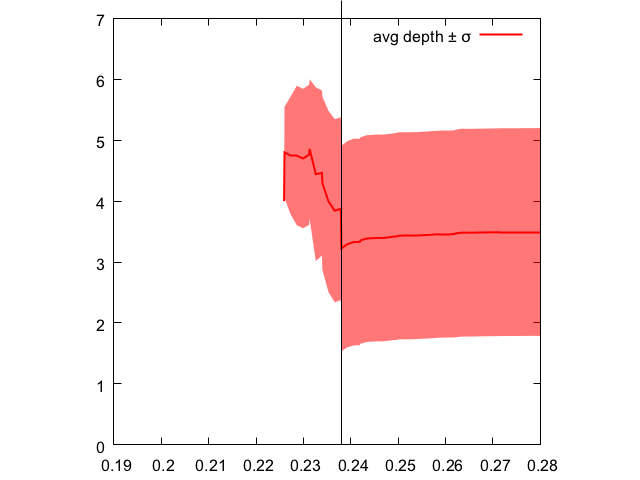}  \hspacegss & \hspacegss  \pushgraphicsLeafDepth{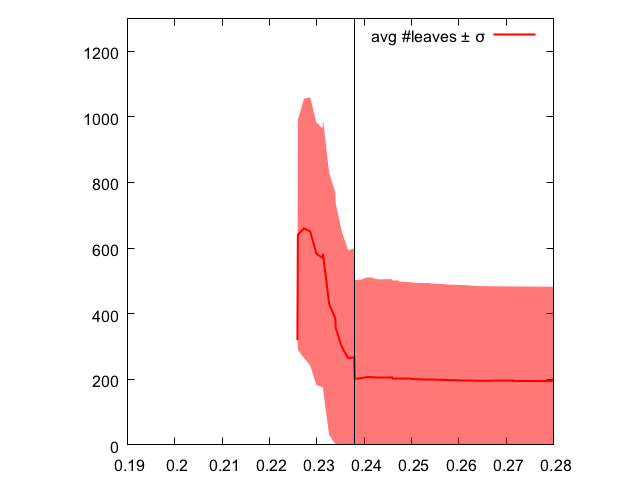}  \\ \hline
 depth \hspacegss & \hspacegss $\#$leaves & depth \hspacegss & \hspacegss $\#$leaves \\ \hline \hline
\end{tabular}
\end{center}
\caption{UCI domain \texttt{transfusion}: $x$ = test error values, $y$
  = cumulated expected depth (left plots) or number of leaves (right
  plots) for the models having test error $\leq x$, aggregated over
  all runs ($\pm$ standard deviation) -- the vertical
  black bar depicts the test error of the default class. \textit{Left
    panel}: w/o DP; \textit{Left
    panel}: with DP; values are
  aggregated over all varying parameters (left: $\alpha$; right:
  $\alpha$, $\varepsilon$, [ $\betatree|\betapred$ ]). }
  \label{f-s-transfusion}
\end{figure}

\begin{figure}[h]
\begin{center}
\begin{tabular}{cc||cc}\hline \hline
\multicolumn{2}{c||}{performances wrt $\alpha$s} & \multicolumn{2}{c}{performances wrt $\epsilon$s (with DP)}\\ \hline
w/o DP \hspacegss & \hspacegss with DP & full
  & crop \\ \hline
\includegraphics[trim=90bp 5bp 75bp
15bp,clip,width=\sgg\linewidth]{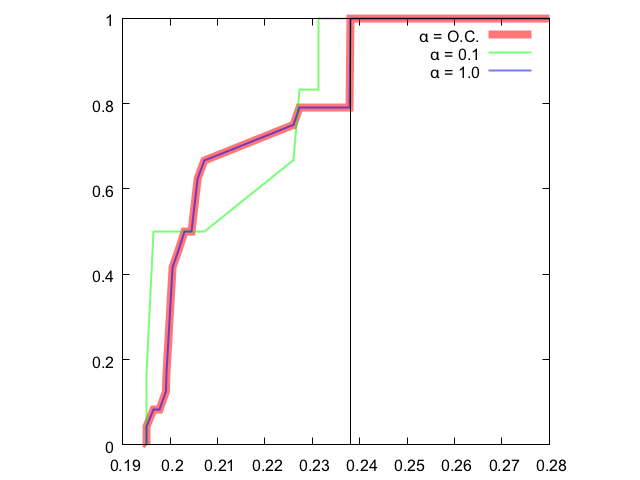}
  \hspacegss & \hspacegss \includegraphics[trim=90bp 5bp 75bp
15bp,clip,width=\sgg\linewidth]{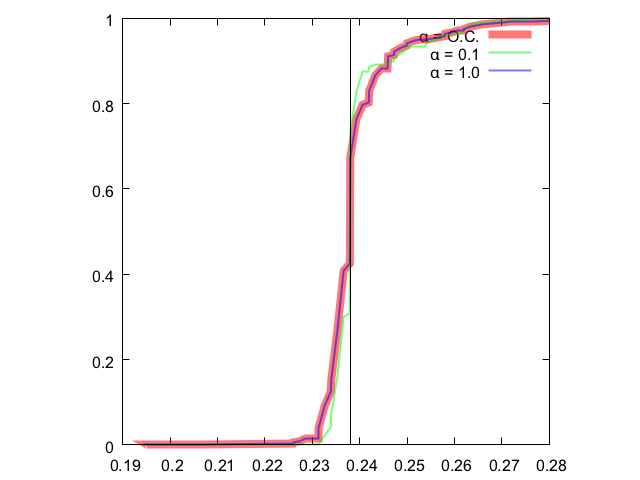} & \includegraphics[trim=90bp 5bp 75bp
15bp,clip,width=\sgg\linewidth]{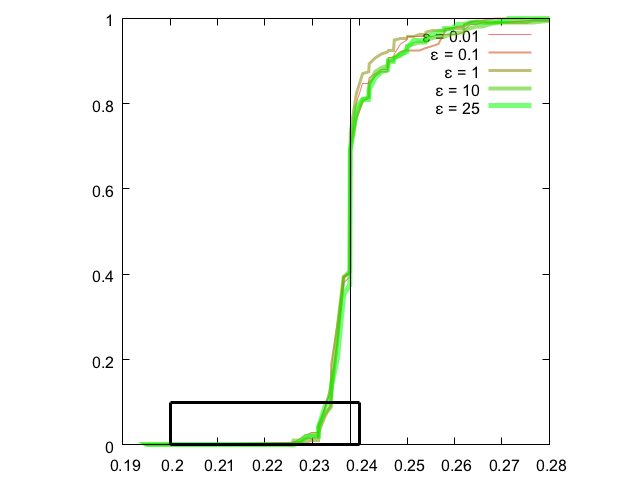} \hspacegss & \hspacegss \includegraphics[trim=90bp 5bp 75bp
15bp,clip,width=\sgg\linewidth]{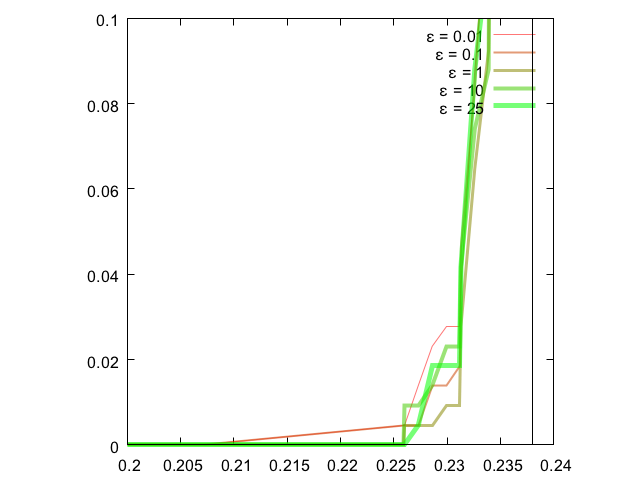} \\ \hline \hline
\end{tabular}
\end{center}
\caption{UCI domain \texttt{transfusion}: $x$ = test error values and
  $y$ = aggregated percentage of runs having error no
  less than $x$ -- the vertical
  black bar depicts the test error of the default class;
  \textit{Left pane}: performances as a function of $\alpha$ (O.C =
  objective calibration), without (left plot) or with DP (right plot);
\textit{Right pane}: performances as a function of $\varepsilon$,
either displaying the full plot (left plot) or a crop over the best
results (right plot). The crop panel is indicated in the left plot.}
  \label{f-s-transfusion2}
\end{figure}

\begin{figure}[h]
\begin{center}
\begin{tabular}{ccccc}\hline \hline
$\varepsilon=0.01$ \hspacegs & \hspacegs $\varepsilon=0.1$ \hspacegs & \hspacegs $\varepsilon=1$ \hspacegs & \hspacegs
                                                        $\varepsilon=10$ \hspacegs &
                                                                    \hspacegs $\varepsilon=25$ \hspacegs \\ \hline 
\includegraphics[trim=90bp 5bp 75bp
15bp,clip,width=\sg\linewidth]{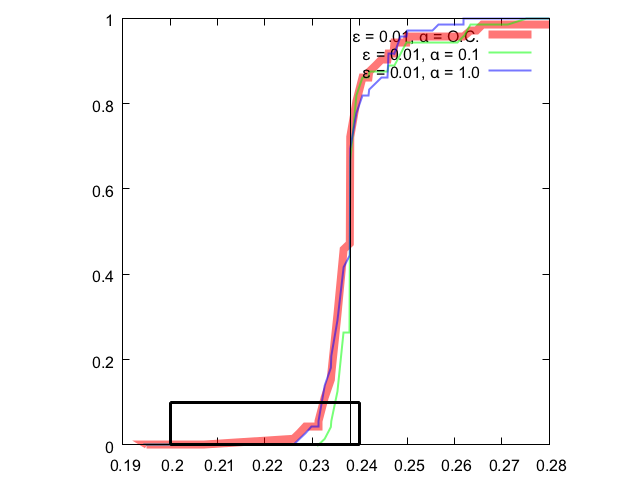}
  \hspacegs & \hspacegs \includegraphics[trim=90bp 5bp 75bp
15bp,clip,width=\sg\linewidth]{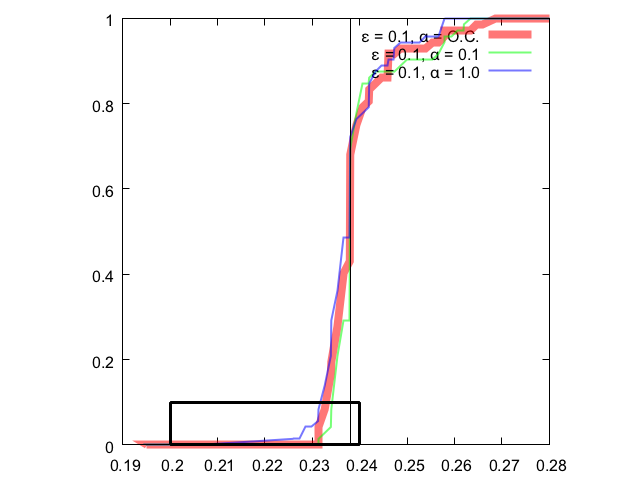}  \hspacegs & \hspacegs \includegraphics[trim=90bp 5bp 75bp
15bp,clip,width=\sg\linewidth]{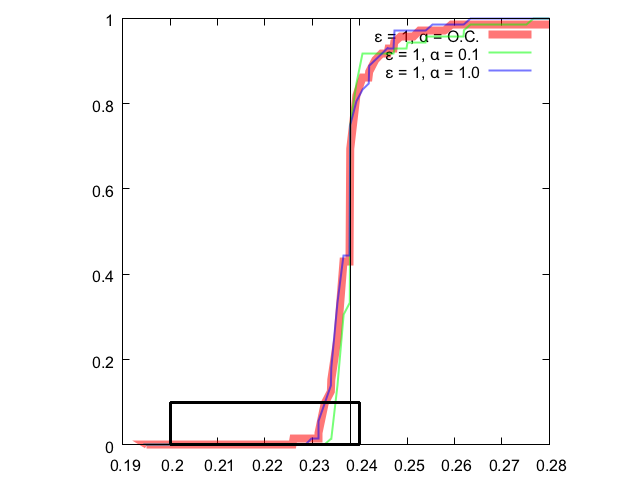}  \hspacegs & \hspacegs \includegraphics[trim=90bp 5bp 75bp
15bp,clip,width=\sg\linewidth]{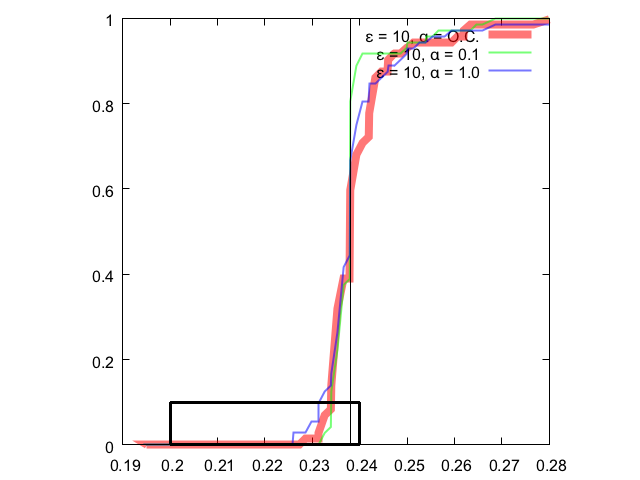}  \hspacegs & \hspacegs \includegraphics[trim=90bp 5bp 75bp
15bp,clip,width=\sg\linewidth]{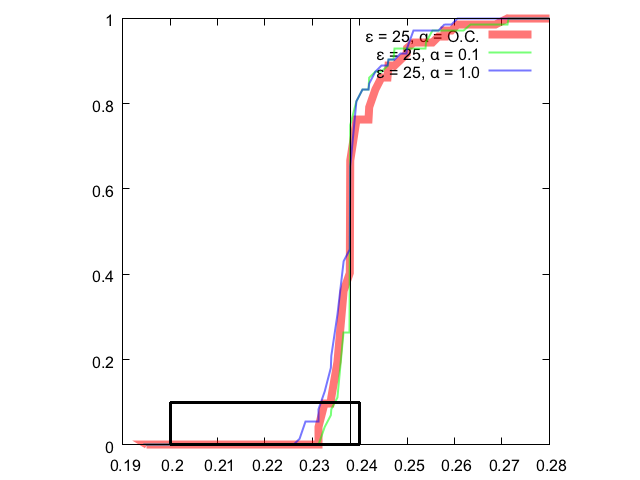}
  \\ 
\includegraphics[trim=90bp 5bp 75bp
15bp,clip,width=\sg\linewidth]{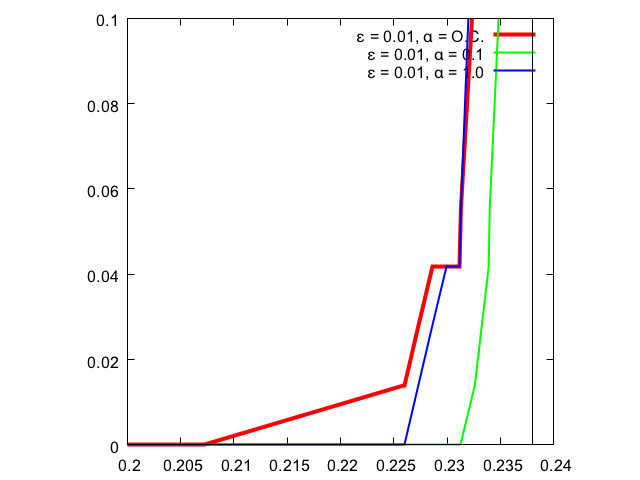}
  \hspacegs & \hspacegs \includegraphics[trim=90bp 5bp 75bp
15bp,clip,width=\sg\linewidth]{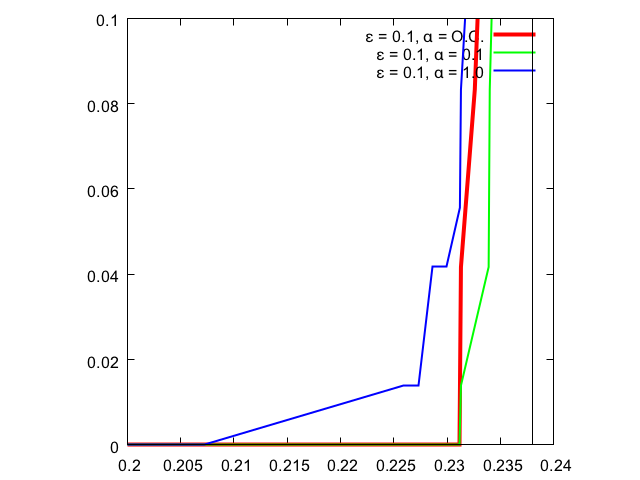}  \hspacegs & \hspacegs \includegraphics[trim=90bp 5bp 75bp
15bp,clip,width=\sg\linewidth]{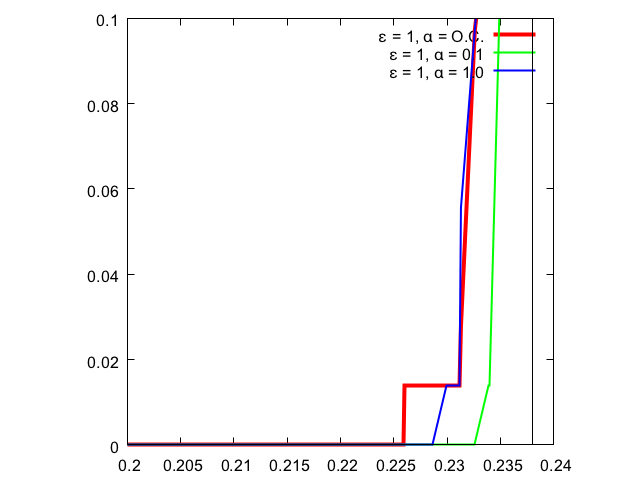}  \hspacegs & \hspacegs \includegraphics[trim=90bp 5bp 75bp
15bp,clip,width=\sg\linewidth]{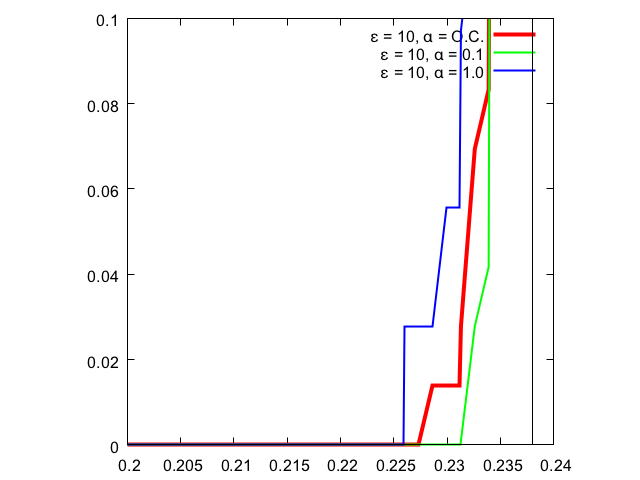}  \hspacegs & \hspacegs \includegraphics[trim=90bp 5bp 75bp
15bp,clip,width=\sg\linewidth]{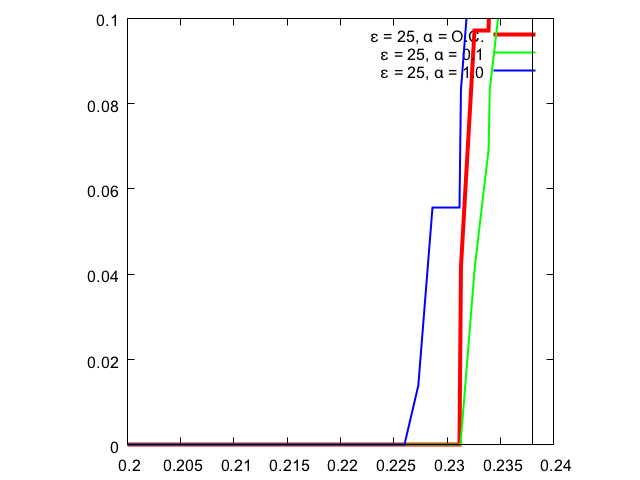}
  \\ \hline \hline
\end{tabular}
\end{center}
\caption{UCI domain \texttt{transfusion}: $x$ = test error values and
  $y$ = aggregated percentage of runs having error no
  less than $x$ -- the vertical
  black bar depicts the test error of the default class; \textit{top row}: performances as a function of
  $\alpha$ showing the full plot for each value of $\varepsilon$;
  \textit{bottom row}: crop of the best results from the top row (the crop panel is indicated in the left plot).}
  \label{f-s-transfusion3}
\end{figure}

\clearpage
\newpage
\subsection*{$\triangleright$ UCI \texttt{banknote}}\label{sub-sec-banknote}

\begin{figure}[h]
\begin{center}
\begin{tabular}{cc||cc}\hline \hline
\multicolumn{2}{c||}{without DP} & \multicolumn{2}{c}{with DP}\\
  \hline
\pushgraphicsLeafDepth{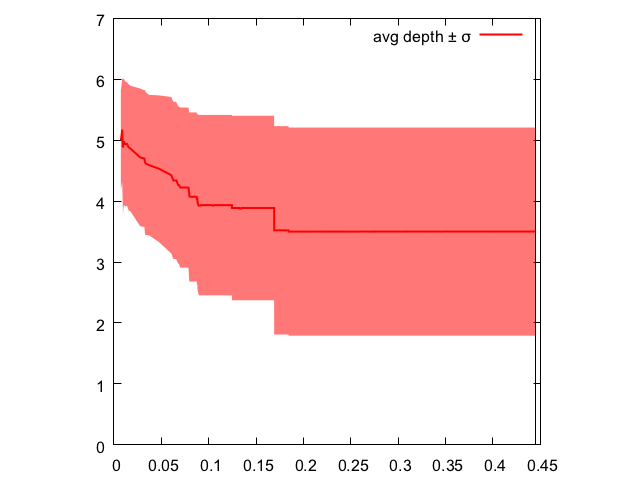}
  \hspacegss & \hspacegss \pushgraphicsLeafDepth{SubFigs/NCT_10/banknote/PostProcess10/plot_binBanknote_avg_dpfree_leaves.png}  &  \pushgraphicsLeafDepth{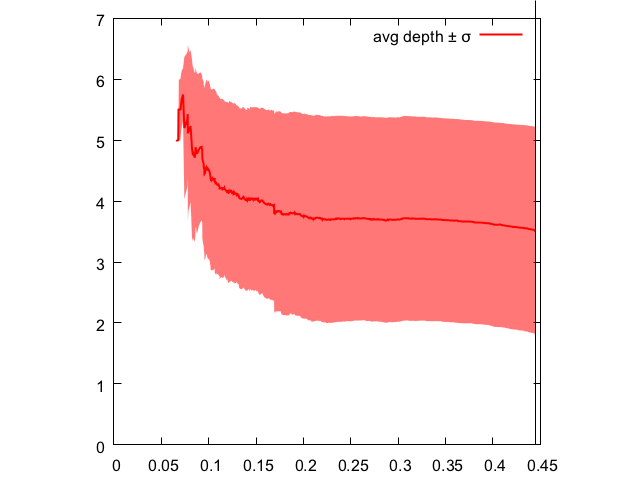}  \hspacegss & \hspacegss  \pushgraphicsLeafDepth{SubFigs/NCT_10/banknote/PostProcess10/plot_binBanknote_avg_dp_leaves.png}  \\ \hline
 depth \hspacegss & \hspacegss $\#$leaves & depth \hspacegss & \hspacegss $\#$leaves \\ \hline \hline
\end{tabular}
\end{center}
\caption{UCI domain \texttt{banknote}, conventions identical as in
  Figure \ref{f-s-transfusion}. }
  \label{f-s-banknote}
\end{figure}

\begin{figure}[h]
\begin{center}
\begin{tabular}{cc||cc}\hline \hline
\multicolumn{2}{c||}{performances wrt $\alpha$s} & \multicolumn{2}{c}{performances wrt $\epsilon$s (with DP)}\\ \hline
w/o DP \hspacegss & \hspacegss with DP & full
  & crop \\ \hline
\includegraphics[trim=90bp 5bp 75bp
15bp,clip,width=\sgg\linewidth]{SubFigs/NCT_10/banknote/PostProcess10/plot_binBanknote_cumulated_dpfree_alphas_ytotal.png}
  \hspacegss & \hspacegss \includegraphics[trim=90bp 5bp 75bp
15bp,clip,width=\sgg\linewidth]{SubFigs/NCT_10/banknote/PostProcess10/plot_binBanknote_cumulated_dp_alphas_ytotal.png} & \includegraphics[trim=90bp 5bp 75bp
15bp,clip,width=\sgg\linewidth]{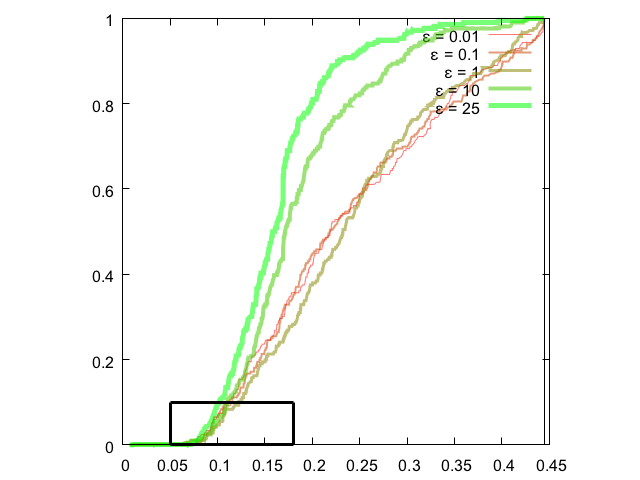} \hspacegss & \hspacegss \includegraphics[trim=90bp 5bp 75bp
15bp,clip,width=\sgg\linewidth]{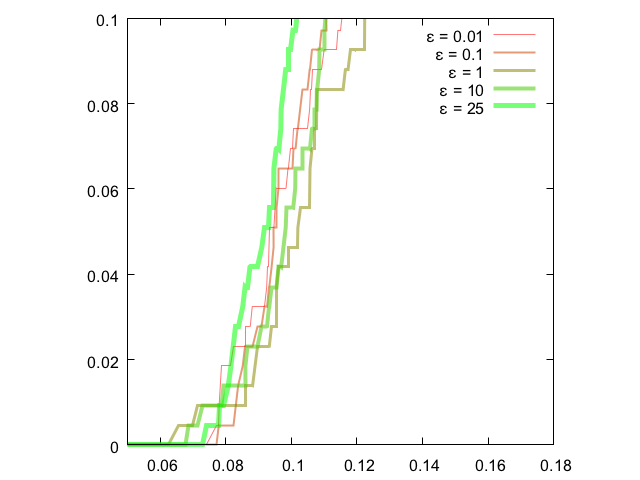} \\ \hline \hline
\end{tabular}
\end{center}
\caption{UCI domain \texttt{banknote}, conventions identical as in
  Figure \ref{f-s-transfusion2}.}
  \label{f-s-banknote2}
\end{figure}

\begin{figure}[h]
\begin{center}
\begin{tabular}{ccccc}\hline \hline
$\varepsilon=0.01$ \hspacegs & \hspacegs $\varepsilon=0.1$ \hspacegs & \hspacegs $\varepsilon=1$ \hspacegs & \hspacegs
                                                        $\varepsilon=10$ \hspacegs &
                                                                    \hspacegs $\varepsilon=25$ \hspacegs \\ \hline 
\includegraphics[trim=90bp 5bp 75bp
15bp,clip,width=\sg\linewidth]{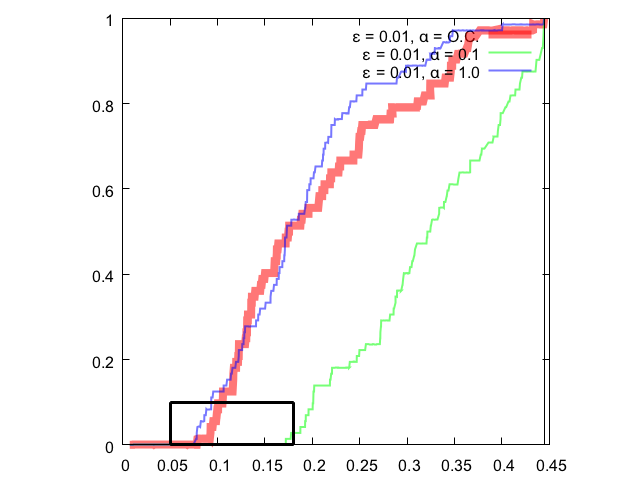}
  \hspacegs & \hspacegs \includegraphics[trim=90bp 5bp 75bp
15bp,clip,width=\sg\linewidth]{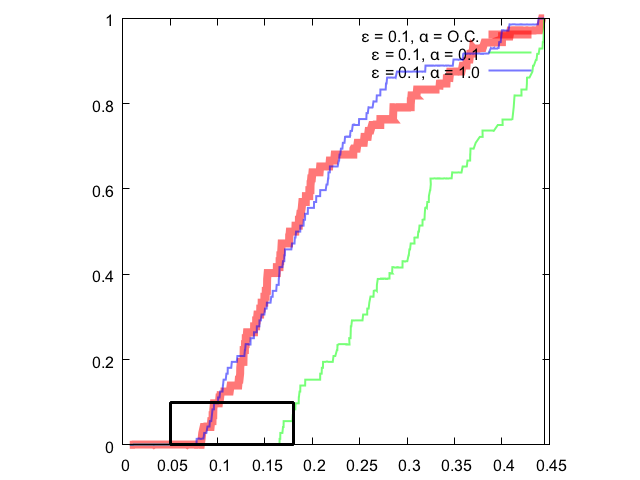}  \hspacegs & \hspacegs \includegraphics[trim=90bp 5bp 75bp
15bp,clip,width=\sg\linewidth]{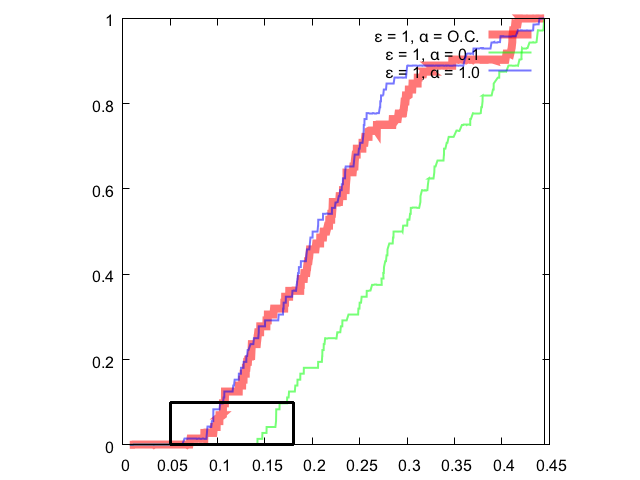}  \hspacegs & \hspacegs \includegraphics[trim=90bp 5bp 75bp
15bp,clip,width=\sg\linewidth]{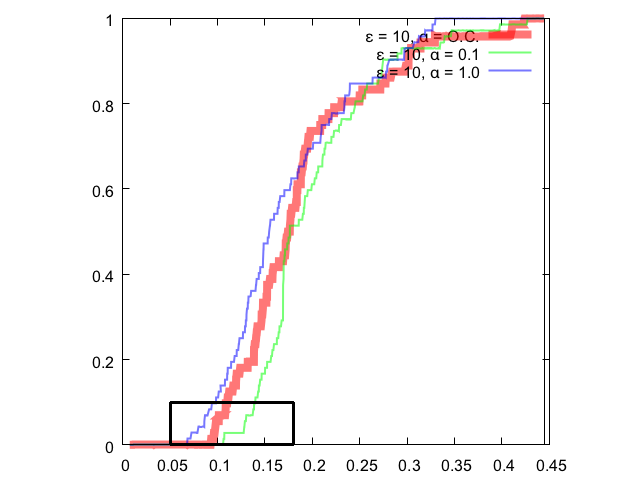}  \hspacegs & \hspacegs \includegraphics[trim=90bp 5bp 75bp
15bp,clip,width=\sg\linewidth]{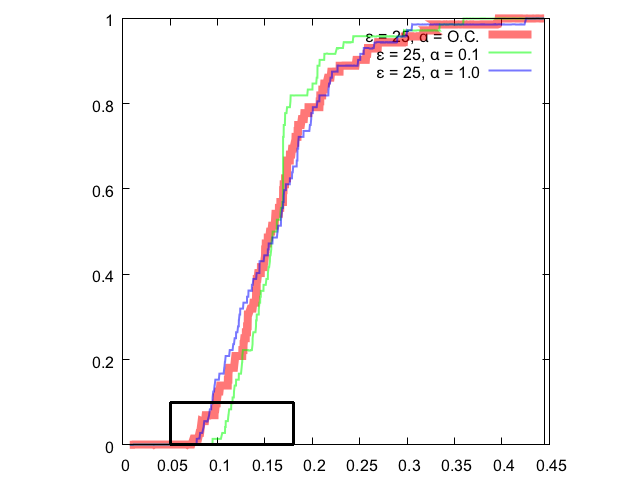}
  \\
\hline \hline
\end{tabular}
\end{center}
\caption{UCI domain \texttt{banknote}, conventions identical as in
  Figure \ref{f-s-transfusion3}.}
  \label{f-s-banknote3}
\end{figure}

\clearpage
\newpage
\subsection*{$\triangleright$ UCI \texttt{breastwisc}}\label{sub-sec-breastwisc}

\begin{figure}[h]
\begin{center}
\begin{tabular}{cc||cc}\hline \hline
\multicolumn{2}{c||}{without DP} & \multicolumn{2}{c}{with DP}\\
  \hline
\pushgraphicsLeafDepth{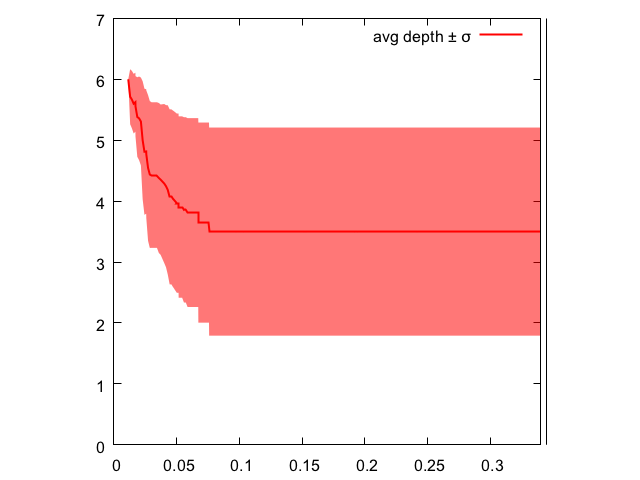}
  \hspacegss & \hspacegss \pushgraphicsLeafDepth{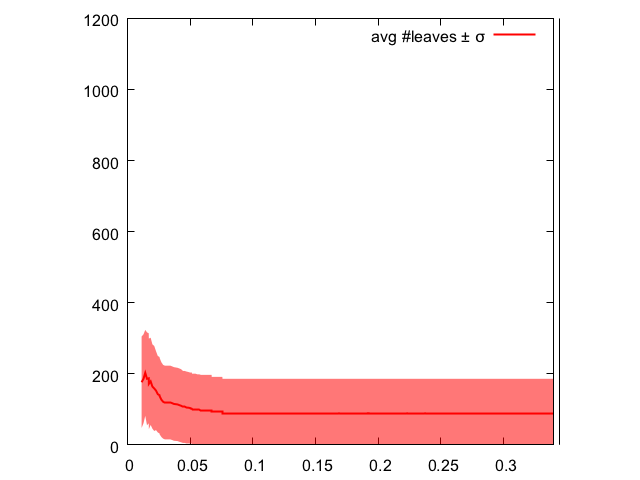}  &  \pushgraphicsLeafDepth{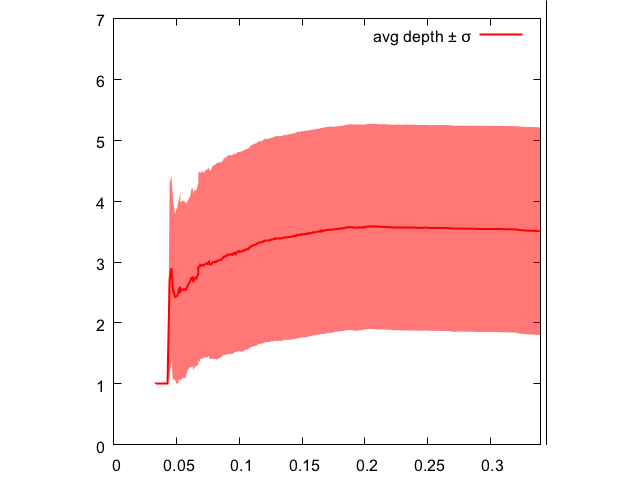}  \hspacegss & \hspacegss  \pushgraphicsLeafDepth{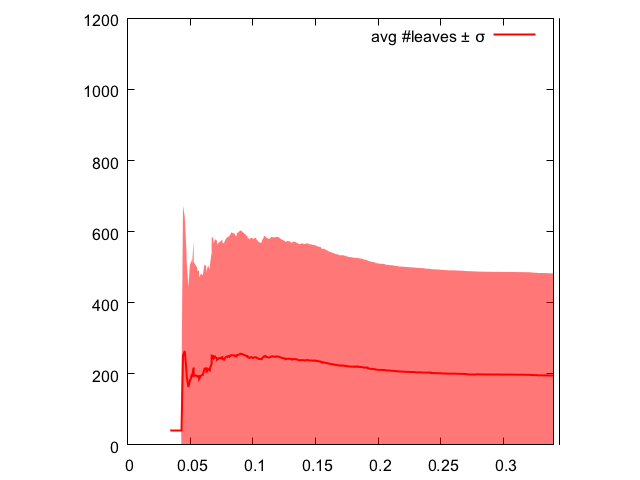}  \\ \hline
 depth \hspacegss & \hspacegss $\#$leaves & depth \hspacegss & \hspacegss $\#$leaves \\ \hline \hline
\end{tabular}
\end{center}
\caption{UCI domain \texttt{breastwisc}, conventions identical as in
  Figure \ref{f-s-transfusion}. }
  \label{f-s-breastwisc}
\end{figure}

\begin{figure}[h]
\begin{center}
\begin{tabular}{cc||cc}\hline \hline
\multicolumn{2}{c||}{performances wrt $\alpha$s} & \multicolumn{2}{c}{performances wrt $\epsilon$s (with DP)}\\ \hline
w/o DP \hspacegss & \hspacegss with DP & full
  & crop \\ \hline
\includegraphics[trim=90bp 5bp 75bp
15bp,clip,width=\sgg\linewidth]{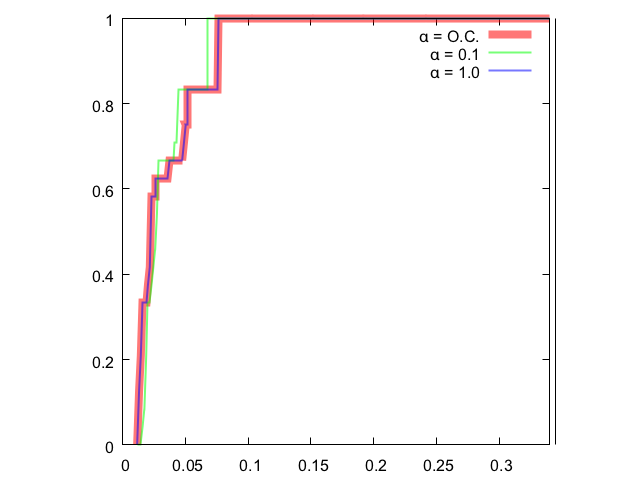}
  \hspacegss & \hspacegss \includegraphics[trim=90bp 5bp 75bp
15bp,clip,width=\sgg\linewidth]{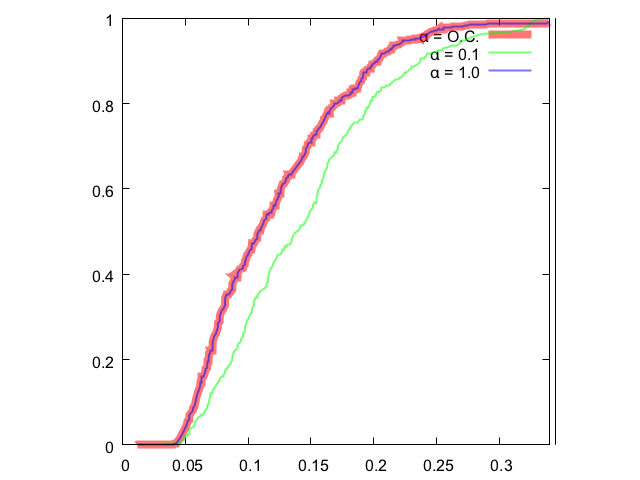} & \includegraphics[trim=90bp 5bp 75bp
15bp,clip,width=\sgg\linewidth]{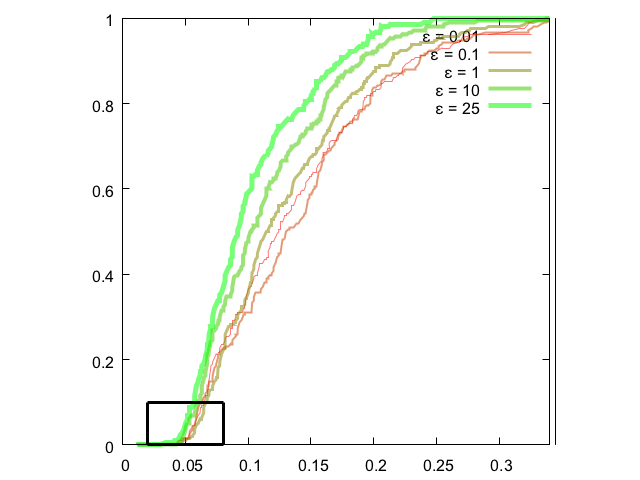} \hspacegss & \hspacegss \includegraphics[trim=90bp 5bp 75bp
15bp,clip,width=\sgg\linewidth]{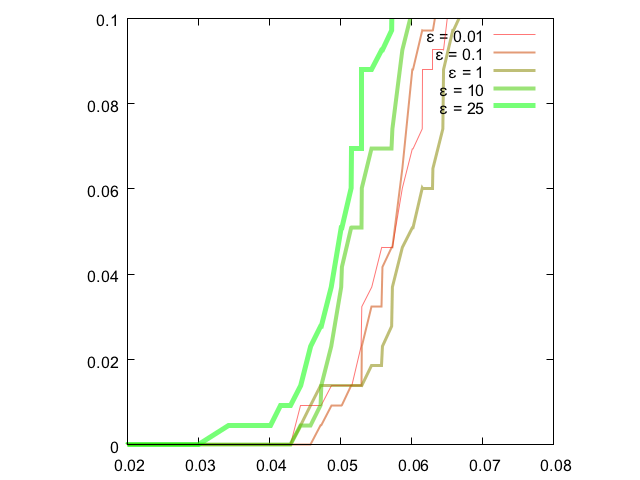} \\ \hline \hline
\end{tabular}
\end{center}
\caption{UCI domain \texttt{breastwisc}, conventions identical as in
  Figure \ref{f-s-transfusion2}.}
  \label{f-s-breastwisc2}
\end{figure}

\begin{figure}[h]
\begin{center}
\begin{tabular}{ccccc}\hline \hline
$\varepsilon=0.01$ \hspacegs & \hspacegs $\varepsilon=0.1$ \hspacegs & \hspacegs $\varepsilon=1$ \hspacegs & \hspacegs
                                                        $\varepsilon=10$ \hspacegs &
                                                                    \hspacegs $\varepsilon=25$ \hspacegs \\ \hline 
\includegraphics[trim=90bp 5bp 75bp
15bp,clip,width=\sg\linewidth]{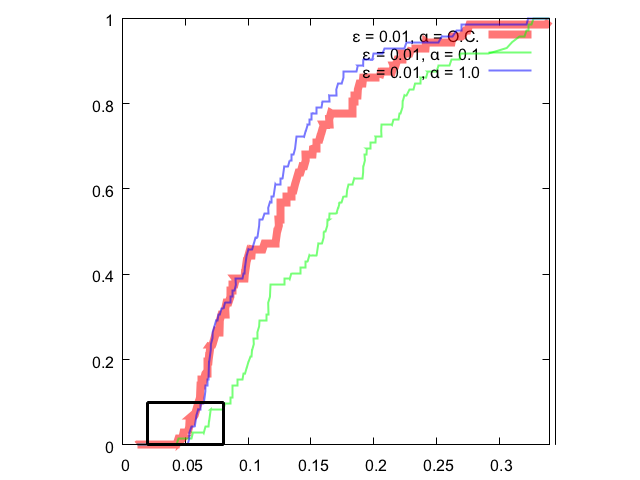}
  \hspacegs & \hspacegs \includegraphics[trim=90bp 5bp 75bp
15bp,clip,width=\sg\linewidth]{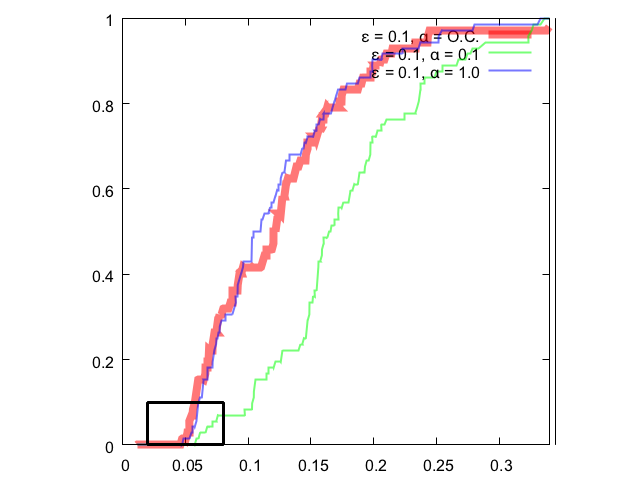}  \hspacegs & \hspacegs \includegraphics[trim=90bp 5bp 75bp
15bp,clip,width=\sg\linewidth]{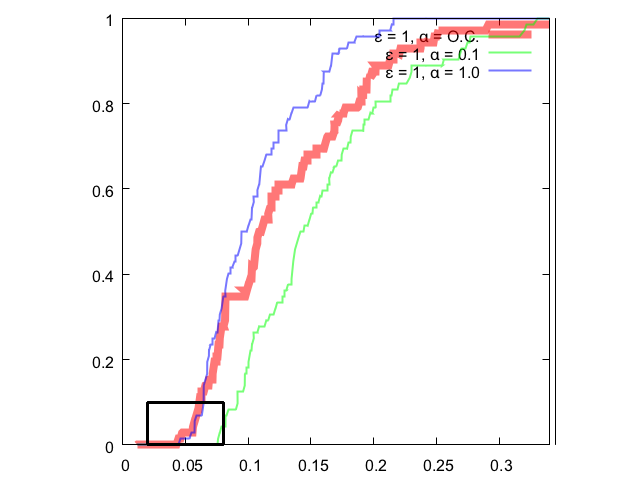}  \hspacegs & \hspacegs \includegraphics[trim=90bp 5bp 75bp
15bp,clip,width=\sg\linewidth]{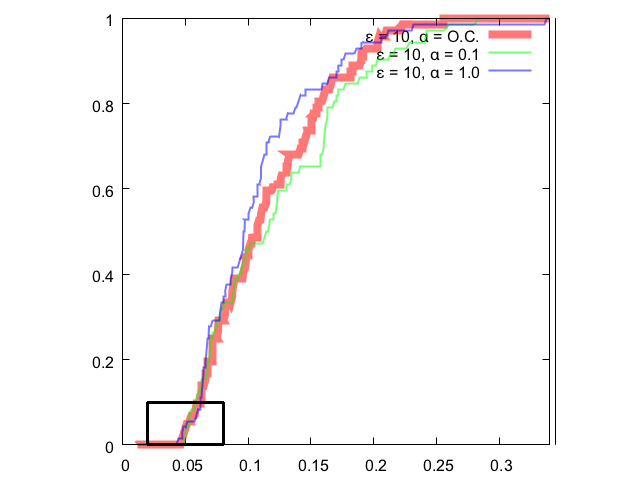}  \hspacegs & \hspacegs \includegraphics[trim=90bp 5bp 75bp
15bp,clip,width=\sg\linewidth]{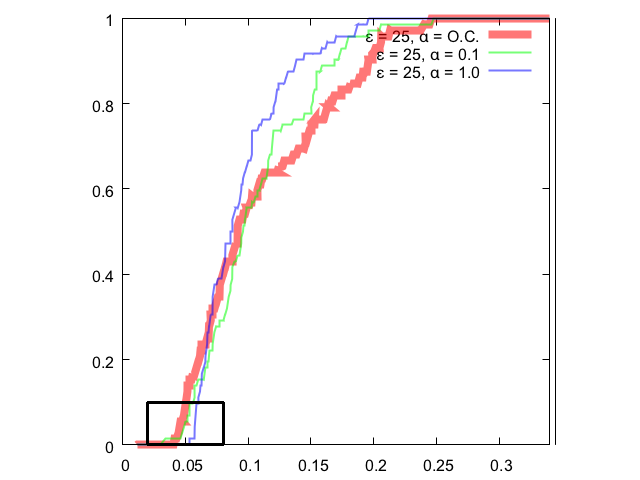}
  \\
\hline \hline
\end{tabular}
\end{center}
\caption{UCI domain \texttt{breastwisc}, conventions identical as in
  Figure \ref{f-s-transfusion3}.}
  \label{f-s-breastwisc3}
\end{figure}

\clearpage
\newpage
\subsection*{$\triangleright$ UCI \texttt{ionosphere}}\label{sub-sec-ionosphere}

\begin{figure}[h]
\begin{center}
\begin{tabular}{cc||cc}\hline \hline
\multicolumn{2}{c||}{without DP} & \multicolumn{2}{c}{with DP}\\ \hline
\pushgraphicsLeafDepth{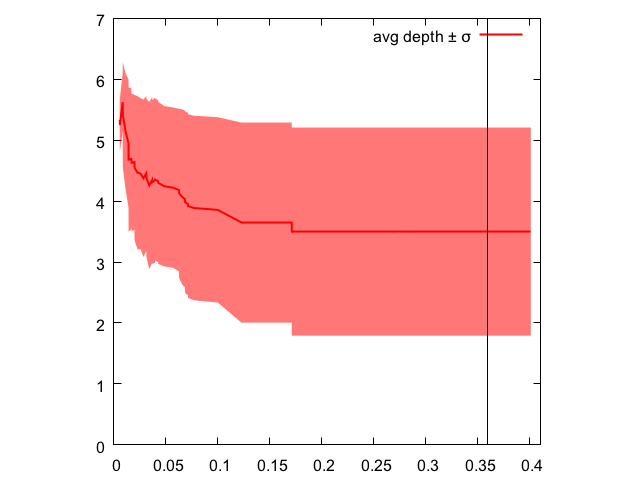}
  \hspacegss & \hspacegss \pushgraphicsLeafDepth{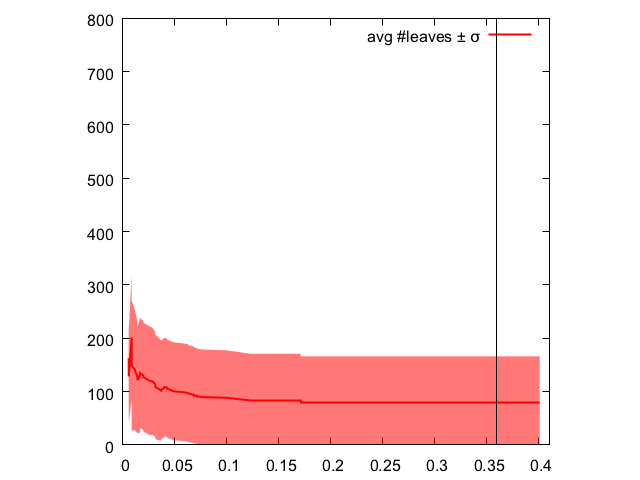}  &  \pushgraphicsLeafDepth{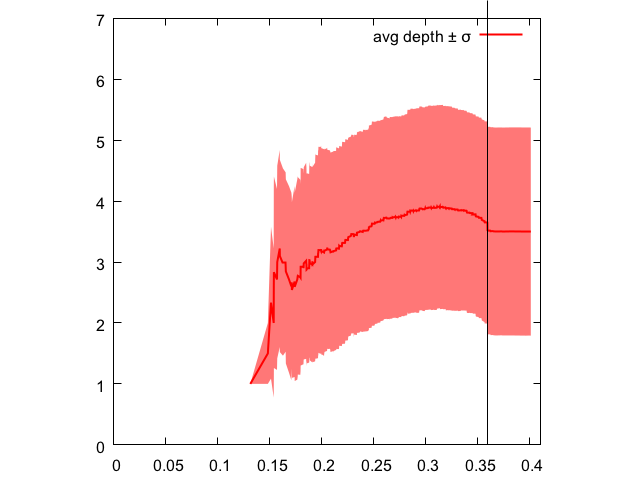}  \hspacegss & \hspacegss  \pushgraphicsLeafDepth{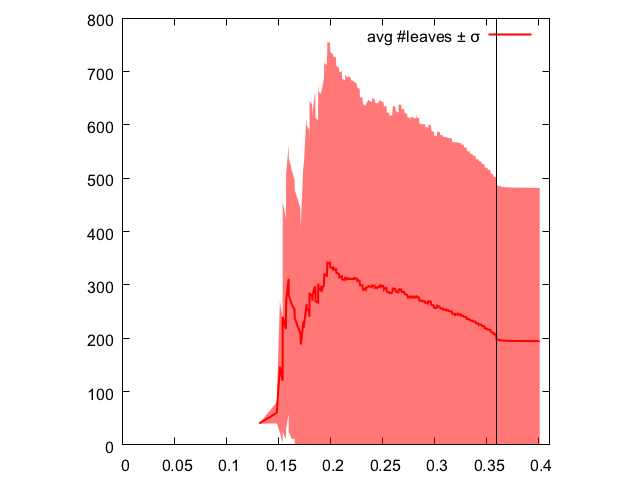}  \\ \hline
 depth \hspacegss & \hspacegss $\#$leaves & depth \hspacegss & \hspacegss $\#$leaves \\ \hline \hline
\end{tabular}
\end{center}
\caption{UCI domain \texttt{ionosphere}, conventions identical as in
  Figure \ref{f-s-transfusion}. }
  \label{f-s-ionosphere}
\end{figure}

\begin{figure}[h]
\begin{center}
\begin{tabular}{cc||cc}\hline \hline
\multicolumn{2}{c||}{performances wrt $\alpha$s} & \multicolumn{2}{c}{performances wrt $\epsilon$s (with DP)}\\ \hline
w/o DP \hspacegss & \hspacegss with DP & full
  & crop \\ \hline
\includegraphics[trim=90bp 5bp 75bp
15bp,clip,width=\sgg\linewidth]{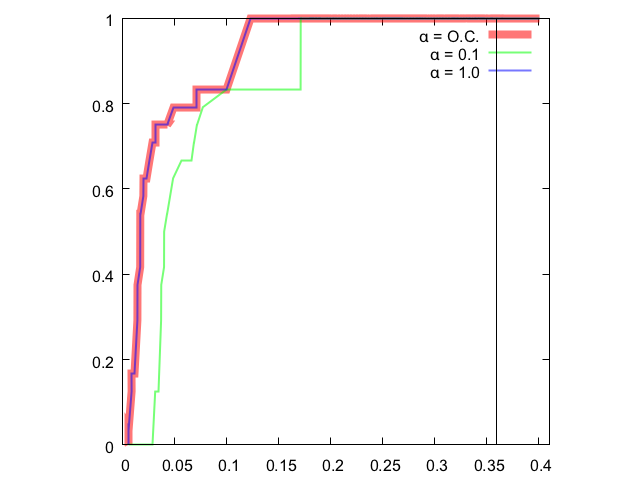}
  \hspacegss & \hspacegss \includegraphics[trim=90bp 5bp 75bp
15bp,clip,width=\sgg\linewidth]{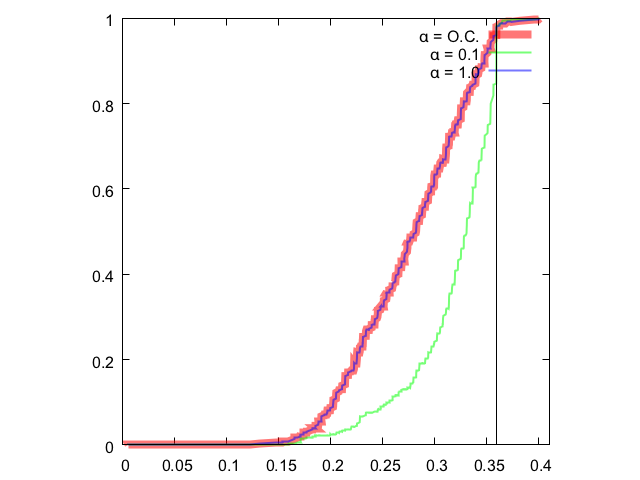} & \includegraphics[trim=90bp 5bp 75bp
15bp,clip,width=\sgg\linewidth]{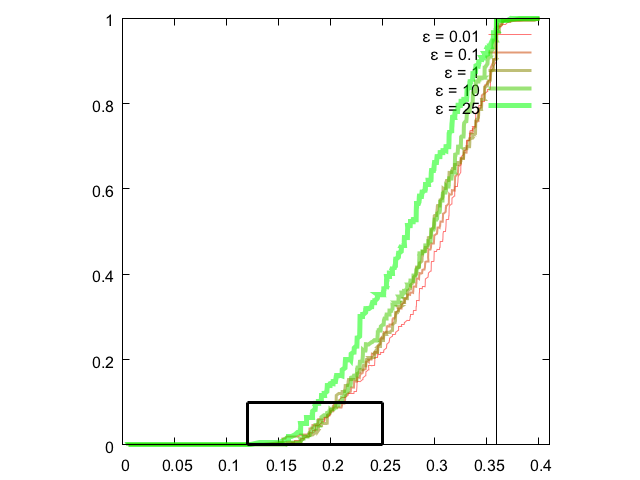} \hspacegss & \hspacegss \includegraphics[trim=90bp 5bp 75bp
15bp,clip,width=\sgg\linewidth]{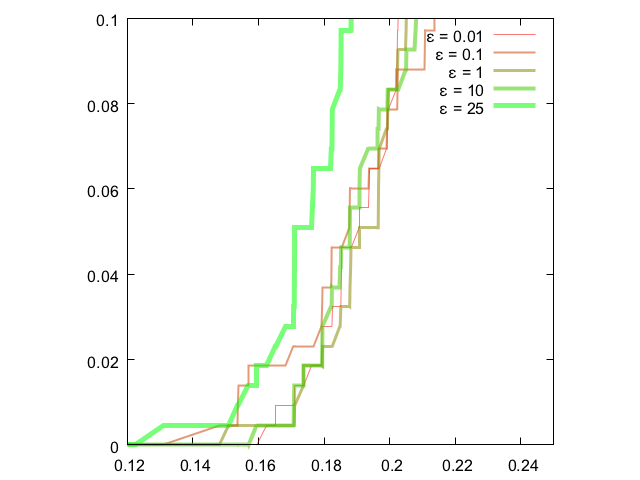} \\ \hline \hline
\end{tabular}
\end{center}
\caption{UCI domain \texttt{ionosphere}, conventions identical as in
  Figure \ref{f-s-transfusion2}.}
  \label{f-s-ionosphere2}
\end{figure}

\begin{figure}[h]
\begin{center}
\begin{tabular}{ccccc}\hline \hline
$\varepsilon=0.01$ \hspacegs & \hspacegs $\varepsilon=0.1$ \hspacegs & \hspacegs $\varepsilon=1$ \hspacegs & \hspacegs
                                                        $\varepsilon=10$ \hspacegs &
                                                                    \hspacegs $\varepsilon=25$ \hspacegs \\ \hline 
\includegraphics[trim=90bp 5bp 75bp
15bp,clip,width=\sg\linewidth]{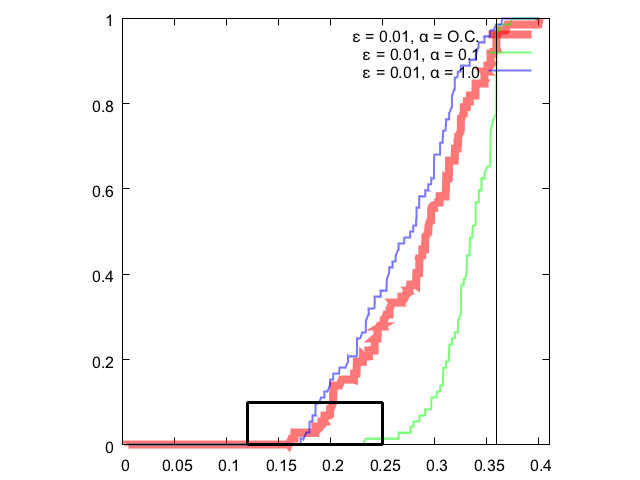}
  \hspacegs & \hspacegs \includegraphics[trim=90bp 5bp 75bp
15bp,clip,width=\sg\linewidth]{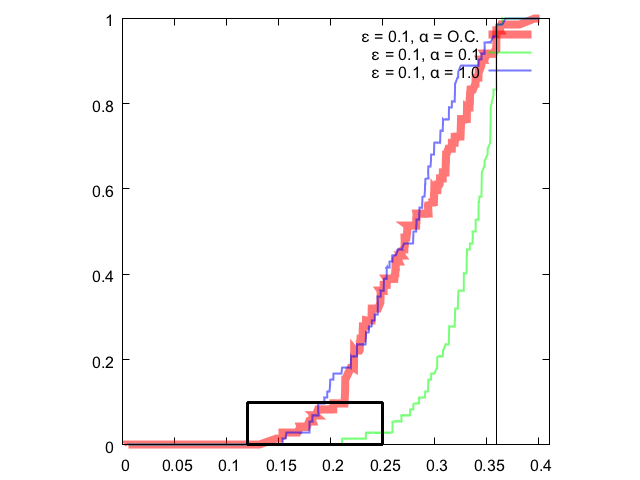}  \hspacegs & \hspacegs \includegraphics[trim=90bp 5bp 75bp
15bp,clip,width=\sg\linewidth]{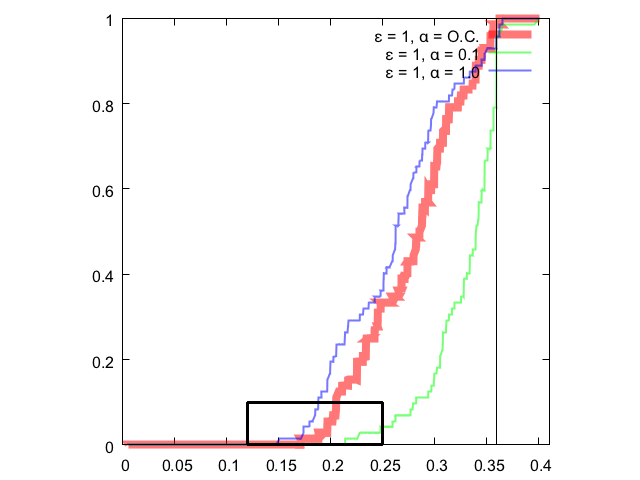}  \hspacegs & \hspacegs \includegraphics[trim=90bp 5bp 75bp
15bp,clip,width=\sg\linewidth]{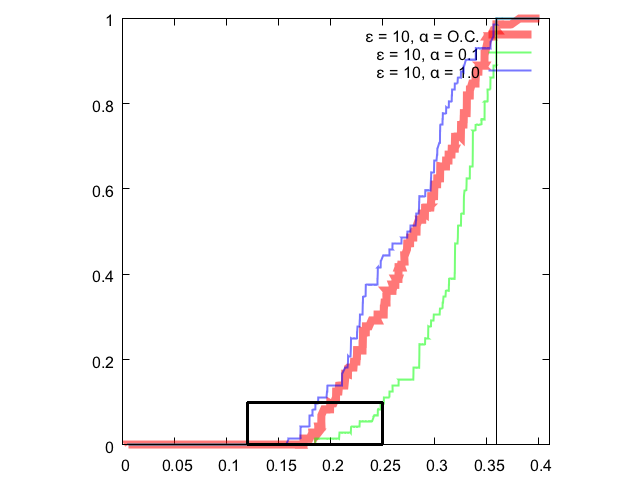}  \hspacegs & \hspacegs \includegraphics[trim=90bp 5bp 75bp
15bp,clip,width=\sg\linewidth]{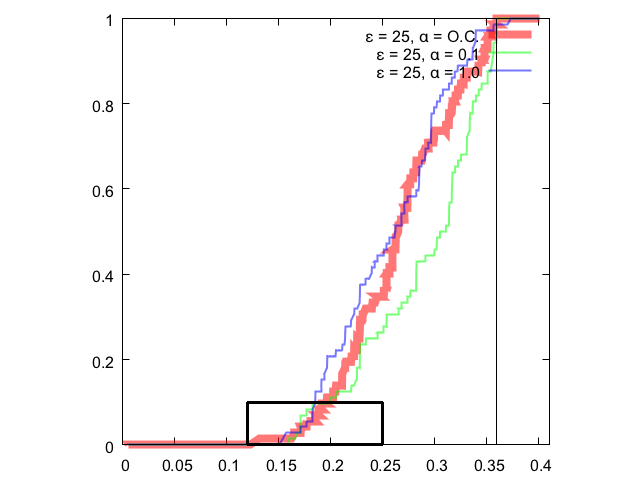}
  \\
\hline \hline
\end{tabular}
\end{center}
\caption{UCI domain \texttt{ionosphere}, conventions identical as in
  Figure \ref{f-s-transfusion3}.}
  \label{f-s-ionosphere3}
\end{figure}

\clearpage
\newpage
\subsection*{$\triangleright$ UCI \texttt{sonar}}\label{sub-sec-sonar}

\begin{figure}[h]
\begin{center}
\begin{tabular}{cc||cc}\hline \hline
\multicolumn{2}{c||}{without DP} & \multicolumn{2}{c}{with DP}\\
  \hline
\pushgraphicsLeafDepth{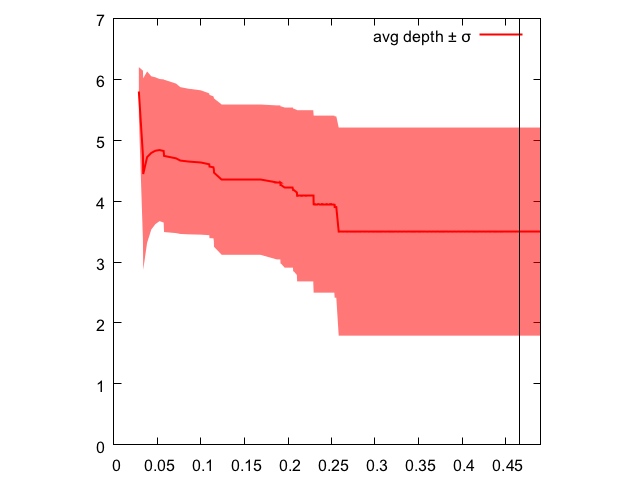}
  \hspacegss & \hspacegss \pushgraphicsLeafDepth{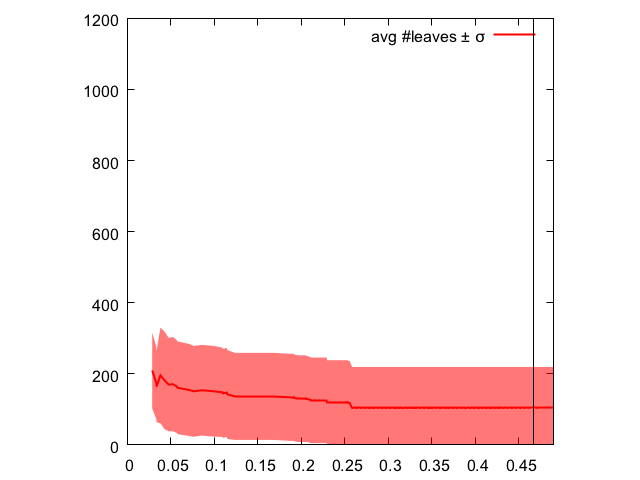}  &  \pushgraphicsLeafDepth{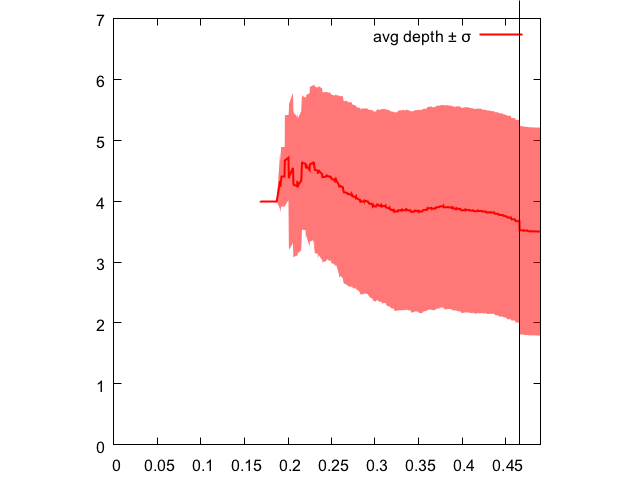}  \hspacegss & \hspacegss  \pushgraphicsLeafDepth{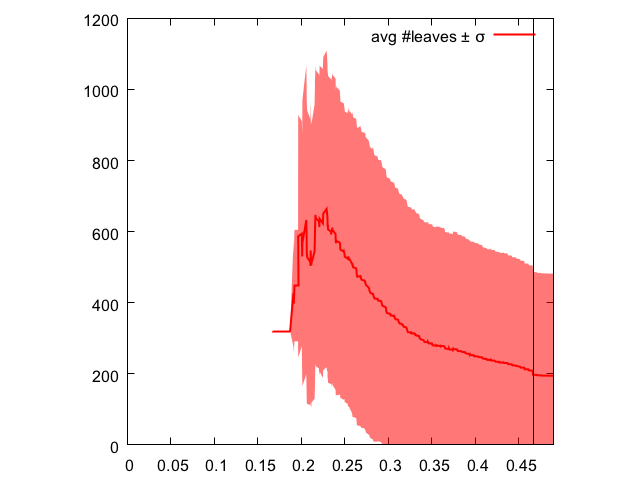}  \\ \hline
 depth \hspacegss & \hspacegss $\#$leaves & depth \hspacegss & \hspacegss $\#$leaves \\ \hline \hline
\end{tabular}
\end{center}
\caption{UCI domain \texttt{sonar}, conventions identical as in
  Figure \ref{f-s-transfusion}. }
  \label{f-s-sonar}
\end{figure}

\begin{figure}[h]
\begin{center}
\begin{tabular}{cc||cc}\hline \hline
\multicolumn{2}{c||}{performances wrt $\alpha$s} & \multicolumn{2}{c}{performances wrt $\epsilon$s (with DP)}\\ \hline
w/o DP \hspacegss & \hspacegss with DP & full
  & crop \\ \hline
\includegraphics[trim=90bp 5bp 75bp
15bp,clip,width=\sgg\linewidth]{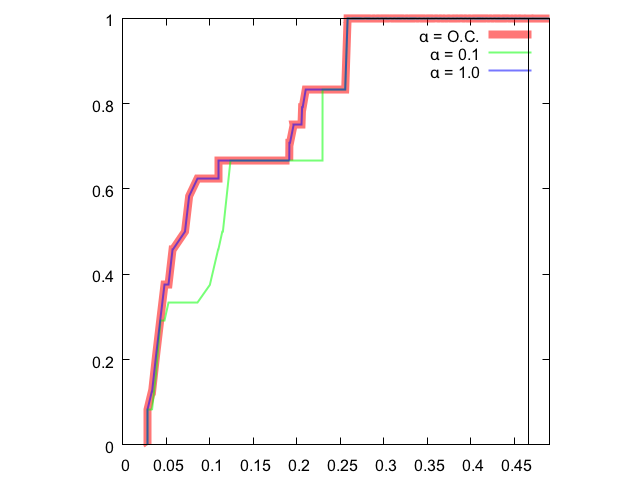}
  \hspacegss & \hspacegss \includegraphics[trim=90bp 5bp 75bp
15bp,clip,width=\sgg\linewidth]{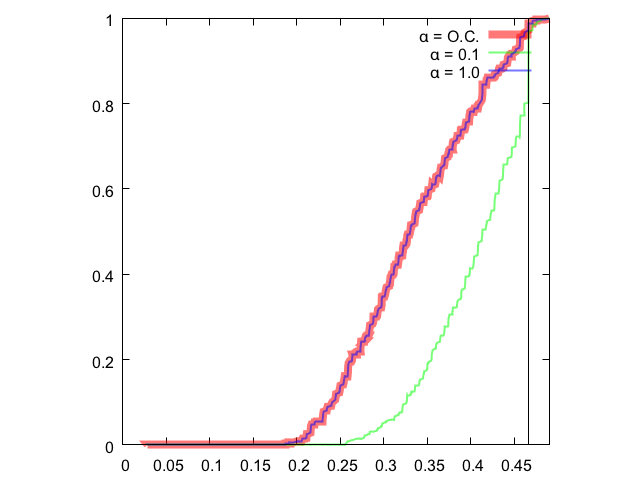} & \includegraphics[trim=90bp 5bp 75bp
15bp,clip,width=\sgg\linewidth]{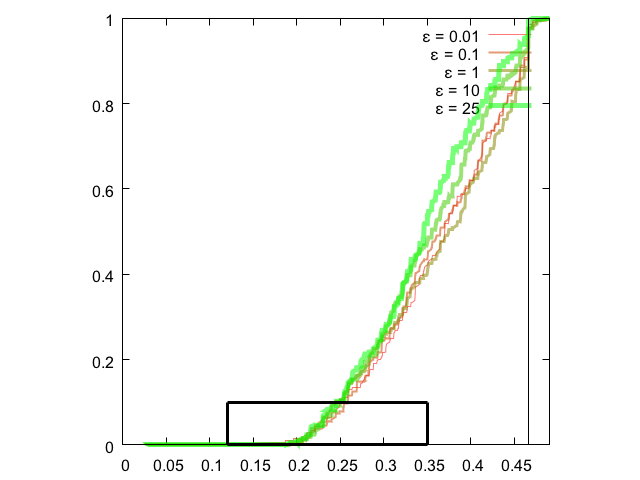} \hspacegss & \hspacegss \includegraphics[trim=90bp 5bp 75bp
15bp,clip,width=\sgg\linewidth]{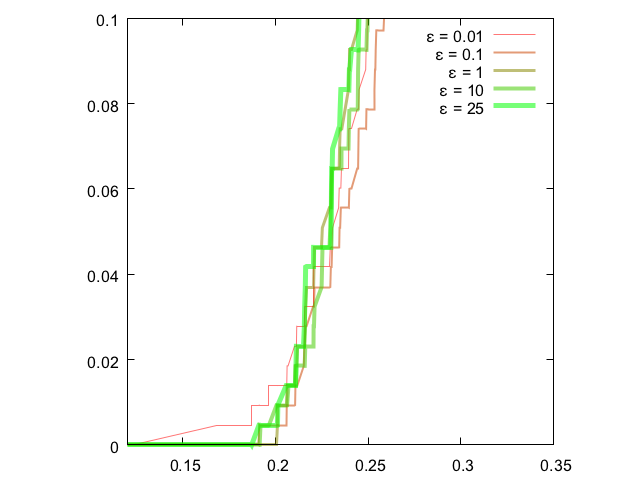} \\ \hline \hline
\end{tabular}
\end{center}
\caption{UCI domain \texttt{sonar}, conventions identical as in
  Figure \ref{f-s-transfusion2}.}
  \label{f-s-sonar2}
\end{figure}

\begin{figure}[h]
\begin{center}
\begin{tabular}{ccccc}\hline \hline
$\varepsilon=0.01$ \hspacegs & \hspacegs $\varepsilon=0.1$ \hspacegs & \hspacegs $\varepsilon=1$ \hspacegs & \hspacegs
                                                        $\varepsilon=10$ \hspacegs &
                                                                    \hspacegs $\varepsilon=25$ \hspacegs \\ \hline 
\includegraphics[trim=90bp 5bp 75bp
15bp,clip,width=\sg\linewidth]{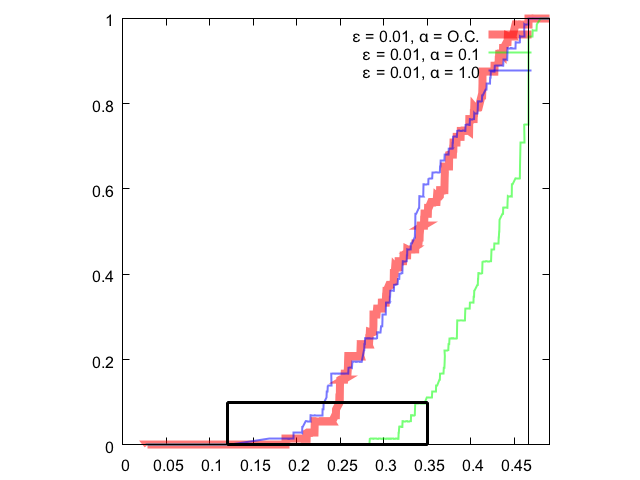}
  \hspacegs & \hspacegs \includegraphics[trim=90bp 5bp 75bp
15bp,clip,width=\sg\linewidth]{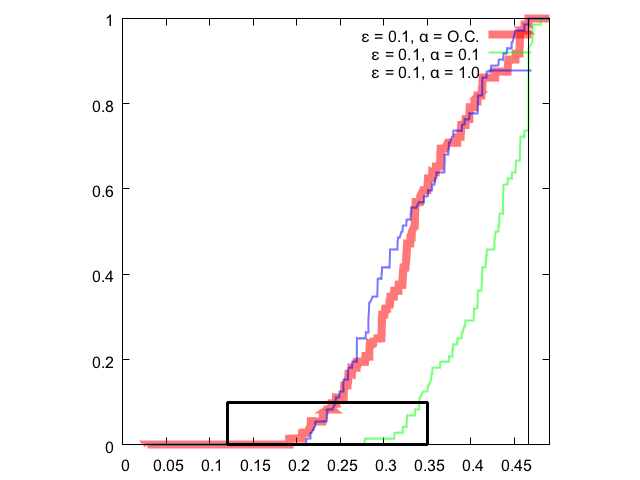}  \hspacegs & \hspacegs \includegraphics[trim=90bp 5bp 75bp
15bp,clip,width=\sg\linewidth]{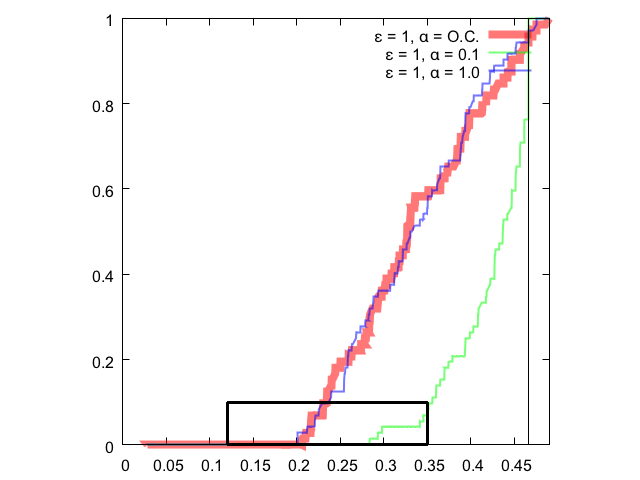}  \hspacegs & \hspacegs \includegraphics[trim=90bp 5bp 75bp
15bp,clip,width=\sg\linewidth]{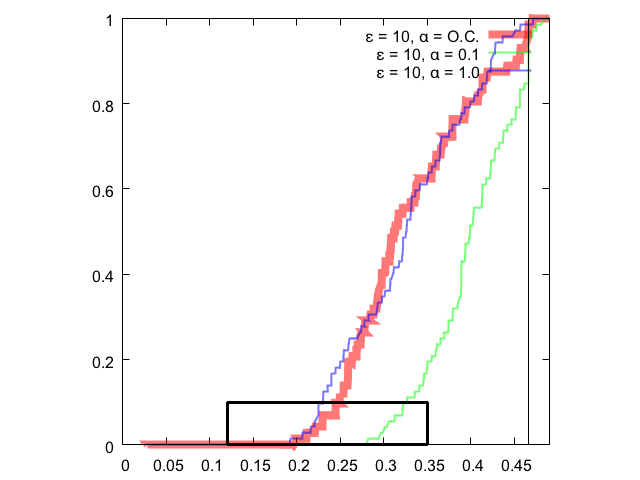}  \hspacegs & \hspacegs \includegraphics[trim=90bp 5bp 75bp
15bp,clip,width=\sg\linewidth]{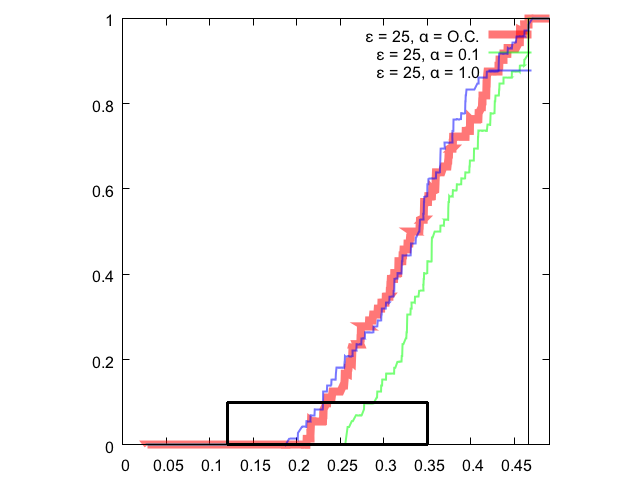}
  \\
\hline \hline
\end{tabular}
\end{center}
\caption{UCI domain \texttt{sonar}, conventions identical as in
  Figure \ref{f-s-transfusion3}.}
  \label{f-s-sonar3}
\end{figure}

\clearpage
\newpage
\subsection*{$\triangleright$ UCI \texttt{yeast}}\label{sub-sec-yeast}

\begin{figure}[h]
\begin{center}
\begin{tabular}{cc||cc}\hline \hline
\multicolumn{2}{c||}{without DP} & \multicolumn{2}{c}{with DP}\\ \hline
\pushgraphicsLeafDepth{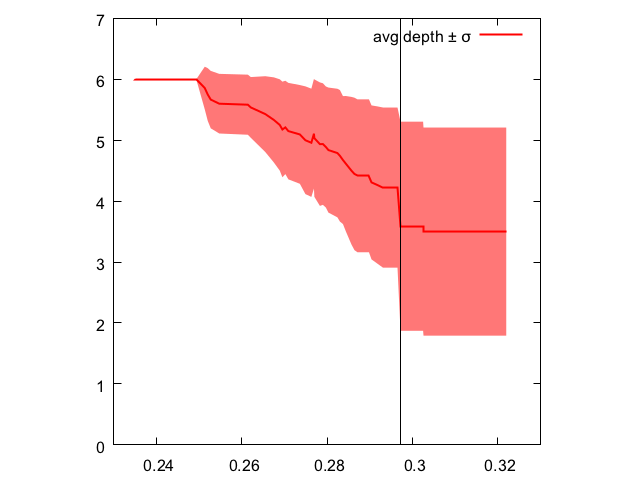}
  \hspacegss & \hspacegss \pushgraphicsLeafDepth{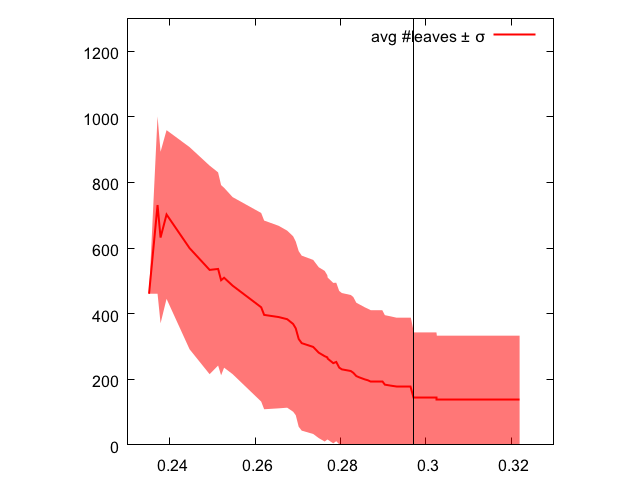}  &  \pushgraphicsLeafDepth{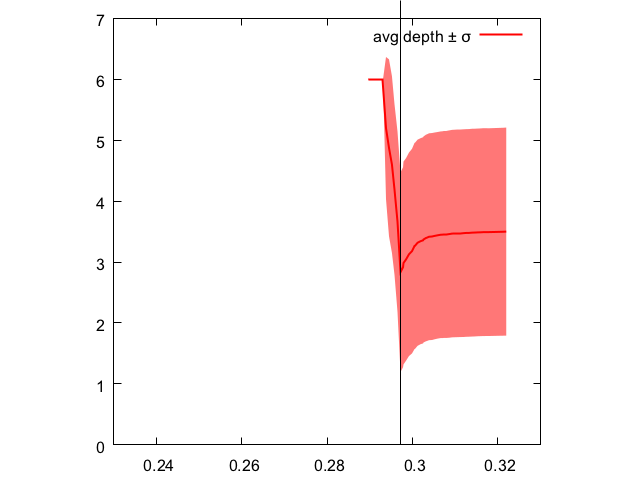}  \hspacegss & \hspacegss  \pushgraphicsLeafDepth{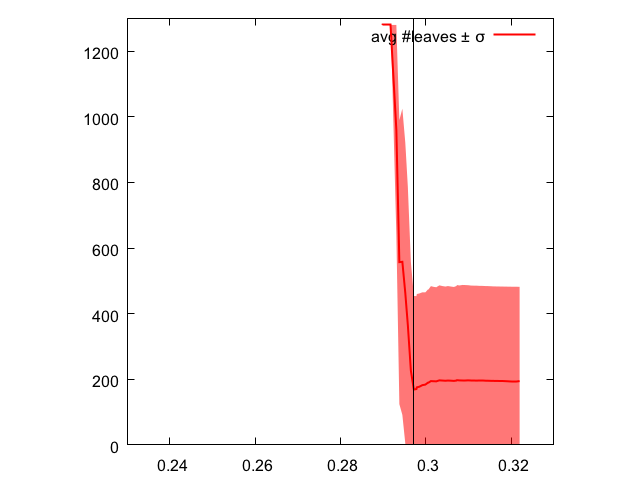}  \\ \hline
 depth \hspacegss & \hspacegss $\#$leaves & depth \hspacegss & \hspacegss $\#$leaves \\ \hline \hline
\end{tabular}
\end{center}
\caption{UCI domain \texttt{yeast}, conventions identical as in
  Figure \ref{f-s-transfusion}. }
  \label{f-s-yeast}
\end{figure}

\begin{figure}[h]
\begin{center}
\begin{tabular}{cc||cc}\hline \hline
\multicolumn{2}{c||}{performances wrt $\alpha$s} & \multicolumn{2}{c}{performances wrt $\epsilon$s (with DP)}\\ \hline
w/o DP \hspacegss & \hspacegss with DP & full
  & crop \\ \hline
\includegraphics[trim=90bp 5bp 75bp
15bp,clip,width=\sgg\linewidth]{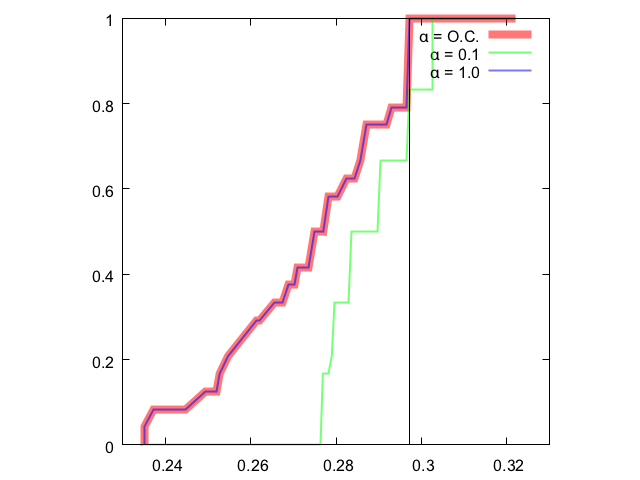}
  \hspacegss & \hspacegss \includegraphics[trim=90bp 5bp 75bp
15bp,clip,width=\sgg\linewidth]{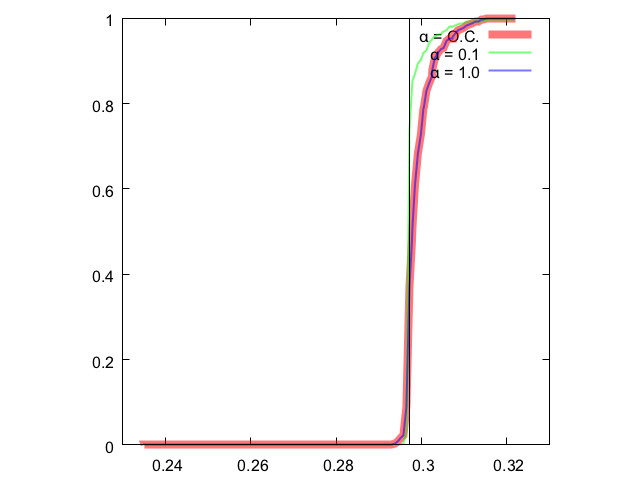} & \includegraphics[trim=90bp 5bp 75bp
15bp,clip,width=\sgg\linewidth]{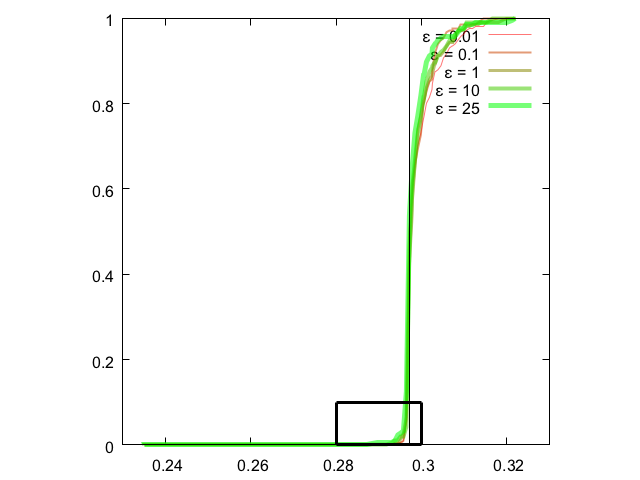} \hspacegss & \hspacegss \includegraphics[trim=90bp 5bp 75bp
15bp,clip,width=\sgg\linewidth]{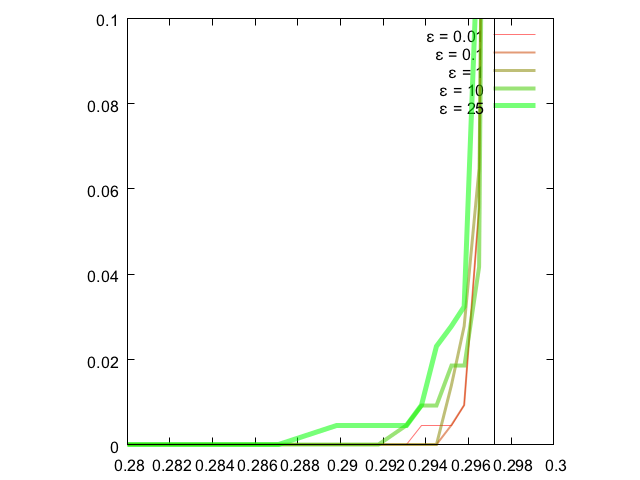} \\ \hline \hline
\end{tabular}
\end{center}
\caption{UCI domain \texttt{yeast}, conventions identical as in
  Figure \ref{f-s-transfusion2}.}
  \label{f-s-yeast2}
\end{figure}

\begin{figure}[h]
\begin{center}
\begin{tabular}{ccccc}\hline \hline
$\varepsilon=0.01$ \hspacegs & \hspacegs $\varepsilon=0.1$ \hspacegs & \hspacegs $\varepsilon=1$ \hspacegs & \hspacegs
                                                        $\varepsilon=10$ \hspacegs &
                                                                    \hspacegs $\varepsilon=25$ \hspacegs \\ \hline 
\includegraphics[trim=90bp 5bp 75bp
15bp,clip,width=\sg\linewidth]{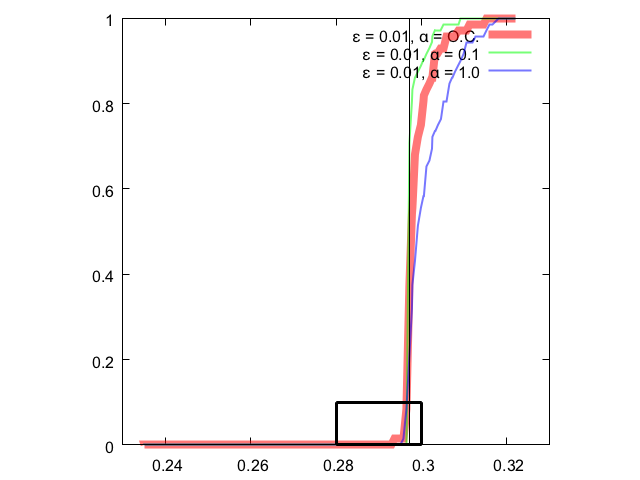}
  \hspacegs & \hspacegs \includegraphics[trim=90bp 5bp 75bp
15bp,clip,width=\sg\linewidth]{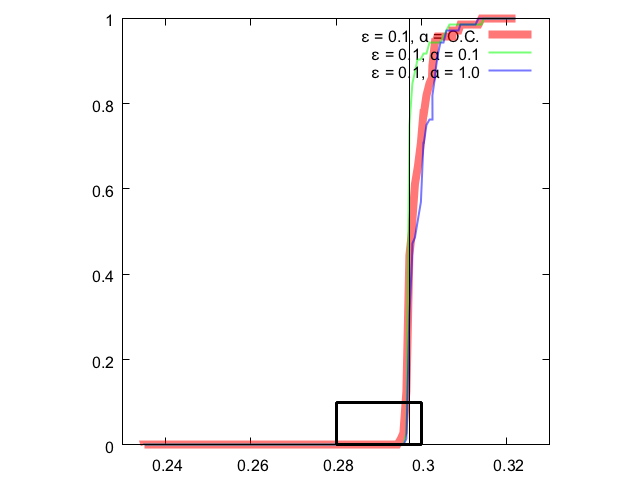}  \hspacegs & \hspacegs \includegraphics[trim=90bp 5bp 75bp
15bp,clip,width=\sg\linewidth]{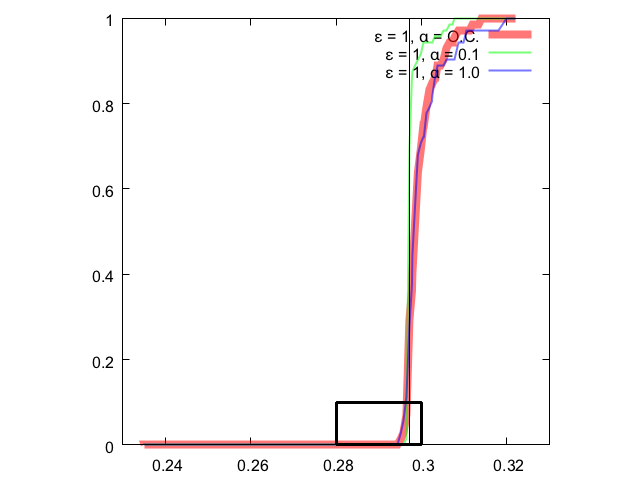}  \hspacegs & \hspacegs \includegraphics[trim=90bp 5bp 75bp
15bp,clip,width=\sg\linewidth]{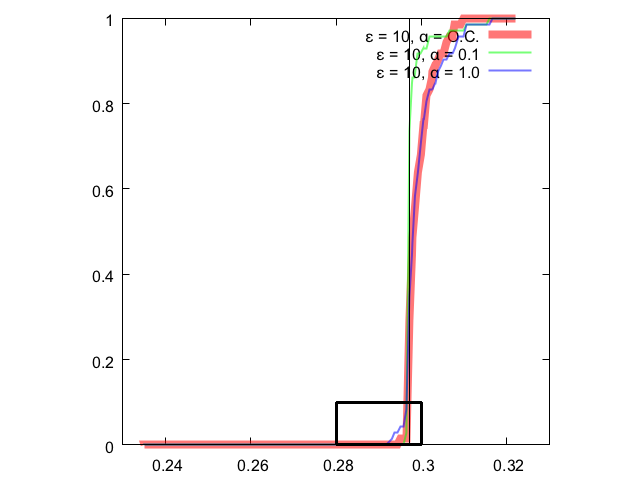}  \hspacegs & \hspacegs \includegraphics[trim=90bp 5bp 75bp
15bp,clip,width=\sg\linewidth]{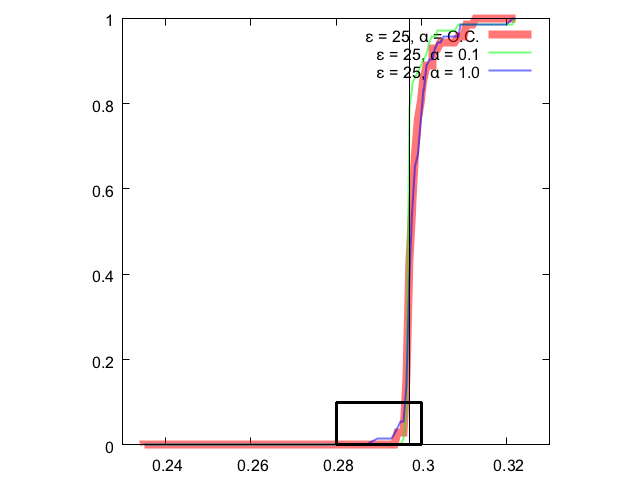}
  \\
\hline \hline
\end{tabular}
\end{center}
\caption{UCI domain \texttt{yeast}, conventions identical as in
  Figure \ref{f-s-transfusion3}.}
  \label{f-s-yeast3}
\end{figure}

\clearpage
\newpage
\subsection*{$\triangleright$ UCI \texttt{winered}}\label{sub-sec-winered}

\begin{figure}[h]
\begin{center}
\begin{tabular}{cc||cc}\hline \hline
\multicolumn{2}{c||}{without DP} & \multicolumn{2}{c}{with DP}\\ \hline
\pushgraphicsLeafDepth{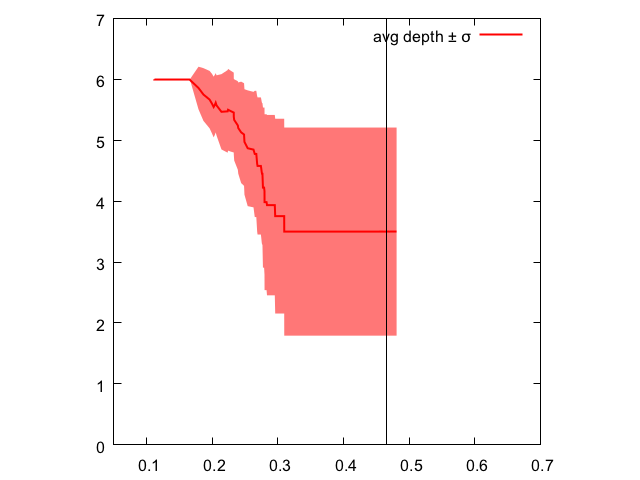}
  \hspacegss & \hspacegss \pushgraphicsLeafDepth{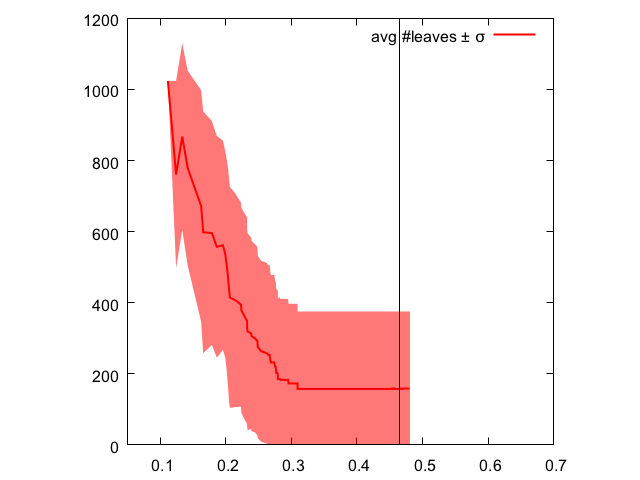}  &  \pushgraphicsLeafDepth{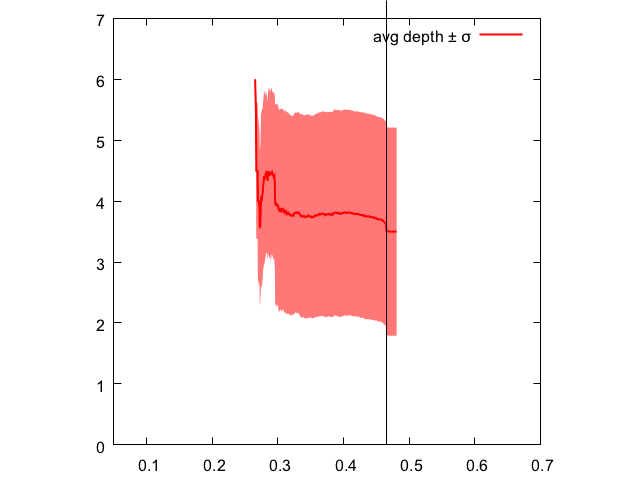}  \hspacegss & \hspacegss  \pushgraphicsLeafDepth{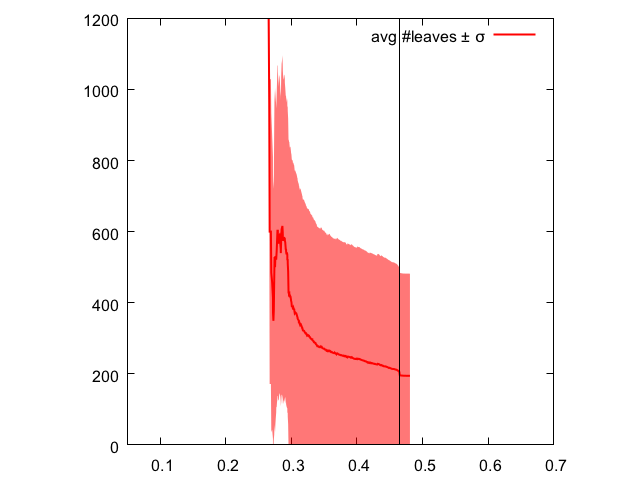}  \\ \hline
 depth \hspacegss & \hspacegss $\#$leaves & depth \hspacegss & \hspacegss $\#$leaves \\ \hline \hline
\end{tabular}
\end{center}
\caption{UCI domain \texttt{winered}, conventions identical as in
  Figure \ref{f-s-transfusion}. }
  \label{f-s-winered}
\end{figure}

\begin{figure}[h]
\begin{center}
\begin{tabular}{cc||cc}\hline \hline
\multicolumn{2}{c||}{performances wrt $\alpha$s} & \multicolumn{2}{c}{performances wrt $\epsilon$s (with DP)}\\ \hline
w/o DP \hspacegss & \hspacegss with DP & full
  & crop \\ \hline
\includegraphics[trim=90bp 5bp 75bp
15bp,clip,width=\sgg\linewidth]{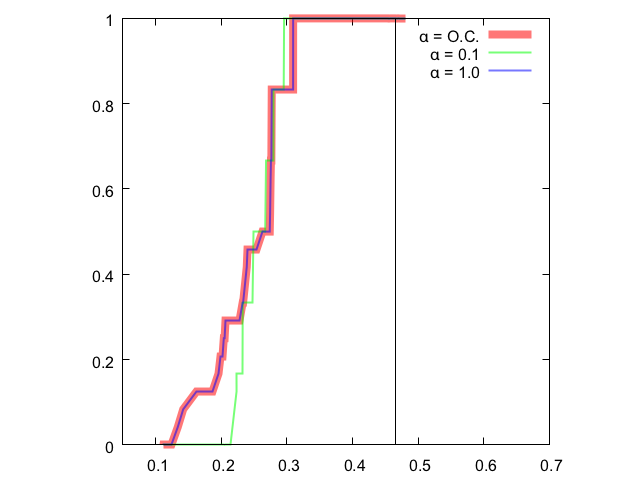}
  \hspacegss & \hspacegss \includegraphics[trim=90bp 5bp 75bp
15bp,clip,width=\sgg\linewidth]{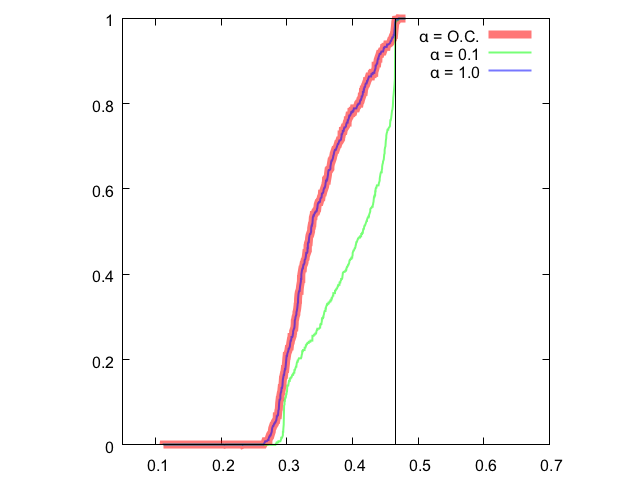} & \includegraphics[trim=90bp 5bp 75bp
15bp,clip,width=\sgg\linewidth]{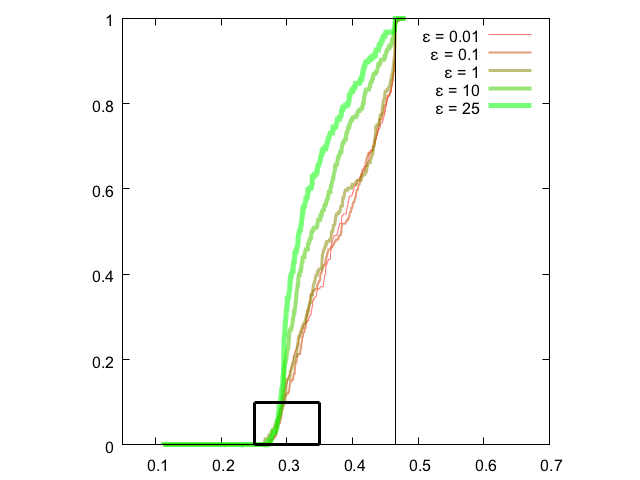} \hspacegss & \hspacegss \includegraphics[trim=90bp 5bp 75bp
15bp,clip,width=\sgg\linewidth]{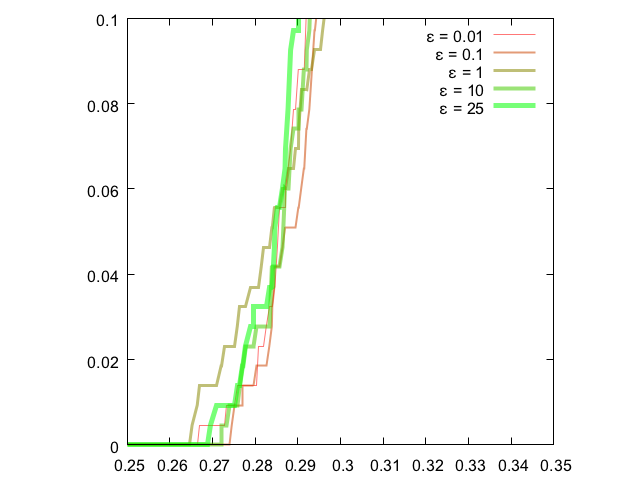} \\ \hline \hline
\end{tabular}
\end{center}
\caption{UCI domain \texttt{winered}, conventions identical as in
  Figure \ref{f-s-transfusion2}.}
  \label{f-s-winered2}
\end{figure}

\begin{figure}[h]
\begin{center}
\begin{tabular}{ccccc}\hline \hline
$\varepsilon=0.01$ \hspacegs & \hspacegs $\varepsilon=0.1$ \hspacegs & \hspacegs $\varepsilon=1$ \hspacegs & \hspacegs
                                                        $\varepsilon=10$ \hspacegs &
                                                                    \hspacegs $\varepsilon=25$ \hspacegs \\ \hline 
\includegraphics[trim=90bp 5bp 75bp
15bp,clip,width=\sg\linewidth]{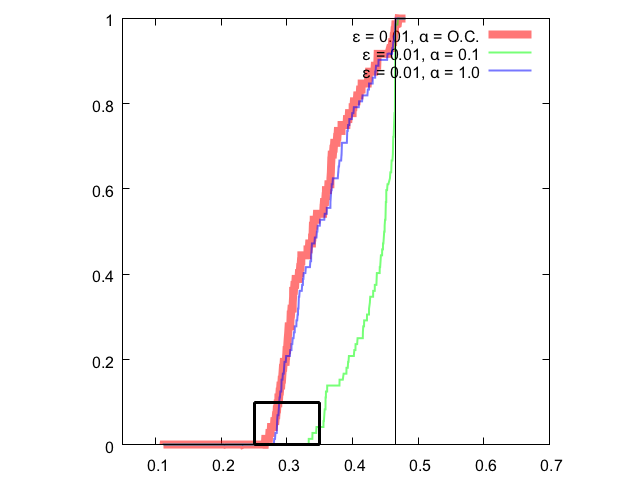}
  \hspacegs & \hspacegs \includegraphics[trim=90bp 5bp 75bp
15bp,clip,width=\sg\linewidth]{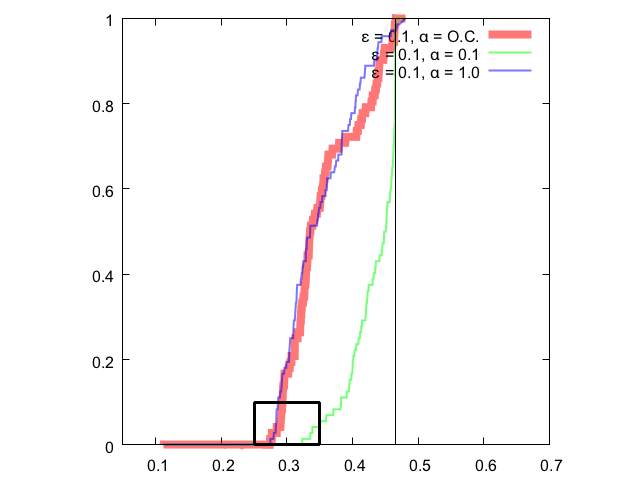}  \hspacegs & \hspacegs \includegraphics[trim=90bp 5bp 75bp
15bp,clip,width=\sg\linewidth]{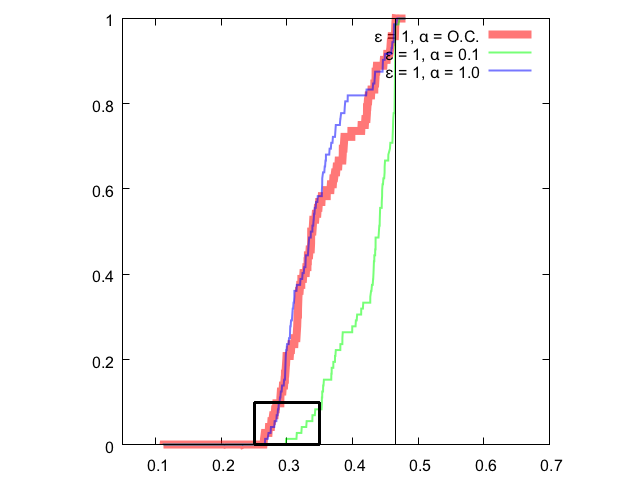}  \hspacegs & \hspacegs \includegraphics[trim=90bp 5bp 75bp
15bp,clip,width=\sg\linewidth]{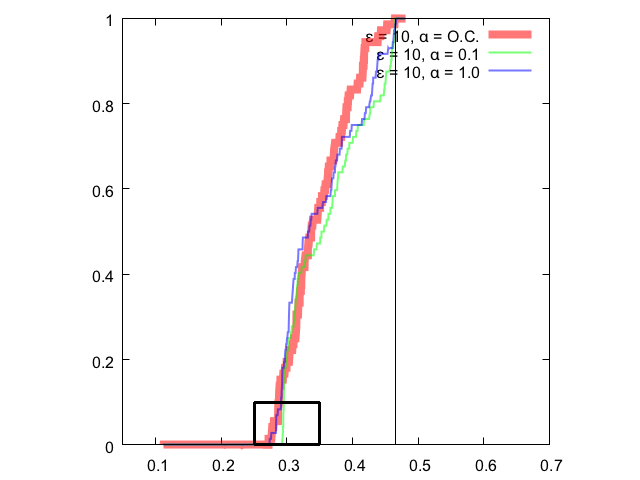}  \hspacegs & \hspacegs \includegraphics[trim=90bp 5bp 75bp
15bp,clip,width=\sg\linewidth]{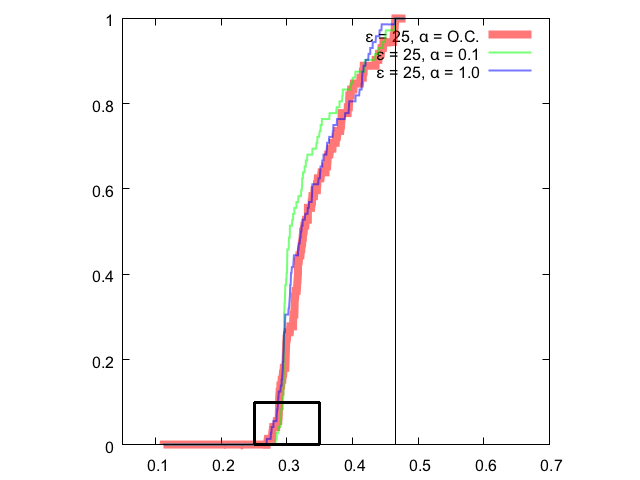}
  \\
\hline \hline
\end{tabular}
\end{center}
\caption{UCI domain \texttt{winered}, conventions identical as in
  Figure \ref{f-s-transfusion3}.}
  \label{f-s-winered3}
\end{figure}

\clearpage
\newpage
\subsection*{$\triangleright$ UCI \texttt{cardiotocography}}\label{sub-sec-cardiotocography}

\begin{figure}[h]
\begin{center}
\begin{tabular}{cc||cc}\hline \hline
\multicolumn{2}{c||}{without DP} & \multicolumn{2}{c}{with DP}\\ \hline
\pushgraphicsLeafDepth{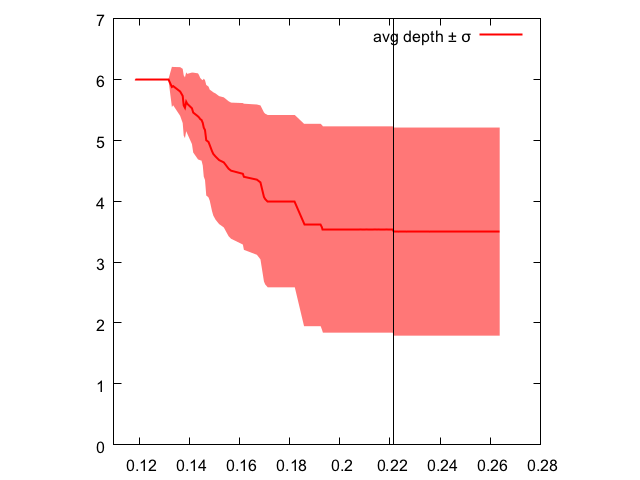}
  \hspacegss & \hspacegss \pushgraphicsLeafDepth{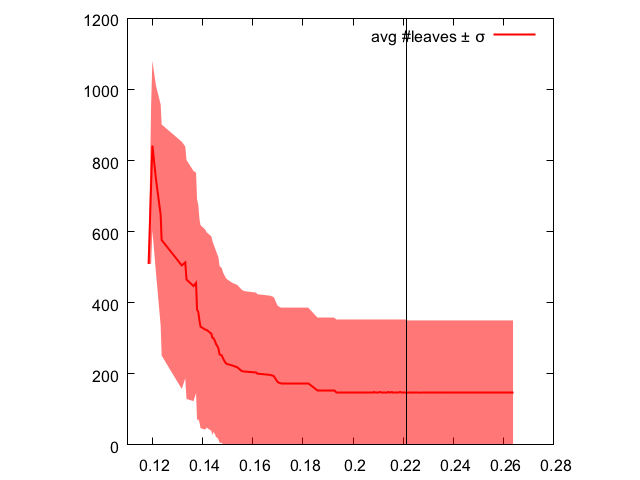}  &  \pushgraphicsLeafDepth{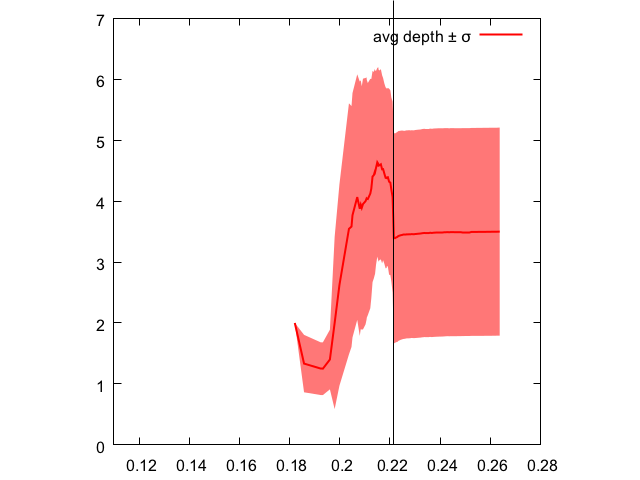}  \hspacegss & \hspacegss  \pushgraphicsLeafDepth{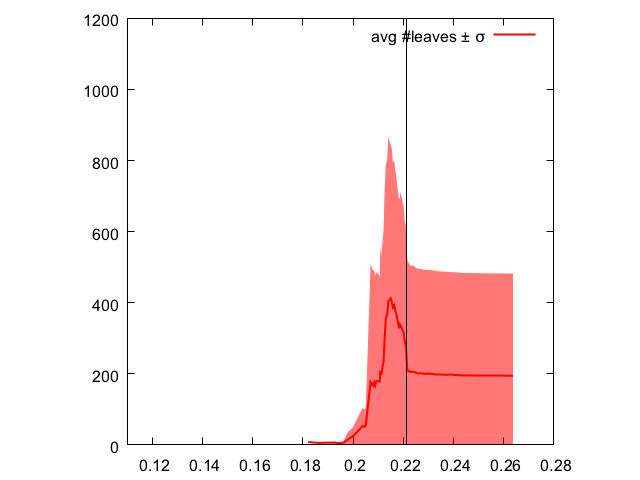}  \\ \hline
 depth \hspacegss & \hspacegss $\#$leaves & depth \hspacegss & \hspacegss $\#$leaves \\ \hline \hline
\end{tabular}
\end{center}
\caption{UCI domain \texttt{cardiotocography}, conventions identical as in
  Figure \ref{f-s-transfusion}. }
  \label{f-s-cardiotocography}
\end{figure}

\begin{figure}[h]
\begin{center}
\begin{tabular}{cc||cc}\hline \hline
\multicolumn{2}{c||}{performances wrt $\alpha$s} & \multicolumn{2}{c}{performances wrt $\epsilon$s (with DP)}\\ \hline
w/o DP \hspacegss & \hspacegss with DP & full
  & crop \\ \hline
\includegraphics[trim=90bp 5bp 75bp
15bp,clip,width=\sgg\linewidth]{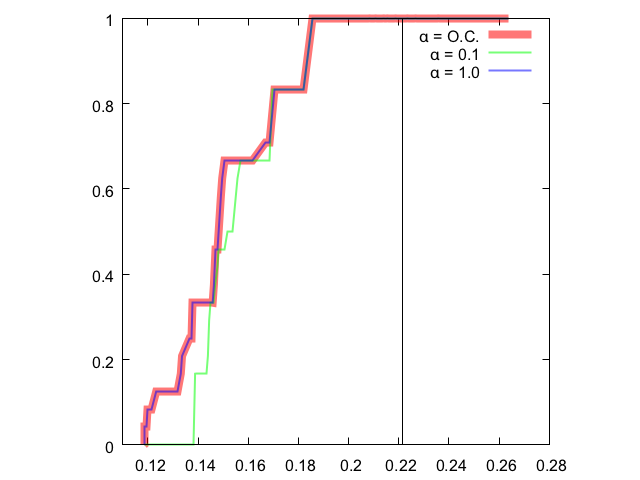}
  \hspacegss & \hspacegss \includegraphics[trim=90bp 5bp 75bp
15bp,clip,width=\sgg\linewidth]{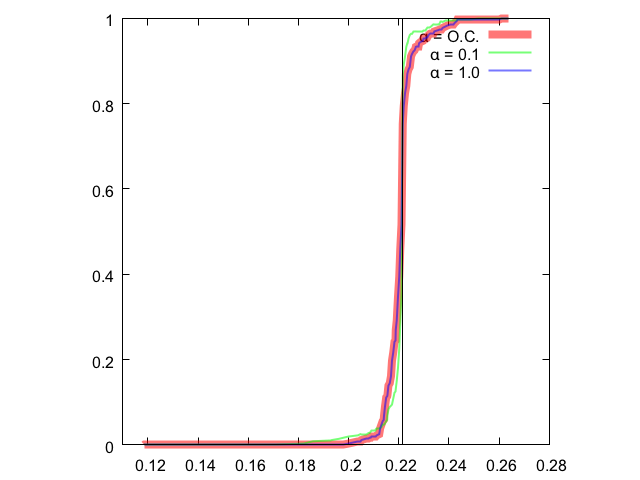} & \includegraphics[trim=90bp 5bp 75bp
15bp,clip,width=\sgg\linewidth]{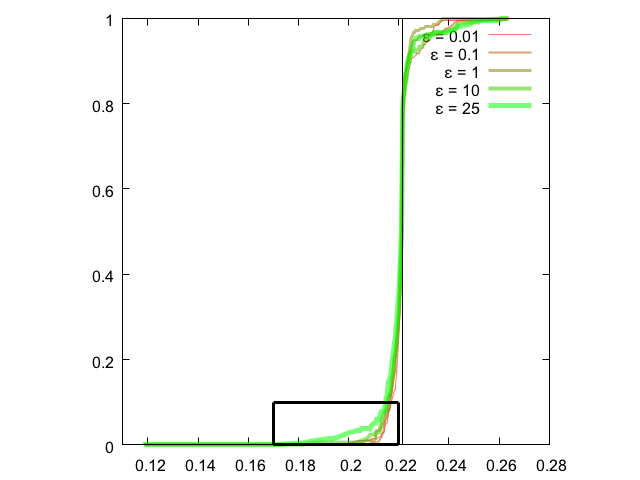} \hspacegss & \hspacegss \includegraphics[trim=90bp 5bp 75bp
15bp,clip,width=\sgg\linewidth]{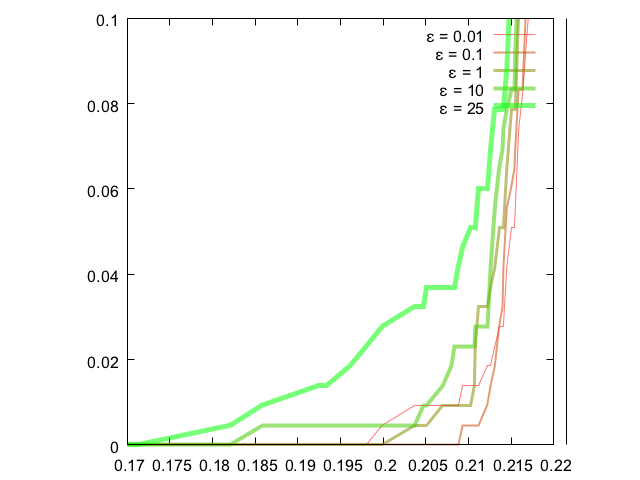} \\ \hline \hline
\end{tabular}
\end{center}
\caption{UCI domain \texttt{cardiotocography}, conventions identical as in
  Figure \ref{f-s-transfusion2}.}
  \label{f-s-cardiotocography2}
\end{figure}

\begin{figure}[h]
\begin{center}
\begin{tabular}{ccccc}\hline \hline
$\varepsilon=0.01$ \hspacegs & \hspacegs $\varepsilon=0.1$ \hspacegs & \hspacegs $\varepsilon=1$ \hspacegs & \hspacegs
                                                        $\varepsilon=10$ \hspacegs &
                                                                    \hspacegs $\varepsilon=25$ \hspacegs \\ \hline 
\includegraphics[trim=90bp 5bp 75bp
15bp,clip,width=\sg\linewidth]{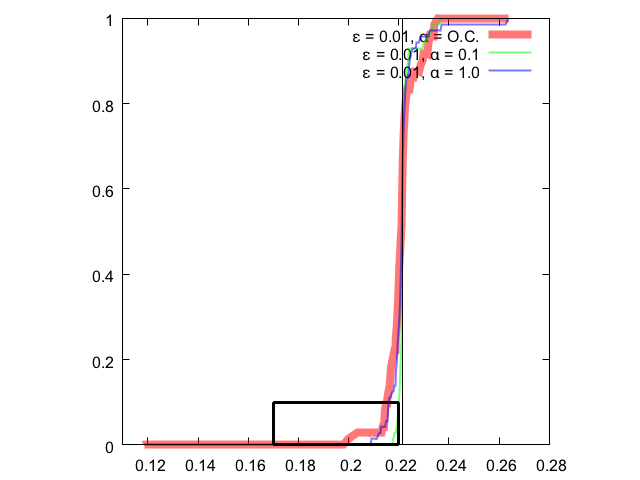}
  \hspacegs & \hspacegs \includegraphics[trim=90bp 5bp 75bp
15bp,clip,width=\sg\linewidth]{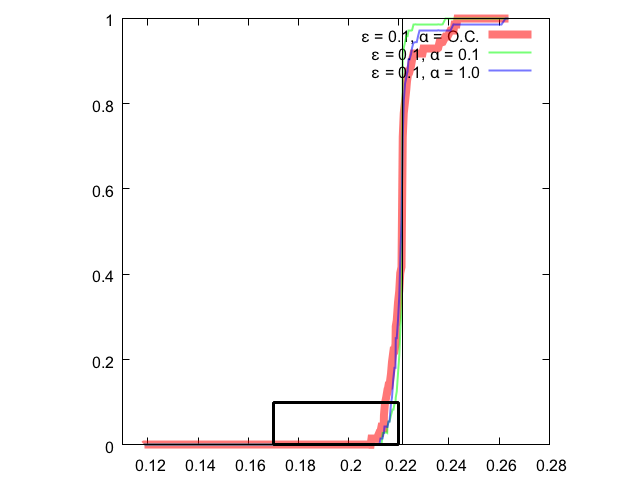}  \hspacegs & \hspacegs \includegraphics[trim=90bp 5bp 75bp
15bp,clip,width=\sg\linewidth]{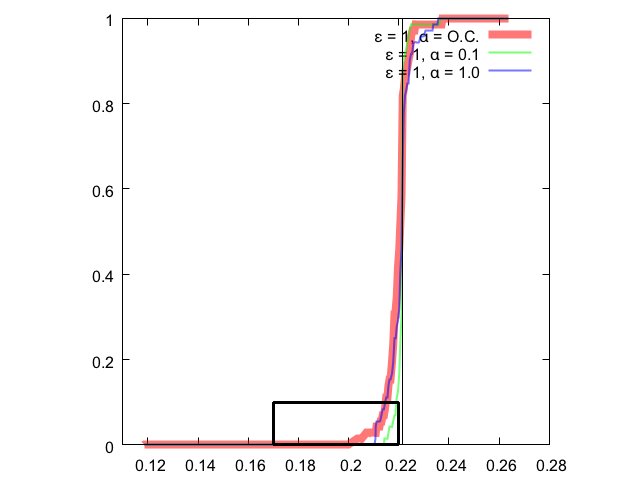}  \hspacegs & \hspacegs \includegraphics[trim=90bp 5bp 75bp
15bp,clip,width=\sg\linewidth]{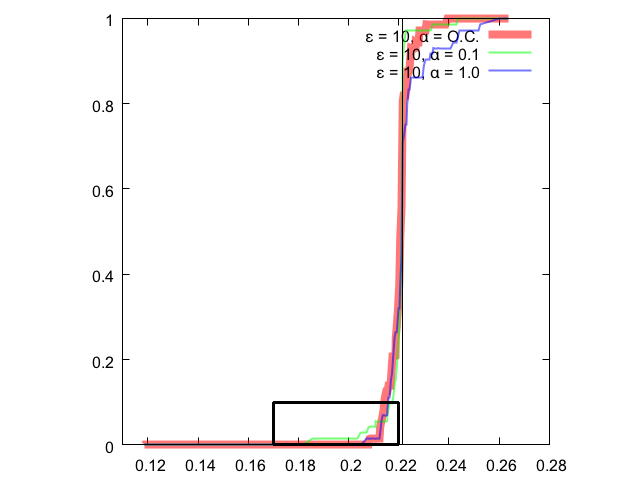}  \hspacegs & \hspacegs \includegraphics[trim=90bp 5bp 75bp
15bp,clip,width=\sg\linewidth]{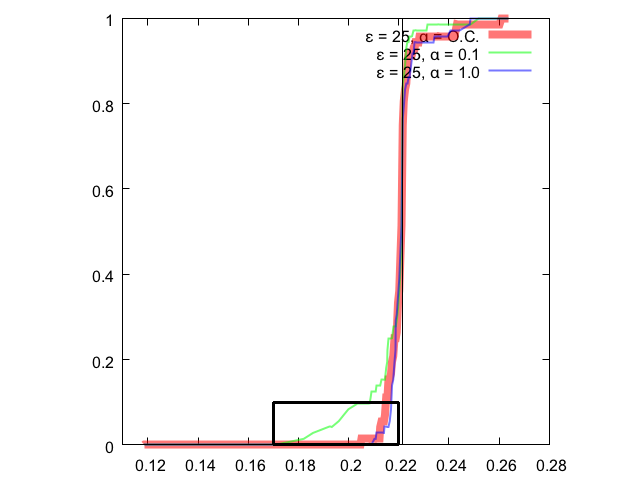}
  \\
\hline \hline
\end{tabular}
\end{center}
\caption{UCI domain \texttt{cardiotocography}, conventions identical as in
  Figure \ref{f-s-transfusion3}.}
  \label{f-s-cardiotocography3}
\end{figure}

\clearpage
\newpage
\subsection*{$\triangleright$ UCI \texttt{creditcardsmall}}\label{sub-sec-creditcardsmall}

\begin{figure}[h]
\begin{center}
\begin{tabular}{cc||cc}\hline \hline
\multicolumn{2}{c||}{without DP} & \multicolumn{2}{c}{with DP}\\ \hline
\pushgraphicsLeafDepth{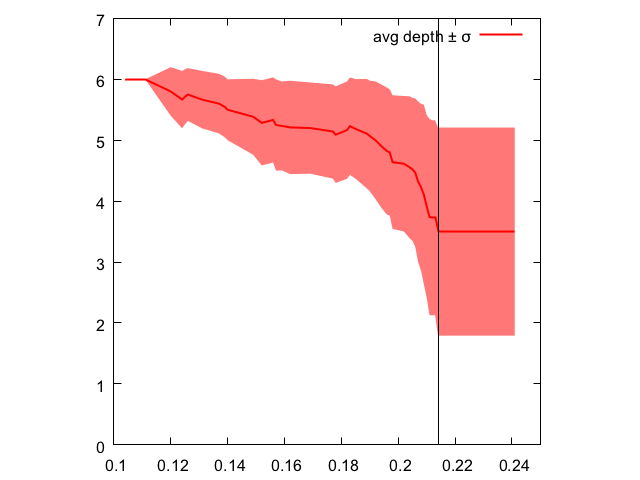}
  \hspacegss & \hspacegss \pushgraphicsLeafDepth{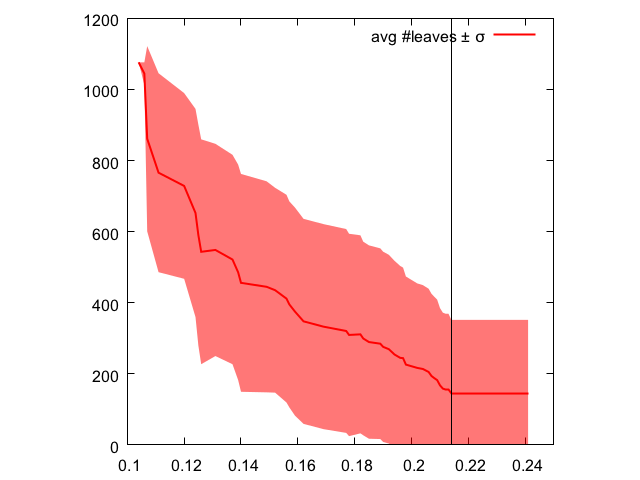}  &  \pushgraphicsLeafDepth{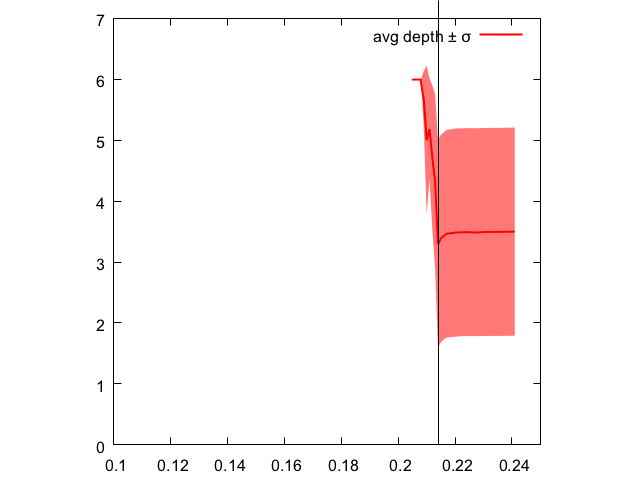}  \hspacegss & \hspacegss  \pushgraphicsLeafDepth{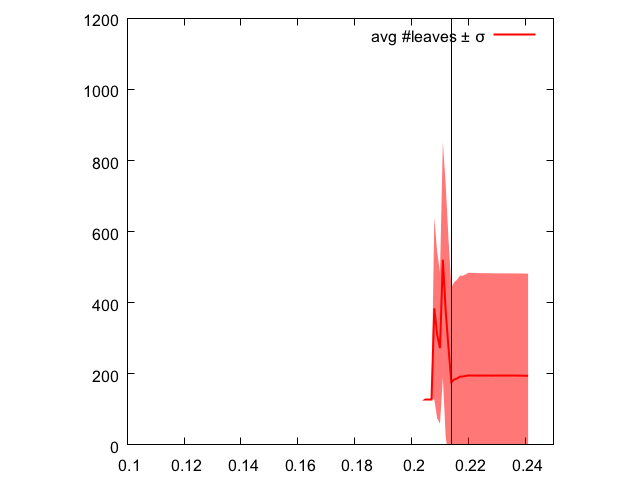}  \\ \hline
 depth \hspacegss & \hspacegss $\#$leaves & depth \hspacegss & \hspacegss $\#$leaves \\ \hline \hline
\end{tabular}
\end{center}
\caption{UCI domain \texttt{creditcardsmall}, conventions identical as in
  Figure \ref{f-s-transfusion}. }
  \label{f-s-creditcardsmall}
\end{figure}

\begin{figure}[h]
\begin{center}
\begin{tabular}{cc||cc}\hline \hline
\multicolumn{2}{c||}{performances wrt $\alpha$s} & \multicolumn{2}{c}{performances wrt $\epsilon$s (with DP)}\\ \hline
w/o DP \hspacegss & \hspacegss with DP & full
  & crop \\ \hline
\includegraphics[trim=90bp 5bp 75bp
15bp,clip,width=\sgg\linewidth]{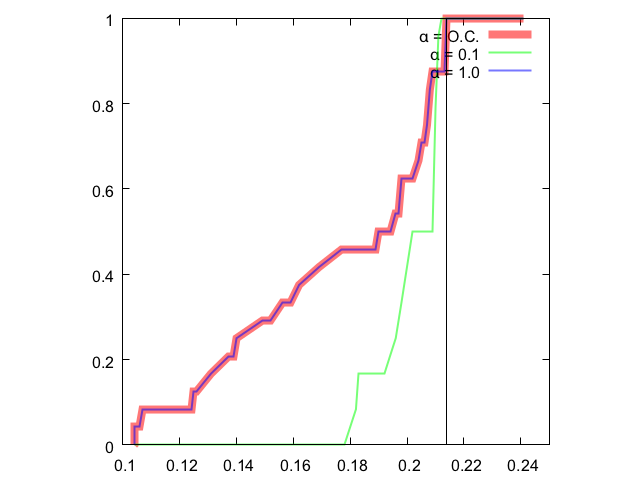}
  \hspacegss & \hspacegss \includegraphics[trim=90bp 5bp 75bp
15bp,clip,width=\sgg\linewidth]{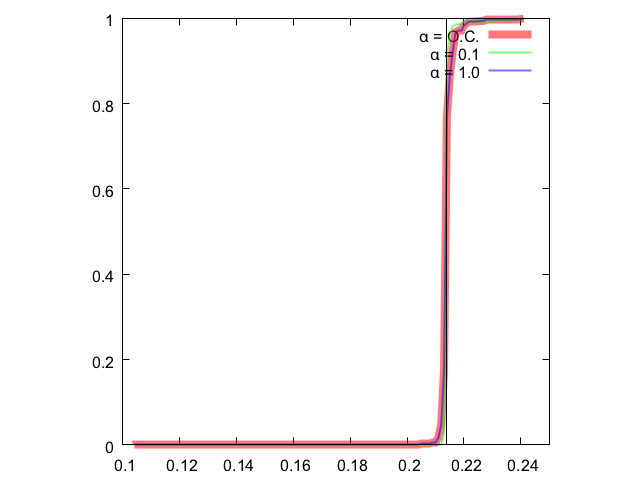} & \includegraphics[trim=90bp 5bp 75bp
15bp,clip,width=\sgg\linewidth]{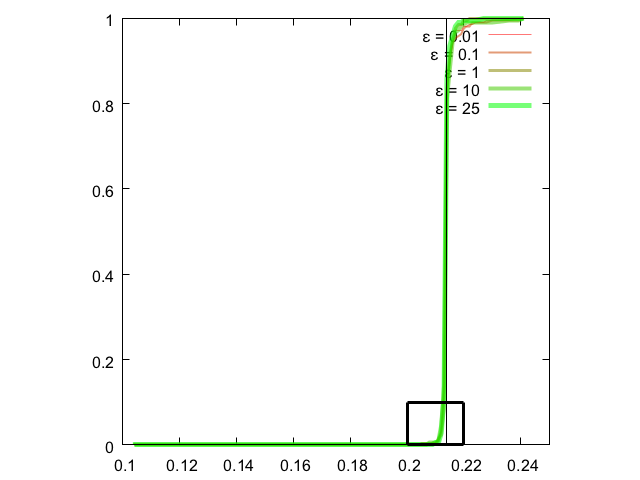} \hspacegss & \hspacegss \includegraphics[trim=90bp 5bp 75bp
15bp,clip,width=\sgg\linewidth]{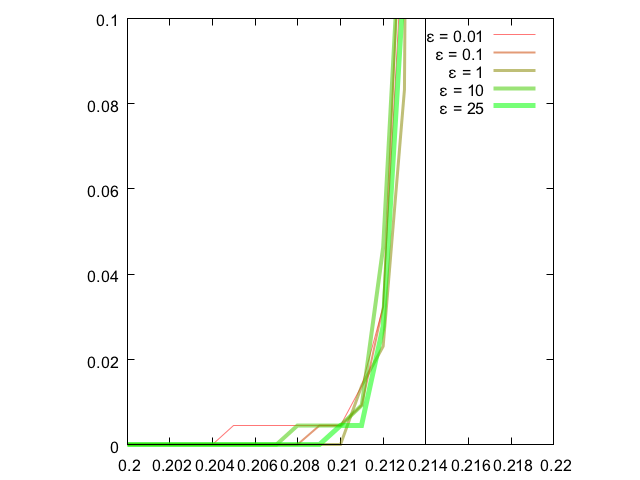} \\ \hline \hline
\end{tabular}
\end{center}
\caption{UCI domain \texttt{creditcardsmall}, conventions identical as in
  Figure \ref{f-s-transfusion2}.}
  \label{f-s-creditcardsmall2}
\end{figure}

\begin{figure}[h]
\begin{center}
\begin{tabular}{ccccc}\hline \hline
$\varepsilon=0.01$ \hspacegs & \hspacegs $\varepsilon=0.1$ \hspacegs & \hspacegs $\varepsilon=1$ \hspacegs & \hspacegs
                                                        $\varepsilon=10$ \hspacegs &
                                                                    \hspacegs $\varepsilon=25$ \hspacegs \\ \hline 
\includegraphics[trim=90bp 5bp 75bp
15bp,clip,width=\sg\linewidth]{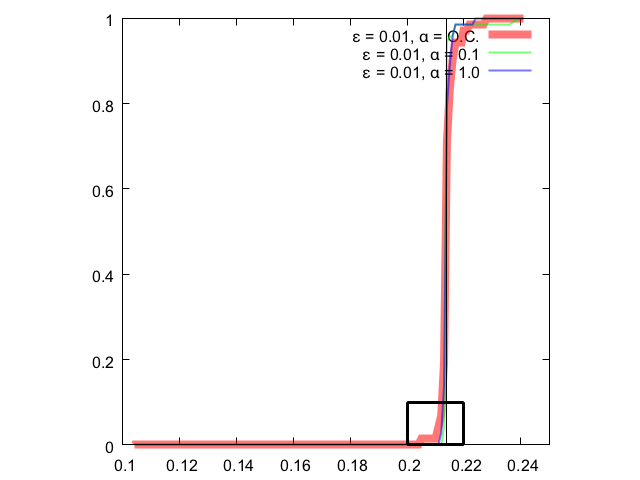}
  \hspacegs & \hspacegs \includegraphics[trim=90bp 5bp 75bp
15bp,clip,width=\sg\linewidth]{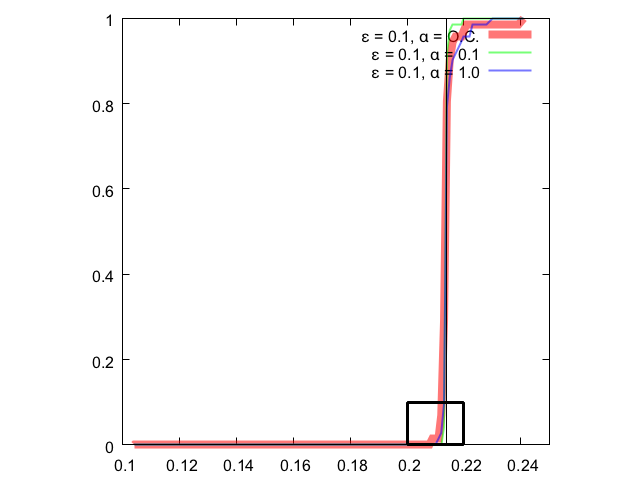}  \hspacegs & \hspacegs \includegraphics[trim=90bp 5bp 75bp
15bp,clip,width=\sg\linewidth]{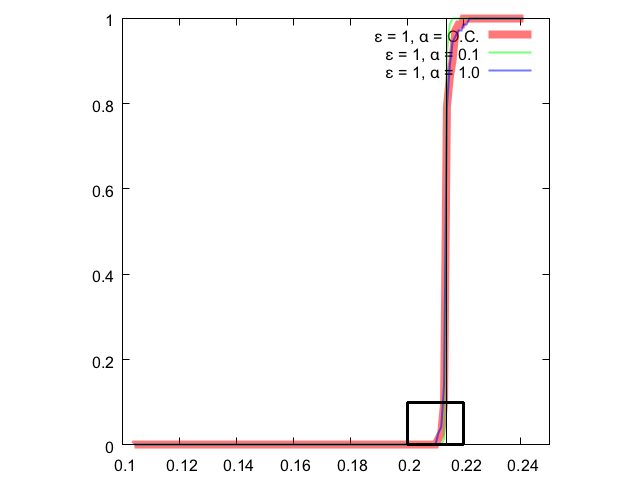}  \hspacegs & \hspacegs \includegraphics[trim=90bp 5bp 75bp
15bp,clip,width=\sg\linewidth]{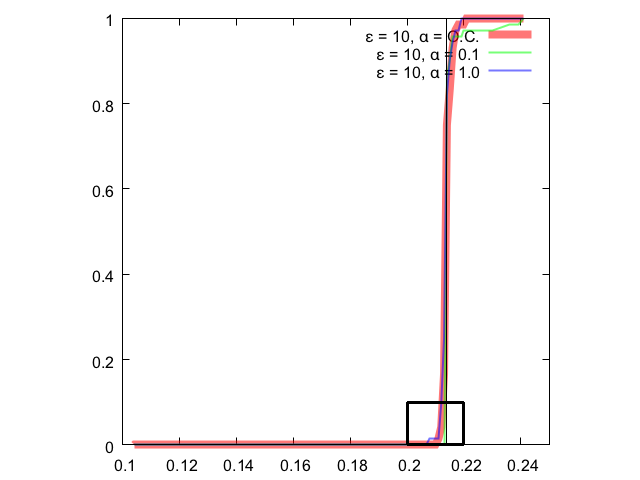}  \hspacegs & \hspacegs \includegraphics[trim=90bp 5bp 75bp
15bp,clip,width=\sg\linewidth]{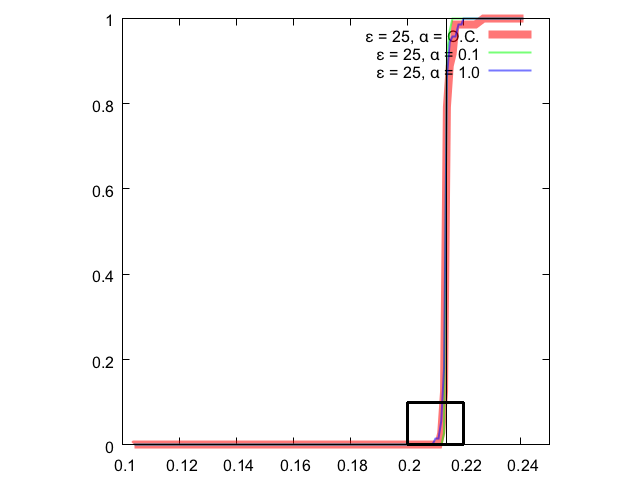}
  \\
\hline \hline
\end{tabular}
\end{center}
\caption{UCI domain \texttt{creditcardsmall}, conventions identical as in
  Figure \ref{f-s-transfusion3}.}
  \label{f-s-creditcardsmall3}
\end{figure}

\clearpage
\newpage
\subsection*{$\triangleright$ UCI \texttt{abalone}}\label{sub-sec-abalone}

\begin{figure}[h]
\begin{center}
\begin{tabular}{cc||cc}\hline \hline
\multicolumn{2}{c||}{without DP} & \multicolumn{2}{c}{with DP}\\ \hline
\pushgraphicsLeafDepth{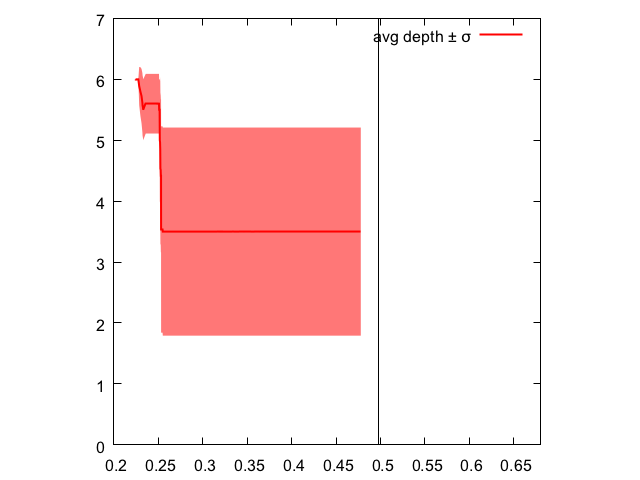}
  \hspacegss & \hspacegss \pushgraphicsLeafDepth{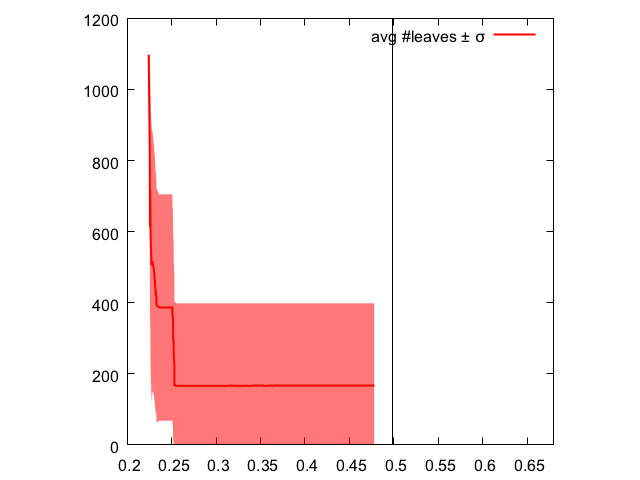}  &  \pushgraphicsLeafDepth{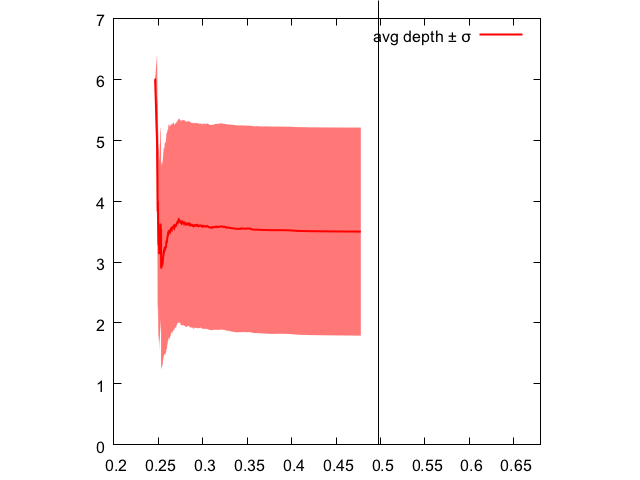}  \hspacegss & \hspacegss  \pushgraphicsLeafDepth{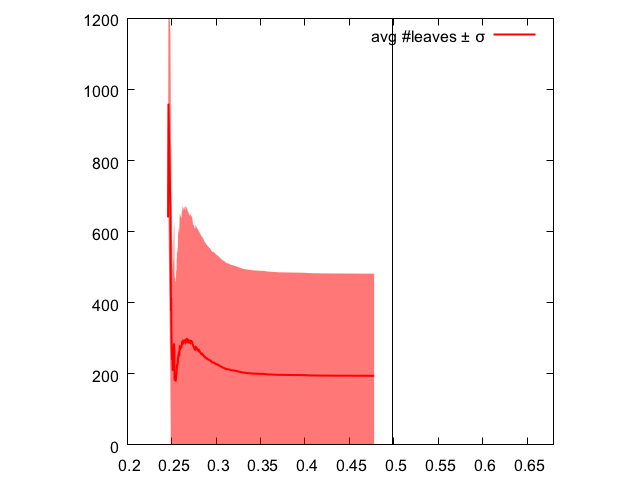}  \\ \hline
 depth \hspacegss & \hspacegss $\#$leaves & depth \hspacegss & \hspacegss $\#$leaves \\ \hline \hline
\end{tabular}
\end{center}
\caption{UCI domain \texttt{abalone}, conventions identical as in
  Figure \ref{f-s-transfusion}. }
  \label{f-s-abalone}
\end{figure}

\begin{figure}[h]
\begin{center}
\begin{tabular}{cc||cc}\hline \hline
\multicolumn{2}{c||}{performances wrt $\alpha$s} & \multicolumn{2}{c}{performances wrt $\epsilon$s (with DP)}\\ \hline
w/o DP \hspacegss & \hspacegss with DP & full
  & crop \\ \hline
\includegraphics[trim=90bp 5bp 75bp
15bp,clip,width=\sgg\linewidth]{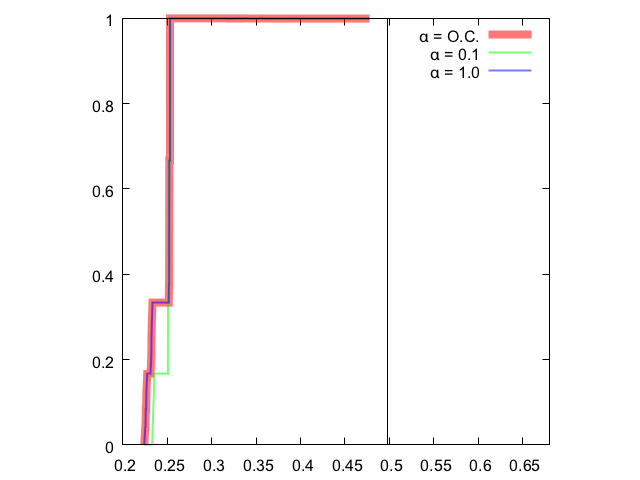}
  \hspacegss & \hspacegss \includegraphics[trim=90bp 5bp 75bp
15bp,clip,width=\sgg\linewidth]{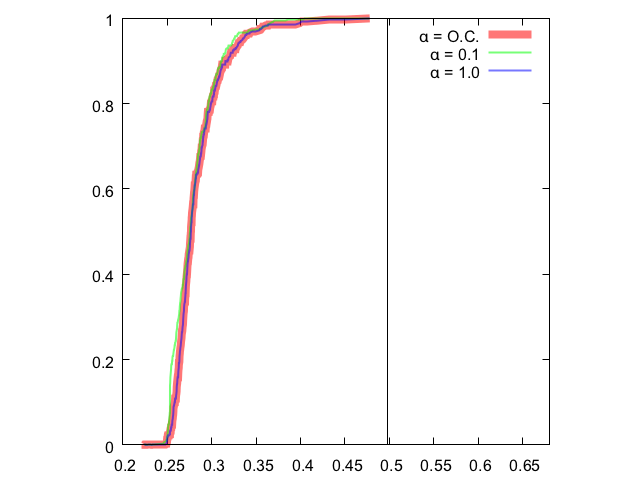} & \includegraphics[trim=90bp 5bp 75bp
15bp,clip,width=\sgg\linewidth]{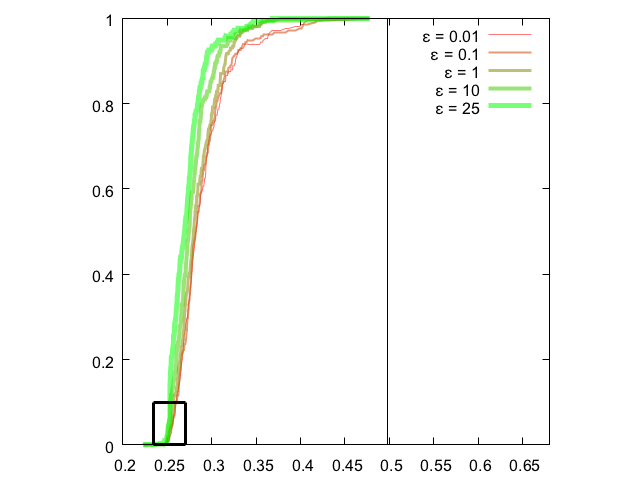} \hspacegss & \hspacegss \includegraphics[trim=90bp 5bp 75bp
15bp,clip,width=\sgg\linewidth]{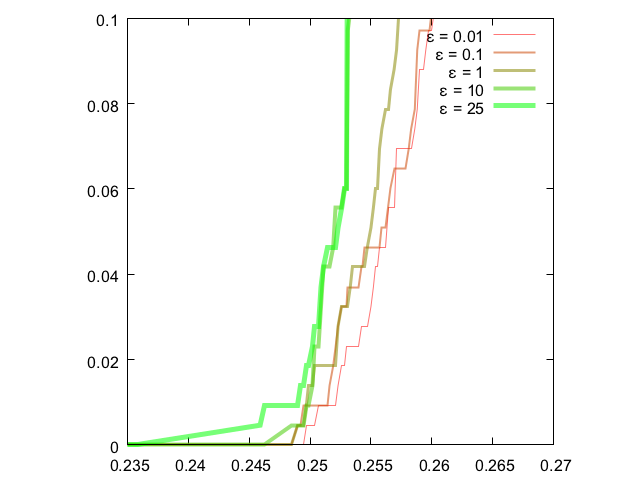} \\ \hline \hline
\end{tabular}
\end{center}
\caption{UCI domain \texttt{abalone}, conventions identical as in
  Figure \ref{f-s-transfusion2}.}
  \label{f-s-abalone2}
\end{figure}

\begin{figure}[h]
\begin{center}
\begin{tabular}{ccccc}\hline \hline
$\varepsilon=0.01$ \hspacegs & \hspacegs $\varepsilon=0.1$ \hspacegs & \hspacegs $\varepsilon=1$ \hspacegs & \hspacegs
                                                        $\varepsilon=10$ \hspacegs &
                                                                    \hspacegs $\varepsilon=25$ \hspacegs \\ \hline 
\includegraphics[trim=90bp 5bp 75bp
15bp,clip,width=\sg\linewidth]{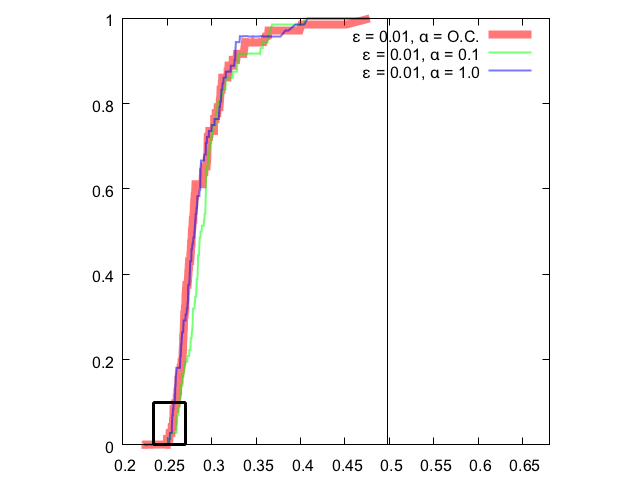}
  \hspacegs & \hspacegs \includegraphics[trim=90bp 5bp 75bp
15bp,clip,width=\sg\linewidth]{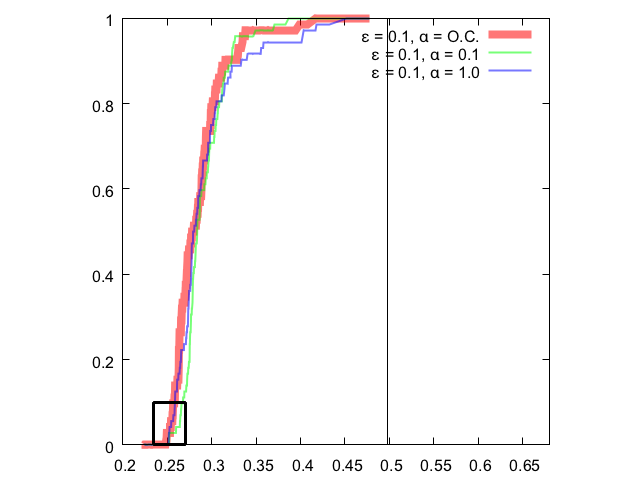}  \hspacegs & \hspacegs \includegraphics[trim=90bp 5bp 75bp
15bp,clip,width=\sg\linewidth]{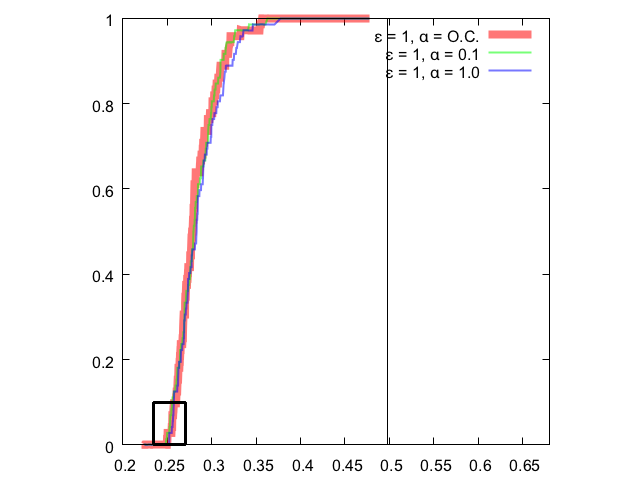}  \hspacegs & \hspacegs \includegraphics[trim=90bp 5bp 75bp
15bp,clip,width=\sg\linewidth]{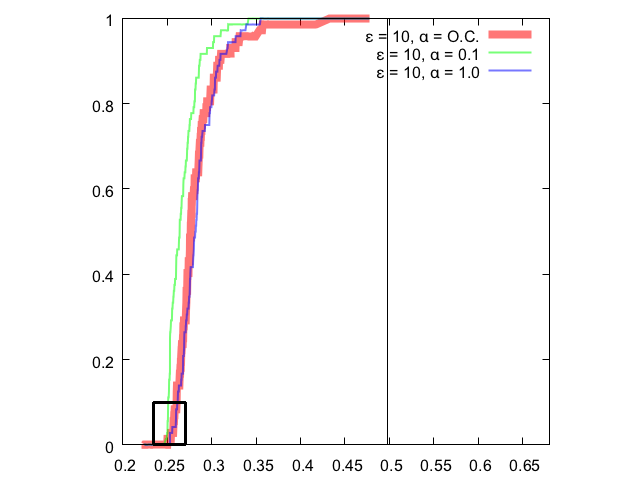}  \hspacegs & \hspacegs \includegraphics[trim=90bp 5bp 75bp
15bp,clip,width=\sg\linewidth]{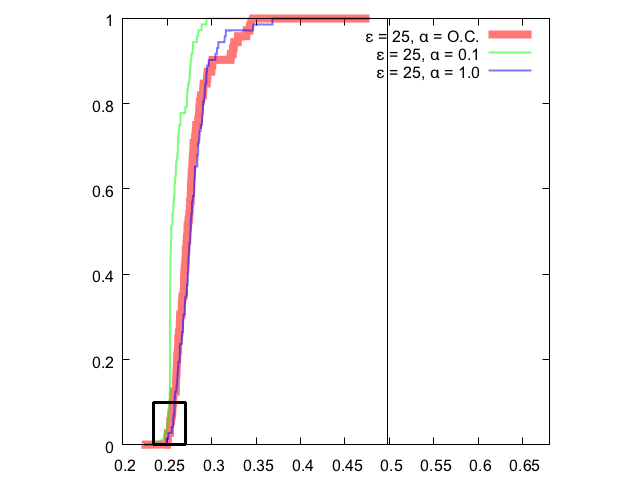}
  \\
\hline \hline
\end{tabular}
\end{center}
\caption{UCI domain \texttt{abalone}, conventions identical as in
  Figure \ref{f-s-transfusion3}.}
  \label{f-s-abalone3}
\end{figure}

\clearpage
\newpage
\subsection*{$\triangleright$ UCI \texttt{qsar}}\label{sub-sec-qsar}

\begin{figure}[h]
\begin{center}
\begin{tabular}{cc||cc}\hline \hline
\multicolumn{2}{c||}{without DP} & \multicolumn{2}{c}{with DP}\\ \hline
\pushgraphicsLeafDepth{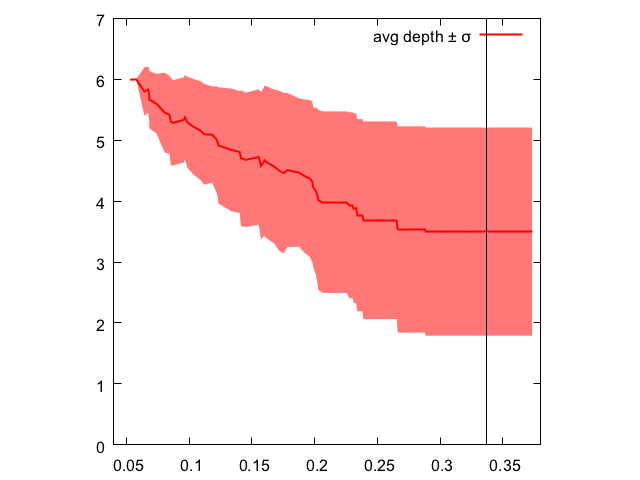}
  \hspacegss & \hspacegss \pushgraphicsLeafDepth{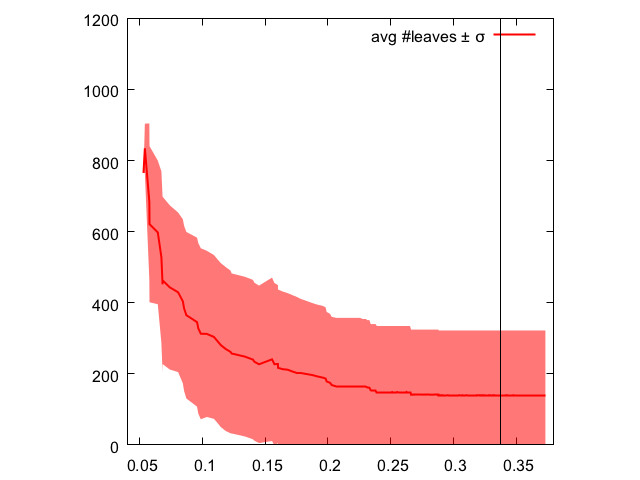}  &  \pushgraphicsLeafDepth{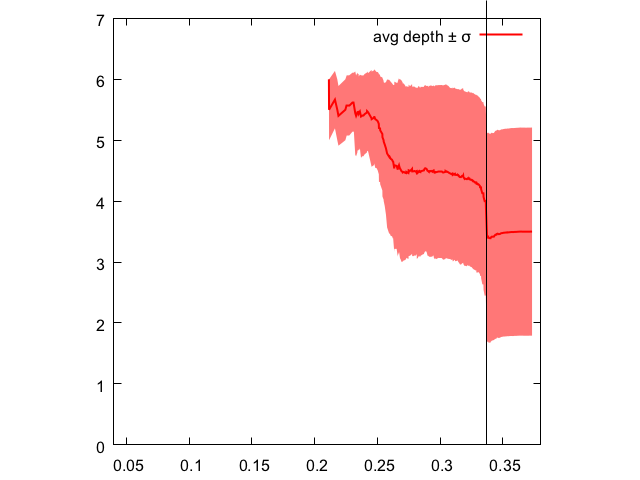}  \hspacegss & \hspacegss  \pushgraphicsLeafDepth{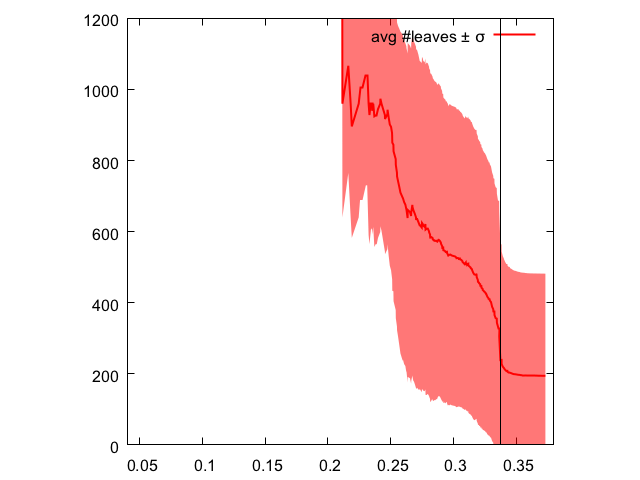}  \\ \hline
 depth \hspacegss & \hspacegss $\#$leaves & depth \hspacegss & \hspacegss $\#$leaves \\ \hline \hline
\end{tabular}
\end{center}
\caption{UCI domain \texttt{qsar}, conventions identical as in
  Figure \ref{f-s-transfusion}. }
  \label{f-s-qsar}
\end{figure}

\begin{figure}[h]
\begin{center}
\begin{tabular}{cc||cc}\hline \hline
\multicolumn{2}{c||}{performances wrt $\alpha$s} & \multicolumn{2}{c}{performances wrt $\epsilon$s (with DP)}\\ \hline
w/o DP \hspacegss & \hspacegss with DP & full
  & crop \\ \hline
\includegraphics[trim=90bp 5bp 75bp
15bp,clip,width=\sgg\linewidth]{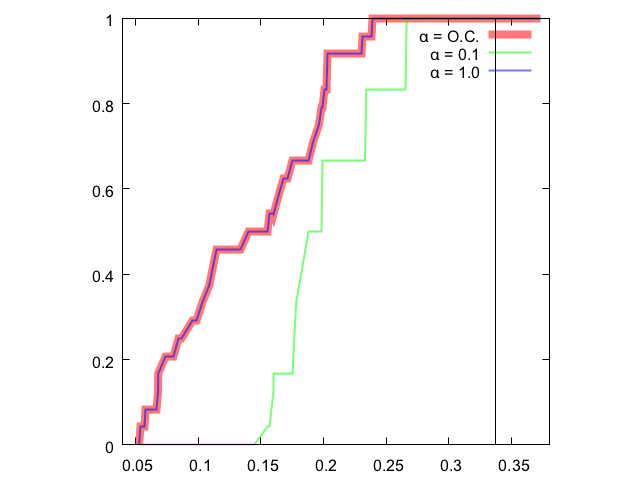}
  \hspacegss & \hspacegss \includegraphics[trim=90bp 5bp 75bp
15bp,clip,width=\sgg\linewidth]{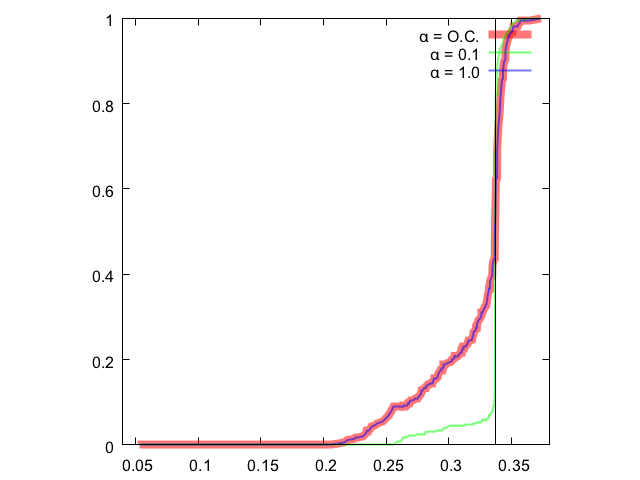} & \includegraphics[trim=90bp 5bp 75bp
15bp,clip,width=\sgg\linewidth]{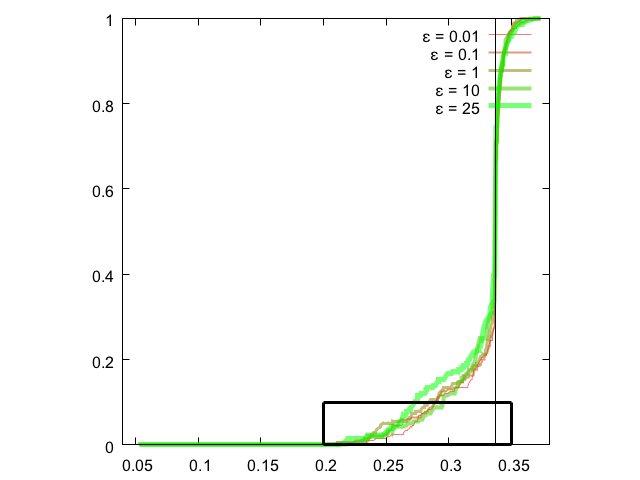} \hspacegss & \hspacegss \includegraphics[trim=90bp 5bp 75bp
15bp,clip,width=\sgg\linewidth]{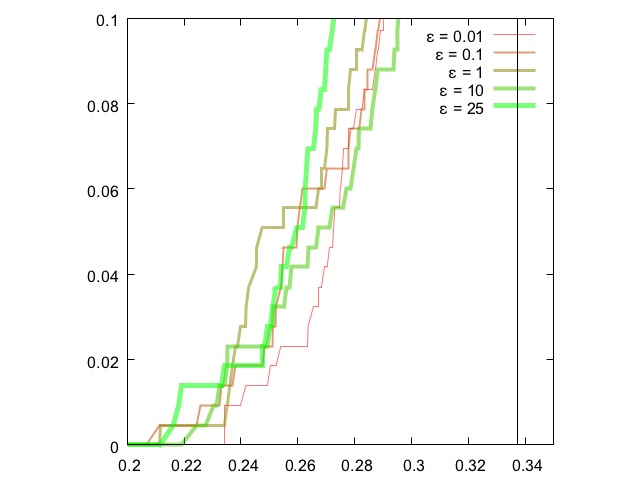} \\ \hline \hline
\end{tabular}
\end{center}
\caption{UCI domain \texttt{qsar}, conventions identical as in
  Figure \ref{f-s-transfusion2}.}
  \label{f-s-qsar2}
\end{figure}

\begin{figure}[h]
\begin{center}
\begin{tabular}{ccccc}\hline \hline
$\varepsilon=0.01$ \hspacegs & \hspacegs $\varepsilon=0.1$ \hspacegs & \hspacegs $\varepsilon=1$ \hspacegs & \hspacegs
                                                        $\varepsilon=10$ \hspacegs &
                                                                    \hspacegs $\varepsilon=25$ \hspacegs \\ \hline 
\includegraphics[trim=90bp 5bp 75bp
15bp,clip,width=\sg\linewidth]{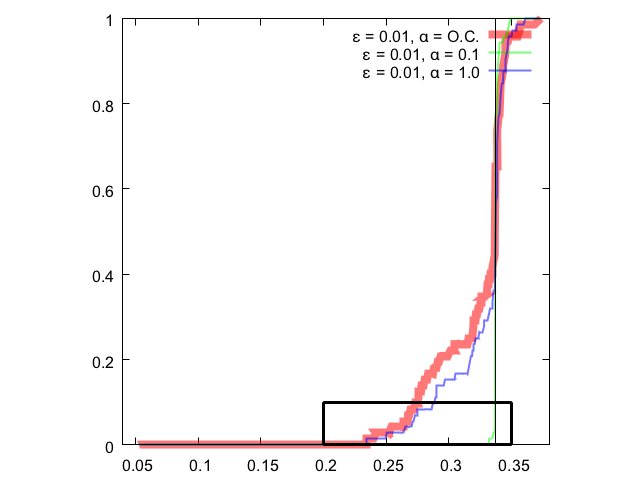}
  \hspacegs & \hspacegs \includegraphics[trim=90bp 5bp 75bp
15bp,clip,width=\sg\linewidth]{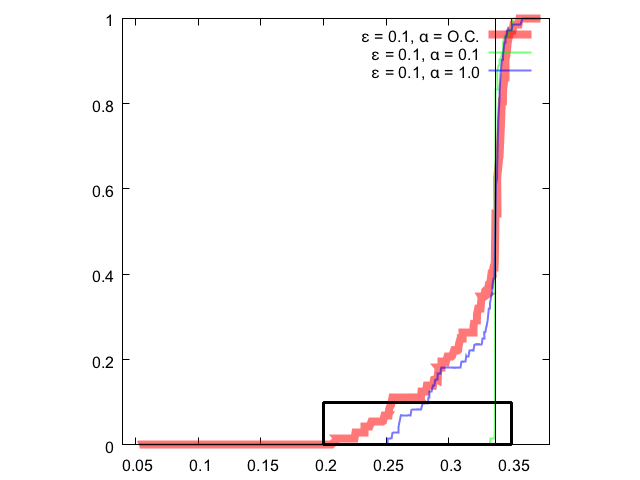}  \hspacegs & \hspacegs \includegraphics[trim=90bp 5bp 75bp
15bp,clip,width=\sg\linewidth]{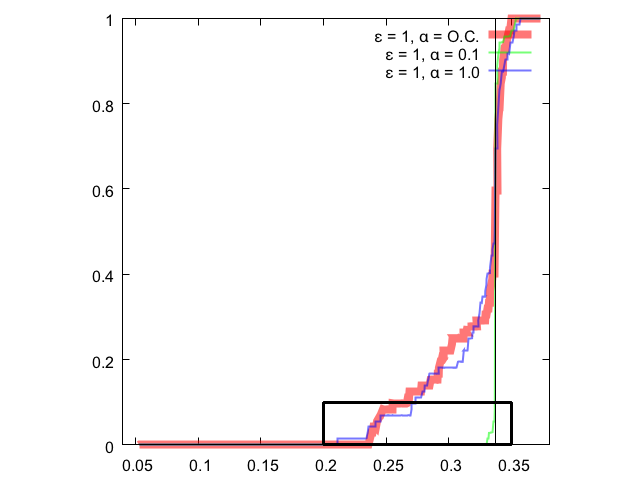}  \hspacegs & \hspacegs \includegraphics[trim=90bp 5bp 75bp
15bp,clip,width=\sg\linewidth]{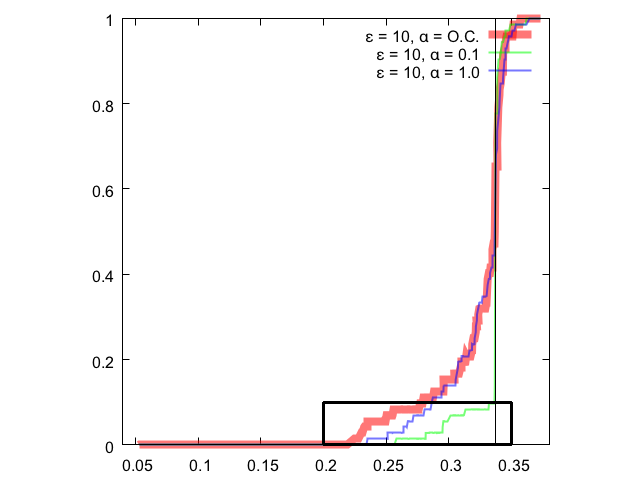}  \hspacegs & \hspacegs \includegraphics[trim=90bp 5bp 75bp
15bp,clip,width=\sg\linewidth]{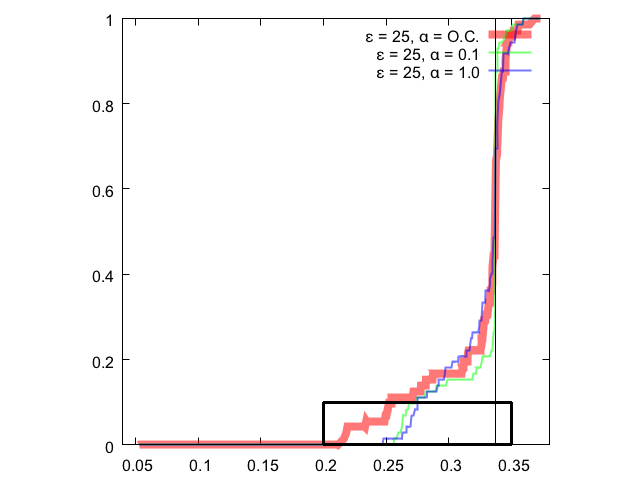}
  \\
\hline \hline
\end{tabular}
\end{center}
\caption{UCI domain \texttt{qsar}, conventions identical as in
  Figure \ref{f-s-transfusion3}.}
  \label{f-s-qsar3}
\end{figure}

\clearpage
\newpage
\subsection*{$\triangleright$ UCI \texttt{page}}\label{sub-sec-page}

\begin{figure}[h]
\begin{center}
\begin{tabular}{cc||cc}\hline \hline
\multicolumn{2}{c||}{without DP} & \multicolumn{2}{c}{with DP}\\ \hline
\pushgraphicsLeafDepth{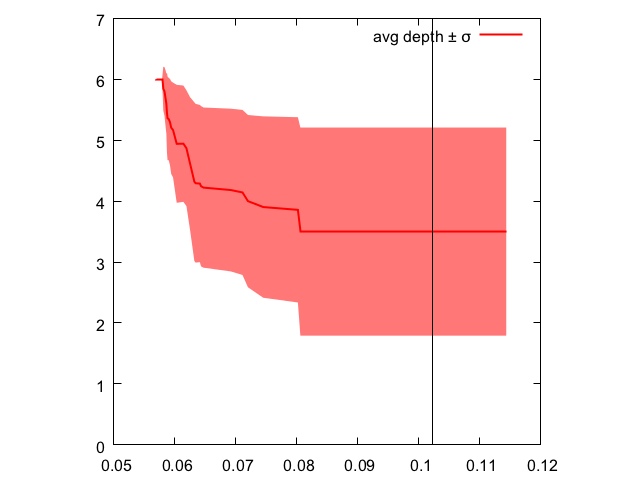}
  \hspacegss & \hspacegss \pushgraphicsLeafDepth{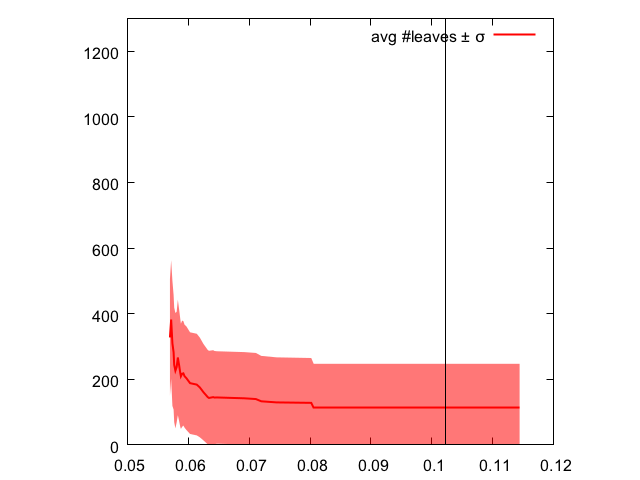}  &  \pushgraphicsLeafDepth{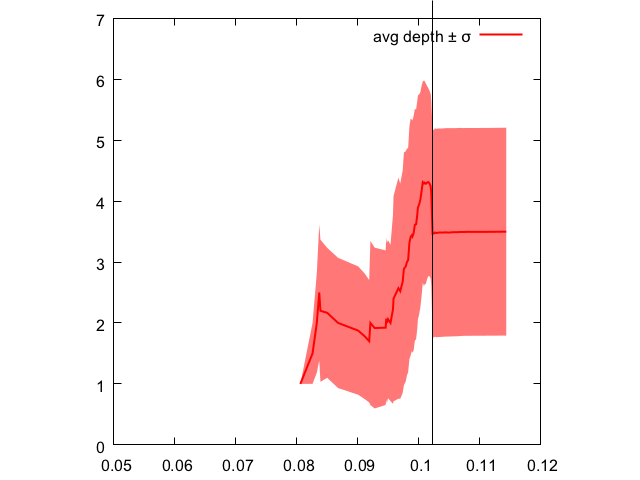}  \hspacegss & \hspacegss  \pushgraphicsLeafDepth{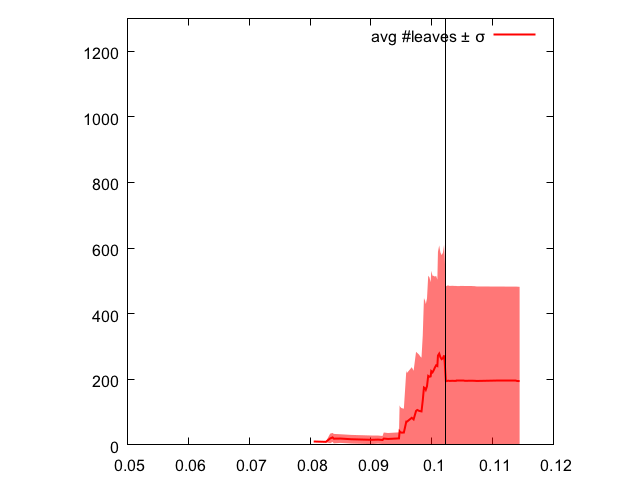}  \\ \hline
 depth \hspacegss & \hspacegss $\#$leaves & depth \hspacegss & \hspacegss $\#$leaves \\ \hline \hline
\end{tabular}
\end{center}
\caption{UCI domain \texttt{page}, conventions identical as in
  Figure \ref{f-s-transfusion}. }
  \label{f-s-page}
\end{figure}

\begin{figure}[h]
\begin{center}
\begin{tabular}{cc||cc}\hline \hline
\multicolumn{2}{c||}{performances wrt $\alpha$s} & \multicolumn{2}{c}{performances wrt $\epsilon$s (with DP)}\\ \hline
w/o DP \hspacegss & \hspacegss with DP & full
  & crop \\ \hline
\includegraphics[trim=90bp 5bp 75bp
15bp,clip,width=\sgg\linewidth]{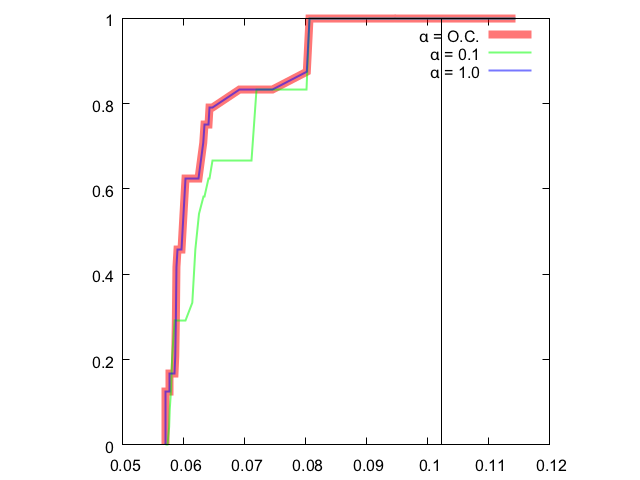}
  \hspacegss & \hspacegss \includegraphics[trim=90bp 5bp 75bp
15bp,clip,width=\sgg\linewidth]{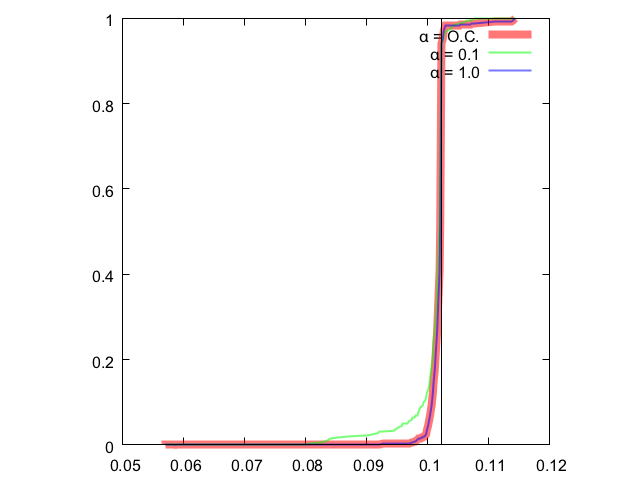} & \includegraphics[trim=90bp 5bp 75bp
15bp,clip,width=\sgg\linewidth]{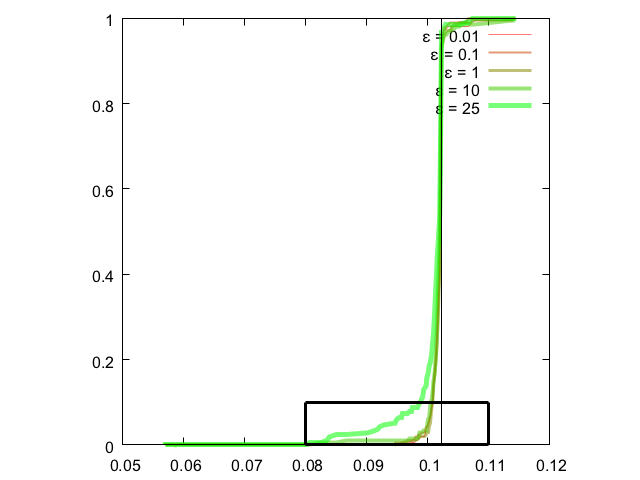} \hspacegss & \hspacegss \includegraphics[trim=90bp 5bp 75bp
15bp,clip,width=\sgg\linewidth]{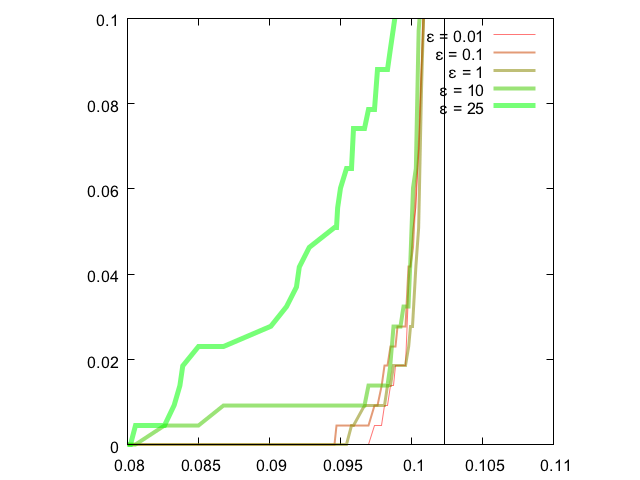} \\ \hline \hline
\end{tabular}
\end{center}
\caption{UCI domain \texttt{page}, conventions identical as in
  Figure \ref{f-s-transfusion2}.}
  \label{f-s-page2}
\end{figure}

\begin{figure}[h]
\begin{center}
\begin{tabular}{ccccc}\hline \hline
$\varepsilon=0.01$ \hspacegs & \hspacegs $\varepsilon=0.1$ \hspacegs & \hspacegs $\varepsilon=1$ \hspacegs & \hspacegs
                                                        $\varepsilon=10$ \hspacegs &
                                                                    \hspacegs $\varepsilon=25$ \hspacegs \\ \hline 
\includegraphics[trim=90bp 5bp 75bp
15bp,clip,width=\sg\linewidth]{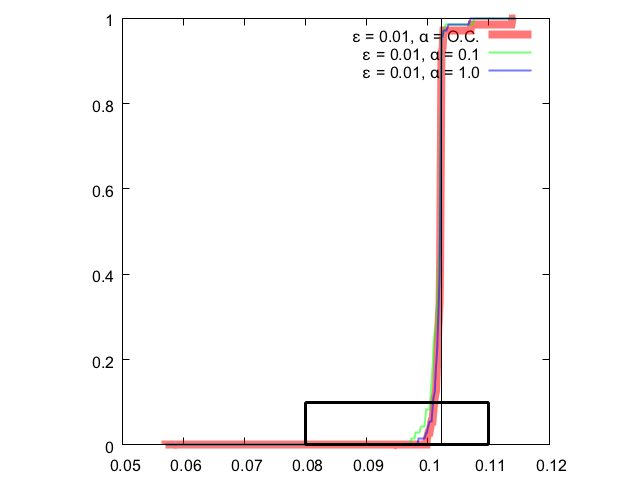}
  \hspacegs & \hspacegs \includegraphics[trim=90bp 5bp 75bp
15bp,clip,width=\sg\linewidth]{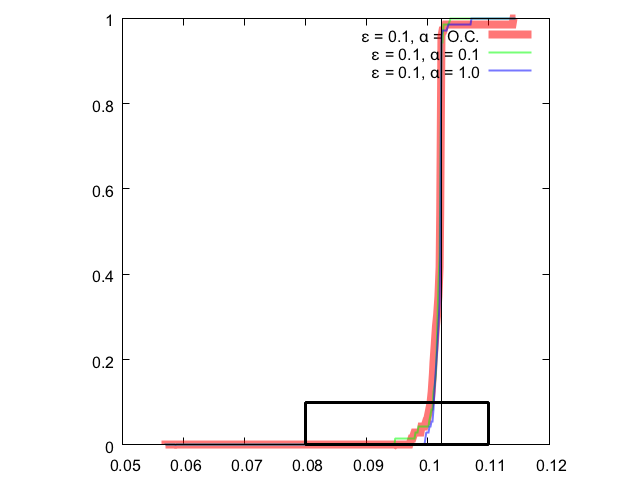}  \hspacegs & \hspacegs \includegraphics[trim=90bp 5bp 75bp
15bp,clip,width=\sg\linewidth]{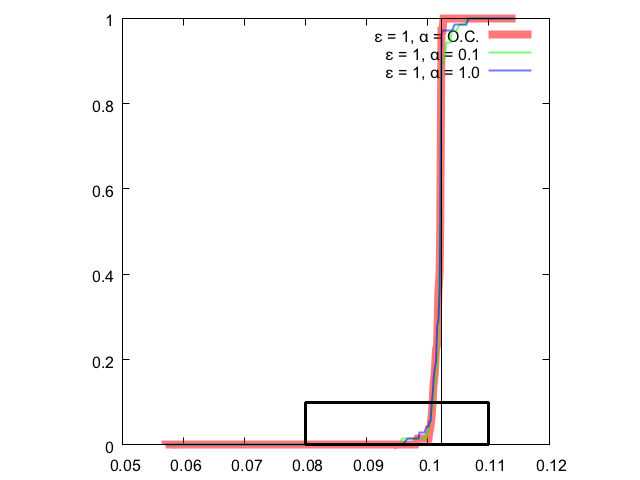}  \hspacegs & \hspacegs \includegraphics[trim=90bp 5bp 75bp
15bp,clip,width=\sg\linewidth]{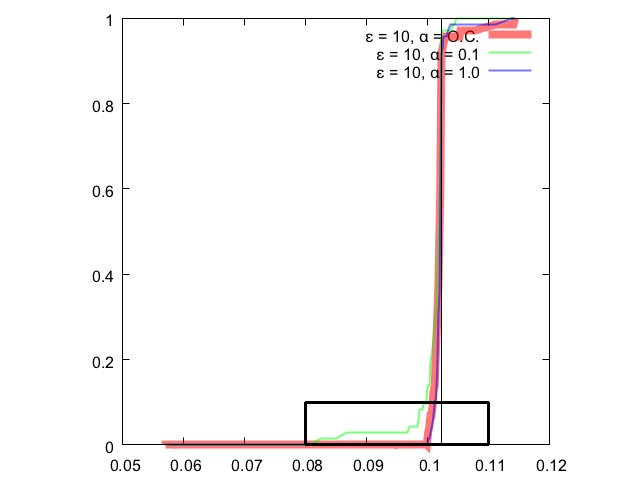}  \hspacegs & \hspacegs \includegraphics[trim=90bp 5bp 75bp
15bp,clip,width=\sg\linewidth]{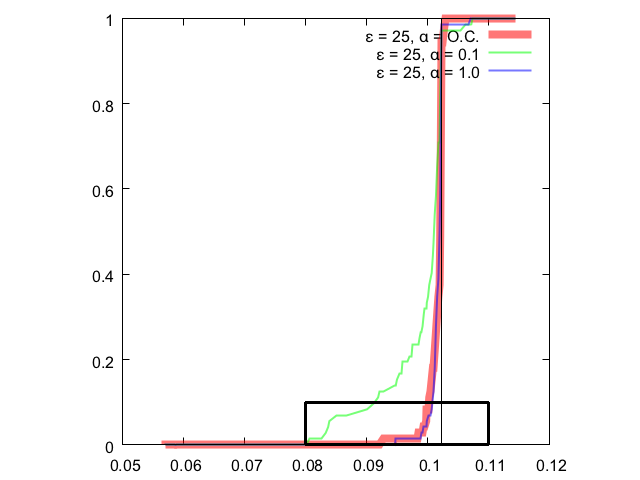}
  \\
\hline \hline
\end{tabular}
\end{center}
\caption{UCI domain \texttt{page}, conventions identical as in
  Figure \ref{f-s-transfusion3}.}
  \label{f-s-page3}
\end{figure}

\clearpage
\newpage
\subsection*{$\triangleright$ UCI \texttt{mice}}\label{sub-sec-mice}

\begin{figure}[h]
\begin{center}
\begin{tabular}{cc||cc}\hline \hline
\multicolumn{2}{c||}{without DP} & \multicolumn{2}{c}{with DP}\\ \hline
\pushgraphicsLeafDepth{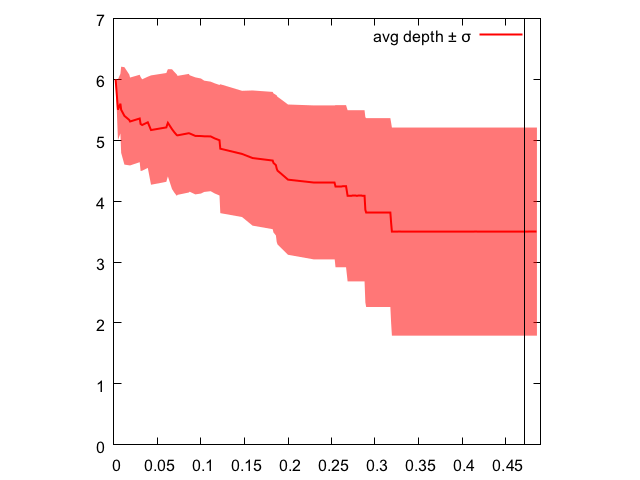}
  \hspacegss & \hspacegss \pushgraphicsLeafDepth{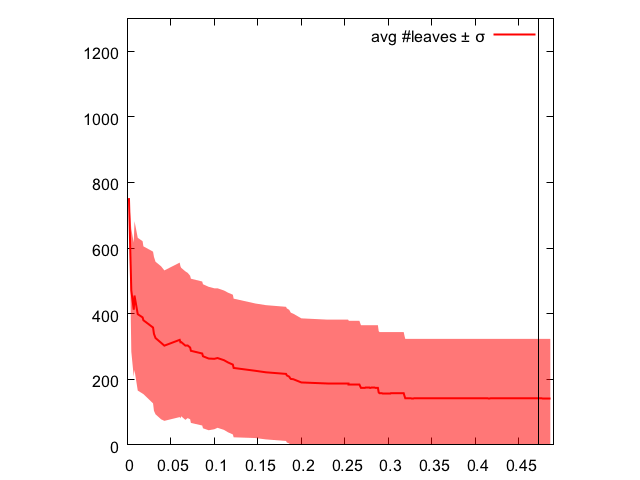}  &  \pushgraphicsLeafDepth{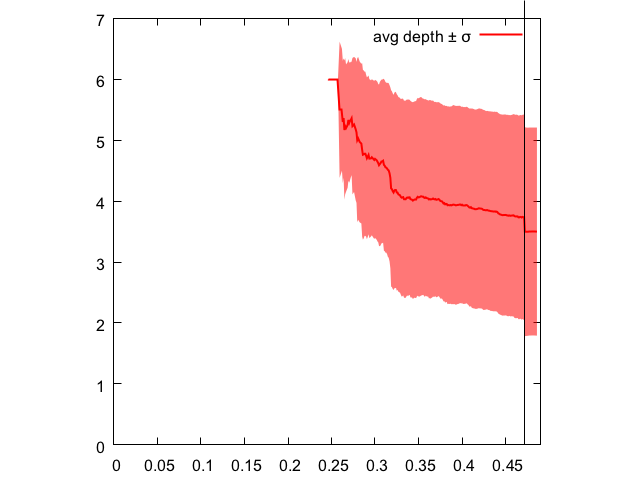}  \hspacegss & \hspacegss  \pushgraphicsLeafDepth{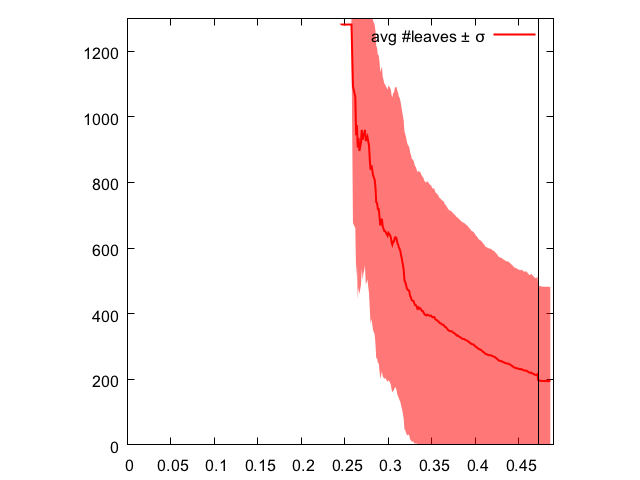}  \\ \hline
 depth \hspacegss & \hspacegss $\#$leaves & depth \hspacegss & \hspacegss $\#$leaves \\ \hline \hline
\end{tabular}
\end{center}
\caption{UCI domain \texttt{mice}, conventions identical as in
  Figure \ref{f-s-transfusion}. }
  \label{f-s-mice}
\end{figure}

\begin{figure}[h]
\begin{center}
\begin{tabular}{cc||cc}\hline \hline
\multicolumn{2}{c||}{performances wrt $\alpha$s} & \multicolumn{2}{c}{performances wrt $\epsilon$s (with DP)}\\ \hline
w/o DP \hspacegss & \hspacegss with DP & full
  & crop \\ \hline
\includegraphics[trim=90bp 5bp 75bp
15bp,clip,width=\sgg\linewidth]{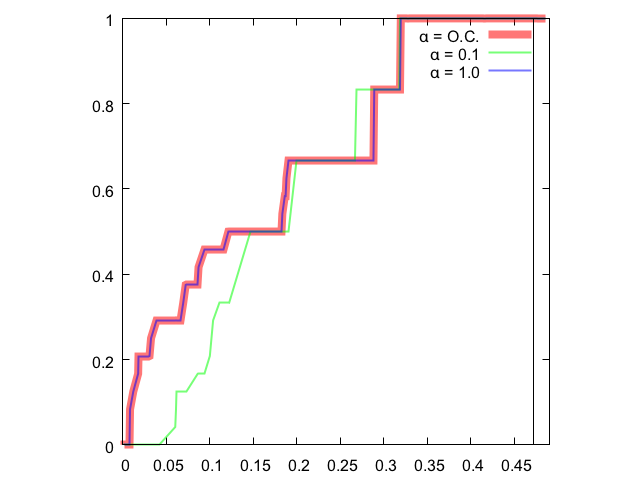}
  \hspacegss & \hspacegss \includegraphics[trim=90bp 5bp 75bp
15bp,clip,width=\sgg\linewidth]{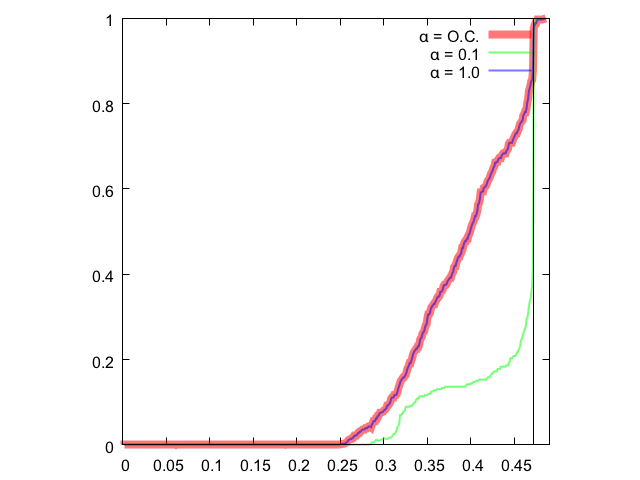} & \includegraphics[trim=90bp 5bp 75bp
15bp,clip,width=\sgg\linewidth]{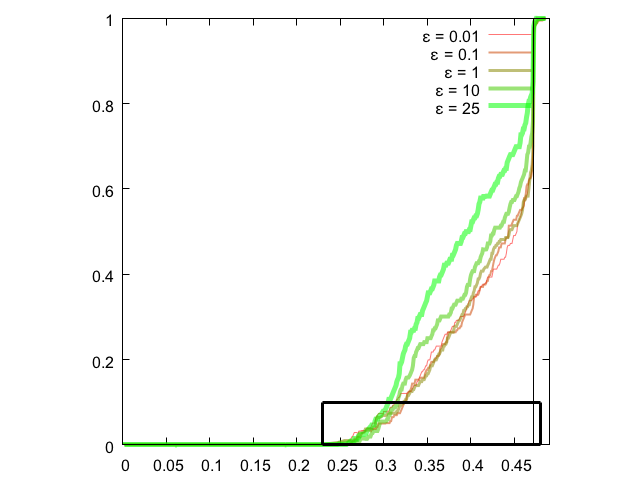} \hspacegss & \hspacegss \includegraphics[trim=90bp 5bp 75bp
15bp,clip,width=\sgg\linewidth]{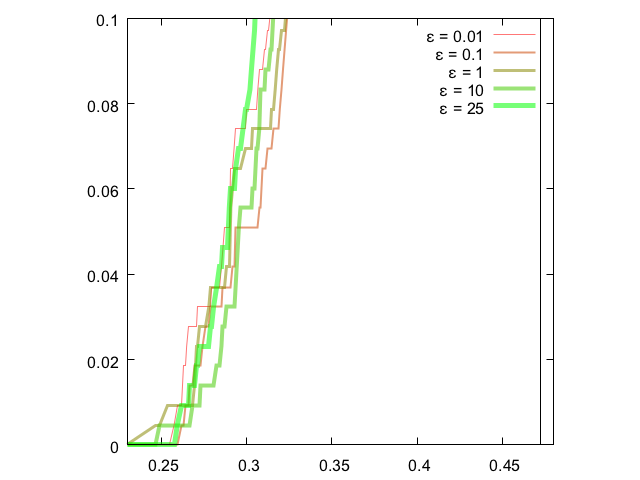} \\ \hline \hline
\end{tabular}
\end{center}
\caption{UCI domain \texttt{mice}, conventions identical as in
  Figure \ref{f-s-transfusion2}.}
  \label{f-s-mice2}
\end{figure}

\begin{figure}[h]
\begin{center}
\begin{tabular}{ccccc}\hline \hline
$\varepsilon=0.01$ \hspacegs & \hspacegs $\varepsilon=0.1$ \hspacegs & \hspacegs $\varepsilon=1$ \hspacegs & \hspacegs
                                                        $\varepsilon=10$ \hspacegs &
                                                                    \hspacegs $\varepsilon=25$ \hspacegs \\ \hline 
\includegraphics[trim=90bp 5bp 75bp
15bp,clip,width=\sg\linewidth]{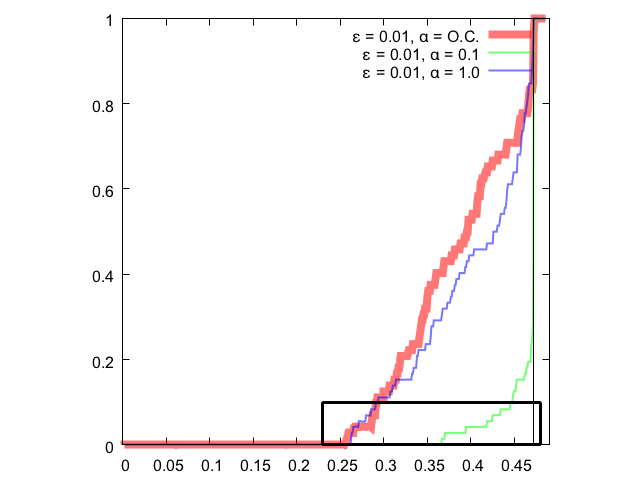}
  \hspacegs & \hspacegs \includegraphics[trim=90bp 5bp 75bp
15bp,clip,width=\sg\linewidth]{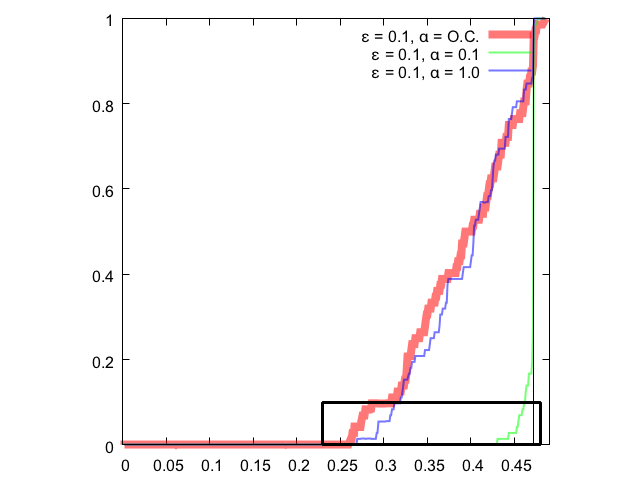}  \hspacegs & \hspacegs \includegraphics[trim=90bp 5bp 75bp
15bp,clip,width=\sg\linewidth]{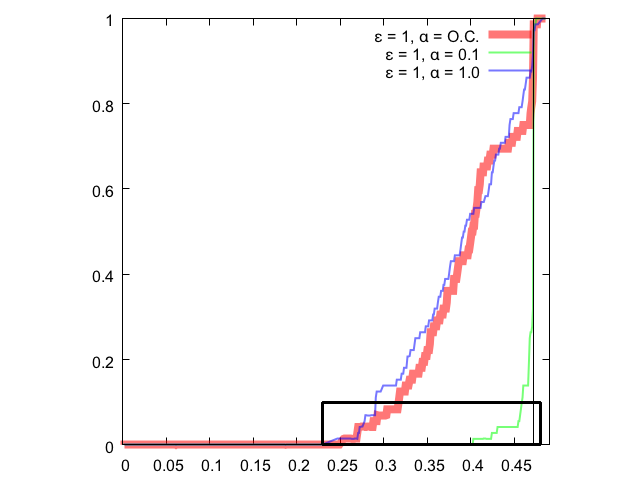}  \hspacegs & \hspacegs \includegraphics[trim=90bp 5bp 75bp
15bp,clip,width=\sg\linewidth]{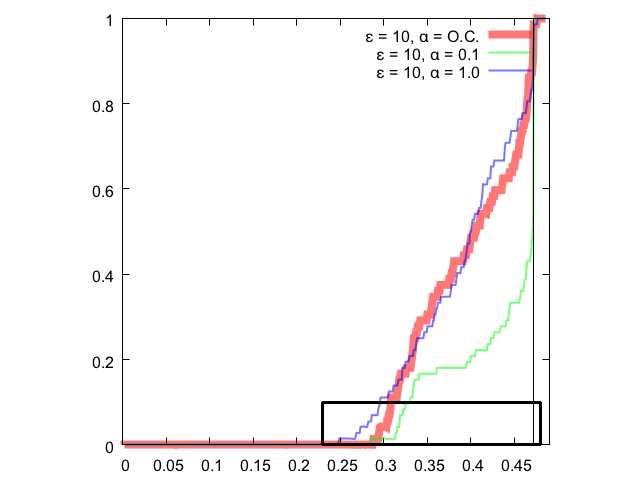}  \hspacegs & \hspacegs \includegraphics[trim=90bp 5bp 75bp
15bp,clip,width=\sg\linewidth]{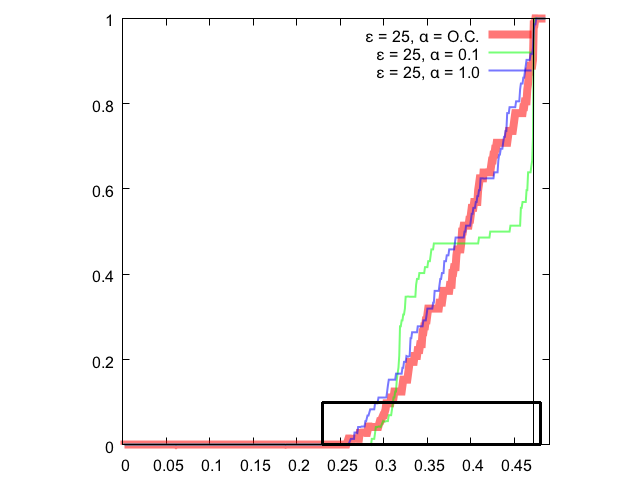}
  \\
\hline \hline
\end{tabular}
\end{center}
\caption{UCI domain \texttt{mice}, conventions identical as in
  Figure \ref{f-s-transfusion3}.}
  \label{f-s-mice3}
\end{figure}

\clearpage
\newpage
\subsection*{$\triangleright$ UCI \texttt{hill$+$noise}}\label{sub-sec-hillnoise}

\begin{figure}[h]
\begin{center}
\begin{tabular}{cc||cc}\hline \hline
\multicolumn{2}{c||}{without DP} & \multicolumn{2}{c}{with DP}\\ \hline
\pushgraphicsLeafDepth{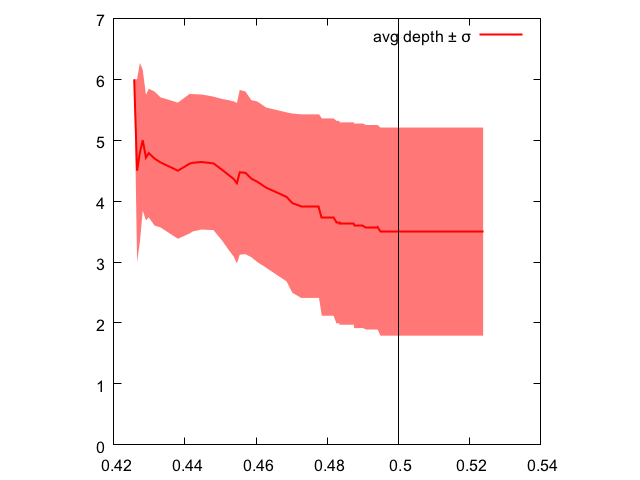}
  \hspacegss & \hspacegss \pushgraphicsLeafDepth{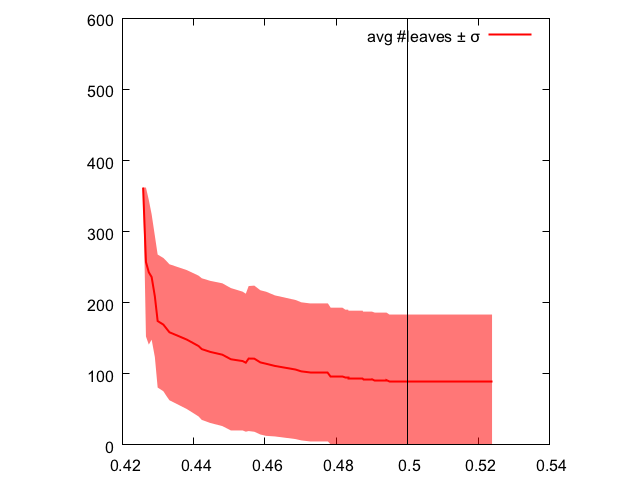}  &  \pushgraphicsLeafDepth{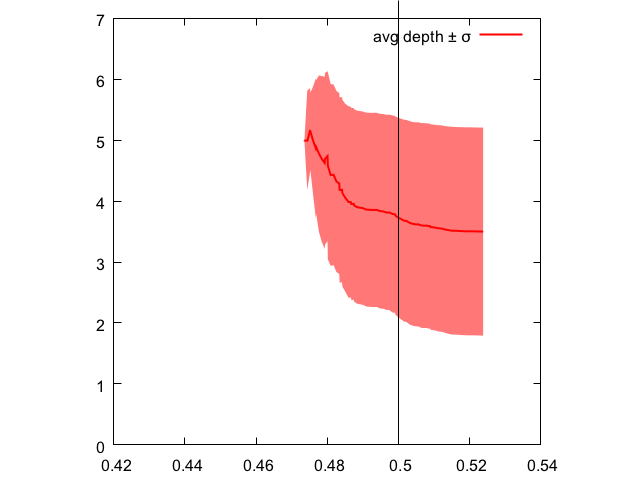}  \hspacegss & \hspacegss  \pushgraphicsLeafDepth{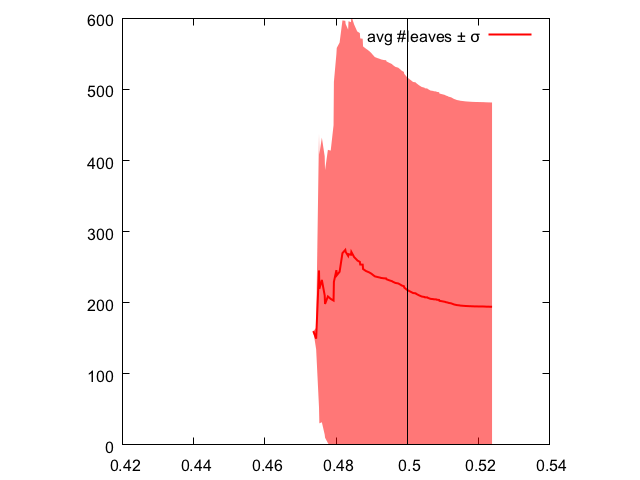}  \\ \hline
 depth \hspacegss & \hspacegss $\#$leaves & depth \hspacegss & \hspacegss $\#$leaves \\ \hline \hline
\end{tabular}
\end{center}
\caption{UCI domain \texttt{hill$+$noise}, conventions identical as in
  Figure \ref{f-s-transfusion}. }
  \label{f-s-hillnoise}
\end{figure}

\begin{figure}[h]
\begin{center}
\begin{tabular}{cc||cc}\hline \hline
\multicolumn{2}{c||}{performances wrt $\alpha$s} & \multicolumn{2}{c}{performances wrt $\epsilon$s (with DP)}\\ \hline
w/o DP \hspacegss & \hspacegss with DP & full
  & crop \\ \hline
\includegraphics[trim=90bp 5bp 75bp
15bp,clip,width=\sgg\linewidth]{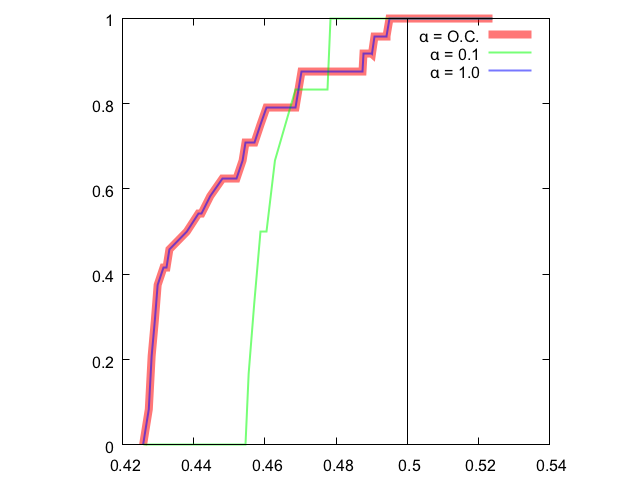}
  \hspacegss & \hspacegss \includegraphics[trim=90bp 5bp 75bp
15bp,clip,width=\sgg\linewidth]{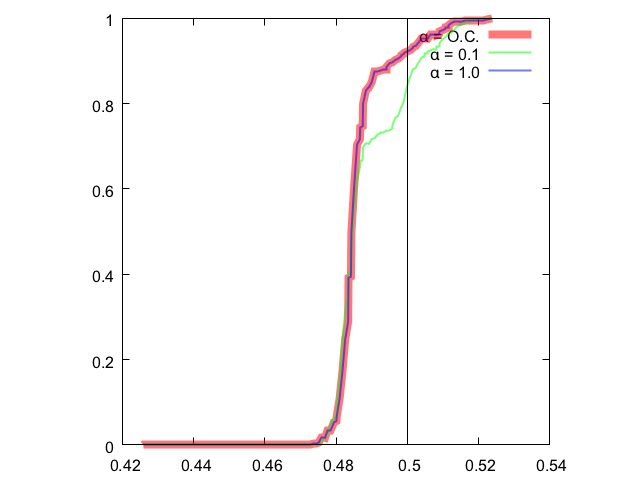} & \includegraphics[trim=90bp 5bp 75bp
15bp,clip,width=\sgg\linewidth]{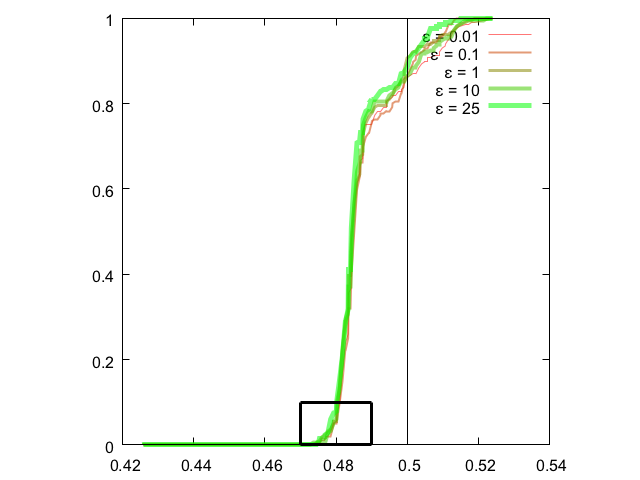} \hspacegss & \hspacegss \includegraphics[trim=90bp 5bp 75bp
15bp,clip,width=\sgg\linewidth]{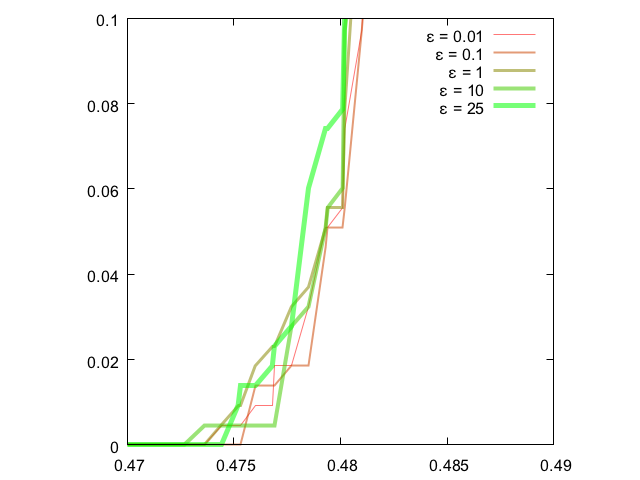} \\ \hline \hline
\end{tabular}
\end{center}
\caption{UCI domain \texttt{hill$+$noise}, conventions identical as in
  Figure \ref{f-s-transfusion2}.}
  \label{f-s-hillnoise2}
\end{figure}

\begin{figure}[h]
\begin{center}
\begin{tabular}{ccccc}\hline \hline
$\varepsilon=0.01$ \hspacegs & \hspacegs $\varepsilon=0.1$ \hspacegs & \hspacegs $\varepsilon=1$ \hspacegs & \hspacegs
                                                        $\varepsilon=10$ \hspacegs &
                                                                    \hspacegs $\varepsilon=25$ \hspacegs \\ \hline 
\includegraphics[trim=90bp 5bp 75bp
15bp,clip,width=\sg\linewidth]{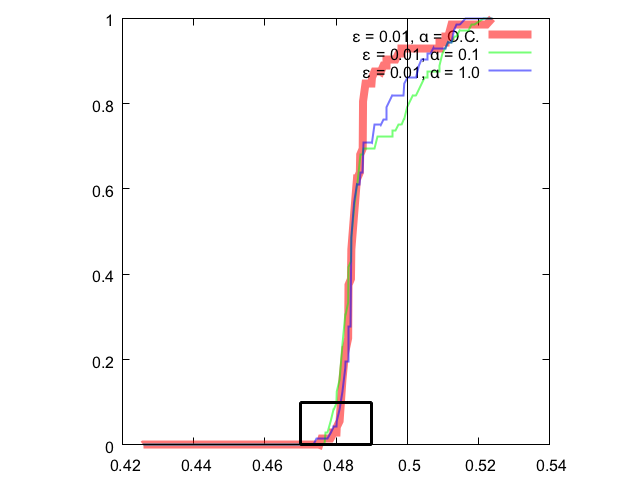}
  \hspacegs & \hspacegs \includegraphics[trim=90bp 5bp 75bp
15bp,clip,width=\sg\linewidth]{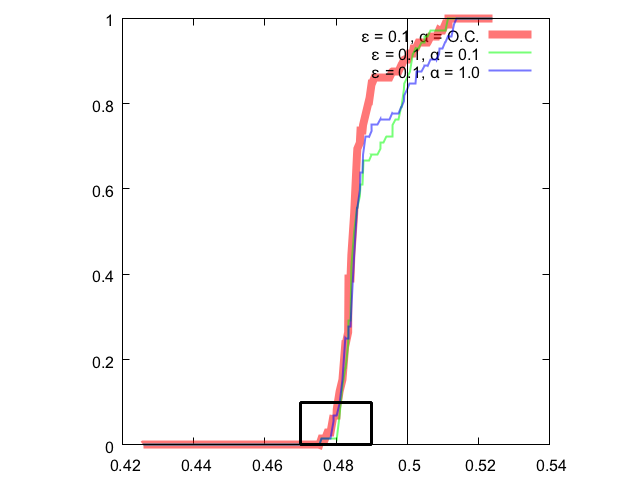}  \hspacegs & \hspacegs \includegraphics[trim=90bp 5bp 75bp
15bp,clip,width=\sg\linewidth]{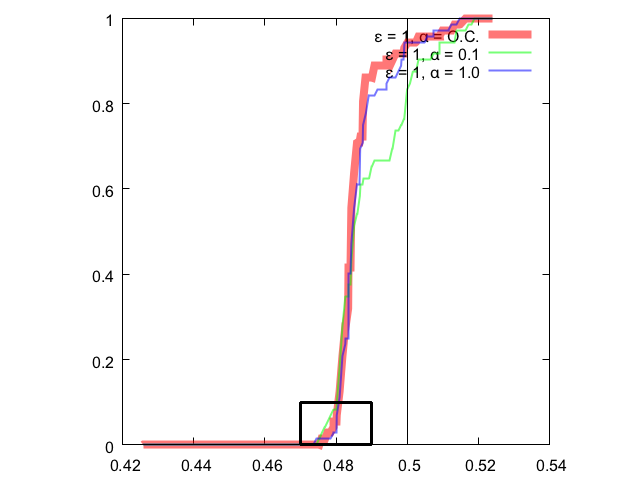}  \hspacegs & \hspacegs \includegraphics[trim=90bp 5bp 75bp
15bp,clip,width=\sg\linewidth]{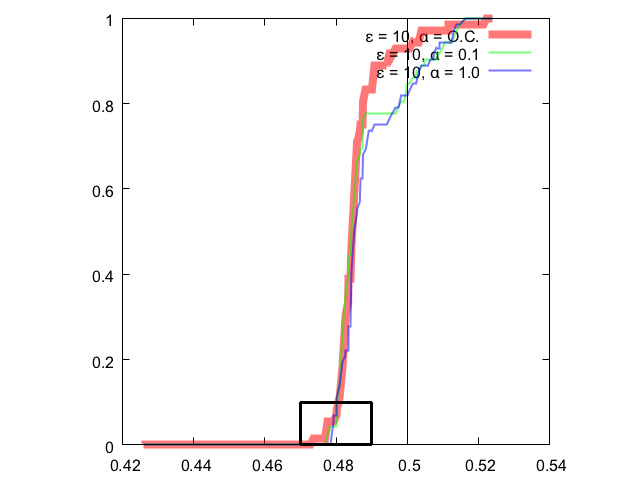}  \hspacegs & \hspacegs \includegraphics[trim=90bp 5bp 75bp
15bp,clip,width=\sg\linewidth]{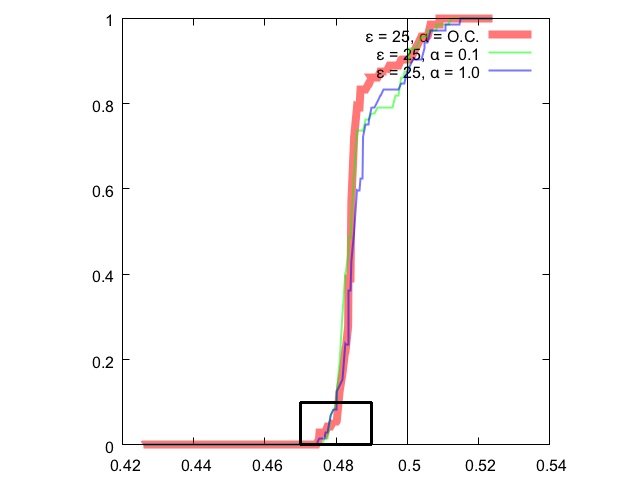}
  \\
\hline \hline
\end{tabular}
\end{center}
\caption{UCI domain \texttt{hill$+$noise}, conventions identical as in
  Figure \ref{f-s-transfusion3}.}
  \label{f-s-hillnoise3}
\end{figure}

\clearpage
\newpage
\subsection*{$\triangleright$ UCI \texttt{hill$+$nonoise}}\label{sub-sec-hillnonoise}

\begin{figure}[h]
\begin{center}
\begin{tabular}{cc||cc}\hline \hline
\multicolumn{2}{c||}{without DP} & \multicolumn{2}{c}{with DP}\\ \hline
\pushgraphicsLeafDepth{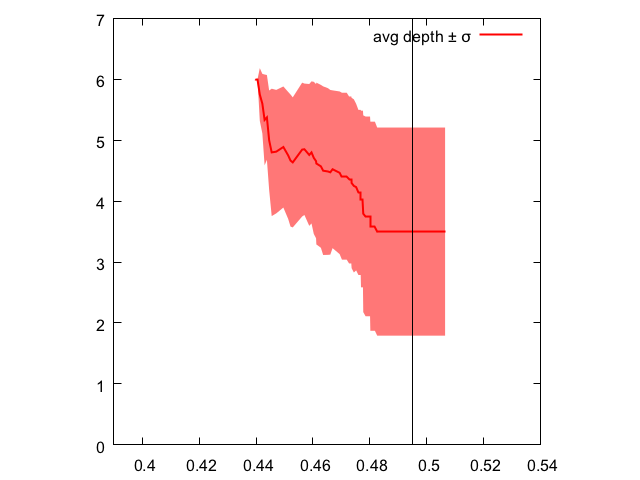}
  \hspacegss & \hspacegss \pushgraphicsLeafDepth{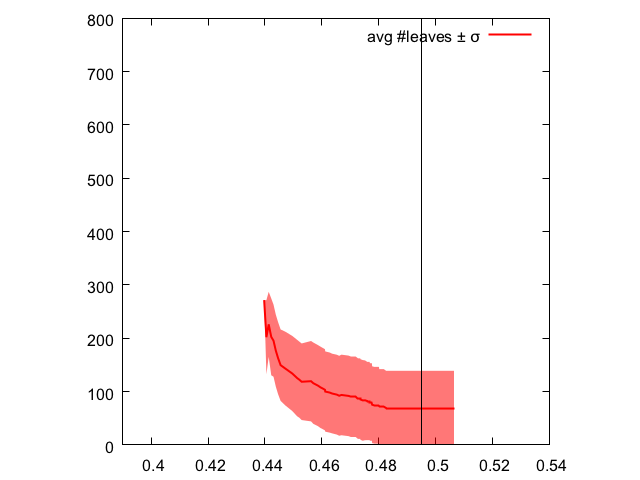}  &  \pushgraphicsLeafDepth{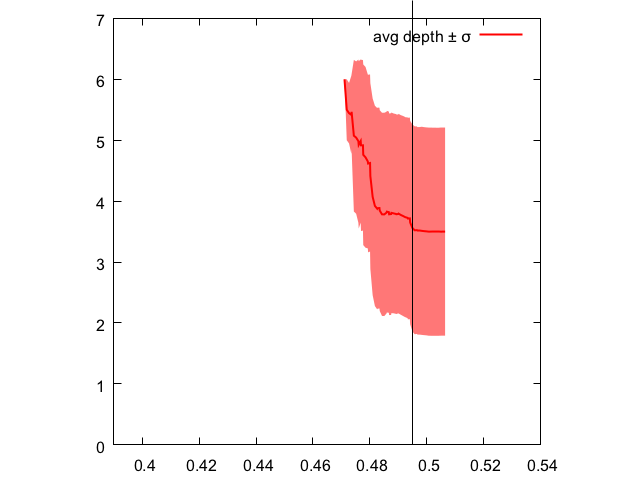}  \hspacegss & \hspacegss  \pushgraphicsLeafDepth{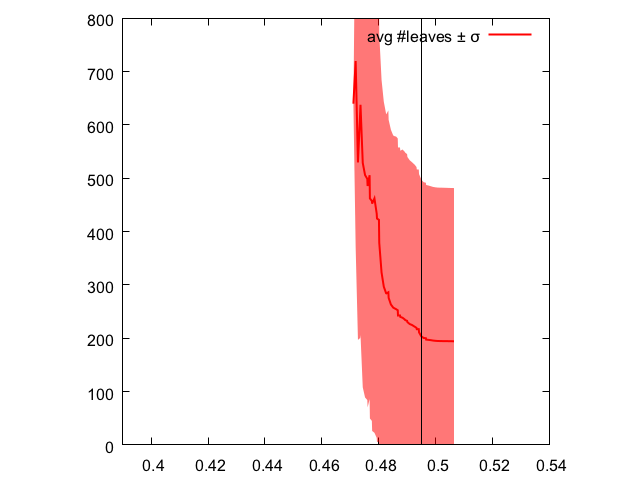}  \\ \hline
 depth \hspacegss & \hspacegss $\#$leaves & depth \hspacegss & \hspacegss $\#$leaves \\ \hline \hline
\end{tabular}
\end{center}
\caption{UCI domain \texttt{hill$+$nonoise}, conventions identical as in
  Figure \ref{f-s-transfusion}. }
  \label{f-s-hillnonoise}
\end{figure}

\begin{figure}[h]
\begin{center}
\begin{tabular}{cc||cc}\hline \hline
\multicolumn{2}{c||}{performances wrt $\alpha$s} & \multicolumn{2}{c}{performances wrt $\epsilon$s (with DP)}\\ \hline
w/o DP \hspacegss & \hspacegss with DP & full
  & crop \\ \hline
\includegraphics[trim=90bp 5bp 75bp
15bp,clip,width=\sgg\linewidth]{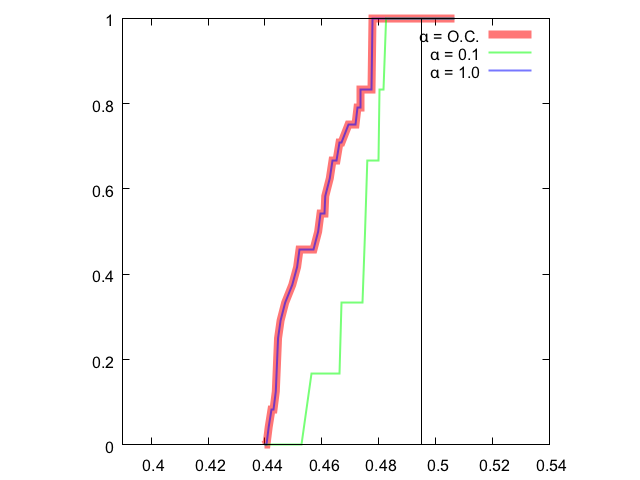}
  \hspacegss & \hspacegss \includegraphics[trim=90bp 5bp 75bp
15bp,clip,width=\sgg\linewidth]{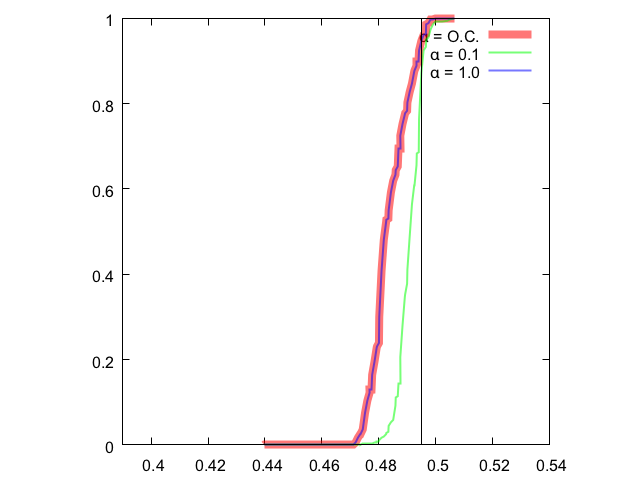} & \includegraphics[trim=90bp 5bp 75bp
15bp,clip,width=\sgg\linewidth]{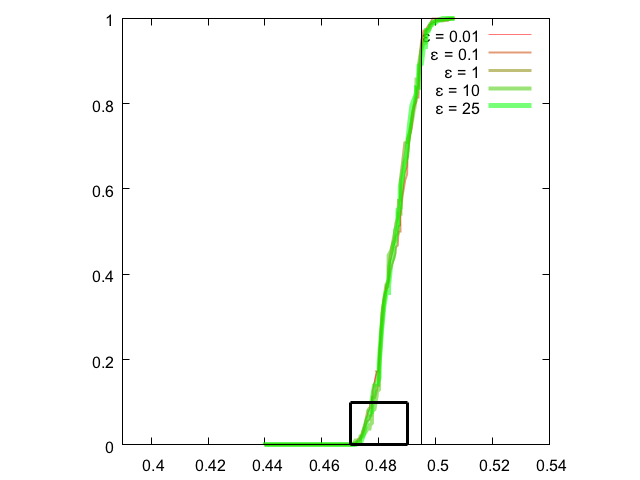} \hspacegss & \hspacegss \includegraphics[trim=90bp 5bp 75bp
15bp,clip,width=\sgg\linewidth]{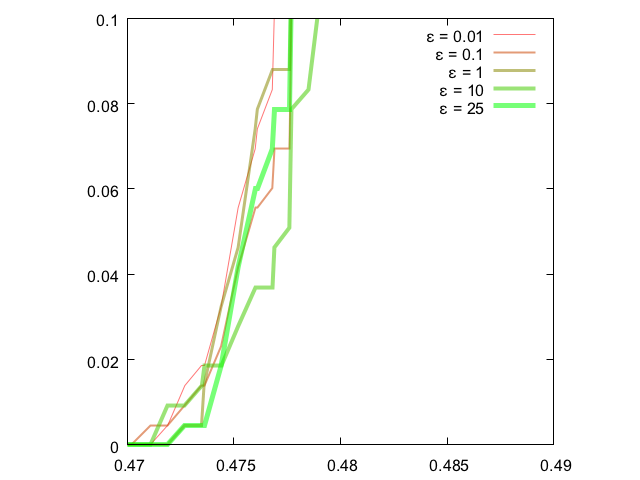} \\ \hline \hline
\end{tabular}
\end{center}
\caption{UCI domain \texttt{hill$+$nonoise}, conventions identical as in
  Figure \ref{f-s-transfusion2}.}
  \label{f-s-hillnonoise2}
\end{figure}

\begin{figure}[h]
\begin{center}
\begin{tabular}{ccccc}\hline \hline
$\varepsilon=0.01$ \hspacegs & \hspacegs $\varepsilon=0.1$ \hspacegs & \hspacegs $\varepsilon=1$ \hspacegs & \hspacegs
                                                        $\varepsilon=10$ \hspacegs &
                                                                    \hspacegs $\varepsilon=25$ \hspacegs \\ \hline 
\includegraphics[trim=90bp 5bp 75bp
15bp,clip,width=\sg\linewidth]{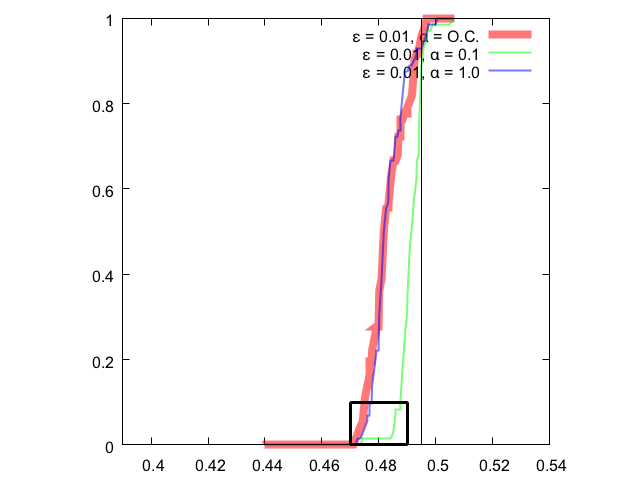}
  \hspacegs & \hspacegs \includegraphics[trim=90bp 5bp 75bp
15bp,clip,width=\sg\linewidth]{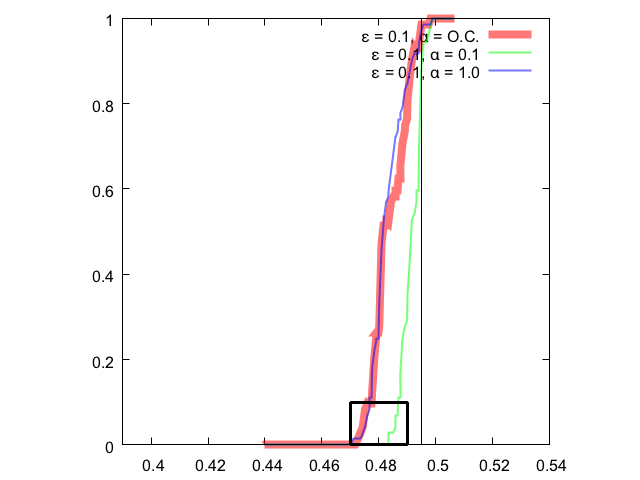}  \hspacegs & \hspacegs \includegraphics[trim=90bp 5bp 75bp
15bp,clip,width=\sg\linewidth]{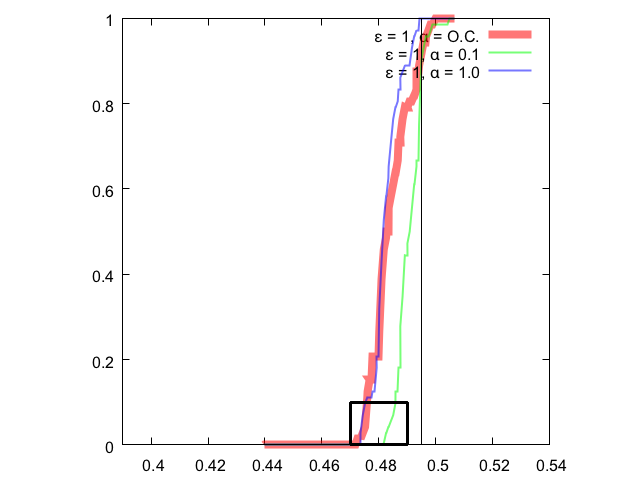}  \hspacegs & \hspacegs \includegraphics[trim=90bp 5bp 75bp
15bp,clip,width=\sg\linewidth]{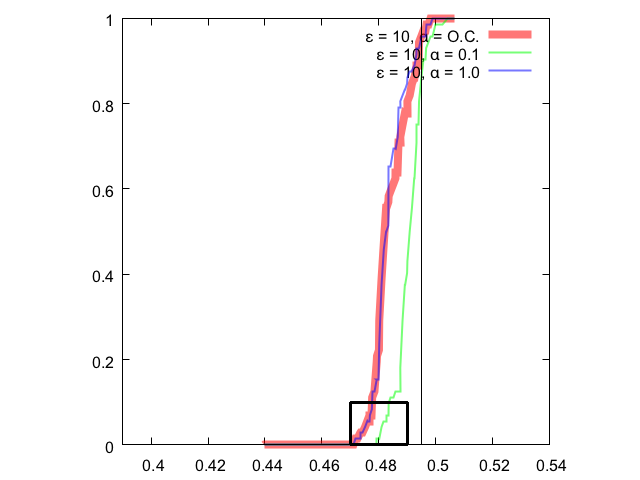}  \hspacegs & \hspacegs \includegraphics[trim=90bp 5bp 75bp
15bp,clip,width=\sg\linewidth]{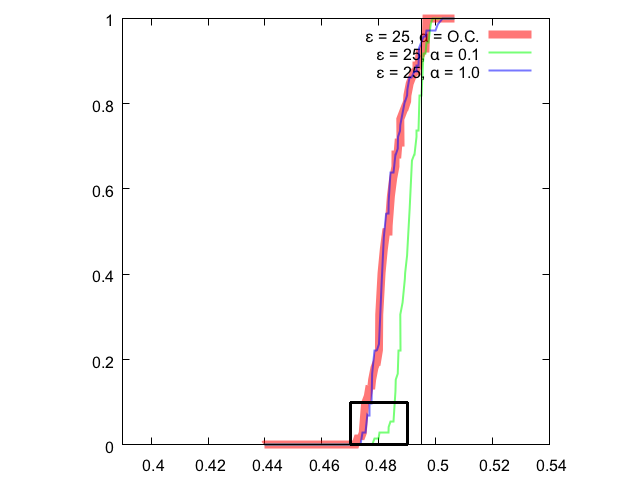}
  \\
\hline \hline
\end{tabular}
\end{center}
\caption{UCI domain \texttt{hill$+$nonoise}, conventions identical as in
  Figure \ref{f-s-transfusion3}.}
  \label{f-s-hillnonoise3}
\end{figure}

\clearpage
\newpage
\subsection*{$\triangleright$ UCI \texttt{firmteacher}}\label{sub-sec-firmteacher}

\begin{figure}[h]
\begin{center}
\begin{tabular}{cc||cc}\hline \hline
\multicolumn{2}{c||}{without DP} & \multicolumn{2}{c}{with DP}\\ \hline
\pushgraphicsLeafDepth{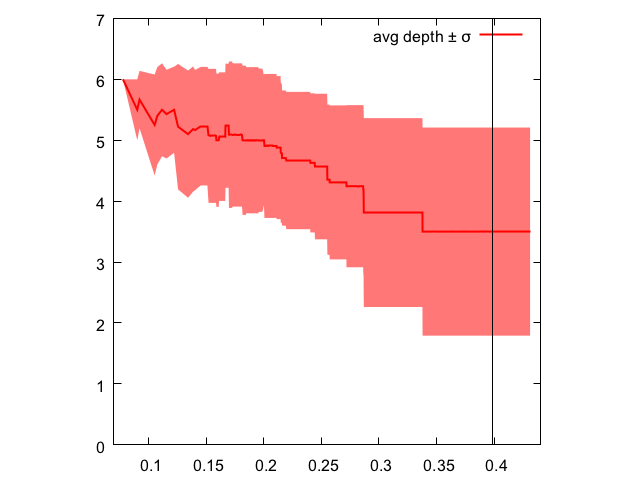}
  \hspacegss & \hspacegss \pushgraphicsLeafDepth{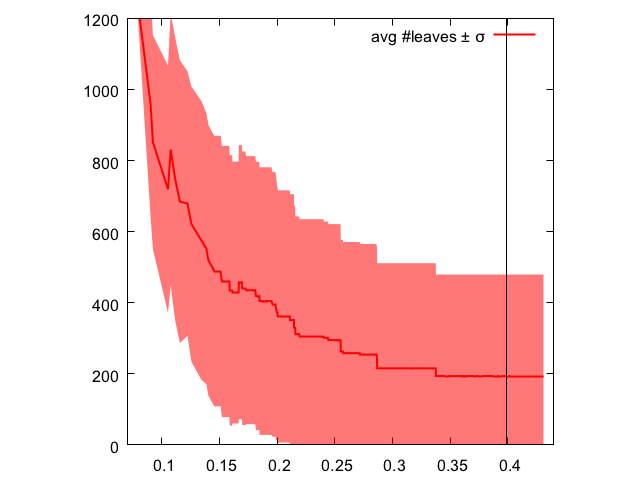}  &  \pushgraphicsLeafDepth{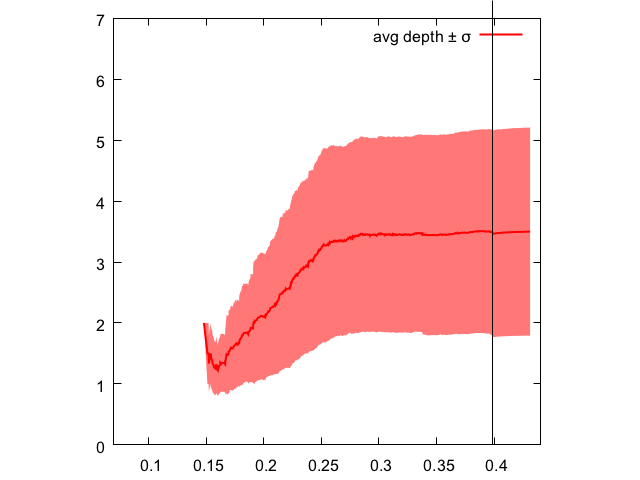}  \hspacegss & \hspacegss  \pushgraphicsLeafDepth{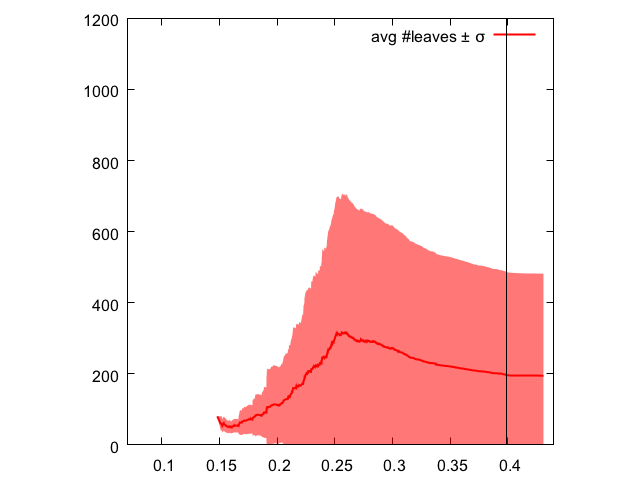}  \\ \hline
 depth \hspacegss & \hspacegss $\#$leaves & depth \hspacegss & \hspacegss $\#$leaves \\ \hline \hline
\end{tabular}
\end{center}
\caption{UCI domain \texttt{firmteacher}, conventions identical as in
  Figure \ref{f-s-transfusion}. }
  \label{f-s-firmteacher2}
\end{figure}

\begin{figure}[h]
\begin{center}
\begin{tabular}{cc||cc}\hline \hline
\multicolumn{2}{c||}{performances wrt $\alpha$s} & \multicolumn{2}{c}{performances wrt $\epsilon$s (with DP)}\\ \hline
w/o DP \hspacegss & \hspacegss with DP & full
  & crop \\ \hline
\includegraphics[trim=90bp 5bp 75bp
15bp,clip,width=\sgg\linewidth]{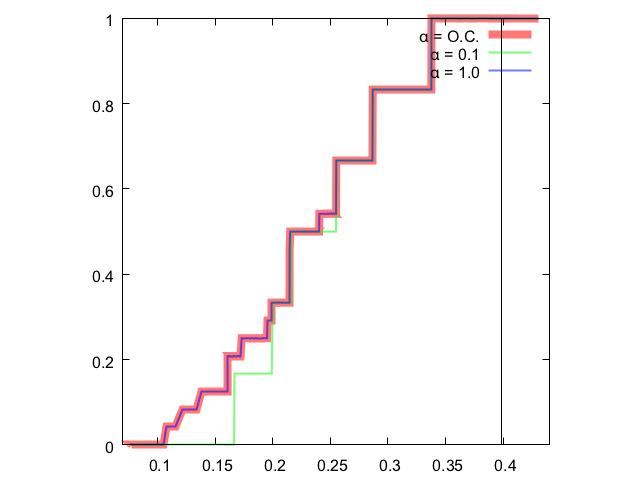}
  \hspacegss & \hspacegss \includegraphics[trim=90bp 5bp 75bp
15bp,clip,width=\sgg\linewidth]{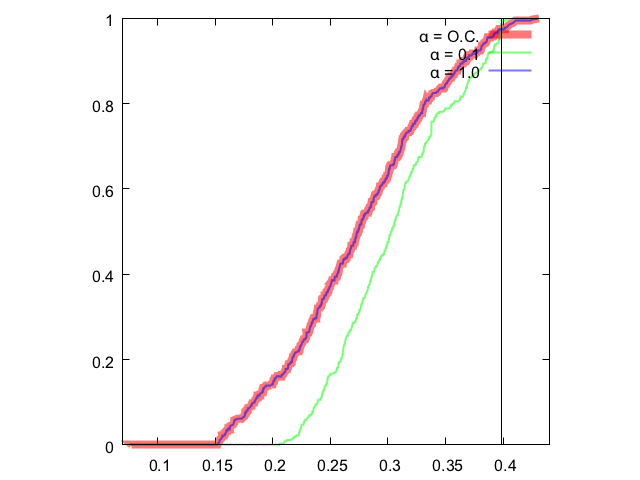} & \includegraphics[trim=90bp 5bp 75bp
15bp,clip,width=\sgg\linewidth]{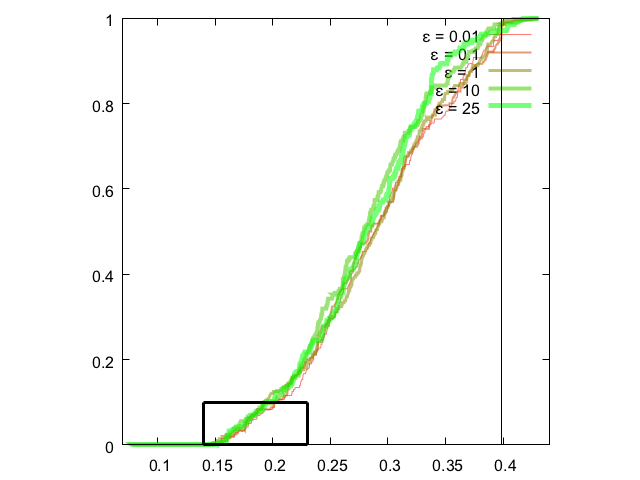} \hspacegss & \hspacegss \includegraphics[trim=90bp 5bp 75bp
15bp,clip,width=\sgg\linewidth]{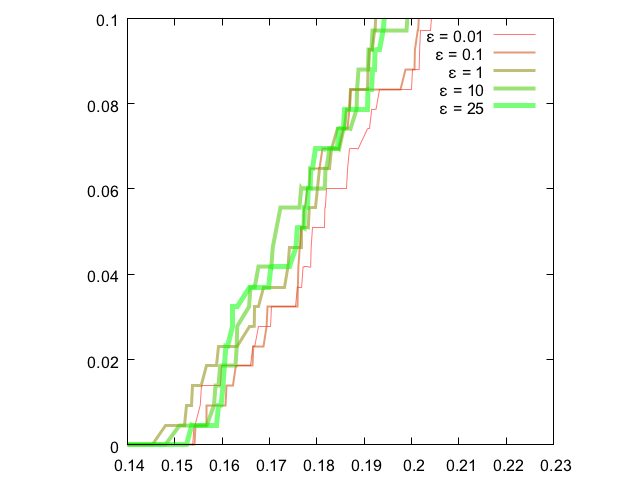} \\ \hline \hline
\end{tabular}
\end{center}
\caption{UCI domain \texttt{firmteacher}, conventions identical as in
  Figure \ref{f-s-transfusion2}.}
  \label{f-s-firmteacher22}
\end{figure}

\begin{figure}[h]
\begin{center}
\begin{tabular}{ccccc}\hline \hline
$\varepsilon=0.01$ \hspacegs & \hspacegs $\varepsilon=0.1$ \hspacegs & \hspacegs $\varepsilon=1$ \hspacegs & \hspacegs
                                                        $\varepsilon=10$ \hspacegs &
                                                                    \hspacegs $\varepsilon=25$ \hspacegs \\ \hline 
\includegraphics[trim=90bp 5bp 75bp
15bp,clip,width=\sg\linewidth]{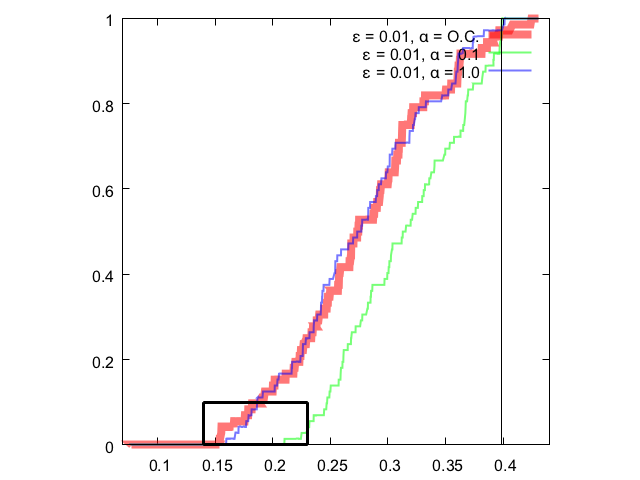}
  \hspacegs & \hspacegs \includegraphics[trim=90bp 5bp 75bp
15bp,clip,width=\sg\linewidth]{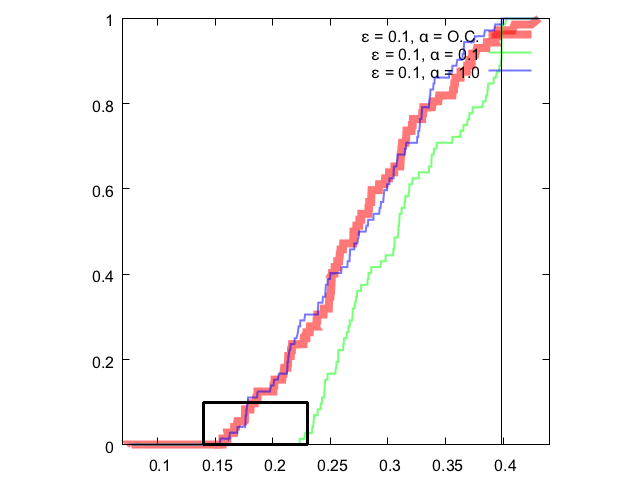}  \hspacegs & \hspacegs \includegraphics[trim=90bp 5bp 75bp
15bp,clip,width=\sg\linewidth]{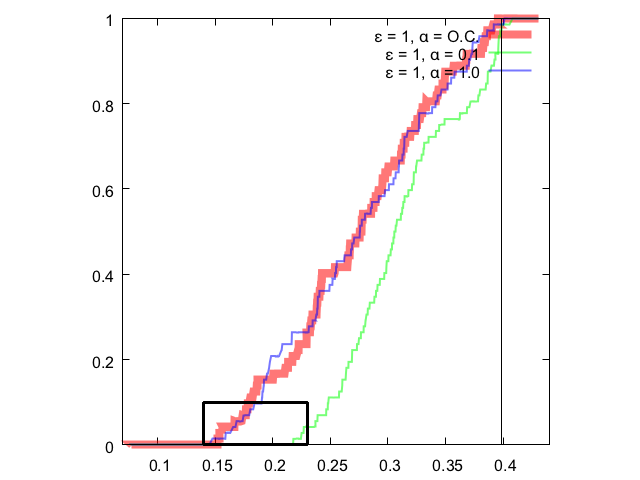}  \hspacegs & \hspacegs \includegraphics[trim=90bp 5bp 75bp
15bp,clip,width=\sg\linewidth]{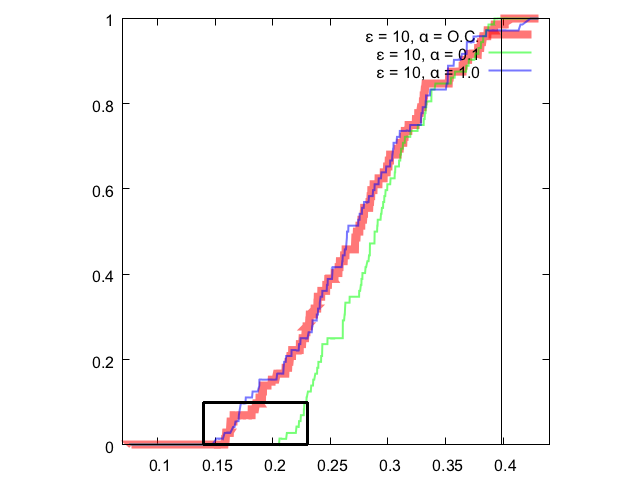}  \hspacegs & \hspacegs \includegraphics[trim=90bp 5bp 75bp
15bp,clip,width=\sg\linewidth]{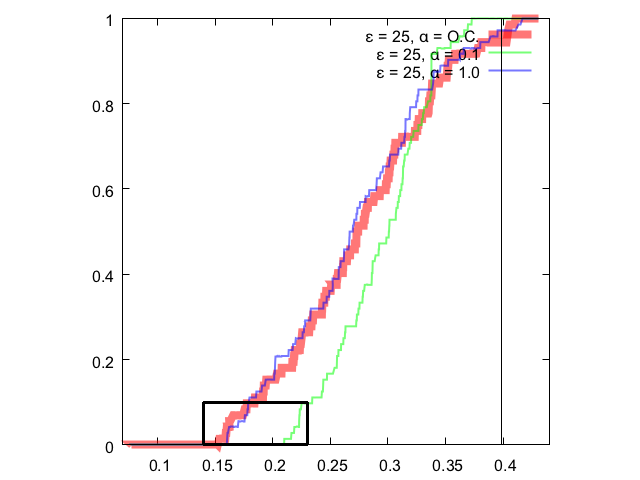}
  \\
\hline \hline
\end{tabular}
\end{center}
\caption{UCI domain \texttt{firmteacher}, conventions identical as in
  Figure \ref{f-s-transfusion3}.}
  \label{f-s-firmteacher23}
\end{figure}

\clearpage
\newpage
\subsection*{$\triangleright$ UCI \texttt{magic}}\label{sub-sec-magic}

\begin{figure}[h]
\begin{center}
\begin{tabular}{cc||cc}\hline \hline
\multicolumn{2}{c||}{without DP} & \multicolumn{2}{c}{with DP}\\ \hline
\pushgraphicsLeafDepth{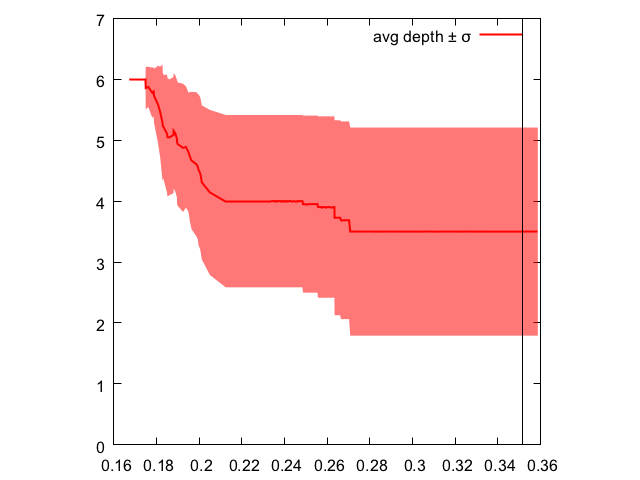}
  \hspacegss & \hspacegss \pushgraphicsLeafDepth{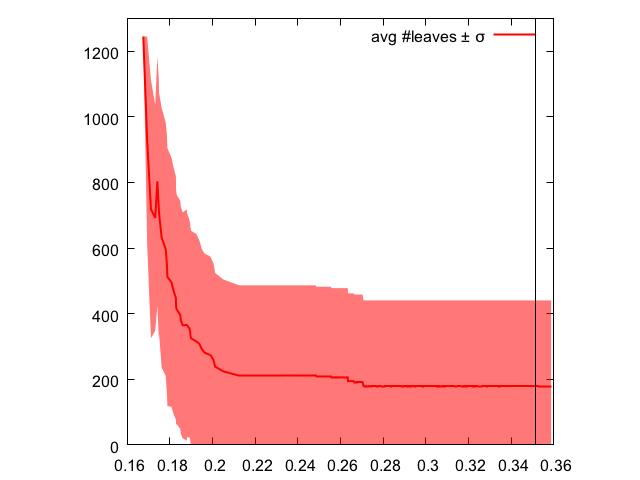}  &  \pushgraphicsLeafDepth{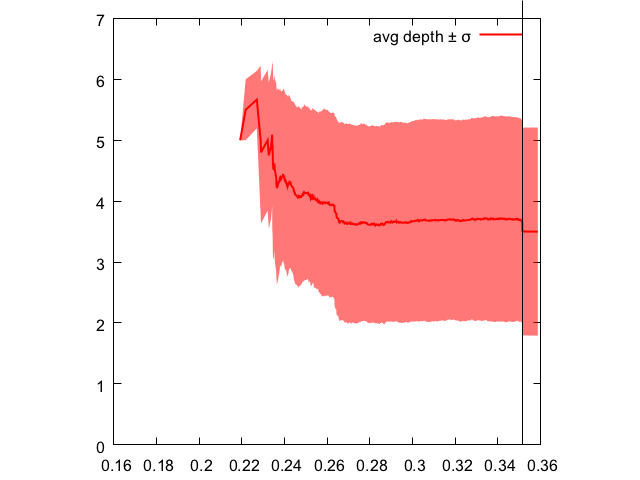}  \hspacegss & \hspacegss  \pushgraphicsLeafDepth{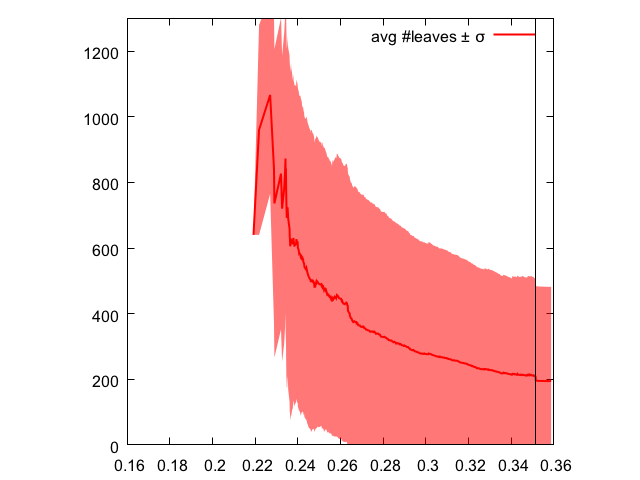}  \\ \hline
 depth \hspacegss & \hspacegss $\#$leaves & depth \hspacegss & \hspacegss $\#$leaves \\ \hline \hline
\end{tabular}
\end{center}
\caption{UCI domain \texttt{magic}, conventions identical as in
  Figure \ref{f-s-transfusion}. }
  \label{f-s-magic}
\end{figure}

\begin{figure}[h]
\begin{center}
\begin{tabular}{cc||cc}\hline \hline
\multicolumn{2}{c||}{performances wrt $\alpha$s} & \multicolumn{2}{c}{performances wrt $\epsilon$s (with DP)}\\ \hline
w/o DP \hspacegss & \hspacegss with DP & full
  & crop \\ \hline
\includegraphics[trim=90bp 5bp 75bp
15bp,clip,width=\sgg\linewidth]{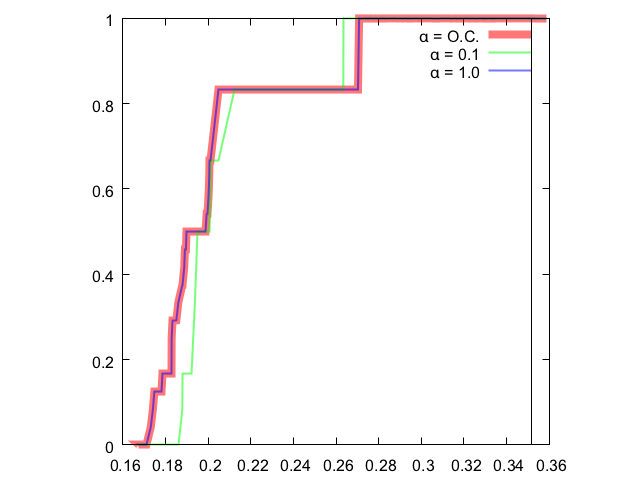}
  \hspacegss & \hspacegss \includegraphics[trim=90bp 5bp 75bp
15bp,clip,width=\sgg\linewidth]{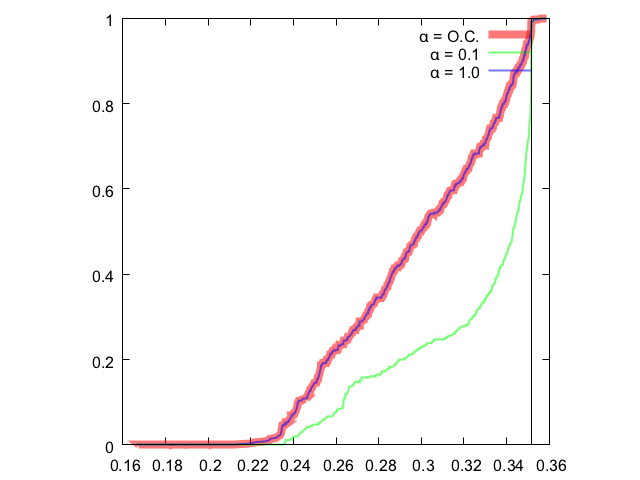} & \includegraphics[trim=90bp 5bp 75bp
15bp,clip,width=\sgg\linewidth]{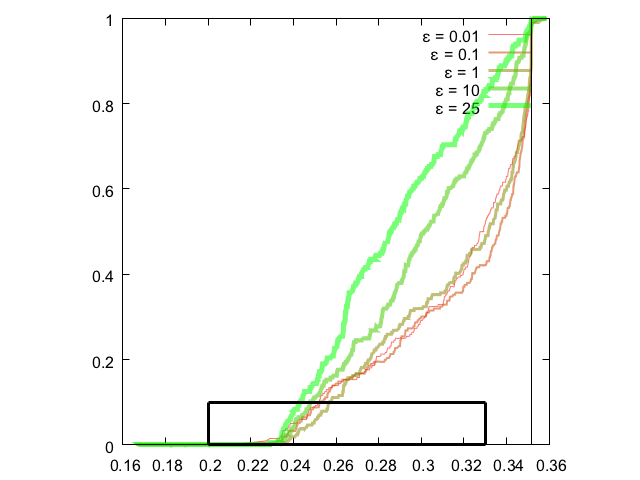} \hspacegss & \hspacegss \includegraphics[trim=90bp 5bp 75bp
15bp,clip,width=\sgg\linewidth]{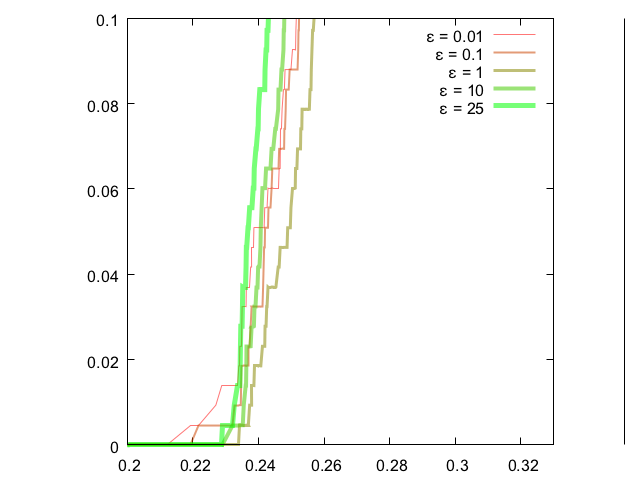} \\ \hline \hline
\end{tabular}
\end{center}
\caption{UCI domain \texttt{magic}, conventions identical as in
  Figure \ref{f-s-transfusion2}.}
  \label{f-s-magic2}
\end{figure}

\begin{figure}[h]
\begin{center}
\begin{tabular}{ccccc}\hline \hline
$\varepsilon=0.01$ \hspacegs & \hspacegs $\varepsilon=0.1$ \hspacegs & \hspacegs $\varepsilon=1$ \hspacegs & \hspacegs
                                                        $\varepsilon=10$ \hspacegs &
                                                                    \hspacegs $\varepsilon=25$ \hspacegs \\ \hline 
\includegraphics[trim=90bp 5bp 75bp
15bp,clip,width=\sg\linewidth]{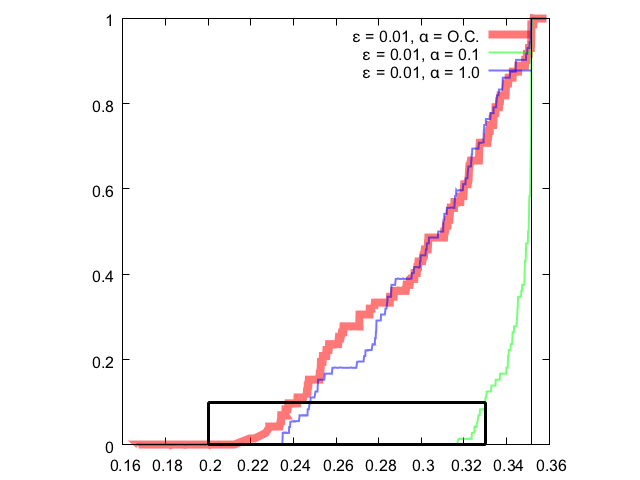}
  \hspacegs & \hspacegs \includegraphics[trim=90bp 5bp 75bp
15bp,clip,width=\sg\linewidth]{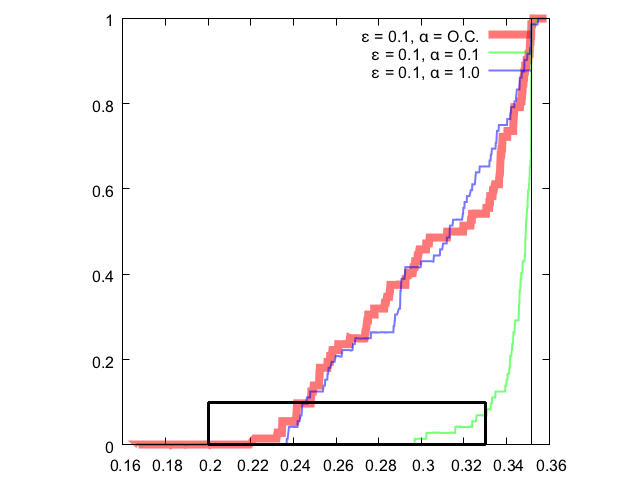}  \hspacegs & \hspacegs \includegraphics[trim=90bp 5bp 75bp
15bp,clip,width=\sg\linewidth]{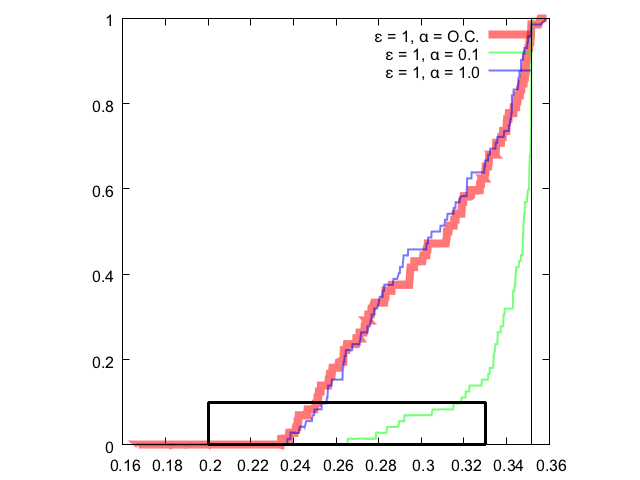}  \hspacegs & \hspacegs \includegraphics[trim=90bp 5bp 75bp
15bp,clip,width=\sg\linewidth]{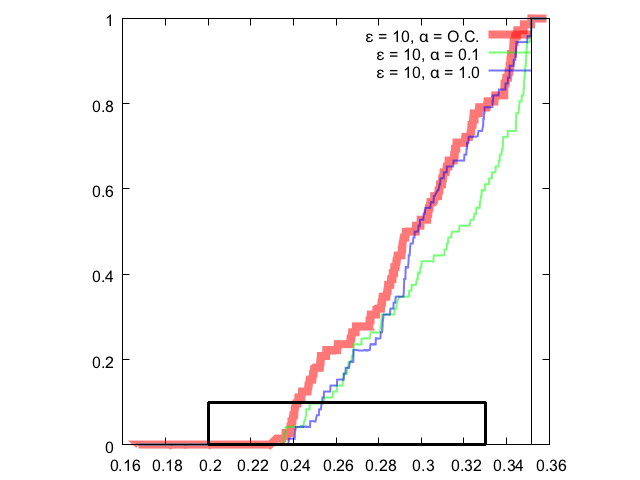}  \hspacegs & \hspacegs \includegraphics[trim=90bp 5bp 75bp
15bp,clip,width=\sg\linewidth]{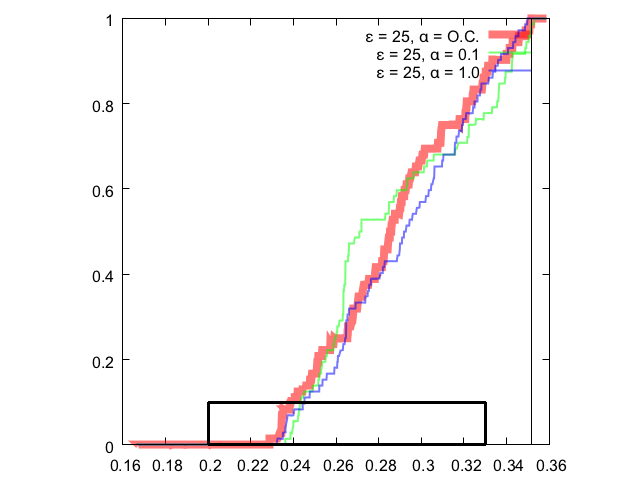}
  \\
\hline \hline
\end{tabular}
\end{center}
\caption{UCI domain \texttt{magic}, conventions identical as in
  Figure \ref{f-s-transfusion3}.}
  \label{f-s-magic3}
\end{figure}

\clearpage
\newpage
\subsection*{$\triangleright$ UCI \texttt{eeg}}\label{sub-sec-eeg}

\begin{figure}[h]
\begin{center}
\begin{tabular}{cc||cc}\hline \hline
\multicolumn{2}{c||}{without DP} & \multicolumn{2}{c}{with DP}\\ \hline
\pushgraphicsLeafDepth{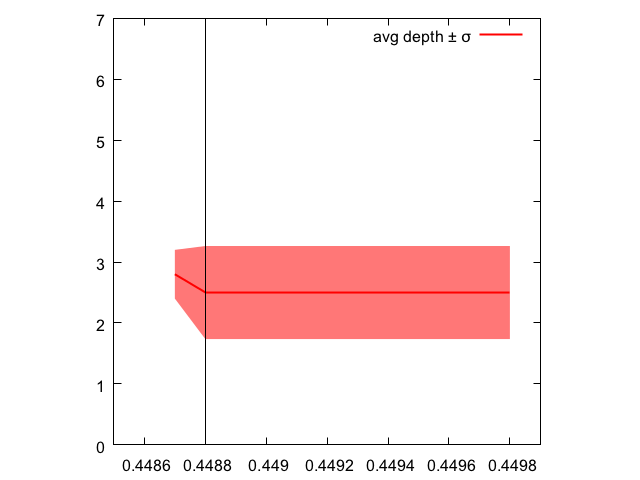}
  \hspacegss & \hspacegss \pushgraphicsLeafDepth{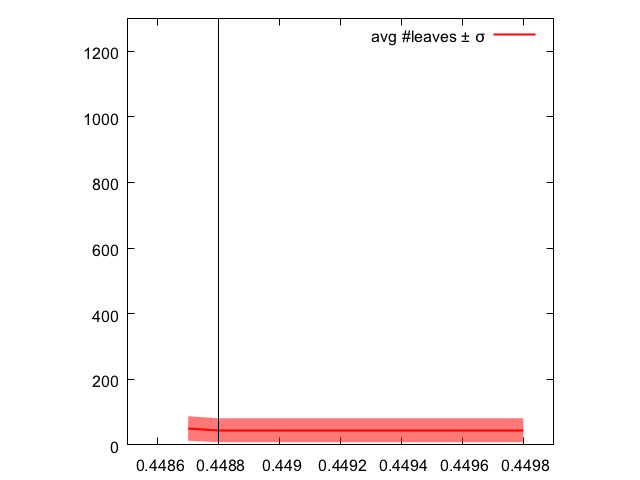}  &  \pushgraphicsLeafDepth{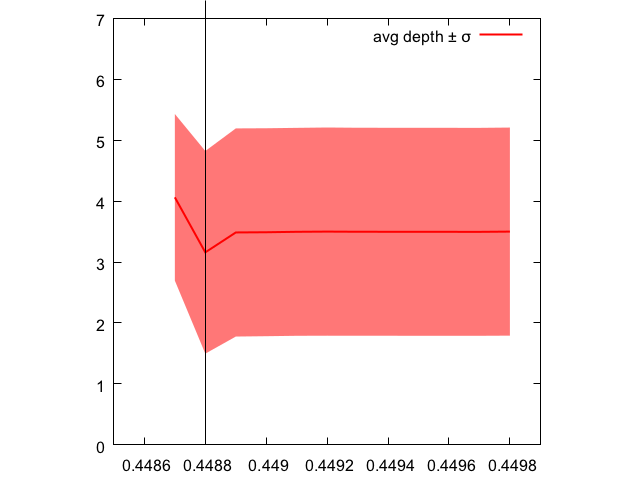}  \hspacegss & \hspacegss  \pushgraphicsLeafDepth{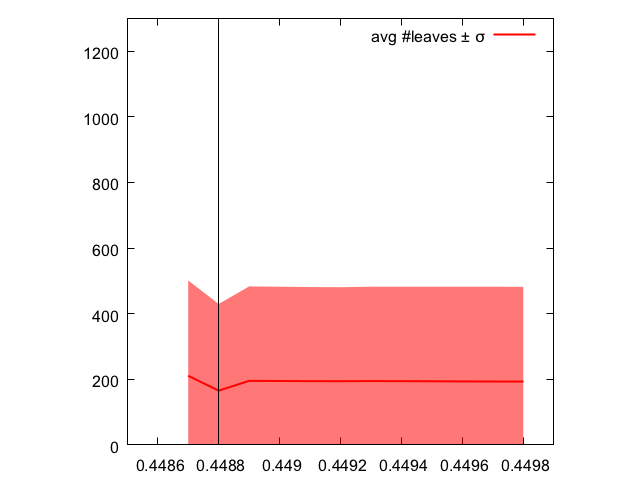}  \\ \hline
 depth \hspacegss & \hspacegss $\#$leaves & depth \hspacegss & \hspacegss $\#$leaves \\ \hline \hline
\end{tabular}
\end{center}
\caption{UCI domain \texttt{eeg}, conventions identical as in
  Figure \ref{f-s-transfusion}. }
  \label{f-s-eeg}
\end{figure}

\begin{figure}[h]
\begin{center}
\begin{tabular}{cc||cc}\hline \hline
\multicolumn{2}{c||}{performances wrt $\alpha$s} & \multicolumn{2}{c}{performances wrt $\epsilon$s (with DP)}\\ \hline
w/o DP \hspacegss & \hspacegss with DP & full
  & crop \\ \hline
\includegraphics[trim=90bp 5bp 75bp
15bp,clip,width=\sgg\linewidth]{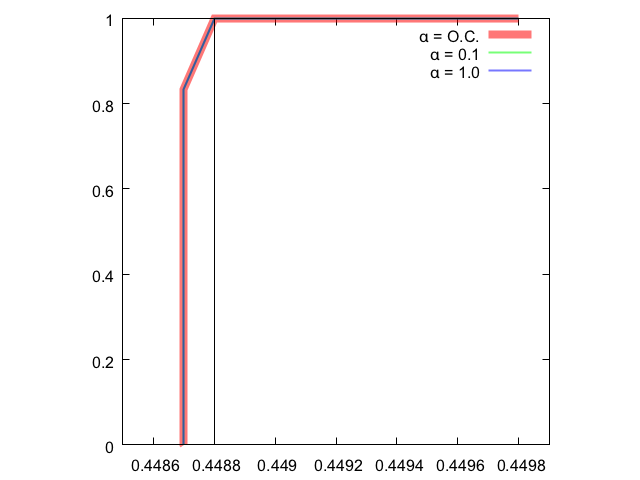}
  \hspacegss & \hspacegss \includegraphics[trim=90bp 5bp 75bp
15bp,clip,width=\sgg\linewidth]{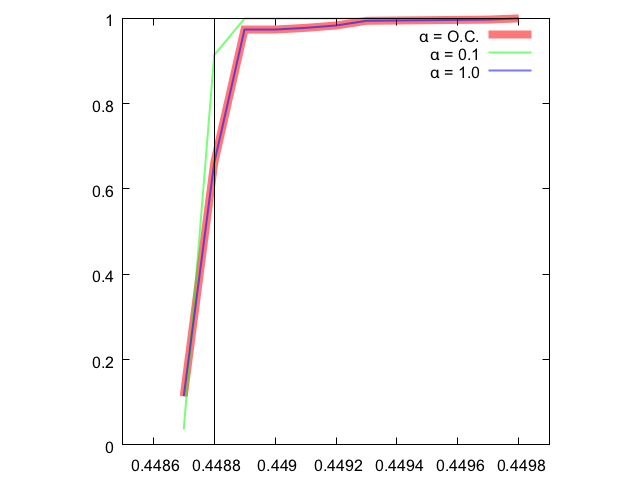} & \includegraphics[trim=90bp 5bp 75bp
15bp,clip,width=\sgg\linewidth]{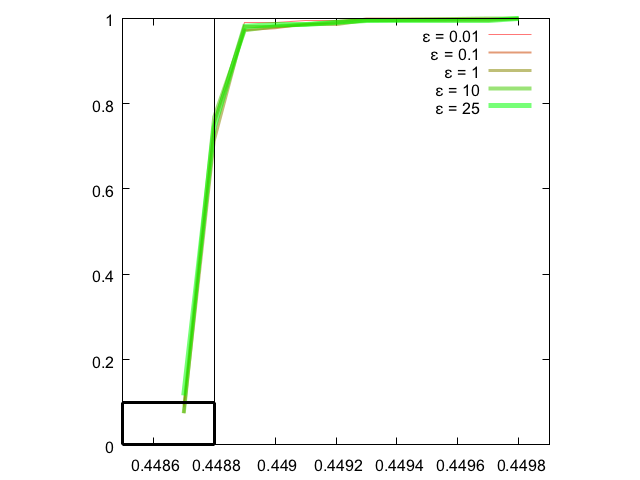} \hspacegss & \hspacegss \includegraphics[trim=90bp 5bp 75bp
15bp,clip,width=\sgg\linewidth]{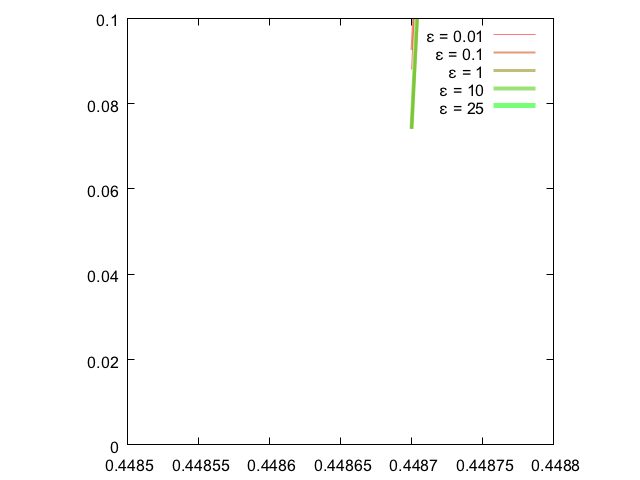} \\ \hline \hline
\end{tabular}
\end{center}
\caption{UCI domain \texttt{eeg}, conventions identical as in
  Figure \ref{f-s-transfusion2}.}
  \label{f-s-eeg2}
\end{figure}

\begin{figure}[h]
\begin{center}
\begin{tabular}{ccccc}\hline \hline
$\varepsilon=0.01$ \hspacegs & \hspacegs $\varepsilon=0.1$ \hspacegs & \hspacegs $\varepsilon=1$ \hspacegs & \hspacegs
                                                        $\varepsilon=10$ \hspacegs &
                                                                    \hspacegs $\varepsilon=25$ \hspacegs \\ \hline 
\includegraphics[trim=90bp 5bp 75bp
15bp,clip,width=\sg\linewidth]{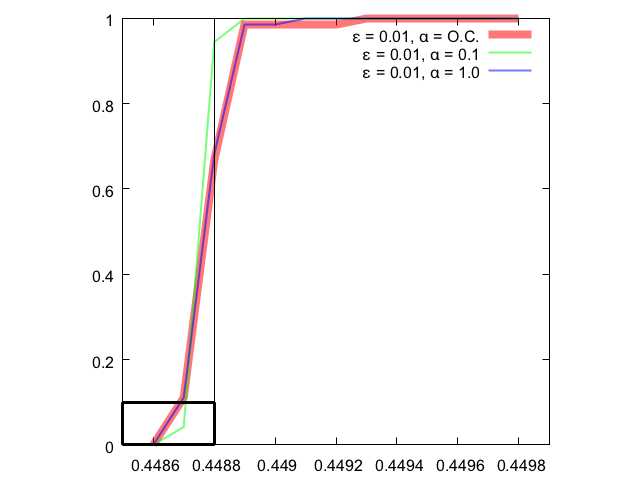}
  \hspacegs & \hspacegs \includegraphics[trim=90bp 5bp 75bp
15bp,clip,width=\sg\linewidth]{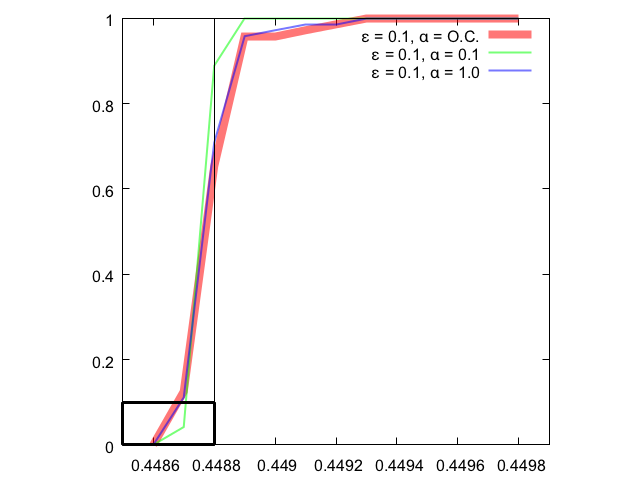}  \hspacegs & \hspacegs \includegraphics[trim=90bp 5bp 75bp
15bp,clip,width=\sg\linewidth]{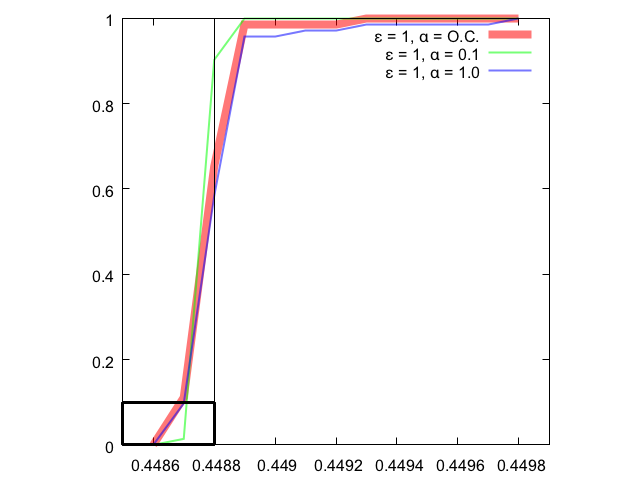}  \hspacegs & \hspacegs \includegraphics[trim=90bp 5bp 75bp
15bp,clip,width=\sg\linewidth]{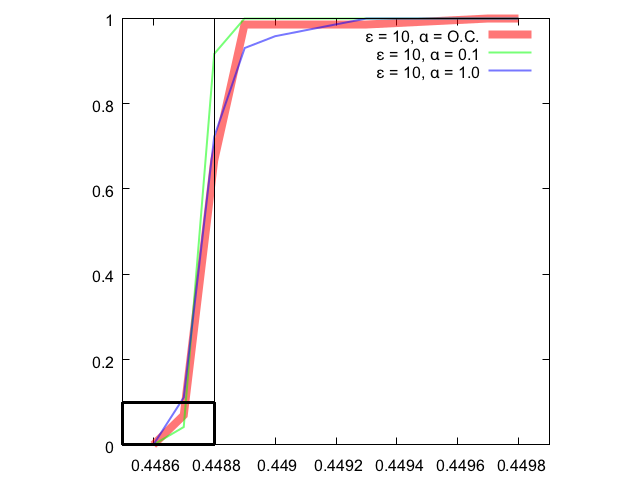}  \hspacegs & \hspacegs \includegraphics[trim=90bp 5bp 75bp
15bp,clip,width=\sg\linewidth]{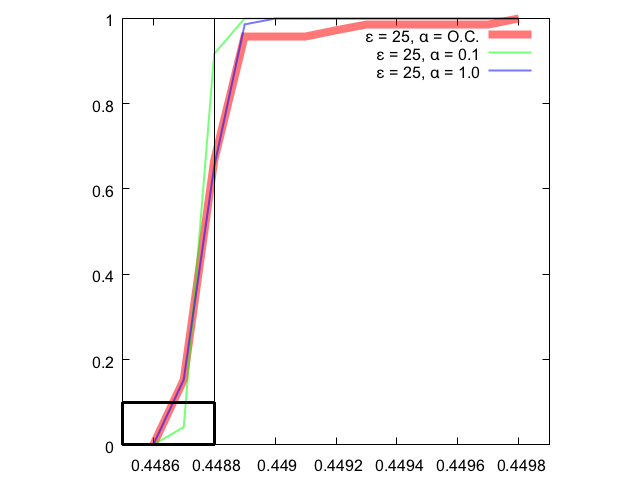}
  \\
\hline \hline
\end{tabular}
\end{center}
\caption{UCI domain \texttt{eeg}, conventions identical as in
  Figure \ref{f-s-transfusion3}.}
  \label{f-s-eeg3}
\end{figure}

\clearpage
\newpage
\subsection*{Summary in $d, T$ for the best DP in \alphaboost}\label{sub-sec-sumdT}

Table \ref{t-sumdT} roughly summarizes the optimal regimes for $d$
(depth) and $T$ (number of trees) for the best DP results in \alphaboost.

\begin{table}[t]
\begin{center}
\scalebox{.83}{%
}
\end{center}
\caption{Localisation of each of the 19 domains in terms of the model
  complexity parameters ($d,T$) allowing to get the best DP results,
  as observed from the "\texttt{results}*" file (see above, Section \ref{sub-sec-sum-imp}).}
  \label{t-sumdT}
\end{table}

\clearpage
\newpage
\subsection*{Summary of the comparison \alphaboost~vs RFs with DP}\label{sub-sec-AlphavsRF}

\begin{table}[t]
\begin{center}
\hspacegs \includegraphics[trim=5bp 5bp 5bp
5bp,clip,width=0.5\linewidth]{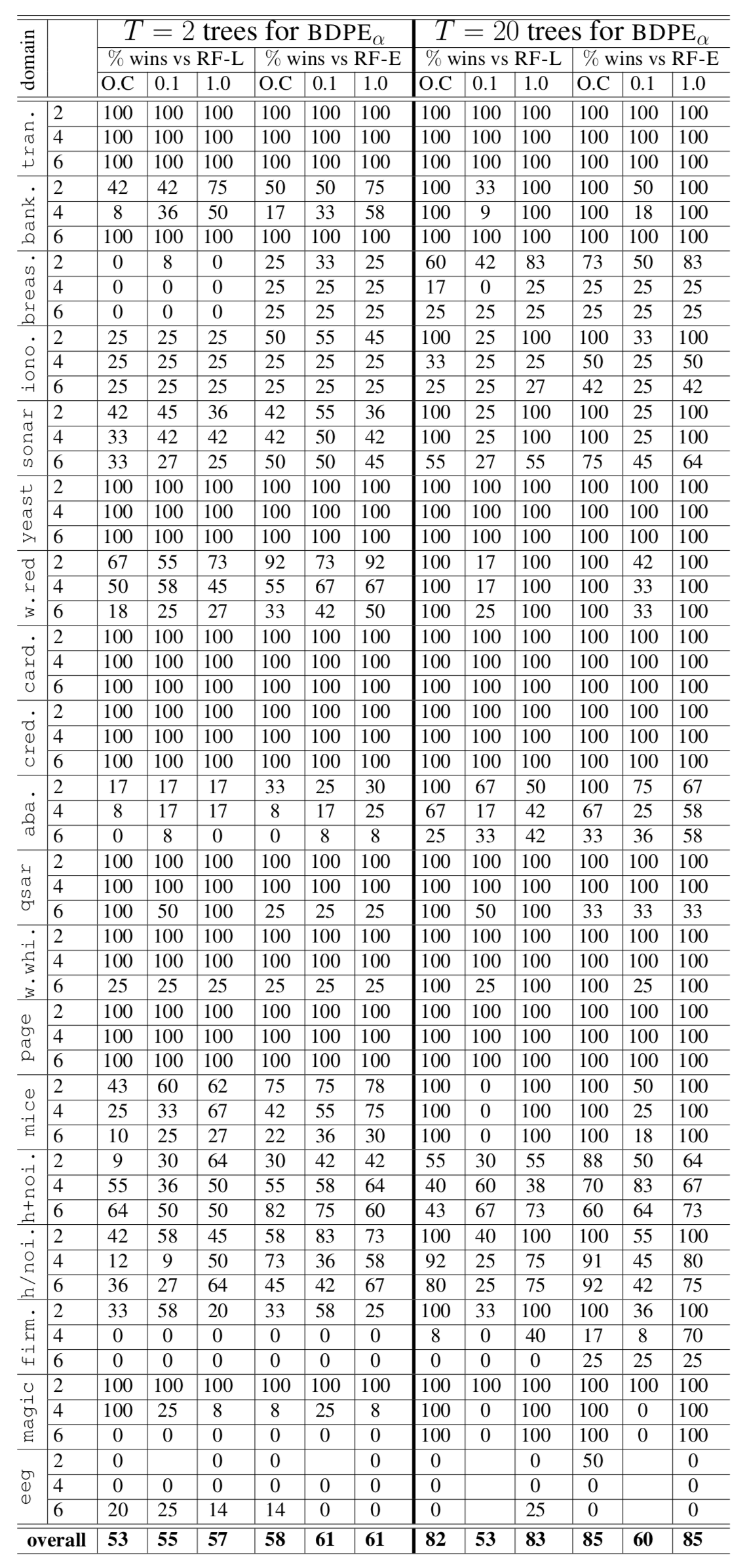}
\end{center} 
\caption{Comparison of \alphaboost~vs two SOTA random forest (RFs)
  approaches, each inducing $T=21$ random trees. For each domain and
  each depth value in $\{2,4,6\}$, we compute the number of runs where
one algorithm significantly (evaluated with a Student's $t$ test and all counts get
p-value $p<0.01$) beats the other and then compute the
percentage of those where \alphaboost~is the lead, for several values
of $\alpha$ and a number of trees $T\in \{2,20\}$ (left and right
tables, resp.) for \alphaboost.}
  \label{t-bvsrf}
\end{table}

\section*{Summary comparison $\nvpriv=10$ vs $\nvpriv=50$ ($M = 10$)}\label{sub-sec-1050}

\begin{figure}[h]
\begin{center}
\begin{tabular}{c|cc||cc}\hline \hline
 & \multicolumn{2}{c||}{performances wrt $\epsilon$s} &
                                                     \multicolumn{2}{c}{high
                                                     privacy
                                                     performances
                                                     ($\varepsilon = 0.01$)}\\
 & $\nvpriv = 10$ & $\nvpriv = 50$ & $\nvpriv = 10$ & $\nvpriv = 50$ \\ \hline
\rotatebox[origin=l]{90}{\texttt{banknote}} & \includegraphics[trim=90bp 5bp 75bp
  15bp,clip,width=\sgg\linewidth]{SubFigs/NCT_10/banknote/PostProcess10/plot_binBanknote_cumulated_dp_epsilons_ytotal.png}
  \hspacegss & \hspacegss \includegraphics[trim=90bp 5bp 75bp
               15bp,clip,width=\sgg\linewidth]{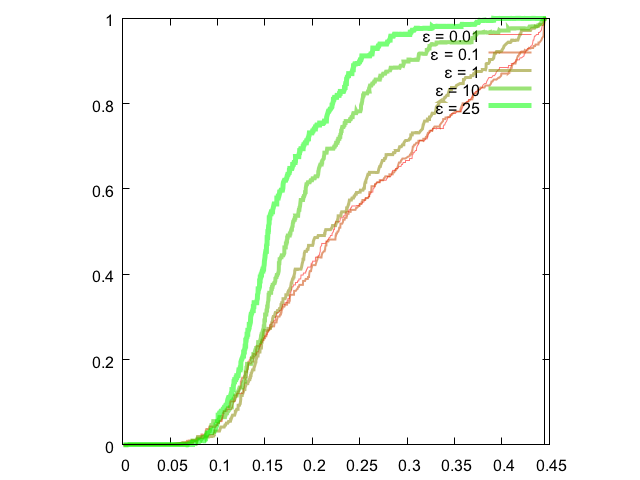}
  & \includegraphics[trim=90bp 5bp 75bp
  15bp,clip,width=\sgg\linewidth]{SubFigs/NCT_10/banknote/PostProcess10/plot_binBanknote_EPS0_01_cumulated_dp_epsilonxalphas_ytotal.png}
  \hspacegss & \hspacegss \includegraphics[trim=90bp 5bp 75bp
               15bp,clip,width=\sgg\linewidth]{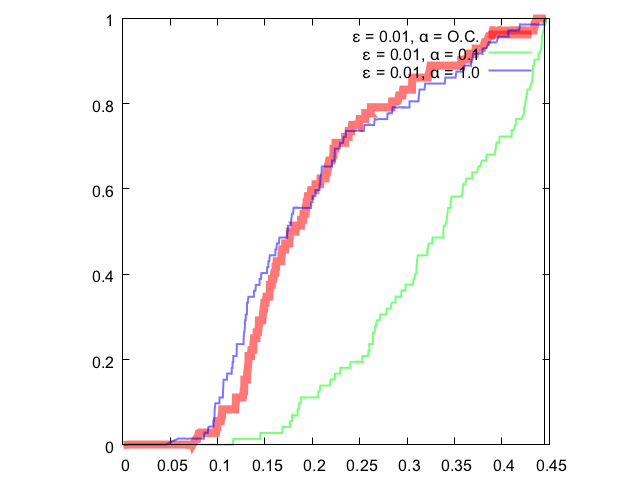}\\ \hline
\rotatebox[origin=l]{90}{\texttt{winered}} & \includegraphics[trim=90bp 5bp 75bp
  15bp,clip,width=\sgg\linewidth]{SubFigs/NCT_10/winered/PostProcess10/plot_binWinered_cumulated_dp_epsilons_ytotal.png}
  \hspacegss & \hspacegss \includegraphics[trim=90bp 5bp 75bp
               15bp,clip,width=\sgg\linewidth]{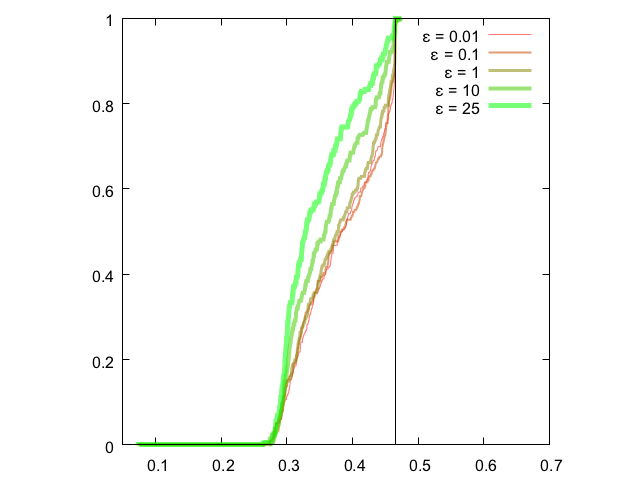}
  & \includegraphics[trim=90bp 5bp 75bp
  15bp,clip,width=\sgg\linewidth]{SubFigs/NCT_10/winered/PostProcess10/plot_binWinered_EPS0_01_cumulated_dp_epsilonxalphas_ytotal.png}
  \hspacegss & \hspacegss \includegraphics[trim=90bp 5bp 75bp
               15bp,clip,width=\sgg\linewidth]{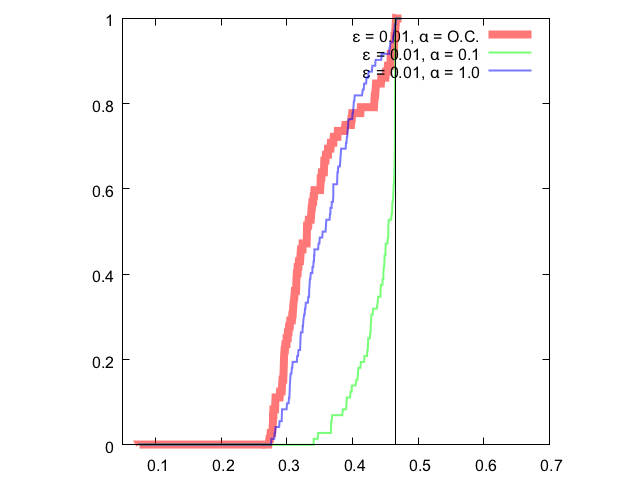}\\ \hline
\rotatebox[origin=l]{90}{\texttt{qsar}} & \includegraphics[trim=90bp 5bp 75bp
  15bp,clip,width=\sgg\linewidth]{SubFigs/NCT_10/qsar/PostProcess10/plot_binQsar_cumulated_dp_epsilons_ytotal.png}
  \hspacegss & \hspacegss \includegraphics[trim=90bp 5bp 75bp
               15bp,clip,width=\sgg\linewidth]{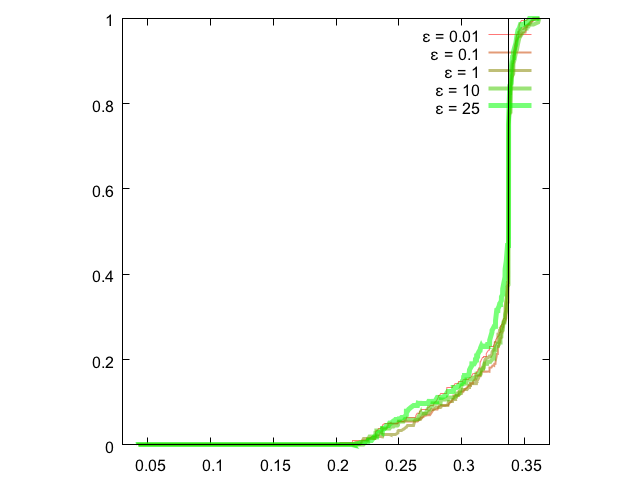}
  & \includegraphics[trim=90bp 5bp 75bp
  15bp,clip,width=\sgg\linewidth]{SubFigs/NCT_10/qsar/PostProcess10/plot_binQsar_EPS0_01_cumulated_dp_epsilonxalphas_ytotal.png}
  \hspacegss & \hspacegss \includegraphics[trim=90bp 5bp 75bp
               15bp,clip,width=\sgg\linewidth]{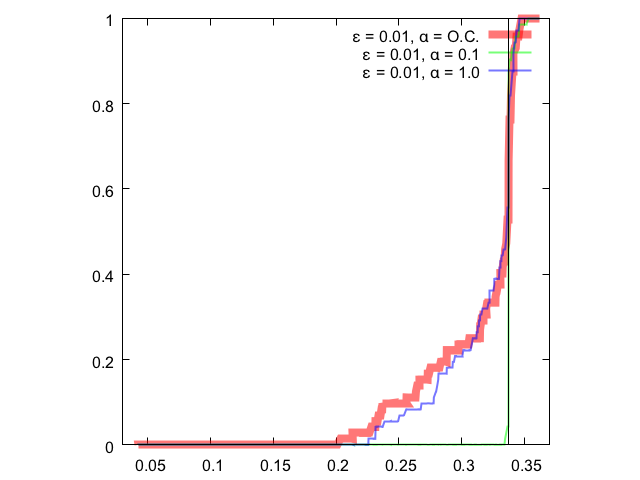}\\ \hline
\rotatebox[origin=l]{90}{\texttt{winewhite}} & \includegraphics[trim=90bp 5bp 75bp
  15bp,clip,width=\sgg\linewidth]{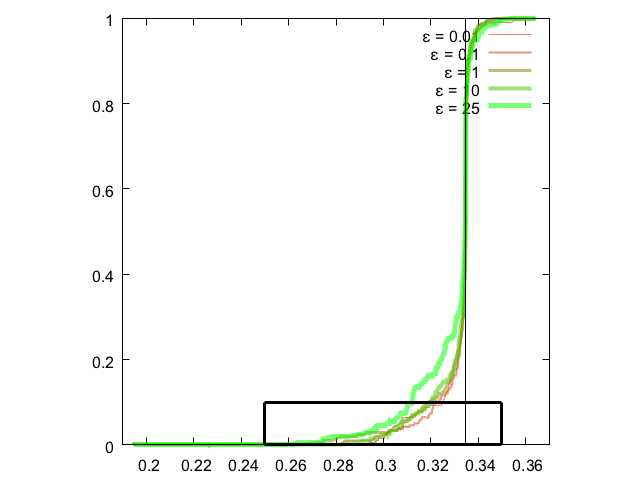}
  \hspacegss & \hspacegss \includegraphics[trim=90bp 5bp 75bp
               15bp,clip,width=\sgg\linewidth]{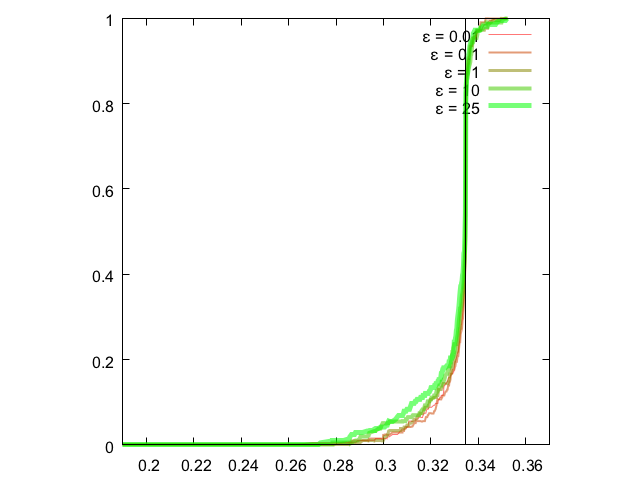}
  & \includegraphics[trim=90bp 5bp 75bp
  15bp,clip,width=\sgg\linewidth]{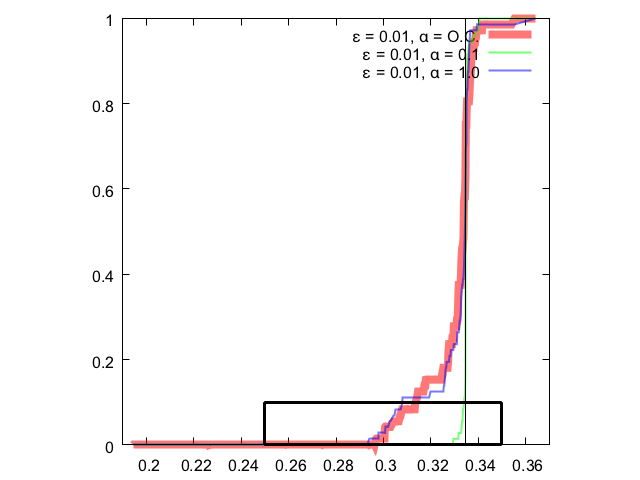}
  \hspacegss & \hspacegss \includegraphics[trim=90bp 5bp 75bp
               15bp,clip,width=\sgg\linewidth]{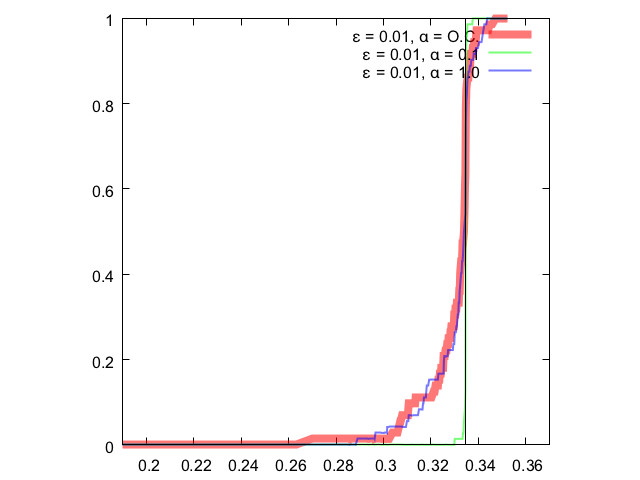}\\ \hline
\end{tabular}
\end{center}
\caption{Extract of the comparison between quantization in
  $\nvpriv=10$ vs $\nvpriv=50$ values for continuous attributes, for
  both the overall privacy results (left subtable) and results
  as a function of $\alpha$ for high privacy regime ($\varepsilon =
  0.01$, right subtable). Conventions follow Figures
  \ref{f-s-transfusion} and \ref{f-s-transfusion2}. The vertical black
  line is the test error of the majority class.}
  \label{f-summary1050}
\end{figure}